\documentclass{article}

\usepackage{natbib}
\setcitestyle{numbers, compress}

\usepackage[final]{neurips_2023}

\usepackage[utf8]{inputenc} 
\usepackage[T1]{fontenc}    
\usepackage{hyperref}       
\usepackage{url}            
\usepackage{booktabs}       
\usepackage{amsfonts}       
\usepackage{nicefrac}       
\usepackage{microtype}      
\usepackage{xcolor}         
\usepackage{pifont}

\usepackage{makecell}

\usepackage[ruled,vlined]{algorithm2e}

\hypersetup{colorlinks=true,linkcolor=blue,linktocpage=false,citebordercolor=blue,citecolor=blue,anchorcolor=blue}

\usepackage{subfigure}
\usepackage{etoc}

\usepackage{amsmath}
\usepackage{amssymb}
\usepackage{mathtools}
\usepackage{amsthm}
\theoremstyle{plain}
\newtheorem{theorem}{Theorem}[section]

\newtheorem{lemma}[theorem]{Lemma}

\theoremstyle{definition}
\newtheorem{definition}[theorem]{Definition}
\newtheorem{assumption}[theorem]{Assumption}
\newtheorem{remark}[theorem]{Remark}
\newtheorem*{main result}{Main Theorem}
\allowdisplaybreaks[4]

\usepackage{tikz}
\usetikzlibrary{calc, arrows.meta, intersections, patterns, positioning, shapes.misc, fadings, through,decorations.pathreplacing}

\tikzstyle{descript} = [text = black,align=center, minimum height=1.8cm, align=center, outer sep=0pt,font = \footnotesize]
\tikzstyle{activity} =[align=center,outer sep=1pt]

\newcommand*{\circled}[1]{\lower.7ex\hbox{\tikz\draw (0pt, 0pt)%
    circle (.5em) node {\makebox[1em][c]{\small #1}};}}

\newcommand{\bb}{\boldsymbol{b}}

\newcommand{\bv}{\boldsymbol{v}}
\newcommand{\bw}{\boldsymbol{w}}
\newcommand{\bx}{\boldsymbol{x}}

\newcommand{\sgrad}{\partial^{\circ}}
\newcommand{\bF}{\boldsymbol{F}}

\newcommand{\bI}{\boldsymbol{I}}

\newcommand{\bP}{\boldsymbol{P}}

\newcommand{\bmu}{\boldsymbol{\mu}}

\newcommand{\btheta}{\boldsymbol{\theta}}

\newcommand{\bSigma}{\boldsymbol{\Sigma}}

\newcommand{\rd}{\mathrm{d}}

\newcommand{\cI}{\mathcal{I}}
\newcommand{\cJ}{\mathcal{J}}
\newcommand{\cK}{\mathcal{K}}
\newcommand{\cL}{\mathcal{L}}
\newcommand{\cM}{\mathcal{M}}

\newcommand{\cO}{\mathcal{O}}
\newcommand{\cP}{\mathcal{P}}

\newcommand{\cS}{\mathcal{S}}

\newcommand{\cU}{\mathcal{U}}
\newcommand{\cV}{\mathcal{V}}

\newcommand{\cZ}{\mathcal{Z}}

\newcommand{\bbI}{\mathbb{I}}

\newcommand{\bbR}{\mathbb{R}}
\newcommand{\bbS}{\mathbb{S}}

\newcommand{\bzero}{\mathbf{0}}

\newcommand{\norm}[1]{\ensuremath{\left\| #1 \right\|}}
\newcommand{\bracket}[1]{\ensuremath{\left( #1 \right)}}

\newcommand{\Acc}{{\rm Acc}}

\newcommand{\sgn}{\texttt{sgn}}

\newcommand{\<}{\left\langle}
\renewcommand{\>}{\right\rangle}

\title{Understanding Multi-phase Optimization Dynamics and Rich Nonlinear Behaviors of ReLU Networks}

\author{%
 Mingze Wang \\
  School of Mathematical Sciences \\
  Peking University\\
  Beijing, 100081, P.R. China \\
  \texttt{mingzewang@stu.pku.edu.cn} \\
  \And
  Chao Ma\\
  Department of Mathematics \\
  Stanford University\\
  Stanford, CA 94305 \\
  \texttt{chaoma@stanford.edu} \\
}

\begin{document}

\maketitle


\begin{abstract}
The training process of ReLU neural networks often exhibits complicated nonlinear phenomena. 
The nonlinearity of models and non-convexity of loss pose significant challenges for theoretical analysis. Therefore, most previous theoretical works on the optimization dynamics of neural networks focus either on local analysis (like the end of training) or approximate linear models (like Neural Tangent Kernel). 
In this work, we conduct a complete theoretical characterization of the training process of a two-layer ReLU network trained by Gradient Flow on a linearly separable data. In this specific setting, our analysis captures the whole optimization process starting from random initialization to final convergence. 
Despite the relatively simple model and data that we studied, we reveal four different phases from the whole training process showing a general simplifying-to-complicating learning trend.
Specific nonlinear behaviors can also be precisely identified and captured theoretically, such as
initial condensation, saddle-to-plateau dynamics, plateau escape, changes of activation patterns, 
learning with increasing complexity, etc.
\end{abstract}

\section{Introduction}

Deep learning shows its remarkable capabilities across various fields of applications. However, the theoretical understanding of its great success still has a long way to go. Among all theoretical topics, one of the most crucial aspect is the understanding of the optimization dynamics of deep neural network (NN), particularly the dynamics produced by Gradient Descent (GD) and its variants. This topic is highly challenging due to the highly non-convex loss landscape and existing works usually work with settings that do not align well with realistic practices. For instance, the extensived studied Neural Tangent Kernel (NTK) theory~\citep{jacot2018neural, du2018gradient, du2019gradient, zou2018stochastic, allen2019convergence} proves the global convergence of Stochastic gradient descent (SGD) to zero training error for highly over-parameterized neural networks; however, the optimization behaviors are similar to kernel methods and do not exhibit nonlinear behaviors, because neurons remain close to their initialization throughout training. 

In reality, however, the training of practical networks can exhibit plenty of nonlinear behaviors~\citep{chizat2018global,mei2019mean,woodworth2020kernel}.
In the initial stage of the training, a prevalent nonlinear phenomenon induced by small initialization is \textit{initial condensation}~\citep{maennel2018gradient,luo2021phase}, where neurons condense onto a few isolated orientations. 
At the end of training, NNs trained by GD can \textit{directionally converge} to the KKT points of some constrained max-margin problem~\citep{nacson2019lexicographic,lyu2019gradient,ji2020directional}. However, KKT points are not generally unique, and determining which direction GD converges to can be challenging. 
Nonlinear training behaviors besides initial and terminating stages of optimization are also numerous. 
For example, for square loss, ~\cite{jacot2021saddle} investigates the \textit{saddle-to-saddle} dynamics where GD traverses a sequence of saddles during training, but it is unclear whether similar behavior can occur for classification tasks using exp-tailed loss.
Moreover, while in lazy regime, most activation patterns do not change during training ReLU networks, it remains uncertain when and how \textit{activation patterns evolve} beyond lazy regime.
Additionally, while it is generally conjectured that GD \textit{learns functions of increasing complexity}~\citep{nakkiran2019sgd}, this perspective has yet to be proven.

As reviewed in Section~\ref{section: related works}, works have been done to analyze and explain the nonlinear training behaviors listed above. However, due to the complexity of the training dynamics, most existing works only focus on one phenomenon and conduct analysis on a certain stage of the training process. Few attempts have been done to derive a full characterization of the whole training dynamics from the initialization to convergence, and the settings adopted by these works are usually too simple to capture many nonlinear behaviors~\citep{phuong2021inductive,lyu2021gradient,wang2022early,boursier2022gradient}.


{\bf In this work}, we make an attempt to theoretically describe the whole neural network training dynamics beyond the linear regime, in a setting that many nonlinear behaviors manifest.
Specifically, We analyze the training process of a two-layer ReLU network trained by Gradient Flow (GF) on a linearly separable data. 
In this setting, our analysis captures the whole optimization process starting from random initialization to final convergence. 
Despite the relatively simple model and data that we studied, we reveal multiple phases in training process, and show a general simplifying-to-complicating learning trend by detailed analysis of each phase.
Specifically, by our meticulous theoretical analysis of the whole training process, we precisely identify \textbf{four different phases} that exhibit \textbf{numerous nonlinear behaviors}.
In Phase I, \textit{initial condensation and simplification} occur as living neurons rapidly condense in two different directions. Meanwhile, GF \textit{escapes from the saddle} around initialization. 
In Phase II, GF \textit{gets stuck into the plateau} of training accuracy for a long time, then \textit{escapes}. 
In Phase III, a significant number of neurons are \textit{deactivated}, leading to \textit{self-simplification} of the network, then GF tries to learn using the almost simplest network.
In Phase IV, a considerable number of neurons are \textit{reactivated}, causing \textit{self-complication} of the network. Finally, GF \textit{converges towards an initialization-dependent direction}, and this direction is \textit{not even a local max margin direction}. Overall, the whole training process exhibits a remarkable \textit{simplifying-to-complicating} behavior.

\section{Other Related Works}
\label{section: related works}

Initial condensation phenomenon are studied in \citep{maennel2018gradient,luo2021phase,zhou2022empirical,zhou2022towards,abbe2022merged,abbe2022initial,chen2023phase}. 
Theoretically,~\cite{lyu2021gradient,boursier2022gradient} analyze the condensation directions under their settings,
which are some types of data average. 
Additionally,~\cite{atanasov2021neural} demonstrates that NNs in the rich feature learning regime learn a kernel machine due to the silent alignment phenomenon, similar to the initial condensation.


The end of training is extensively studied for classification tasks. Specifically, for classification with exponentially-tailed loss functions, if all the training data can be classified correctly, NNs trained by GD converge to the KKT directions of some constrained max-margin problem~\citep{nacson2019lexicographic,lyu2019gradient,chizat2020implicit,ji2020directional,kunin2022asymmetric}.
In \citep{phuong2021inductive,lyu2021gradient}, they analyze entire training dynamics and derive specific convergent directions that only depend on the data. 
Furthermore, another famous phenomenon in the end of training is the neural collapse~\citep{papyan2020prevalence,fang2021exploring,zhu2021geometric,han2021neural}, which says the features represented by over-parameterized neural networks for data in a same class will collapse to one point, and such points for all classes converge to a simplex equiangular tight frame.

Saddle-to-saddle dynamics are explored for square loss in ~\citep{jacot2021saddle,zhang2021embedding,boursier2022gradient,pesme2023saddle,abbe2023sgd}. Furthermore, learning of increasing complexity, also called simplifying-to-complicating or frequency-principle, is investigated in~\citep{arpit2017closer,nakkiran2019sgd,xu2019frequency,rahaman2019spectral}.

Beyond lazy regime and local analysis, \cite{phuong2021inductive,lyu2021gradient,wang2022early,boursier2022gradient} also characterize the whole training dynamics and exhibit a few of nonlinear behaviors. Specifically,~\cite{lyu2021gradient} studies the training dynamics of GF on Leaky ReLU networks, which differ from ReLU networks because Leaky ReLU is always activated on any data.
In~\citep{safran2022effective}, they studies the dynamics of GF on one dimensional dataset, and characterizes the effective number of linear regions.
In~\citep{brutzkus2017sgd}, they studies the dynamics of SGD  on Leaky ReLU networks and linearly separable dataset.
Moreover,~\cite{boursier2022gradient} characterizes the dynamics on orthogonally data for square loss.
The studies closest to our work are \cite{phuong2021inductive,wang2022early}, exploring the complete dynamics on classifying orthogonally separable data. However, this data is easy to learn, and all the features can be learned rapidly (accuracy=$100\%$) in initial training, followed by lazy training (activation patterns do not change).
Unfortunately, this simplicity does not hold true for actual tasks on much more complex data, and NNs can only learn some features in initial training, which complicates the overall learning process.
Furthermore, we provide a detailed comparison between our results and these works in Section~\ref{section: comparison}.
Another related work~\citep{saxe2022neural} introduces a novel bias of learning dynamics: toward shared representations. This idea and the view of gating networks are enlightening for extending our two-layer theory to deep ReLU neural networks. 

Our work also investigates the max-margin implicit bias of ReLU neural networks, and related works have been listed above.
Although in homogenized neural networks such as ReLU, GD implicitly converges to a KKT point of the max-margin problem, it is still unclear where it is an actual optimum. A recent work~\citep{vardi2022margin} showed that in many cases, the converged KKT point is not even a local optimum of the max margin problem. Besides, there are many other attempts to explain the implicit bias of deep learning~\citep{vardi2023implicit}.
Another popular implicit bias is the flat minima bias~\citep{hochreiter1997flat,keskar2016large}. Recent studies~\citep{wu2018sgd,blanc2020implicit,ma2021linear,li2021happens,mulayoff2021implicit,wu2022does,wu2023implicitstability} provided explanations for why SGD favors flat minima and flat minima generalize well. 

\section{Preliminaries}
{\bf Basic Notations.} We use bold letters for vectors or matrices and lowercase letters for scalars, e.g. $\bx=(x_1,\cdots,x_d)^\top\in\mathbb{R}^d$ and $\bP=(P_{i j})_{m_1\times m_2}\in\mathbb{R}^{m_1\times m_2}$.
We use $\left<\cdot,\cdot\right>$ for the standard Euclidean inner product between
two vectors, and $\left\|\cdot\right\|$ for the $l_2$ norm of a vector or the spectral norm of a matrix. 
We use progressive representation $\cO,\Omega,\Theta$ to hide absolute positive constants.
For any positive integer $n$, let $[n]=\{1,\cdots,n\}$. Denote by $\mathcal{N}(\bmu,\bSigma)$ the Gaussian distribution with mean $\boldsymbol{\mu}$ and covariance matrix $\mathbf{\Sigma}$, $\mathbb{U}(\cS)$ the uniform distribution on a set $\cS$. Denote by $\mathbb{I}\{E\}$ the indicator function for an event $E$.

\subsection{Binary Classification with Two-layer ReLU Networks}

\textbf{Binary classification.} In this paper, we consider the binary classification problem. We are given $n$ training data $\mathcal{S}=\{(\bx_i,{y}_i)\}_{i=1}^n\subset\mathbb{R}^d\times\{\pm1\}$. Let ${f}(\cdot;\boldsymbol{\theta})$ be a neural network model parameterized by $\btheta$, and aim to minimize the empirical risk given by:
\begin{equation}\label{equ: problem}
        \mathcal{L}(\boldsymbol{\theta})=\frac{1}{n}\sum_{i=1}^n \ell(y_if(\bx_i;\boldsymbol{\theta})),
\end{equation}
where $\ell(\cdot):\bbR\to\bbR$ is the exponential-type loss function \citep{soudry2018implicit,lyu2019gradient} for classification tasks, including the most popular classification losses: exponential loss, logistic loss, and cross-entropy loss. 
Our analysis focuses on the exponential loss $\ell(z)=e^{-z}$, while our method can be extended to logistic loss and cross-entropy loss.

\textbf{Two-layer ReLU Network.} Throughout the following sections, we consider two-layer ReLU neural networks comprising $m$ neurons defined as
\begin{align*}
    f(\bx;\boldsymbol{\theta})=\sum_{k=1}^m a_k\sigma(\boldsymbol{b}_k^\top\bx),
\end{align*}
where $\sigma(z)=\max\{z,0\}$ is the ReLU activation function, $\bb_1\cdots,\bb_m\in\bbR^d$ are the weights in the first layer, $a_1,\cdots,a_m$ are the weights in the second layer. And we consider the case that the weights in the second layer are fixed, which is a common setting used in previous studies \citep{arora2019fine,chatterji2021doesA}. We use $\boldsymbol{\theta}=(\bb_1^\top,\cdots,\bb_m^\top)^\top\in\bbR^{md}$ to denote the concatenation of all trainable weights.

\newpage
\subsection{Gradient Flow Starting from Random Initialization}

\textbf{Gradient Flow.} 
As the limiting dynamics of (Stochastic) Gradient Descent with infinitesimal learning rate \citep{li2017stochastic,li2019stochastic}, we study the following Gradient Flow (GF) on the objective function~\eqref{equ: problem}:
\begin{equation}\label{equ: alg GF}
    \frac{\mathrm{d}\boldsymbol{\theta}(t)}{\mathrm{d}t}\in-\sgrad\mathcal{L}(\boldsymbol{\theta}(t)),\quad t\geq0.
\end{equation}
Notice that the ReLU is not differentiable at $0$, and therefore, the dynamics is defined as a subgradient inclusion flow \citep{bolte2010characterizations}. 
Here, $\partial^{\circ}$ denotes the Clarke
subdifferential, which is a generalization of the derivative for non-differentiable functions. Additionally, to address the potential non-uniqueness of gradient flow trajectories, we adopt the definition of solutions for discontinuous systems~\citep{filippov2013differential}. 
For formal definitions, please refer to Appendix~\ref{appendix: proof preparation},~\ref{section: subdifferential and KKT}, and~\ref{section: discontinuous system solution}.

\textbf{Random Initialization.} We consider GF \eqref{equ: alg GF} starting from the following initialization: 
\begin{align*}
&\boldsymbol{b}_k(0)\overset{\text{i.i.d.}}{\sim}\frac{\kappa_1}{\sqrt{m}}\mathbb{U}(\mathbb{S}^{d-1}) \text{ and } a_k=\mathrm{s}_k\frac{\kappa_2}{\sqrt{m}}\text{ for } k\in[m]; 
\\&\mathrm{s}_k=1\text{ for } k\in[m/2];\ \mathrm{s}_k=-1\text{ for } k\in[m]-[m/2].
\end{align*}
Here, $0<\kappa_1<\kappa_2\leq1$ control the initialization scale. It is worth noting that since the distribution $\mathcal{N}(\boldsymbol{0},\bI_d/d)$ is close to $\mathbb{U}(\mathbb{S}^{d-1})$ in high-dimensional settings, our result can be extended to the initialization $\bb_k\overset{\text{i.i.d.}}{\sim}\mathcal{N}(\boldsymbol{0},\kappa_1^2\bI_d/md)$ with high probability guarantees.

\subsection{Linearly Separable Data beyond Orthogonally Separable}

In previous works \citep{phuong2021inductive,wang2022early}, a special case of the linearly separable dataset was investigated, namely ``orthogonally separable''. A training dataset is orthogonally separable when $\<\bx_i,\bx_j\>\geq0$ for $i,j$ in the same class, and $\<\bx_i,\bx_j\>\leq0$ for $i,j$ in different classes. As mentioned in Section~\ref{section: related works}, in this case, GF can learn all features and achieve $100\%$ training accuracy quickly, followed by lazy training. 
In this work, we consider data that is more difficult to learn, which leads to more complicated optimization dynamics. Specifically, we consider the following data.

\begin{assumption}\label{ass: data} 
Consider the linearly separable dataset $\cS=\{(\bx_i,y_i)\}_{i\in[n]}\subset\bbR^d\times\bbR$ such that $(\bx_i,y_i)=\begin{cases}
(\bx_{+},1),\ i\in[n_{+}]
\\
(\bx_{-},-1),\ i\in[n]-[n_{+}]
\end{cases}$, where $\bx_+,\bx_-\in\bbS^{d-1}$ are two data points with a small angle $\Delta\in(0,{\pi}/{2})$, and $n_+,n_-$ are the numbers of positive and negative samples, respectively, with $n=n_++n_-$. We also use $p:=n_+/n_-$ to denote the ratio of $n_+$ and $n_-$, which measures the class imbalance. Furthermore, we assume $p\cos\Delta>1$. 
\end{assumption}

\begin{remark} 
We focus on the training dataset satisfying Assumption~\ref{ass: data} with a small $\Delta\ll1$. The margin of the dataset is $\sin(\Delta/2)$, which implies that the separability of this data is much weaker than that of orthogonal separable data. 
Additionally, the condition $p\cos\Delta>1$ merely requires a slight imbalance in the data. 
These two properties work together to produce rich nonlinear behaviors during training.
\end{remark}

\section{Characterization of Four-phase Optimization Dynamics}\label{section: GF dynamics}


In this section, we study the whole optimization dynamics of GF~\eqref{equ: alg GF} starting from random initialization when training the two-layer ReLU network on linearly separable dataset satisfying Assumption~\ref{ass: data} and using the loss function~\eqref{equ: problem}. 
To begin with, we introduce some additional notations.

\textbf{Additional Notations.}
First, we identify several crucial data-dependent directions under Assumption \ref{ass: data}. These include two directions that are orthogonal to the data, defined as $\bx_{+}^{\perp}:=\frac{\bx_--\left<\bx_-,\bx_+\right>\bx_+}{\left\|\bx_--\left<\bx_-,\bx_+\right>\bx_+\right\|}$ and $\bx_{-}^{\perp}:=\frac{\bx_+-\<\bx_+,\bx_-\>\bx_-}{\|\bx_+-\<\bx_+,\bx_-\>\bx_-\|}$, which satisfy $\<\bx_+,\bx_+^{\perp}\>=\<\bx_-,\bx_-^{\perp}\>=0$. 
Additionally, we define the label-average data direction as $\boldsymbol{\mu}:=\frac{\boldsymbol{z}}{\norm{\boldsymbol{z}}}$ where $\boldsymbol{z}=\frac{1}{n}\sum_{i=1}^n y_i\bx_i$. One can verify that $\<\bmu,\bx_+\>>0$ and $\<\bmu,\bx_-\>>0$ under the condition $p\cos\Delta>1$. In Figure~\ref{fig: data, loss, acc}, we visualize these directions.

Second, we use the following notations to denote important quantities during the GF training process. We denote the prediction on $\bx_+$ and $\bx_-$ by $f_+(t):=f(\bx_+;\boldsymbol{\theta}(t)),
f_-(t):=f(\bx_-;\boldsymbol{\theta}(t))$. We use $\Acc(t):=\frac{1}{n}\sum_{i=1}^n\bbI\{y_if(\bx_i;\btheta(t))>0\}$ to denote the training accuracy at time $t$. For each neuron $k\in[m]$, we use $\bw_k(t):=\bb(t)/\norm{\bb(t)}$ and $\rho_k(t):=\norm{\bb(t)}$ to denote its direction and norm, respectively. To capture the activation dynamics of each neuron $k\in[m]$ on each data, we use $\sgn_{k}^+(t):=\sgn(\<\bb_k(t),\bx_+\>)$ to record whether the $k$-th neuron is activated with respect to $\bx_{+}$, and $\sgn_{k}^-(t):=\sgn(\<\bb_k(t),\bx_-\>)$ defined similarly, which we call ReLU activation patterns.

\subsection{A Brief Overview of four-phase Optimization Dynamics}

We illustrate different phases in the training dynamics by a numerical example.
Specifically, we train a network on the dataset that satisfies Assumption~\ref{ass: data} with $p=4$ and $\Delta=\pi/15$. The directions and magnitudes of the neurons at some important times are shown in Figure~\ref{fig: dynamics}, reflecting four different phases on the training behavior and activation patterms.
More experiment details and results can be found in Appendix~\ref{appendix: experiment: without noise}.

\begin{figure}[!h]
\vspace{-.2cm}
\begin{center}
    \subfigure[\small $t=0$]
    {\includegraphics[width=3cm]{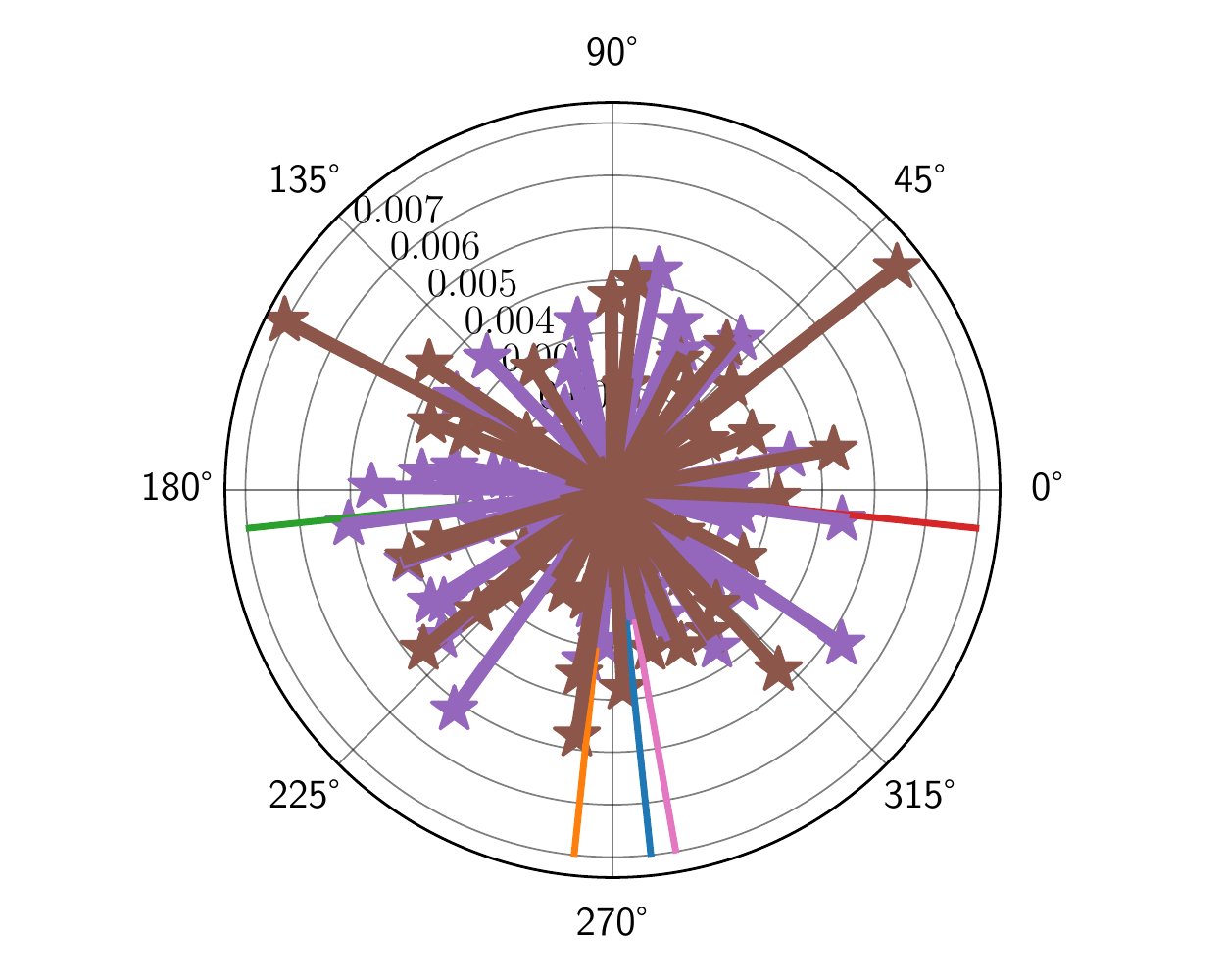}}
    \hspace{-.5cm}
    \subfigure[\small $t=200$]
    {\includegraphics[width=3cm]{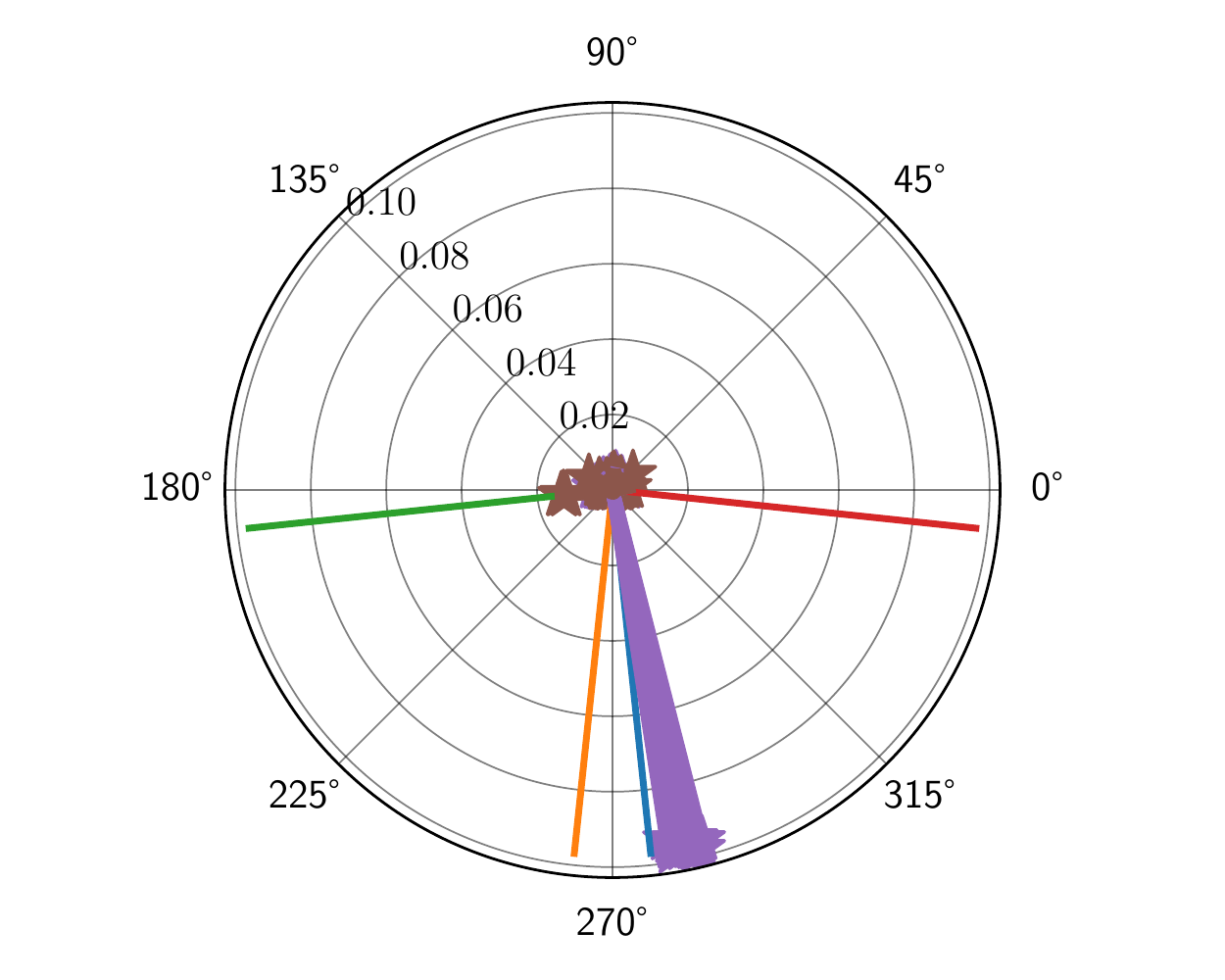}}
    \hspace{-.5cm}
    \subfigure[\small $t=50000$]
    {\includegraphics[width=3cm]{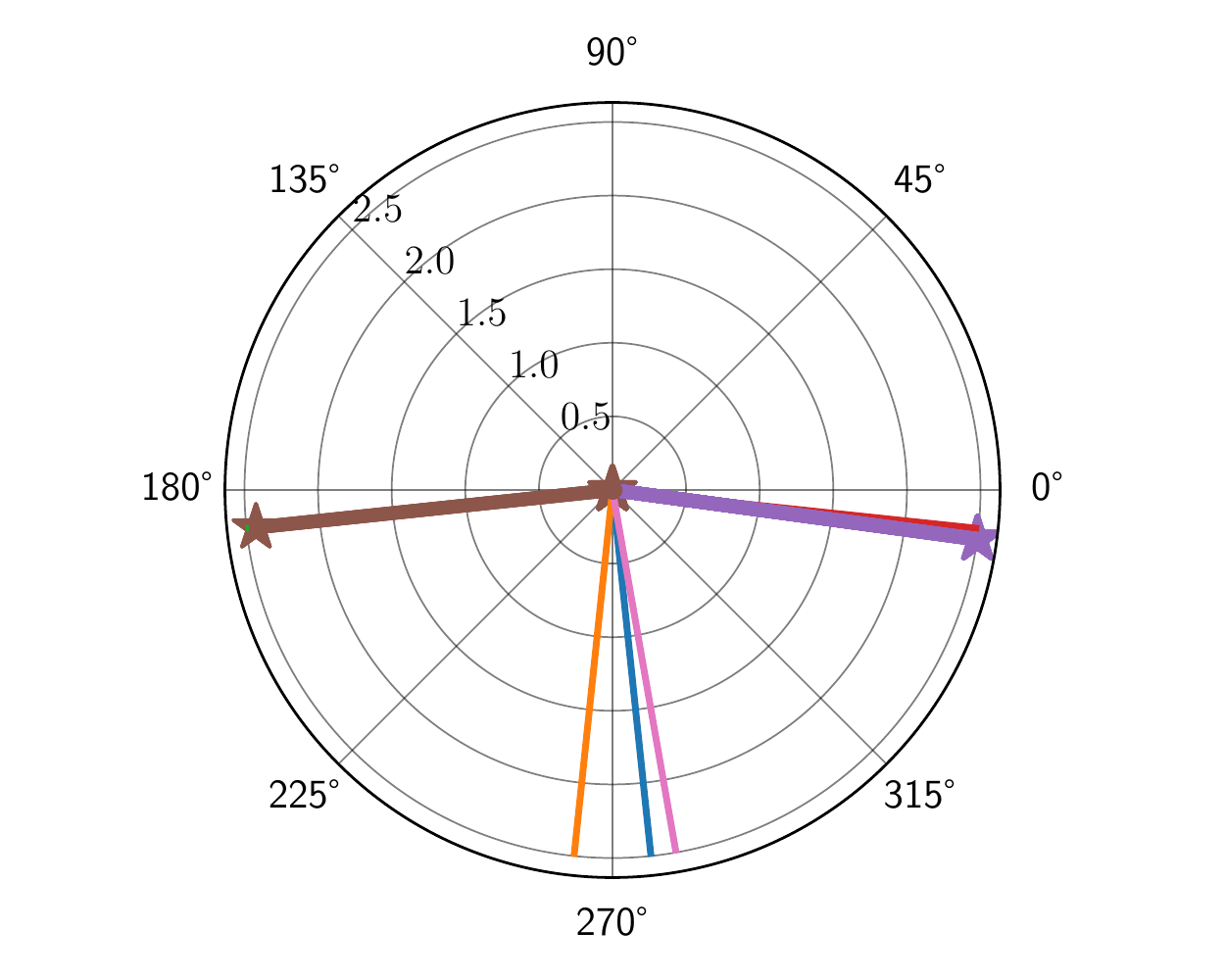}}
    \hspace{-.5cm}
    \subfigure[\small $t=60000$]
    {\includegraphics[width=3cm]{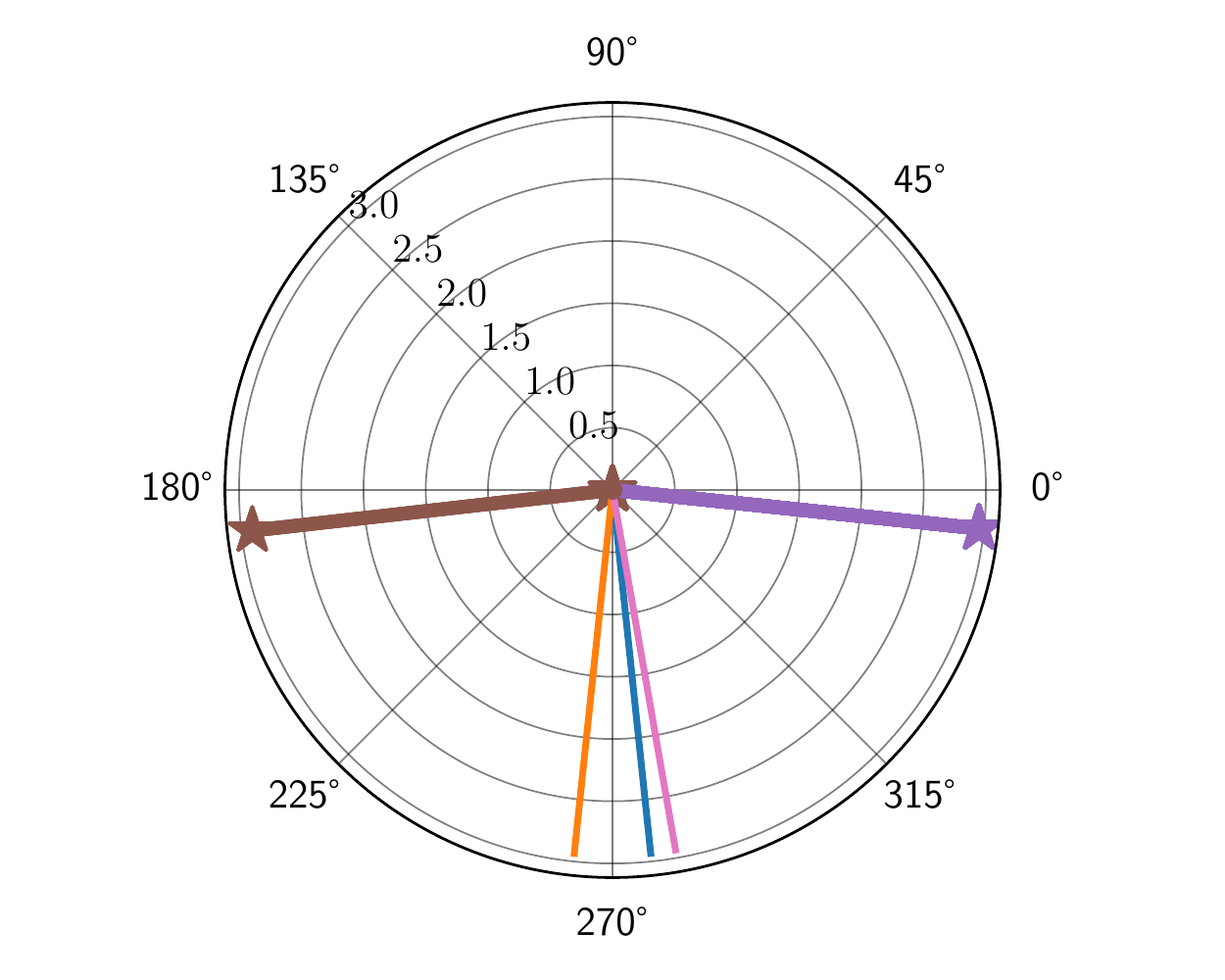}}
    \hspace{-.5cm}
    \subfigure[\small $t=150000$]
    {\includegraphics[width=3cm]{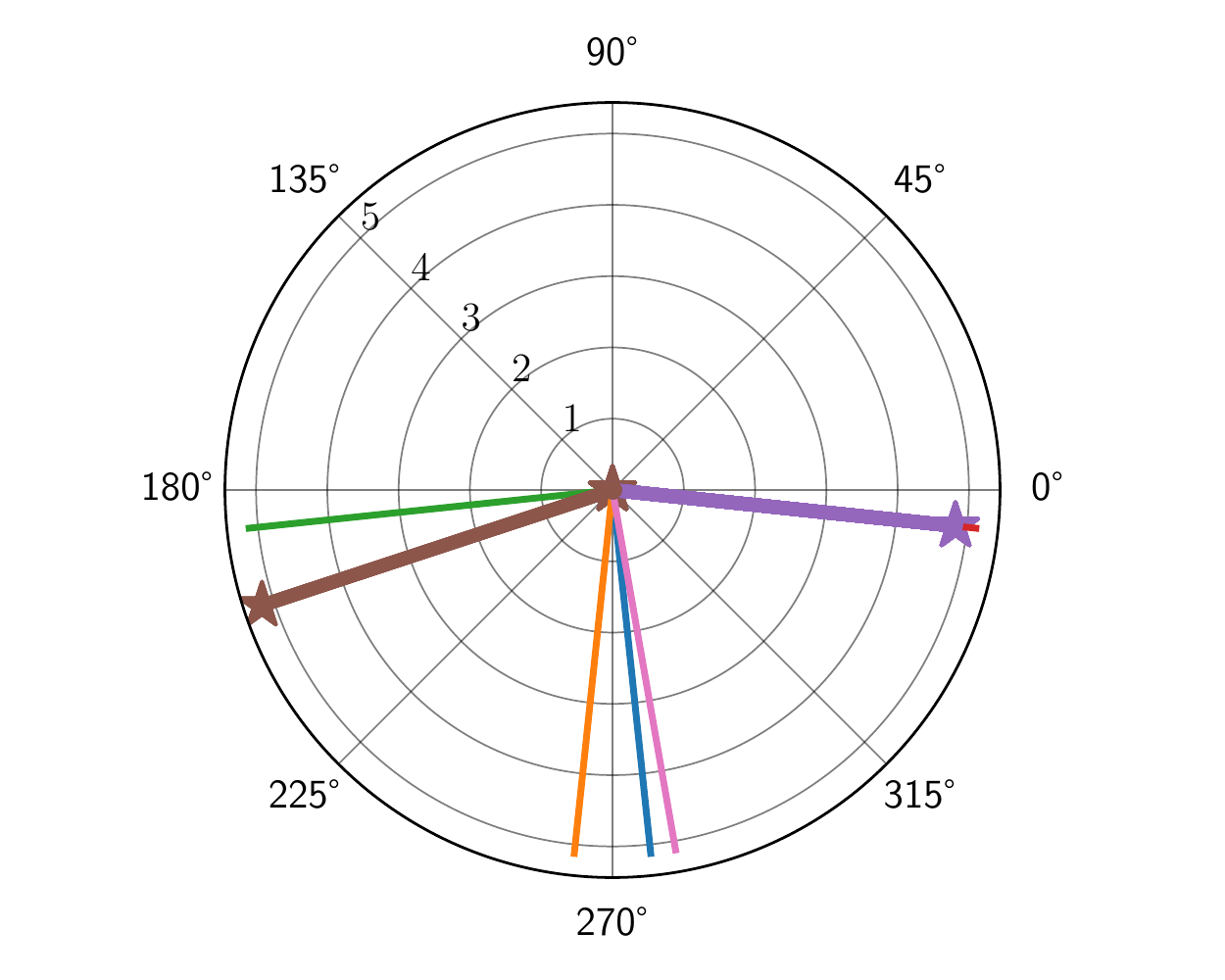}}
\end{center}
\caption{ These figures visualize 
(in polar coordinates) the projections of all neurons $\{\bb_k(t)\}_{k\in[m]}$ onto the $2$d subspace ${\rm span}\{\bx_+,\bx_-\}$ during training. Each purple star represents a positive neuron $(k\in[m/2])$, while each brown star represents a negative neuron $(k\in[m]-[m/2])$. Additionally, the directions of $\bx_+,\bx_-,\bx_+^{\perp},\bx_-^{\perp},\bmu$ are plotted in blue, orange, green, red and pink colors, respectively. The complete version of these figures is Figure~\ref{fig: whole dynamics} in Appendix~\ref{appendix: experiment: without noise}.}
\label{fig: dynamics}
\end{figure}


From Fig~\ref{fig: dynamics}(a) to (b) is the {\bf Phase I} of the dynamics, marked by a condensation of neurons.
Although the initial directions are random, we see that all neurons are rapidly divided into three categories: living positive neurons $(k\in\cK_+)$ and living negative neurons $(k\in\cK_-)$ condense in one direction each ($\bmu$ and $\bx_+^{\perp}$), while other neurons $(k\notin\cK_+\cup\cK_-)$ are deactivated forever.
From the perspective of loss landscape, GF rapidly escapes from the saddle near $\bzero$ where the loss gradient vanishes.

From Fig~\ref{fig: dynamics}(b) to (c) is the {\bf Phase II} of the dynamics, in which GF gets stuck into a plateau with training accuracy $\frac{p}{1+p}$ for a long time $T_{\rm plat}$ before escaping. Once the dynamics escapes from the plateau, the training accuracy rises to a perfect $100\%$. Moreover, activation patterns do not change in this phase.

From Fig~\ref{fig: dynamics}(c) to (d) is the {\bf Phase III} of the dynamics. The phase transition from phase II to phase III
sees a rapid deactivation of all the living positive neurons $k\in\cK_+$ on $\bx_-$ rapidly, while other activation patterns are unchanged. This leads to a simpler network in phase III, in which only living positive neurons (in $\cK_+$) predict $\bx_+$, and only living negative neurons (in $\cK_-$) predict $\bx_-$. Hence, in this phase the GF tries to learn the training data using almost the simplest network by only changing the norms of the neurons. 

Finally, Fig~\ref{fig: dynamics}(d) to (e) shows {\bf Phase IV}, starting from another ``phase transition'' when all the living negative neurons ($k\in\cK_-$) reactivate simultaneously on $\bx_+$.
This leads to a more complicated network. After the phase transition, the activation patterns no longer change, and the neurons eventually converges towards some specific directions dependent on both data and initialization. Additionally, this direction is not even the local optimal max margin direction.

{\bf Overall}, the whole dynamics exhibit a simplifying-to-complicating learning trend.

In the following four subsections, we present a meticulously detailed and comprehensive depiction of the whole optimization dynamics and nonlinear behaviors.
For clarity, in Figure~\ref{fig: timeline}, we first display the timeline of our dynamics and some nonlinear behaviors.

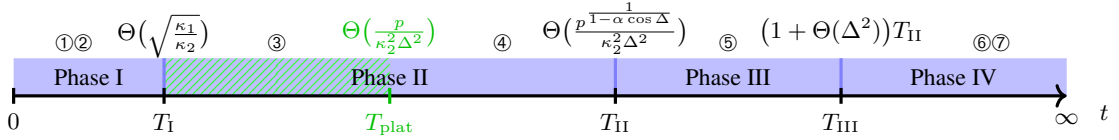
\begin{figure}[!h]
\begin{tikzpicture}[very thick, black]
\small
\coordinate (O) at (-1,0); 
\coordinate (P1) at (1,0);
\coordinate (P2) at (7,0);
\coordinate (P3) at (10,0);
\coordinate (F) at (13,0); 
\coordinate (E1) at (5,0); 
\coordinate (E2) at (0.5,0); 

\fill[color=blue!25] rectangle (O) -- (P1) -- ($(P1)+(0,0.5)$) -- ($(O)+(0,0.5)$); 

\fill[color=blue!25] rectangle (P1) -- (P2) -- ($(P2)+(0,0.5)$) -- ($(P1)+(0,0.5)$); 
\path [pattern color=green!80, pattern=north east lines, line width = 1pt, very thick] rectangle ($(1,0)$) -- ($(4,0)$) -- ($(4,0.5)$) -- ($(1,0.5)$);
\fill[color=blue!25] rectangle (P2) -- (P3) -- ($(P3)+(0,0.5)$) -- ($(P2)+(0,0.5)$); 
\fill[color=blue!25] rectangle (P3) -- (F) -- ($(F)+(0,0.5)$) -- ($(P3)+(0,0.5)$); 

\draw ($(0, 0.25)$) node[activity,black] {Phase I};
\draw ($(4.0, 0.25)$) node[activity,black] {Phase II};
\draw ($(8.5, 0.25)$) node[activity,black] {Phase III};
\draw ($(11.5, 0.25)$) node[activity,black] {Phase IV};


	


\draw[color=blue!50] (P1) -- ($(P1)+(0,0.5)$);
\draw[color=blue!50] (P2) -- ($(P2)+(0,0.5)$);
\draw[color=blue!50] (P3) -- ($(P3)+(0,0.5)$);

\draw[->] (O) -- (F);
\foreach \x in {-1,1,7,10}{\draw(\x cm,3pt) -- (\x cm,-3pt);};
\draw[color=green!75!black] (4cm, 3pt) -- (4cm, -3pt);
\draw[color=green!75!black] (4, 0) node[below=3pt] {$T_{\rm plat}$};
\draw (4, 0.8) node[activity,green!75!black] {$\Theta\big(\frac{p}{\kappa_2^2\Delta^2}\big)$};

\draw (-1, 0) node[below=3pt] {$0$};
\draw (1, 0) node[below=3pt] {$T_{\rm I}$};
\draw (1, 0.8) node[activity,black] {$\Theta\big(\sqrt{\frac{\kappa_1}{\kappa_2}}\big)$};
\draw (7, 0) node[below=3pt] {$T_{\rm II}$};
\draw (7, 0.9) node[activity,black] {$\Theta\big(\frac{p^{\frac{1}{1-\alpha\cos\Delta}}}{\kappa_2^2\Delta^2}\big)$};
\draw (10, 0) node[below=3pt] {$T_{\rm III}$};
\draw (10, 0.85) node[activity,black] {$\big(1+\Theta(\Delta^2)\big)T_{\rm II}$};
\draw (13, 0) node[below=3pt] {$\infty$};
\draw (13.5, 0) node[below=0pt] {$t$};


\draw (-0.2, 0.7) node[activity,black] {\ding{192}\ding{193}};
\draw (2.5, 0.7) node[activity,black] {\ding{194}};
\draw (5.5, 0.7) node[activity,black]  {\ding{195}};
\draw (8.5, 0.7) node[activity,black]  {\ding{196}};
\draw (12.0, 0.7) node[activity,black] {\ding{197}\ding{198}};

\end{tikzpicture}
\caption{
Timeline of the four-phase optimization dynamics, containing some key time points $T_{\rm I},T_{\rm II},T_{\rm III},T_{\rm plat}$ and their theoretical estimates, and some basic nonlinear behaviors: \ding{192} initial condensation, \ding{193} saddle escape, \ding{194} getting stuck in plateau, \ding{195} plateau escape, \ding{196} neuron deactivation, \ding{197} neuron reactivation, \ding{198} initialization-dependent directional convergence. Notice \ding{192}$\sim$\ding{198} are only some basic nonlinear behaviors. Moreover, \ding{193}+\ding{194} is saddle-to-plateau, \ding{192}+\ding{196}+\ding{197} is simplifying-to-complicating.}
\label{fig: timeline}
\end{figure}

In Appendix~\ref{appendix: experiment: bounds}, we further validate our theoretical bounds on the key time points in Figure~\ref{fig: timeline}  numerically. Additionally, in Appendix~\ref{appendix: experiment: with noise}, we relax the data Assumption~\ref{ass: data} by perturbing the data with random noise, and our experimental results illustrate that similar four-phase dynamics and nonlinear behaviors persist. 


\subsection{Phase I. Initial Condensation and Saddle Escape}\label{subsection: phase I}

Let $T_{\rm I}=10\sqrt{\frac{\kappa_1}{\kappa_2}}$, and we call $t\in[0,T_{\rm I}]$ Phase I. The theorem below is our main result in Phase I.

\begin{theorem}[Initial Condensation]\label{thm: GF Phase I}
Let the width $m=\Omega\left(\log(1/\delta)\right)$, the initialization $\kappa_1,\kappa_2=\cO(1)$ and $\kappa_1/\kappa_2=\cO(\Delta^8)$. Then with probability at least $1-\delta$, the following results hold at $T_{\rm I}$:

{\bf (S1)} Let $\mathcal{K}_+$ be the index set of living positive neurons at $T_{\rm I}$, i.e. $\mathcal{K}_+:=\{k\in[m/2]:\text{\rm\sgn}_k^+(T_{\rm I})=1\text{ or }\text{\rm\sgn}_{k}^-(T_{\rm I})=1\}$.
Then, {\bf (i)} $0.21m\leq|\mathcal{K}_+(T_{\rm I})|\leq 0.29m$. Moreover, for any $k\in\mathcal{K}_+$, {\bf (ii)} its norm is small but significant: $\rho_k(T_{\rm I})=\Theta\bracket{\sqrt{\frac{\kappa_1\kappa_2}{m}}}$; {\bf(iii)} Its direction is strongly aligned with $\boldsymbol{\mu}$: $\left<\boldsymbol{w}_k(T_{\rm I}),\boldsymbol{\mu}\right>\geq1-\cO\left(\sqrt{\kappa_1\kappa_2}\right)-\cO\bracket{({\kappa_1}/{\kappa_2})^{0.55}}$; {\bf (iv)} $\text{\rm\sgn}_{k}^+(T_{\rm I})=\text{\rm\sgn}_{k}^-(T_{\rm I})=1$.

{\bf (S2)} Let $\mathcal{K}_-$ be the index set of living negative neurons at $T_{\rm I}$, i.e. $\mathcal{K}_-:=\{k\in[m]-[m/2]:\text{\rm\sgn}_k^+(T_{\rm I})=1\text{ or }\text{\rm\sgn}_{k}^-(T_{\rm I})=1\}$. Then, {\bf (i)} $0.075m\leq|\mathcal{K}_-|\leq 0.205m$. Moreover, for any $k\in\mathcal{K}_-$, {\bf (ii)} its norm is tiny: $\rho_k(T_{\rm I})=\cO\left(\frac{\sqrt{\kappa_1\kappa_2}}{\sqrt{m}}\big(\sqrt{\frac{\kappa_1}{\kappa_2}}+\frac{\Delta}{p}\big)\right)$; {\bf (iii)} its direction is aligned with $\bx_+^{\perp}$: $\left<\boldsymbol{w}_k(T_{\rm I}),\bx_+^{\perp}\right>\geq1-\cO\left((\sqrt{\frac{\kappa_1}{\kappa_2}}\frac{p}{\Delta})^{1.6}\right)$; {\bf (iv)} $\text{\rm\sgn}_{k}^-(T_{\rm I})=1$, but $\text{\rm\sgn}_k^+(T_{\rm I})=0$.

{\bf (S3)} For other neuron $k\notin\cK_+\cup\cK_-$, it dies and
remains unchanged during the remaining training process: $\text{\rm\sgn}_k^+(t)\leq0,\text{\rm\sgn}_k^-(t)\leq0,\bb_k(t)\equiv\bb_k(T_{\rm I}),\ \forall t\geq T_{\rm I}$.

{\bf (S4).} $f_+(T_{\rm I})=\Theta\left(\kappa_2\sqrt{\kappa_1\kappa_2}\right),f_-(T_{\rm I})=\Theta\left(\kappa_2\sqrt{\kappa_1\kappa_2}\right)$, and $\Acc(T_{\rm I})=\frac{p}{1+p}$.
\end{theorem}

{\bf Initial condensation and simplification.}
Theorem \ref{thm: GF Phase I} (S1)(S2)(S3) show that, after a short time $T_{\rm I}=\Theta(\sqrt{\kappa_1/\kappa_2})$, all neurons are implicitly simplified to three categories: $\cK_+$, $\cK_-$ and others. The living positive neurons $k\in\cK_+$ align strongly with $\bmu$, and the living negative neurons $k\in\cK_-$ align with $\bx_+^{\perp}$ and lie on the manifold orthogonal to $\bx_+$. Other neurons die and remain unchanged during the remaining training process. Moreover, we also estimate tight bounds for $|\cK_+|$ and $|\cK_+|$. 
Actually, $\kappa_1/\kappa_2=\cO(1)$ can ensure Theorem \ref{thm: GF Phase I} and initial condensation hold (please refer to Appendix~\ref{appendix: proof: Phase I}), and we write $\kappa_1/\kappa_2=\cO(\Delta^8)$ here to ensure that the dynamics of later phases hold. In Phase I, the dynamics exhibit a fast condensation phenomenon, i.e., in addition to dead neurons, living positive and negative neurons condense in one direction each.

{\bf Saddle-to-Plateau.} 
The network is initially close to the saddle point at $\bzero$ (where the loss gradient vanishes). However, Theorem \ref{thm: GF Phase I} (S1) reveals that despite being small, there is a significant growth in the norm of living positive neuron $k\in\cK_+$ from $\Theta({\kappa_1}/{\sqrt{m}})$ to $\Theta(\sqrt{\kappa_1\kappa_2}/\sqrt{m})$ and the predictions also experience substantial growth (S4). This means that GF escapes from this saddle rapidly. 
Furthermore, it is worth noting that initial training accuracy can randomly be $0,\frac{1}{1+p},\frac{p}{1+p},$ or $1$. 
However, after Phase I, the training accuracy reaches $\Acc(T_{\rm I})=\frac{p}{1+p}$ which we will prove as a plateau in the next subsection. Therefore, Phase I exhibits saddle-to-plateau dynamics.

\begin{remark}\label{rmk: K+, K-, alpha}
Throughout the following subsections, we call the neuron $k\in\cK_+$ the ``living positive neuron'', the neuron $k\in\cK_-$ the ``living negative neuron'', and the neuron $k\notin\cK_+\cup\cK_-$ the ``dead neuron''. Moreover, we denote $m_+:=|\cK_+|$, $m_-:=|\cK_-|$, and $\alpha:=\frac{m_-}{m_+}$. 
Notice that Theorem~\ref{thm: GF Phase I}(S1)(S2) guarantee that $0<\frac{0.075}{0.29}\leq\alpha\leq\frac{0.205}{0.21}<1$.
\end{remark}

\begin{remark} 
The results in the following subsections are all based on the occurrence of the events in Theorem \ref{thm: GF Phase I} and with the same settings as Theorem \ref{thm: GF Phase I}. So they all hold with probability at least $1-\delta$.
\end{remark}

Please refer to Appendix~\ref{appendix: proof: Phase I} for the proof of Phase I.

\subsection{Phase II. Getting Stuck in and Escaping from Plateau}\label{subsection: phase II}

In this phase, we study the dynamics before the patterns of living neurons change again after Phase I. Specifically, we define 
\begin{align*}
    T_{\rm II}:=\inf\{t>T_{\rm I}:\exists k\in\cK_+\cup\cK_-,\text{\rm\sgn}_k^+(t)\ne\text{\rm\sgn}_k^+(T_{\rm I})\text{ or }\text{\rm\sgn}_k^-(t)\ne\text{\rm\sgn}_k^-(T_{\rm I})\},
\end{align*} 
and call $t\in(T_{\rm I},T_{\rm II}]$ Phase II.

\begin{theorem}[End of Phase II]\label{thm: GF Phase II}
{\bf (S1)} $T_{\rm II}=\Theta\left(\frac{p^{\frac{1}{1-\alpha\cos\Delta}}}{\kappa_2^2\Delta^2}\right)$. 
{\bf (S2)} $\cL(\btheta(T_{\rm II}))=\Theta\left(p^{-\frac{1}{1-\alpha\cos\Delta}}\right)$.
{\bf (S3)} One of living positive neuron $k_0\in\cK_+$ precisely changes its pattern on $\bx_-$ at $T_{\rm II}$: $\lim\limits_{t\to T_{\rm II}^-}\text{\rm\sgn}_{k_0}^-(t)=1$ and $\lim\limits_{t\to T_{\rm II}^+}\text{\rm\sgn}_{k_0}^-(t)=0$, while all other activation patterns remain unchanged. 
\end{theorem}

Recalling the results in Theorem \ref{thm: GF Phase I}, during Phase II, the activation patterns do not change with $\sgn_k^+(t)=\sgn_k^-(t)=1$ for $k\in\cK_+$ and $\sgn_k^+(t)=0,\sgn_k^-(t)=1$ for $k\in\cK_-$. Theorem \ref{thm: GF Phase II} demonstrates that at the end of Phase II, except for one of living positive neuron $k_0\in\cK_+$ precisely changes its pattern on $\bx_-$, all other activation patterns remain unchanged. 

\begin{theorem}[Plateau]\label{thm: Plateau Estimate}
We define the hitting time 
$T_{\rm plat}:=\inf\{t\in[T_{\rm I},T_{\rm II}]:\Acc(t)=1\}$. Then, {\bf (S1)} $T_{\rm plat}=\Theta\left(\frac{p}{\kappa_2^2\Delta^2}\right)$; {\bf (S2)} $\forall t\in[T_{\rm I},T_{\rm plat}]$, ${\rm Acc}(t)\equiv\frac{p}{1+p}$; {\bf (S3)} $\forall t\in(T_{\rm plat},T_{\rm II}]$, ${\rm Acc}(t)\equiv1$.
\end{theorem}

{\bf Plateau of training accuracy.} 
According to Theorem \ref{thm: Plateau Estimate}, during Phase II, the training accuracy gets stuck in a long plateau $\frac{p}{1+p}$, which lasts for $\Theta\big(\frac{p}{\kappa_2^2\Delta^2}\big)$ time. 
However, once escaping from this plateau, the training accuracy rises to $100\%$. 
It is worth noting that this plateau is essentially induced by the dataset. All that's required is only mild imbalance ($p$ is slightly greater than 1 such that $p\cos\Delta>1$) and a small margin $\sin(\Delta/2)$ of two data classes. Notably, if the dataset has an extremely tiny margin $(\Delta\to0)$, then the length of this plateau will be significantly prolonged $(T_{\rm plat}\to+\infty)$, which implies how the data separation can affect the training dynamics.
Additionally, using a smaller initialization scale $\kappa_1$ of the input layers cannot avoid this plateau. 

Please refer to Appendix~\ref{appendix: proof: Phase II} for the proof of Phase II.

\subsection{Phase III. Simplifying by Neuron Deactivation, and Trying to Learn by Simplest Network}\label{subsection: phase III}

Building upon Phase II, we demonstrate that within a short time, all the living positive neurons $\cK_+$ change their activation patterns, corresponding to a ``phase transition''. Specifically, we define 
\begin{align*}
    T_{\rm II}^{\rm PT}:=\inf\{t>T_{\rm II}:\forall k\in\cK_+,\sgn_k^-(t)=0\},
\end{align*}
and we call $t\in(T_{\rm II},T_{\rm II}^{\rm PT}]$ the phase transition from Phase II to Phase III.

\begin{theorem}[Phase Transition]\label{thm: GF Phase transition II-to-III} {\bf (S1)} $T_{\rm II}^{\rm PT}=\bracket{1+\cO\left(\sqrt{\kappa_1\kappa_2^3}\right)}T_{\rm II}$; {\bf (S2)} $\text{\rm\sgn}_k^+(T_{\rm II}^{\rm PT})=1$ and $\text{\rm\sgn}_k^-(T_{\rm II}^{\rm PT})=0$ for any $k\in\cK_+$; $\text{\rm\sgn}_k^+(T_{\rm II}^{\rm PT})=0$ and $\text{\rm\sgn}_k^-(T_{\rm II}^{\rm PT})=1$ for any $k\in\cK_-$.
\end{theorem}

{\bf Neuron deactivation.} 
As shown in Theorem~\ref{thm: GF Phase transition II-to-III} (S2), after the phase transition, \textit{all} the living positive neurons $k\in\cK_+$ undergo deactivation for $\bx_-$, i.e., $\sgn_k^-(t)$ changes from 1 to 0, while other activation patterns remain unchanged. Furthermore, Theorem~\ref{thm: GF Phase transition II-to-III} (S1) reveals that the phase transition is completed quite quickly by using sufficiently small initialization value $\kappa_1,\kappa_2$. 
A smaller initialization value leads to a more precise initial condensation $\bw_k(T_{\rm I})\approx\bmu$, causing all living positive neurons to remain closer together before $T_{\rm II}$ and thus changing their patterns nearly simultaneously.

{\bf Self-simplifying.} 
As a result, the network implicitly simplifies itself through the deactivation behavior. At $T_{\rm III}^{\rm PT}$, only living negative neurons $k\in\cK_+$ are used for predicting on $\bx_-$, i.e., $f_-(T_{\rm II}^{\rm PT})=\frac{\kappa_2}{\sqrt{m}}\sum_{k\in\cK_-}\sigma(\<\bb_k(T_{\rm II}^{\rm PT}),\bx_-\>)$. In contrast, during Phase II, both living positive and living negative neurons jointly predict on $\bx_-$, i.e., $f_-(t)=\frac{\kappa_2}{\sqrt{m}}\sum_{k\in\cK_+}\sigma(\<\bb_k(t),\bx_-\>)-\frac{\kappa_2}{\sqrt{m}}\sum_{k\in\cK_-}\sigma(\<\bb_k(t),\bx_-\>)$. As indicated in Table \ref{table: evolution of activation patterns}, two classes of activation patterns are simplified from $(1,0)$ to $(0,0)$, while others do not change.

After this phase transition, we study the dynamics before the patterns of living neurons change again. Specifically, we define 
\begin{align*}
    T_{\rm III}:=\inf\{t>T_{\rm II}^{\rm PT}:\exists k\in\cK_+\cup\cK_-,\text{\rm\sgn}_k^+(t)\ne\text{\rm\sgn}_k^+(T_{\rm II}^{\rm PT})\text{ or }\text{\rm\sgn}_k^-(t)\ne\text{\rm\sgn}_k^-(T_{\rm II}^{\rm PT})\},
\end{align*} 
and call $t\in(T_{\rm II},T_{\rm III}]$ Phase III.

\begin{theorem}[End of Phase III]\label{thm: GF Phase III}
    $T_{\rm III}=\big(1+\Theta(\Delta^2)\big)T_{\rm II}$.
\end{theorem}

{\bf Learning by simplest network.}
During $t\in(T_{\rm II}^{\rm PT},T_{\rm III})$, all activation patterns do not change. This ensures that $f_+(t)=\frac{\kappa_2}{\sqrt{m}}\sum_{k\in\cK_+}\sigma(\<\bb_k(t),\bx_+\>)$, while $f_-(t)=-\frac{\kappa_2}{\sqrt{m}}\sum_{k\in\cK_-}\sigma(\<\bb_k(t),\bx_-\>)$. 
Additionally, by using sufficiently small $\kappa_1$, the neurons in $\cK_+$ and $\cK_-$ keep close together respectively before $T_{\rm III}$, making the network close to a simple two-neuron network consisting of one positive neuron and one negative neuron. Please refer to Appendix~\ref{appendix: proof: Phase III} for more details.
Furthermore, this pattern scheme is almost the "simplest" way to ensure binary classification: the living positive neurons only predict positive data $\bx_+$ while the living negative neurons only predict negative data $\bx_-$. Therefore, GF tries to learn by this almost simplest network in this phase.

Please refer to Appendix~\ref{appendix: proof: Phase III} for the proof of Phase III. 


\subsection{Phase IV. Complicating by Neuron Reactivation, and Directional Convergence}\label{subsection: phase IV}

Phase IV begins with an instantaneous phase transition at time $T_{\rm III}$.

\begin{theorem}[Phase Transition]\label{thm: GF Phase transition III-to-IV}
All living negative neuron $k\in\cK_-$ simultaneously change their patterns on $\bx_+$ at $T_{\rm III}$: $\lim\limits_{t\to T_{\rm III}^-}\text{\rm\sgn}_{k}^+(t)=0$, $\lim\limits_{t\to T_{\rm III}^+}\text{\rm\sgn}_{k}^+(t)=1$, while others remain unchanged.
\end{theorem}

{\bf Neuron reactivation.} According to Theorem~\ref{thm: GF Phase transition III-to-IV}, \textit{all} of living negative neurons $k\in\cK_-$ reactivate \textit{simultaneously} on $\bx_+$ at $T_{\rm III}$: $\sgn_k^+(t)$ changes from 0 to 1, while other activation patterns remain unchanged. 

{\bf Self-Complicating.} 
Along with the reactivation behavior, GF implicitly complicates itself. In Phase III, only living negative neurons $k\in\cK_+$ are used to predict on $\bx_+$, i.e., $f_+(t)=\frac{\kappa_2}{\sqrt{m}}\sum_{k\in\cK_+}\sigma(\<\bb_k(t),\bx_+\>)$. 
In contrast, after the phase transition at $T_{\rm III}$,  both living positive and living negative neurons jointly predict on $\bx_+$, i.e. $f_+(t)=\frac{\kappa_2}{\sqrt{m}}\sum_{k\in\cK_+}\sigma(\<\bb_k(t),\bx_+\>)-\frac{\kappa_2}{\sqrt{m}}\sum_{k\in\cK_-}\sigma(\<\bb_k(t),\bx_+\>)$. 
As indicated in Table~\ref{table: evolution of activation patterns}, two classes of activation patterns are complicated from $(0,0)$ to $(0,1)$, while others do not change.

In this phase, we study the dynamics before activation patterns change again after the phase transition in Theorem~\ref{thm: GF Phase transition III-to-IV}. We define the hitting time:
\begin{align*}
    T_{\rm IV}:=\inf\{t>T_{\rm III}:\exists k\in\cK_+\cup\cK_-,\text{\rm\sgn}_k^+(t)\ne \lim\limits_{s\to T_{\rm III}^+}\text{\rm\sgn}_k^+(s)\text{ or }\text{\rm\sgn}_k^-(t)\ne \lim\limits_{s\to T_{\rm III}^+}\text{\rm\sgn}_k^-(s)\},
\end{align*}
and we call $t\in(T_{\rm III},T_{\rm IV}]$ Phase IV.

\begin{theorem}[Phase IV]\label{thm: GF Phase IV}
{\bf (S1)} $T_{\rm IV}=+\infty$. Moreover, for any $t>T_{\rm III}$,
{\bf (S2)} all activation patterns do not change; 
{\bf (S3)} the loss converges with 
$\cL(\btheta(t))=\Theta\left(\frac{1}{p^{\frac{1}{1-\alpha\cos\Delta}}+\kappa_2^2\Delta^2(t-T_{\rm III})}\right)$.
\end{theorem}

Theorem~\ref{thm: GF Phase IV} illustrates that all activation patterns never change again after the phase transition at $T_{\rm III}$ with $\sgn_k^+(t)=1,\sgn_k^-(t)=0$ for any $k\in\cK_+$ and $\sgn_k^+(t)=\sgn_k^-(t)=1$ for any $k\in\cK_-$. Additionally, the loss converges with the polynomial rate $\Theta(1/\kappa_2^2\Delta^2 t)$. Furthermore, we present the following theorem about the convergent direction of each neuron.

\begin{theorem}[Directional Convergence]\label{thm: GF  directional convergence}
The limit $\lim\limits_{t\to+\infty}\frac{\btheta(t)}{\norm{\btheta(t)}_2}$ exists and denoted by $\overline{\btheta}=(\overline{\bb}_1^\top,\cdots,\overline{\bb}_m^\top)^\top\in\bbS^{md-1}$. Moreover, (i) for any $k\notin\cK_+\cup\cK_-,\overline{\bb}_k=\bzero$; (ii) for any $k\in\cK_+,\overline{\bb}_k\equiv\bv_+=C\big(\bx_+-\bx_-\cos\Delta\big)$; (iii) for any $k\in\cK_-,\overline{\bb}_k\equiv\bv_-=C\big(\big(1+\frac{\sin^2\Delta}{\alpha(1+\cos\Delta)}\big)\bx_--\bx_+\big)$, 
where $C>0$ is a scaling constant such that $\norm{\overline{\btheta}}_2=1$. (iv) Additionally, $f_-(\overline{\btheta})=-f_+(\overline{\btheta})>0$.
\end{theorem}

{\bf Initialization-dependent Directional Convergence.} As an asymptotic result, Theorem~\ref{thm: GF directional convergence} provides the final convergent direction of GF. 
All living positive neurons ($k\in\cK_+$) directionally converge to $\bv_+\parallel\bx_-^{\perp}$ with $\<\bv_+,\bx_+\>>0$ and $\<\bv_+,\bx_-\>=0$, while all living negative neurons ($k\in\cK_-$) directionally converge to $\bv_-\in{\rm span}\{\bx_+,\bx_-\}$ with $\<\bv_-,\bx_+\>>0$ and $\<\bv_-,\bx_-\>>0$. 
It is worth noting that $\bv_-$ directly depends not only on the data but also on the ratio $\alpha=|\cK_-|/|\cK_+|$ (defined in Remark~\ref{rmk: K+, K-, alpha}).
Recalling the results in Phase I, $\alpha$ lies in a certain range with high probability; but it is still a random variable due to its dependence on random initialization. Different initializations may lead to different values $|\cK_-|/|\cK_+|$ at the end of Phase I, eventually causing different convergent directions in Phase IV.

{\bf Non-(Local)-Max-Margin Direction.}
Lastly, we study the implicit bias of the final convergence rate. According to~\cite{lyu2019gradient,ji2020directional} and our results above, $\overline{\btheta}$ in Theorem~\ref{thm: GF  directional convergence} must be a KKT direction of some max-margin optimization problem. 
However, it is not clear whether the direction $\overline{\btheta}$ is actually an actual optimum of this problem.
Surprisingly, in next Theorem, we demonstrate that the final convergent direction is {\bf not even a local optimal} direction of this problem, which enlightens us to rethink the max margin bias of ReLU neural networks. 

\begin{theorem}[Implicit Bias]\label{thm: GF not optimum max margin}
The final convergent direction $\overline{\btheta}$ (in Theorem~\ref{thm: GF  directional convergence}) is a KKT direction of the max-margin problem
$\min:\frac{1}{2}\norm{\btheta}^2\ {\rm s.t.}\ y_if(\bx_i;\btheta)\geq1,i\in[n]$. However, $\overline{\btheta}$ is not even a local optimal direction of this problem.
\end{theorem}

Please refer to Appendix~\ref{appendix: proof: Phase IV} for the proof of Phase IV. 

\section{Discussion and Comparison on Nonlinear Behaviors}\label{section: comparison}

Throughout the whole training process in Section~\ref{section: GF dynamics}, we divide the phases based on the evolution of ReLU activation patterns. During Phase I, as well as the beginning of Phase II and III, numerous activation patterns undergo rapid changes. Table~\ref{table: evolution of activation patterns} summarizes the evolution of activation patterns for all living neurons after Phase I. These results are also numerically validated in Figure~\ref{fig: dynamics}.

\begin{table}[!ht]
\begin{center}
\small
\caption{ The evolution of two classes of activation patterns of living neurons after Phase I. As for other two classes, $\sgn_k^+(t)$ $(k\in\cK_+)$ and $\sgn_k^-(t)$ $(k\in\cK_-)$, they remain equal to $1$ after Phase I.}
\begin{tabular}{c|c|c|c|c}
    \hline\hline
    & $t\in(T_{\rm I},T_{\rm II})$ & $t\in(T_{\rm II},T_{\rm II}^{\rm PT})$ & $t\in(T_{\rm II}^{\rm PT},T_{\rm III})$ & $t\in(T_{\rm III},+\infty)$ 
    \\  \hline
    $\sgn_k^-(t)$ $(k\in\cK_+)$ & 1 & 1 or 0 & 0 & 0
    \\  \hline
    $\sgn_k^+(t)$ $(k\in\cK_-)$ & 0 & 0 & 0 & 1
    \\ \hline\hline
\end{tabular}
\label{table: evolution of activation patterns}
\end{center}
\end{table}

{\bf Simplifying-to-Complicating.} 
In phase I, GF simplifies all the neurons from random directions into three categories: living positive neurons $\cK_+$ and living negative neurons $\cK_-$ condense in one direction each, which other neurons are deactivated forever. After Phase I, as shown in Table \ref{table: evolution of activation patterns}, the two classes of activation patterns change from $(1,0)\overset{\text{simplify}}{\to}(0,0)\overset{\text{complicate}}{\to}(0,1)$, while other patterns remain unchanged. Therefore, the evolution of activation patterns exhibits a simplifying-to-complicating learning trend, which also implies that the network trained by GF learn features in increasing complexity.



{\bf Comparison with NTK.} 
In the lazy regime such as NTK, most neurons keep close to the initialization and most activation patterns do not change during training. Specifically, for any training data $\bx_i$, $\frac{1}{m}\sum_{k\in[m]}\bbI\{\sgn(\<\bb_k(t),\bx_i\>)\ne\sgn(\<\bb_k(0),\bx_i\>)\}= o(1),\forall t>0$~\citep{du2018gradient}.
However, our work stands out from lazy regime as activation patterns undergo numerous changes during training. In Phase I, initial condensation causes substantial changes in activation patterns, which is similarly observed in \citep{phuong2021inductive}. 
Furthermore, even after Phase I, there are notable modifications in activation patterns. 
As shown in Table \ref{table: evolution of activation patterns}, the proportion of changes in activation patterns for any given training data is the $\Theta(1)$, as compared with the $o(1)$ in NTK regime. 
Specifically, at any $t>T_{\rm III}$, $\frac{1}{m}\sum_{k\in[m]}\bbI\{\sgn_k^+(t)\ne\sgn_k^+(T_{\rm I})\}=\frac{1}{m}|\cK_-|=\Theta(1)$ and $\frac{1}{m}\sum_{k\in[m]}\bbI\{\sgn_k^-(t)\ne\sgn_k^-(T_{\rm I})\}=\frac{1}{m}|\cK_+|=\Theta(1)$.
On the other hand, in our analysis, the requirement on the network's width $m$ is only $m=\Omega(\log(1/\delta))$ (Theorem~\ref{thm: GF Phase I}), regardless of data parameters $p,\Delta$, while NTK regime requires a much larger width $m=\Omega(\log(p/\delta)/\Delta^6)$~\citep{ji2019polylogarithmic}.

{\bf Comparison with}~\cite{phuong2021inductive,lyu2021gradient,wang2022early,boursier2022gradient}. Beyond lazy regime and local analysis, these works also characterize the entire training dynamics and analyze a few nonlinear behaviors. Now we compare our results with these works in detail.
(i) While \cite{lyu2021gradient} focuses on training Leaky ReLU NNs, our work and the other three papers study ReLU NNs. It is worth noting that the dynamics of Leaky ReLU NNs differ from ReLU due to the permanent activation of Leaky ReLU $(\sigma'(\cdot)\geq\alpha>0)$.
(ii) Initial condensation is also proven in~\citep{lyu2021gradient,boursier2022gradient}, and the condensation directions are some types of data averages. In our work, neurons can aggregate towards not only the average direction $\bmu$, but also another direction $\bx_+^{\perp}$. Moreover, we also estimate the number of neurons that condense into two directions. 
(iii) Saddle-to-saddle dynamics are proven in \citep{phuong2021inductive} for square loss, where the second saddle is about training loss and caused by incomplete fitting. However, our work focus on classification with exponential loss and exhibit a similar saddle-to-plateau dynamics, where the plateau is about training accuracy, caused by incomplete feature learning.
(iv) Phased feature learning. In \citep{phuong2021inductive,wang2022early}, all features can be rapidly learned in Phase I (accuracy$=100\%$), followed by lazy training. However, for practical tasks on more complex data, NNs can hardly learn all features in such short time. In our work, the data is more difficult to learn, resulting in incomplete feature learning in Phase I (accuracy$<100\%$). Subsequently, NNs experience a long time to learn other features completely. Such multi-phase feature leaning dynamics are closer to practical training process.
(v) Neuron reactivation and deactivation. For ReLU NNs, The evolution of activation patterns is one of the essential causes of nonlinear dynamics. In~\citep{phuong2021inductive,wang2022early}, activation patterns only change rapidly in Phase I, after which they remain unchanged. In~\citep{boursier2022gradient}, their lemma 6 shows that their dynamics lack neuron reactivation. However, in our dynamics, even after Phase I, our dynamics exhibit significant neuron deactivation and reactivation as discussed in Table~\ref{table: evolution of activation patterns}.
(vi) The final convergent directions are also derived in~\citep{phuong2021inductive,lyu2021gradient,boursier2022gradient}, which only depend on the data. However, in our setting, the convergent direction is more complicated, determined by both data and random initialization.
(vii) Furthermore, our four-phase dynamics demonstrate the whole evolution of activation patterns during training and reveal a general simplifying-to-complicating learning trend.

In summary, our whole four-phase optimization dynamics capture more nonlinear behaviors than these works. Furthermore, we conduct a more thorough and detailed theoretical analysis of these nonlinear behaviors, providing a more systematic and comprehensive understanding.

\section{Conclusion and Future Work}
In this work, we study the optimization dynamics of ReLU neural networks trained by GF on a linearly separable data. Our analysis captures the whole optimization process starting from random initialization to final convergence. 
Throughout the whole training process, we reveal four different phases and identify rich nonlinear behaviors theoretically.
However, theoretical understanding of the training of NNs still has a long way to go. For instance, although we conduct a fine-grained analysis of GF, the dynamics of GD are more complex and exhibit other nonlinear behaviors such as progressive sharpening and edge of stability~\citep{wu2018sgd,jastrzkebski2018relation,cohen2021gradient,ma2022beyond,li2022analyzing,damian2022self,zhu2022understanding,ahn2022learning,ahn2022understanding}. 
Additionally, unlike GD, SGD uses only
mini-batches of data and injects noise~\citep{zhu2019anisotropic,thomas2020interplay,feng2021inverse,liu2021noise,ziyin2021minibatch,wu2022does,wojtowytsch2023stochastic,wang2023noise} in each iteration, which can have a pronounced impact on the optimization dynamics and nonlinear behaviors.
Better understanding of the nonlinear behaviors during GD or SGD training is an important direction of future work.

\newpage
\acksection
We thank Prof. Weinan E, Prof. Lei Wu, Prof. Zhi-Qin John Xu and anonymous reviewers for helpful suggestions. Mingze Wang is supported in part by the National Key Basic Research Program of China: 2015CB856000.



\newpage
\appendix

\begin{center}
    \LARGE \textbf{Appendix} 
\end{center}


\part{}
\setcounter{tocdepth}{1}
\localtableofcontents
\newpage

\newpage
\section{Experimental Details}
\label{appendix: experiment}
\vspace{-0.2cm}

All experiments are conducted on a MacBook pro 13 (M2) only using CPU.
See the code at~\href{https://github.com/wmz9/Understanding_Multi-phase_Optimization_NeurIPS2023}{\texttt{https://github.com/wmz9/Understanding\_Multi-phase\_Optimization\_NeurIPS2023}}.

\vspace{-0.2cm}
\subsection{Experiments on standard Dataset}
\label{appendix: experiment: without noise}
\vspace{-0.2cm}

We train the two-layer network on the dataset that satisfies Assumption~\ref{ass: data} with $d=20$, $p=4$ and $\Delta=\pi/15$. 
Specifically, we choose the network width $m=100$; the initialization scale $\kappa_1=0.1,\kappa_2=1$; the small learning rate $\eta=0.01$.

In Figure~\ref{fig: data, loss, acc}, we show some key data directions in this dataset, as well as the training accuracy, which contains a long plateau. Furthermore, in Figure~\ref{fig: whole dynamics}, we provide the evolution of all neurons during training from $t=0$ to $t=150000$, which is a more complete version of Figure~\ref{fig: dynamics}.

\vspace{-0.2cm}
\begin{figure}[!ht]
    \centering
    \includegraphics[width=4.2cm]{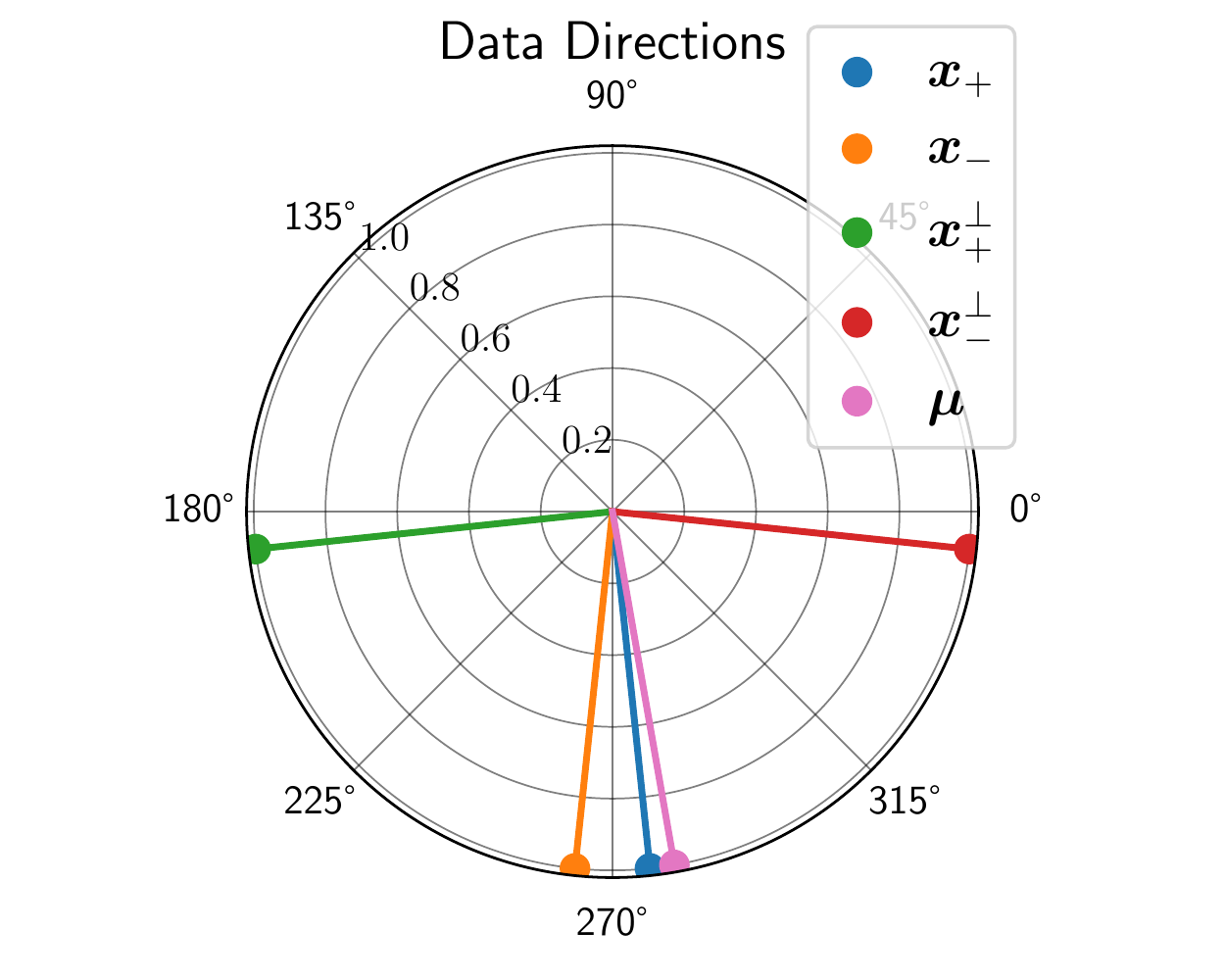}
    \includegraphics[width=4.2cm]{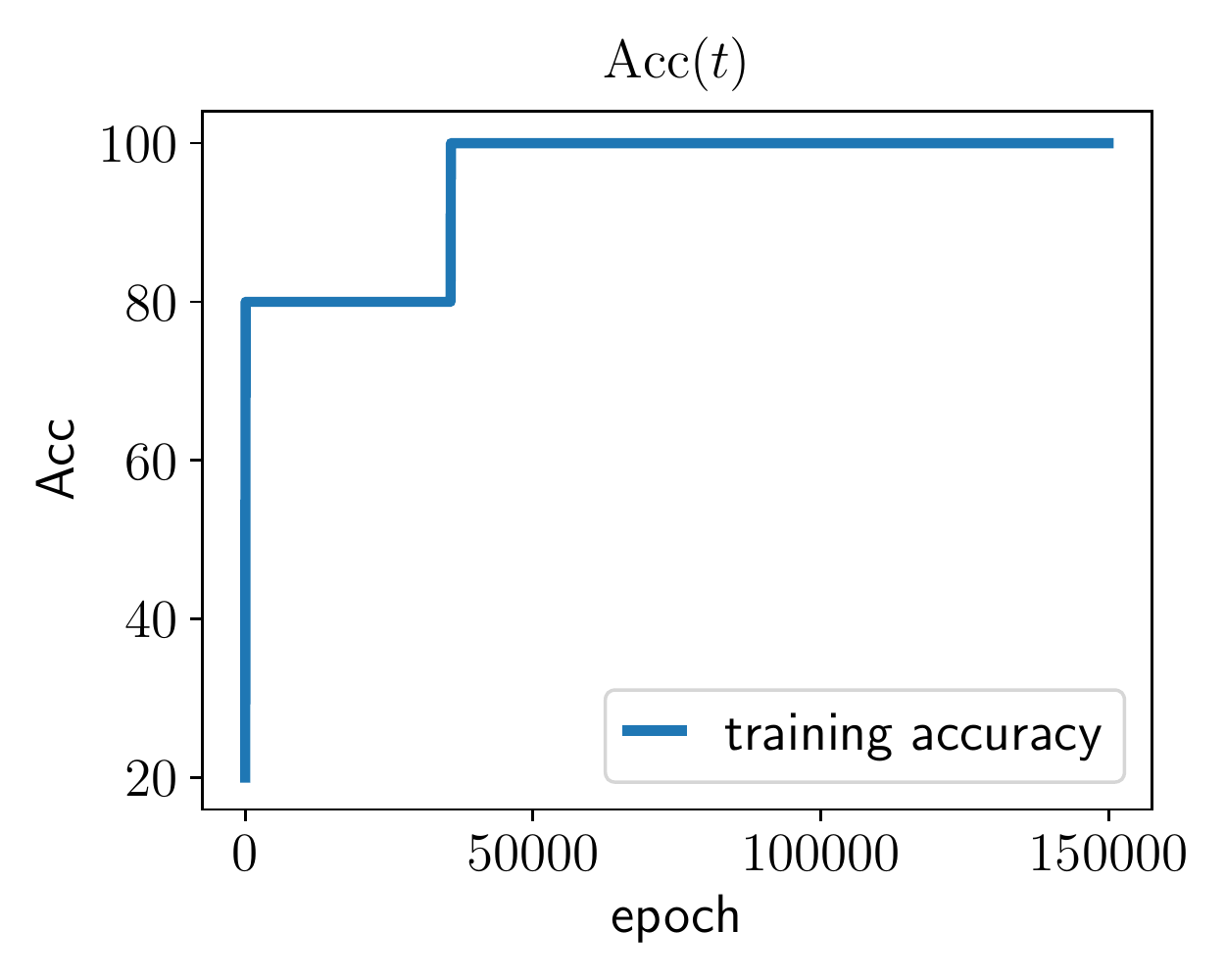}
    \vspace{-.2cm}
    \caption{(left) Some key data directions: the directions of $\bx_+,\bx_-,\bx_+^{\perp},\bx_-^{\perp},\bmu$ are plotted in blue, orange, green, red and pink colors, respectively; (right) The training accuracy.}
    \label{fig: data, loss, acc}
\end{figure}

\vspace{-0.8cm}
\begin{figure}[!ht]
    \centering
    \subfigure[\small $t=0$]
    {\includegraphics[width=4cm]{figures/0.pdf}}
    \subfigure[\small $t=200$]
    {\includegraphics[width=4cm]{figures/200.pdf}}
    \subfigure[\small $t=2000$]
    {\includegraphics[width=4cm]{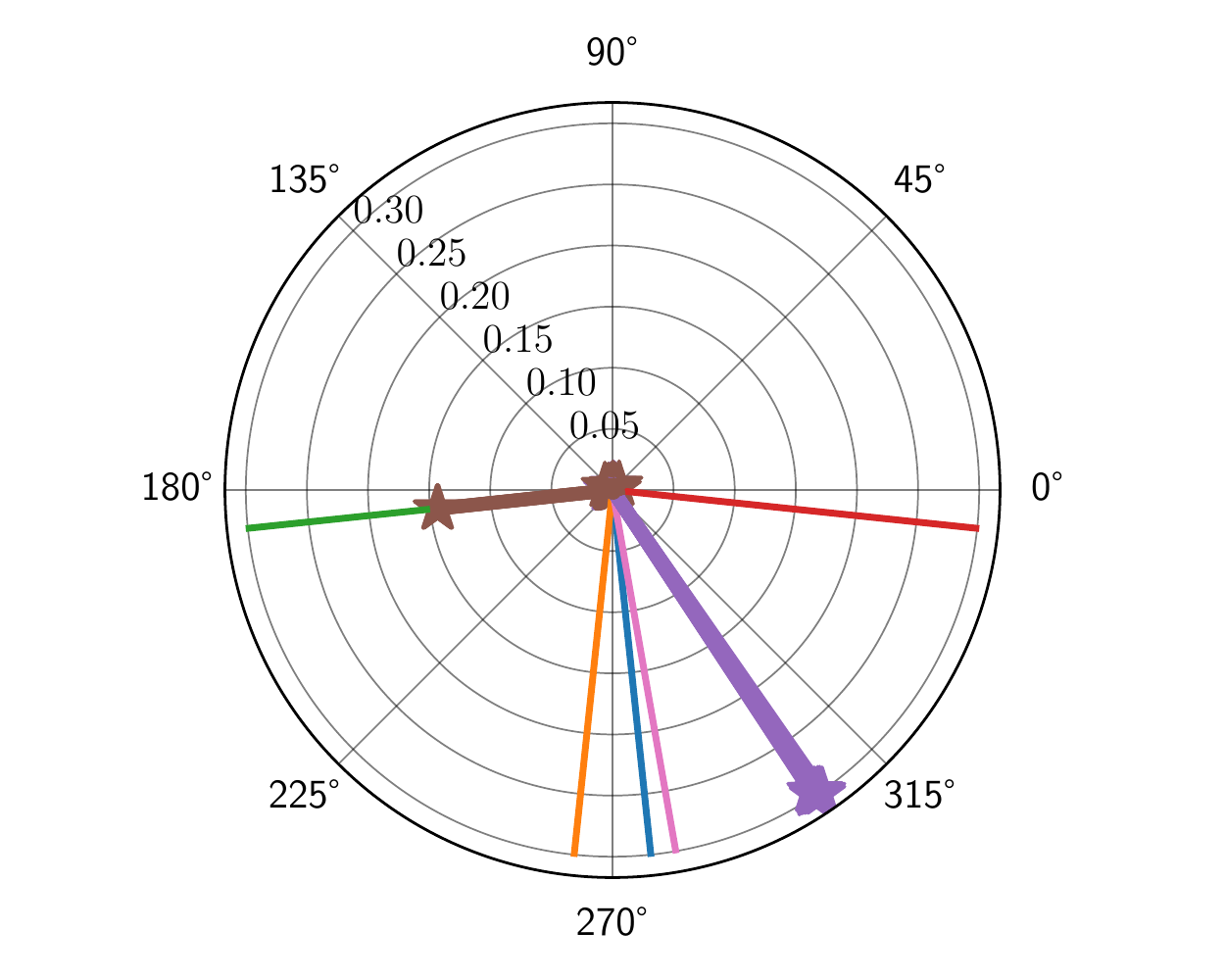}}
    \subfigure[\small $t=30000$]
    {\includegraphics[width=4cm]{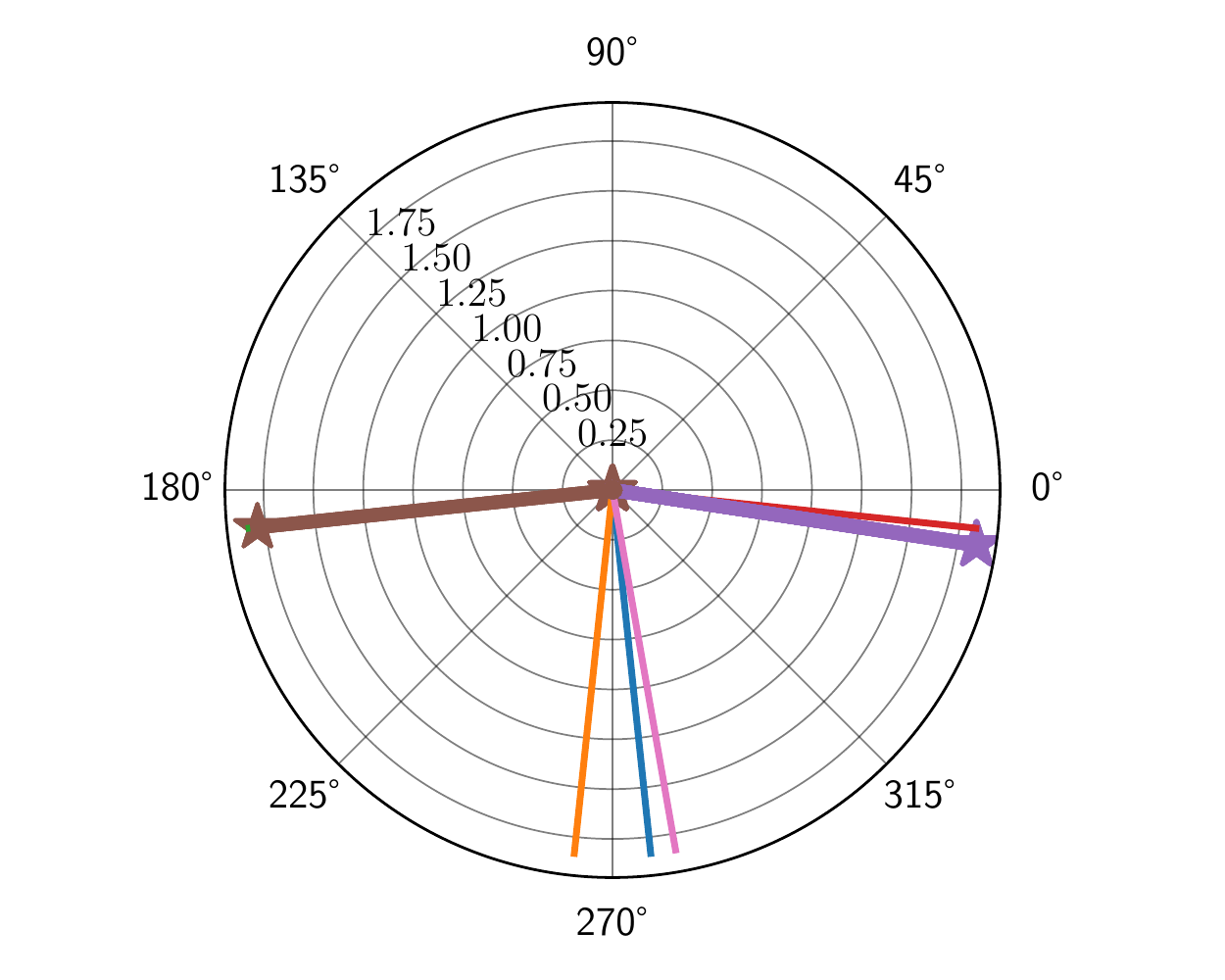}}
    \subfigure[\small $t=50000$]
    {\includegraphics[width=4cm]{figures/50000.pdf}}
    \subfigure[\small $t=60000$]
    {\includegraphics[width=4cm]{figures/60000.pdf}}
    \subfigure[\small $t=80000$]
    {\includegraphics[width=4cm]{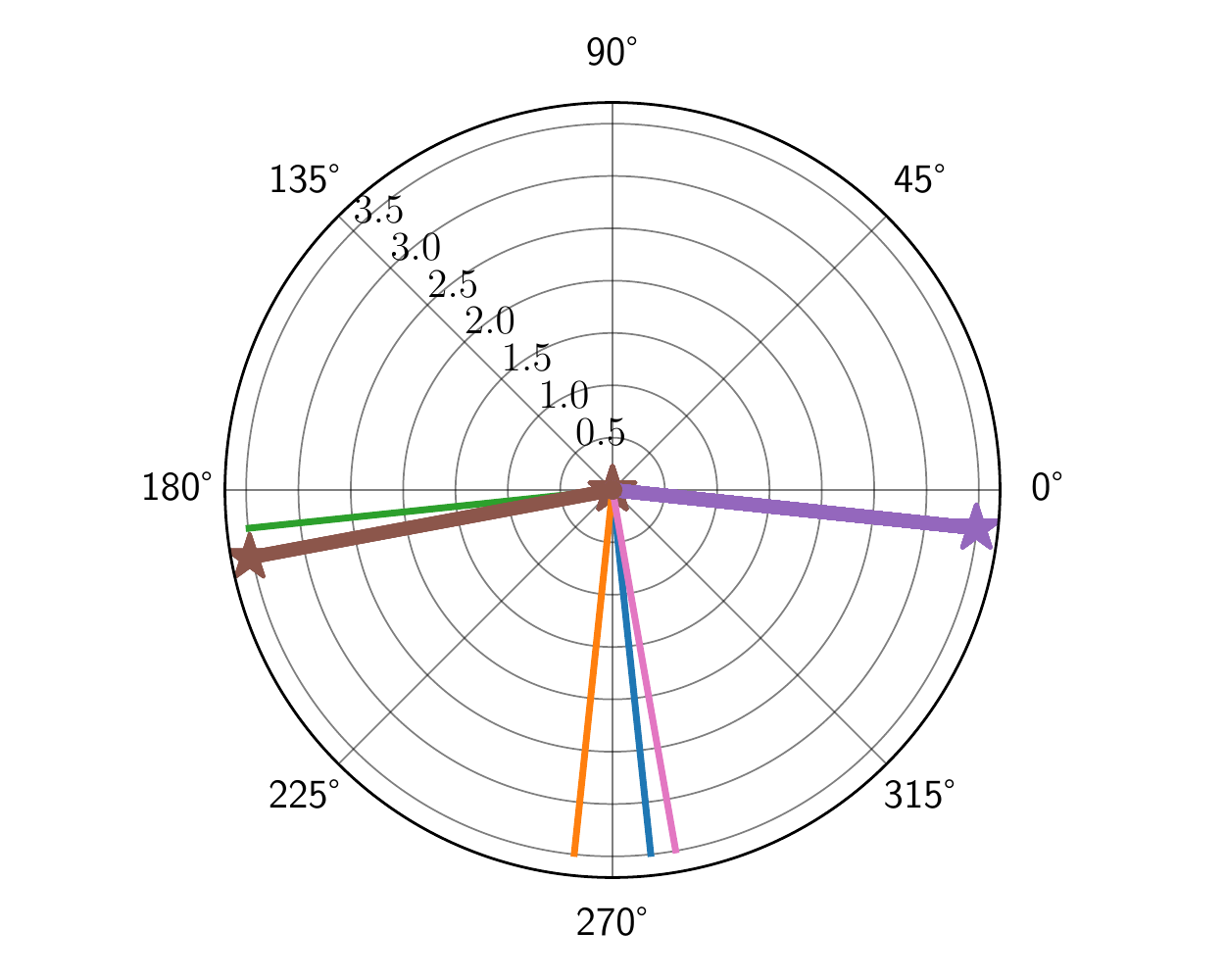}}
    \subfigure[\small $t=120000$]
    {\includegraphics[width=4cm]{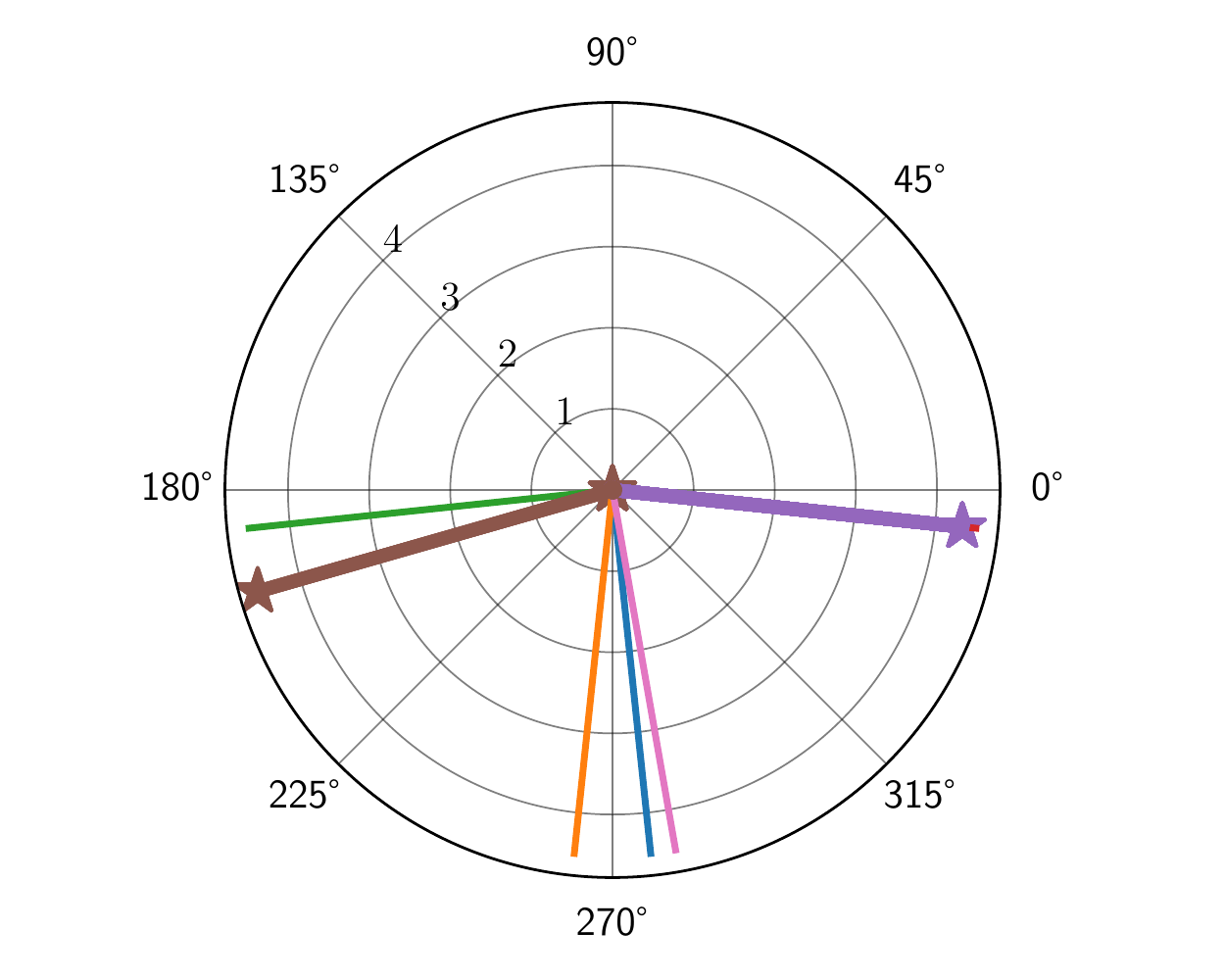}}
    \subfigure[\small $t=150000$]
    {\includegraphics[width=4cm]{figures/150000.pdf}}
    \vspace{-.2cm}
    \caption{(A more complete version of Figure~\ref{fig: dynamics})
    These figures visualize (in polar coordinates) the projections of all neurons $\{\bb_k(t)\}_{k\in[m]}$ onto the $2$d subspace ${\rm span}\{\bx_+,\bx_-\}$ during training (from $t=0$ to $t=150000$). 
    Each purple star represents a positive neuron $(k\in[m/2])$, while each brown star represents a negative neuron $(k\in[m]-[m/2])$. }
    \label{fig: whole dynamics}
\end{figure}

\newpage
\subsection{Numerical validation on our theoretical bounds}
\label{appendix: experiment: bounds}

We conduct experiments to validate our theoretical bounds on $T_{\rm plat}$ and $T_{\rm III}$ under different $p$ and $\Delta$, and the results are shown in ~\ref{table: numerical validation of timeline}.

In the first experiment (1st and 2nd subtable), we fix $p=4$ for change $\Delta$; in the second experiment (3rd and 4th subtable), we fix $\Delta=\pi/15$ and change $p$. As for other hyperparameters (such as $\kappa_1,\kappa_2,d,m$), we keep the same scales as our setups in Appendix~\ref{appendix: experiment: without noise}.

We have two main conclusions: (1) four training phases in our theory persistently exist; (the same as Fig 1 in Appendix A, and be omitted due to the limited space) (2) our theoretical estimates on $T_{\rm plat}$ and $T_{\rm III}$ are relatively tight, basically in the same magnitude as the realistic time.

\begin{table}[!ht]
    \centering
    \caption{The change of our theoretical bounds on $T_{\rm plat}$ and $T_{\rm III}$ under different $p$ and $\Delta$.}
    \begin{tabular}{c|c|c|c|c}
          \hline
         $\Delta$ & $\frac{4\pi}{45}$ & $\frac{4\pi}{45}\cdot\frac{3}{4}=\frac{\pi}{15}$ & $\frac{4\pi}{45}\cdot(\frac{3}{4})^2$ & $\frac{4\pi}{45}\cdot(\frac{3}{4})^3$ \\  \hline
        Realistic $T_{\rm plat}$ & $1.96\times10^4$ & $3.68\times10^4$ & $7.25\times10^4$ & $12.87\times10^4$
        \\ \hline	
        \makecell[c]{Our estimate $\Theta(1/\Delta^2)$: \\ $950/\Delta^2+1520/\Delta-9943$} & $1.90\times10^{4}$ & $3.82\times10^{4}$ & $7.14\times10^{4}$ & $12.89\times10^{4}$ \\ \hline	
    \end{tabular}

    \vspace{.1cm}

    \begin{tabular}{c|c|c|c|c}
      \hline
     $\Delta$ & $\frac{4\pi}{45}$ & $\frac{4\pi}{45}\cdot\frac{3}{4}=\frac{\pi}{15}$ & $\frac{4\pi}{45}\cdot(\frac{3}{4})^2$ & $\frac{4\pi}{45}\cdot(\frac{3}{4})^3$ \\  \hline
     Realistic $T_{\rm III}$ & $4.98\times10^4$ & $6.14\times10^4$ & $9.18\times10^4$ & $15.63\times10^4$ \\ \hline	
        \makecell[c]{Our estimate $\Theta(1/\Delta^2)$: \\
        $1772/\Delta^2-12218/\Delta+67621$}
        & $4.97\times10^{4}$ & $6.17\times10^{4}$ & $9.16\times10^{4}$ & $15.63\times10^{4}$ \\ \hline	
    \end{tabular}
    \vspace{.1cm}

    \begin{tabular}{c|c|c|c|c}
          \hline
         $p$ & $6$ & $8$ & $10$ & $12$  \\  \hline
         Realistic $T_{\rm plat}$ & $6.14\times10^4$ & $9.57\times10^4$ & $13.96\times10^4$ & $17.61\times10^4$ \\ \hline	
        \makecell[c]{Our estimate $\Theta(p)$: \\ $19400p-56400$} & $6.00\times10^{4}$ & $9.88\times10^{4}$ & $13.76\times10^{4}$ & $17.64\times10^{4}$ \\ \hline
\end{tabular}
\vspace{.4cm}

\begin{tabular}{c|c|c|c|c}
      \hline
     $p$ & $6$ & $8$ & $10$ & $12$ \\  \hline
     Realistic $T_{\rm III}$ & $8.92\times10^4$ & $13.40\times10^4$ & $19.72\times10^4$ & $27.68\times10^4$ \\ \hline	
    \makecell[c]{Our estimate $\Theta(p^{\frac{1}{1-\alpha\cos\Delta}})$: \\ $6912p^{1.5}-15897$} & $8.59\times10^{4}$ & $14.05\times10^{4}$ & $20.27\times10^{4}$ & $27.14\times10^{4}$ \\ \hline
\end{tabular}
\label{table: numerical validation of timeline}
\end{table}

\subsection{Experiments on Noisy Dataset}
\label{appendix: experiment: with noise}

We conduct numerical experiments on the setting of adding small stochastic noise on top of $\bx_+$ and $\bx_-$, a little bit more realistic setting. 
Specifically, in ${\rm span}\{\boldsymbol{x}_+,\boldsymbol{x}_-\}$, we perturb the angles of $n_+-1$ instances of $\boldsymbol{x}_+$ and $n_--1$ instances of $\boldsymbol{x}_-$ using stochastic noise $\xi\sim{\rm Unif}([0,\Delta/4])$.

In Figure~\ref{fig: noisy data dynamics}, we visualize (i) the evolution of each neuron throughout the training process; (ii) some key data directions; (iii) the evolution of training accuracy. 

From the numerical results in Figure~\ref{fig: noisy data dynamics}, we have two main conclusions: (1) we ascertain that the same four-phase optimization dynamics and nonlinear behaviors persist, even for our dataset with small stochastic noise; (2) a slight difference is that there is more than one plateau of training accuracy in Phase II. The reason is that for noisy data, GF needs to learn negative data one by one in Phase II. For example, three distinct negative data are employed in this experiment, so three plateaus of training accuracy emerge ($12/15$, $13/15$, and $14/15$).

\newpage
\begin{figure}[!ht]
    \centering
    \subfigure[\small key data directions]
    {\includegraphics[width=4.2cm]{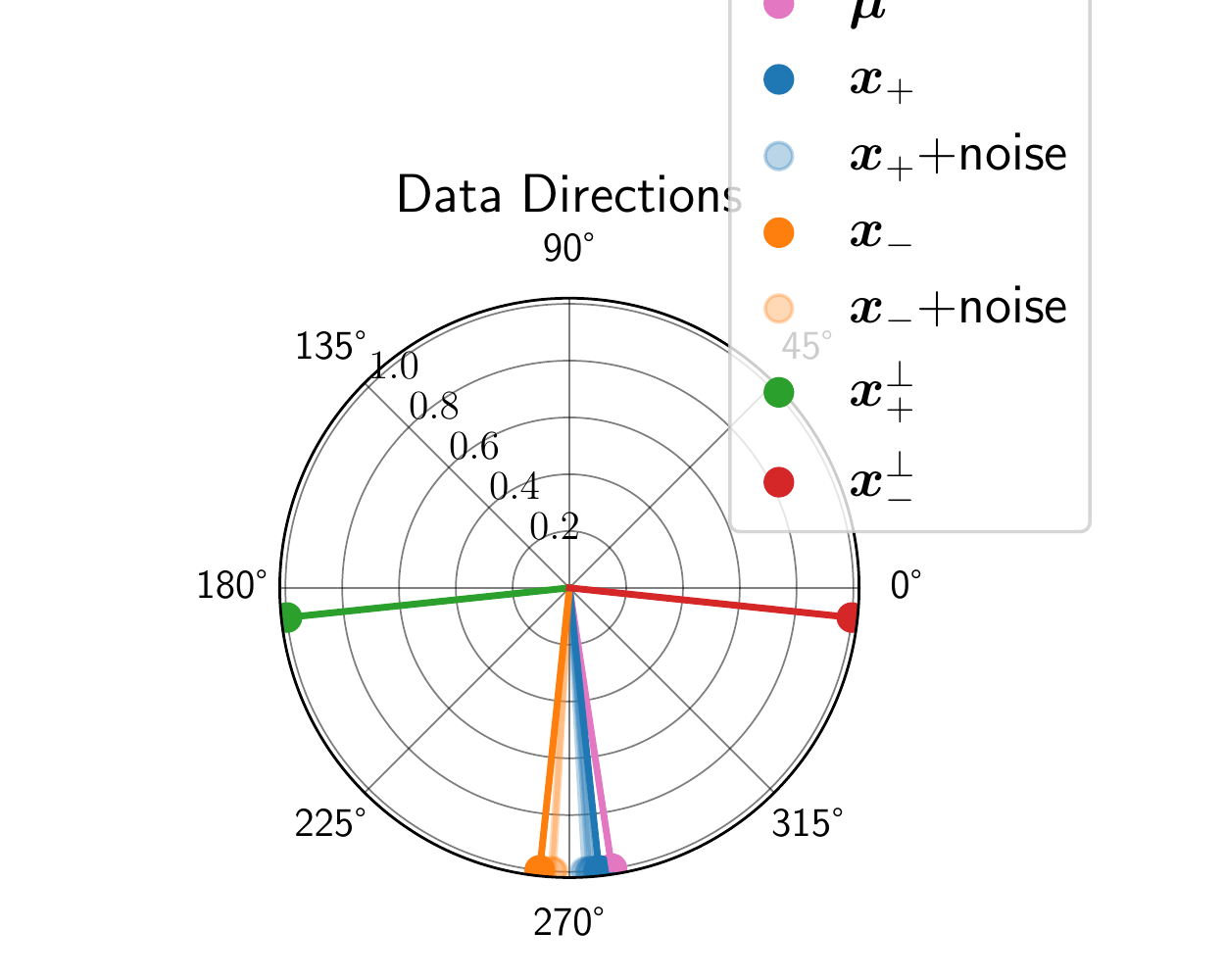}}
    \subfigure[\small training acc]
    {\includegraphics[width=4.2cm]{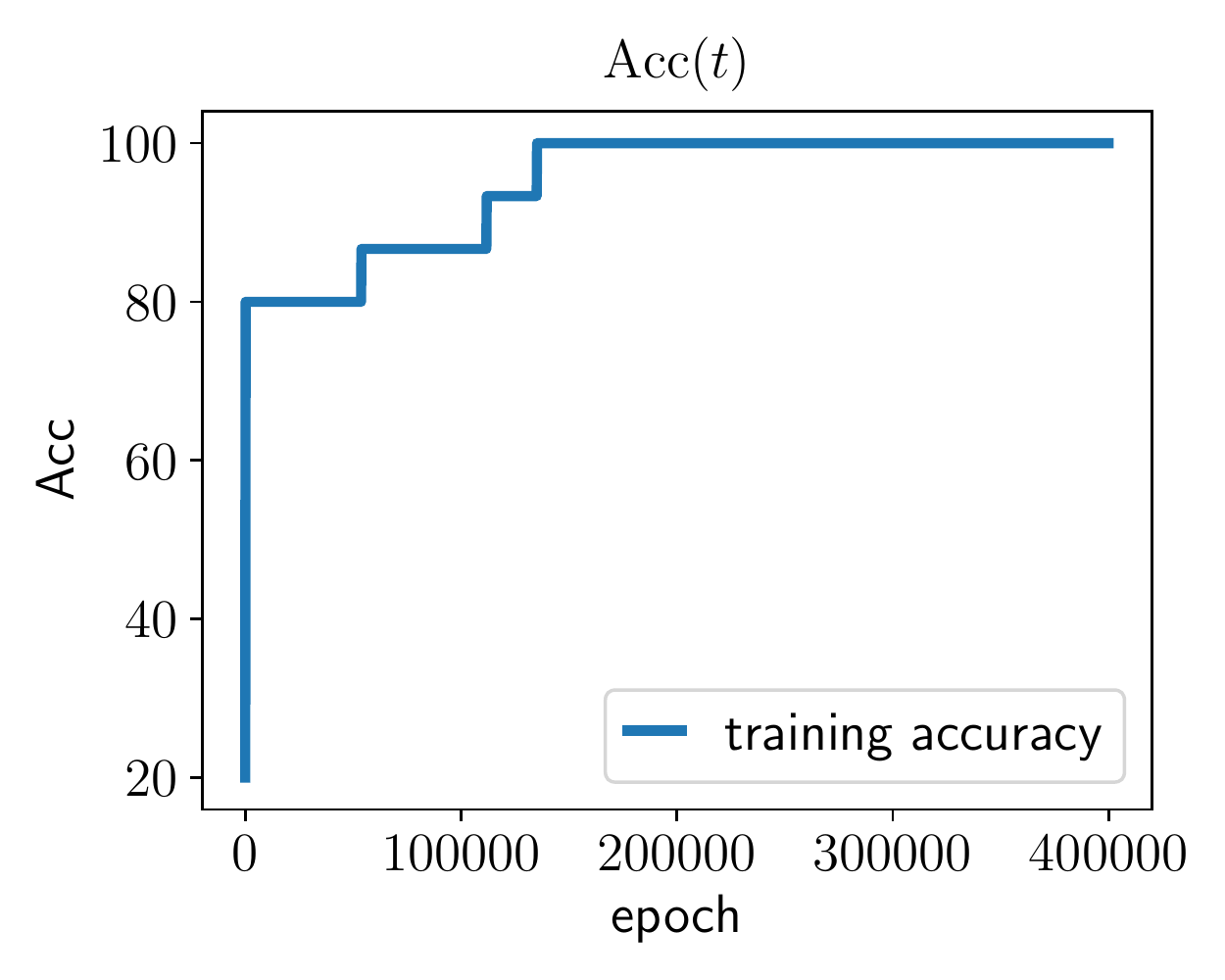}}
    \\
    \subfigure[\small $t=0$]
    {\includegraphics[width=4cm]{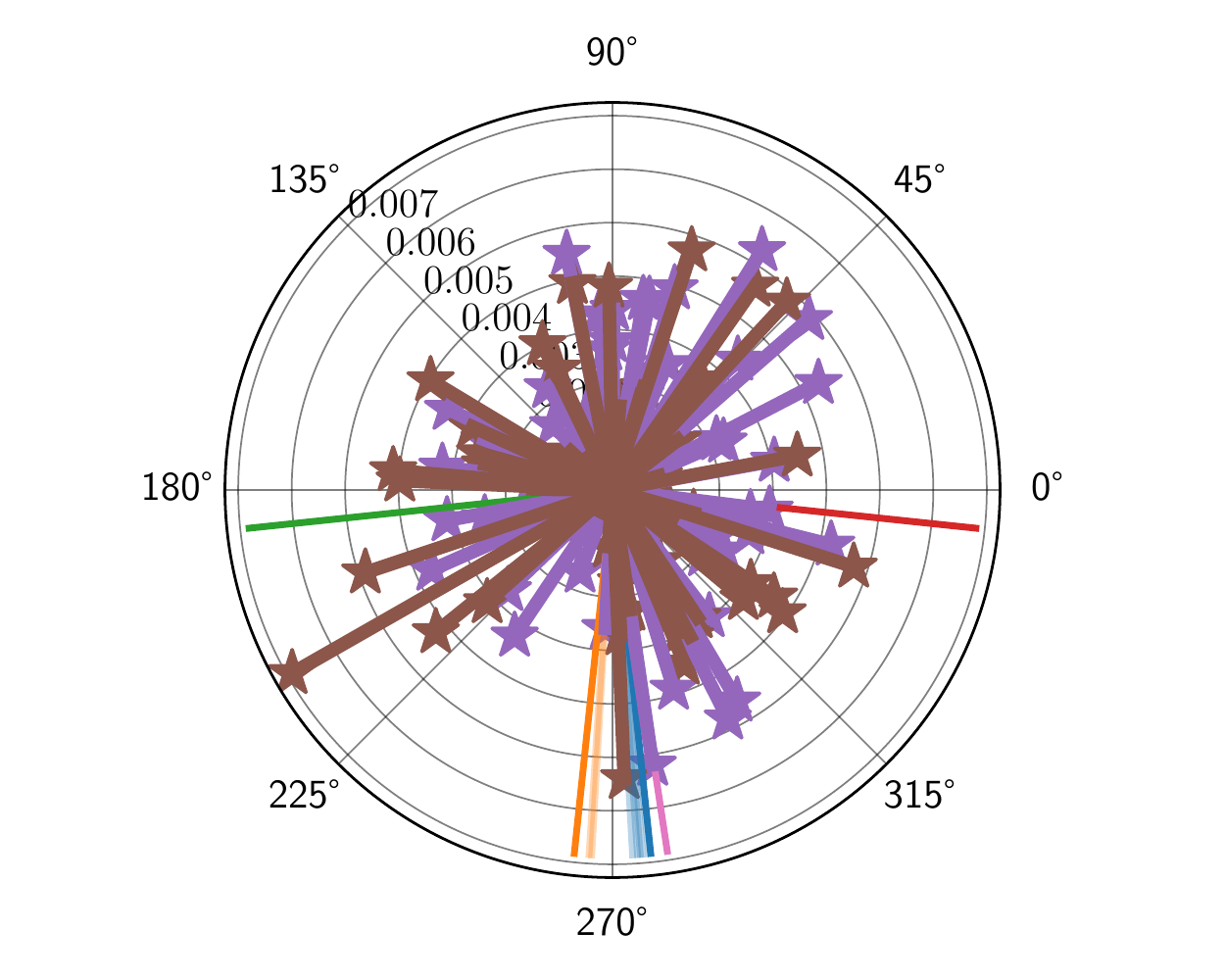}}
    \subfigure[\small $t=200$]
    {\includegraphics[width=4cm]{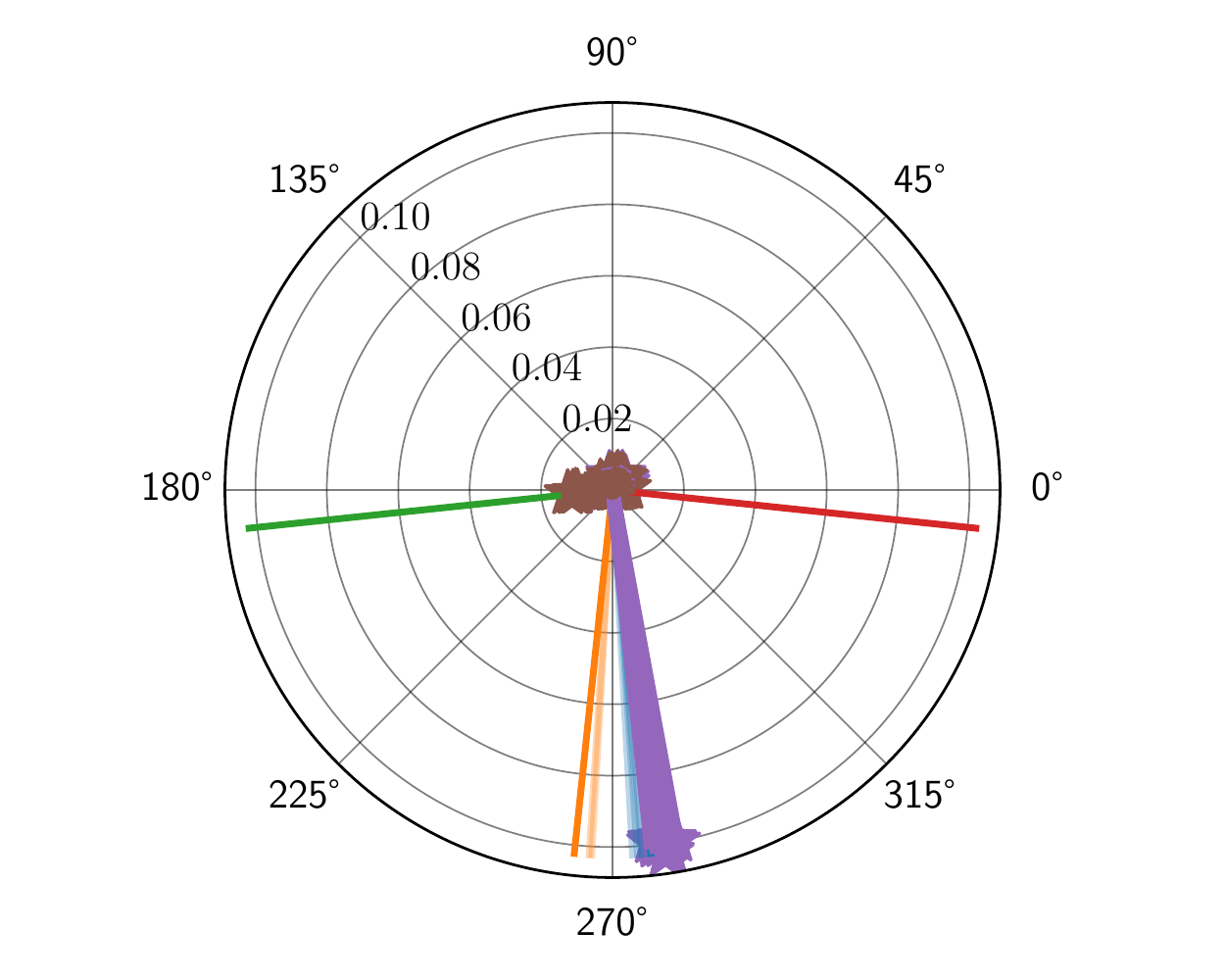}}
    \subfigure[\small $t=2000$]
    {\includegraphics[width=4cm]{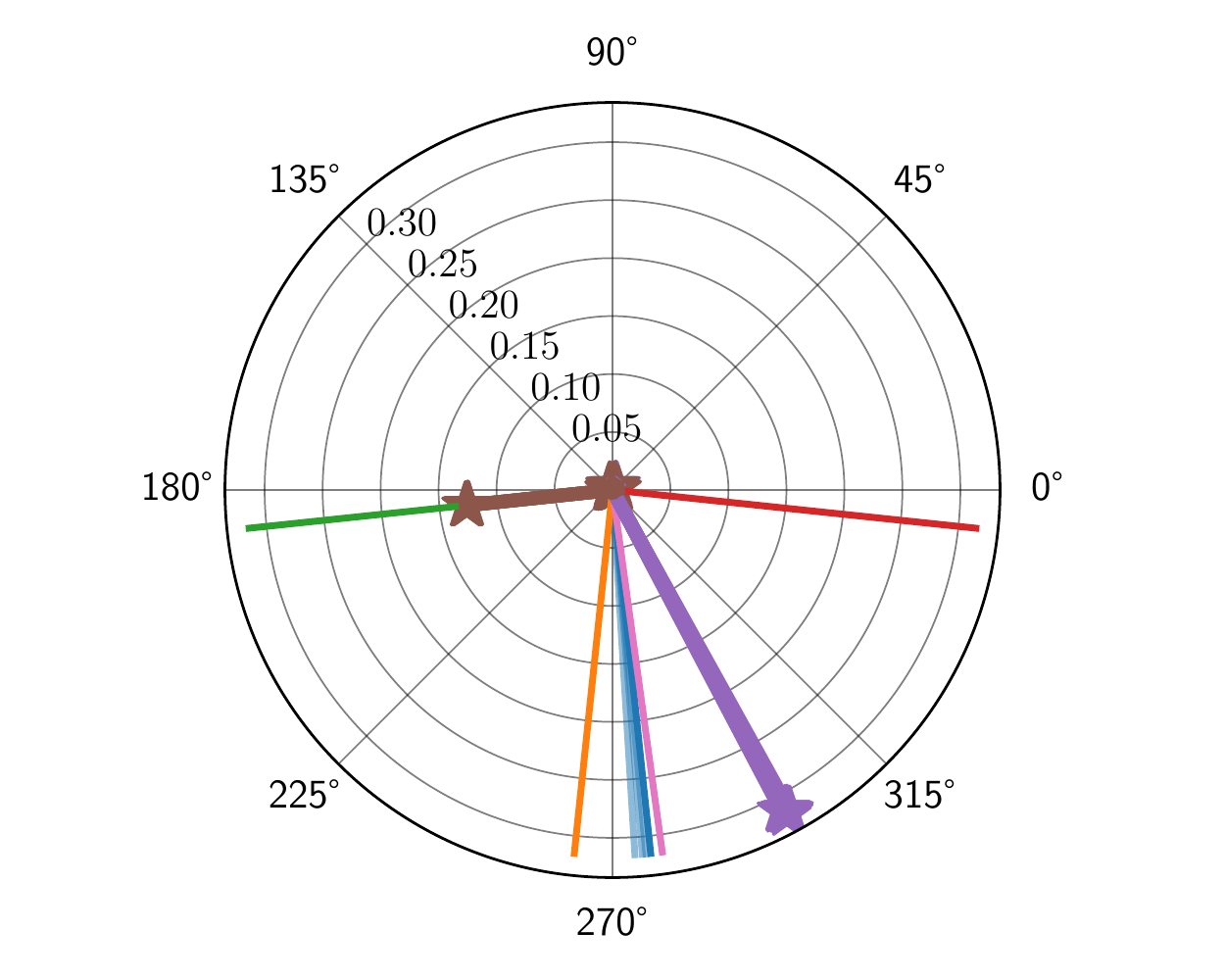}}
    \subfigure[\small $t=30000$]
    {\includegraphics[width=4cm]{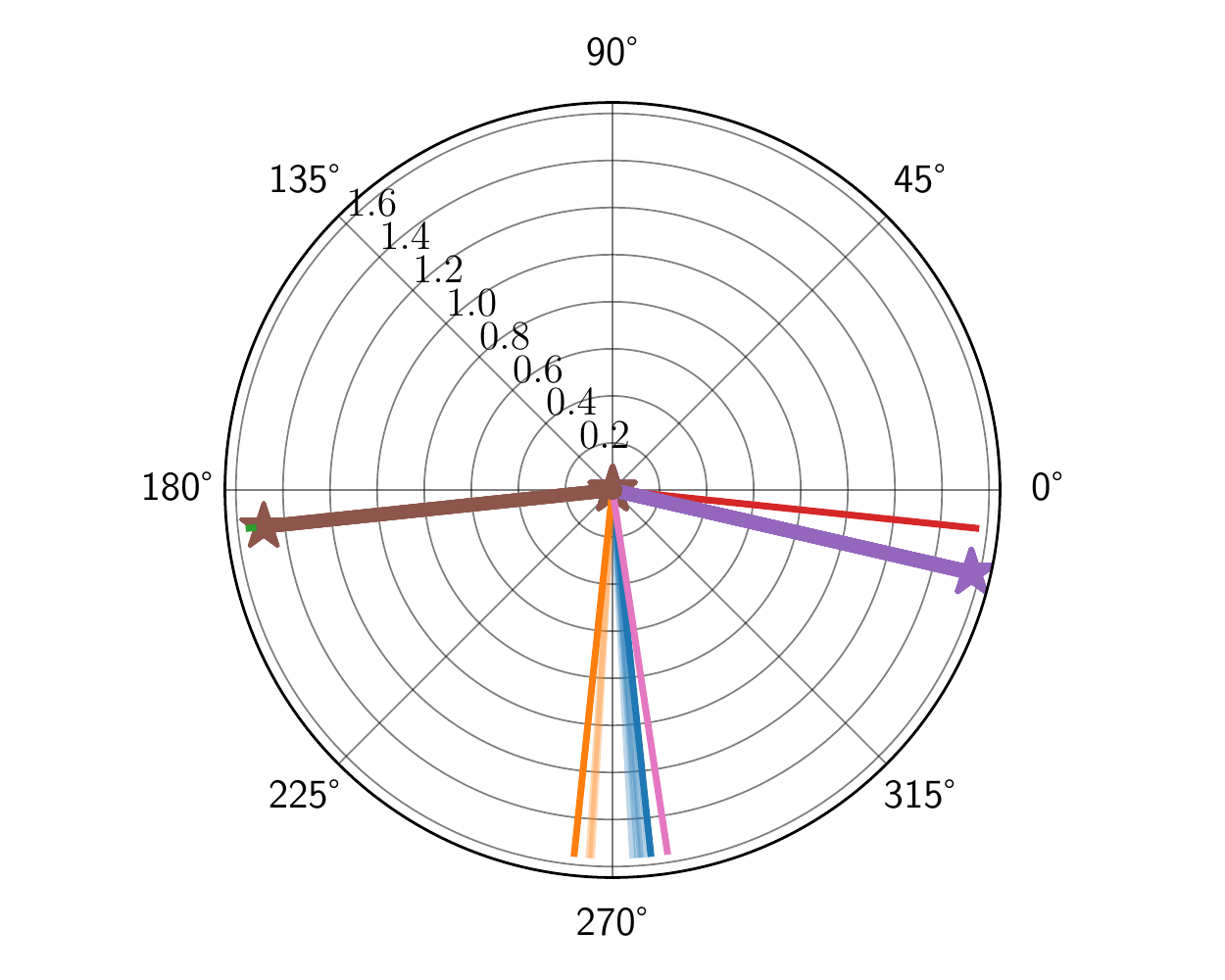}}
    \subfigure[\small $t=65000$]
    {\includegraphics[width=4cm]{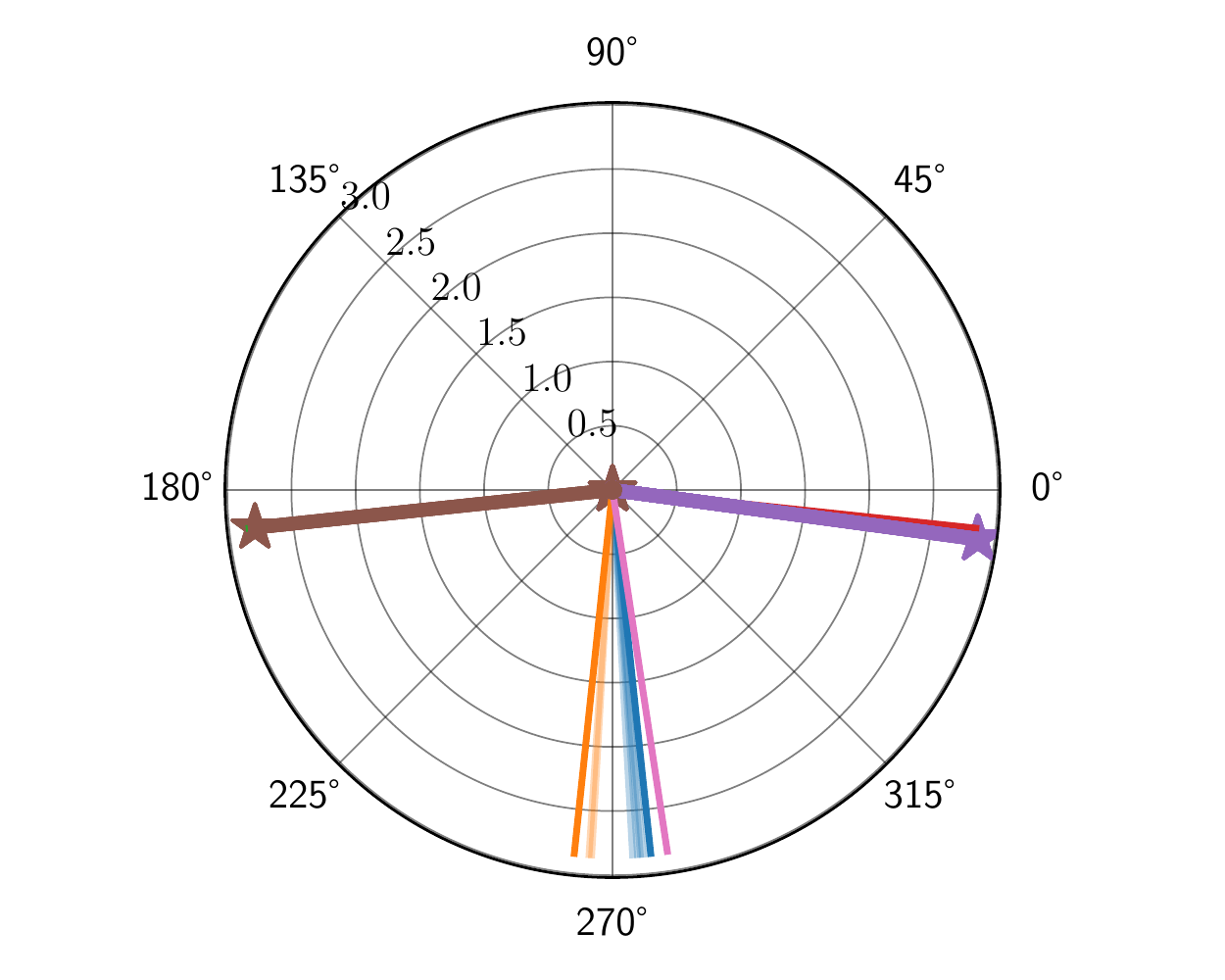}}
    \subfigure[\small $t=90000$]
    {\includegraphics[width=4cm]{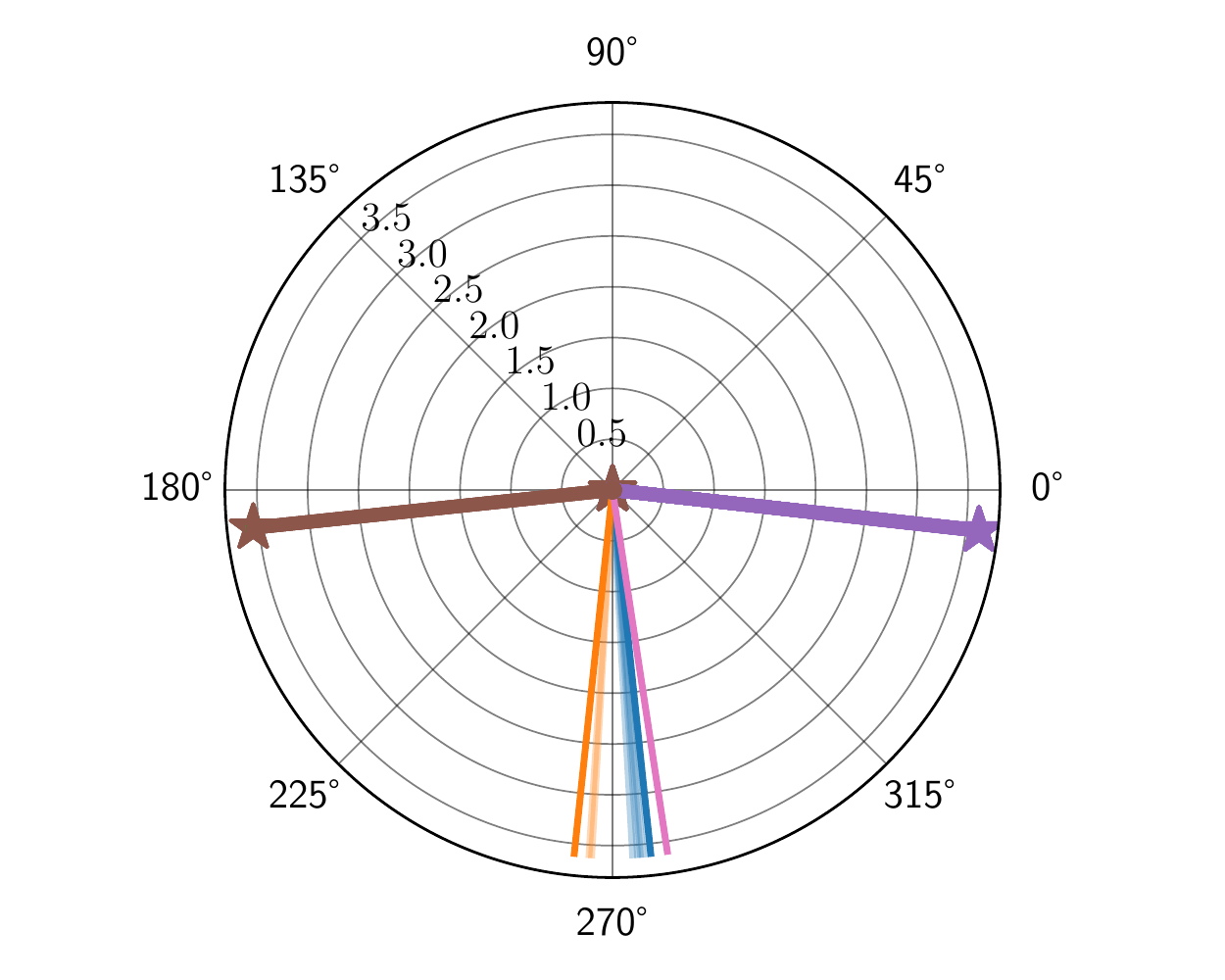}}
    \subfigure[\small $t=200000$]
    {\includegraphics[width=4cm]{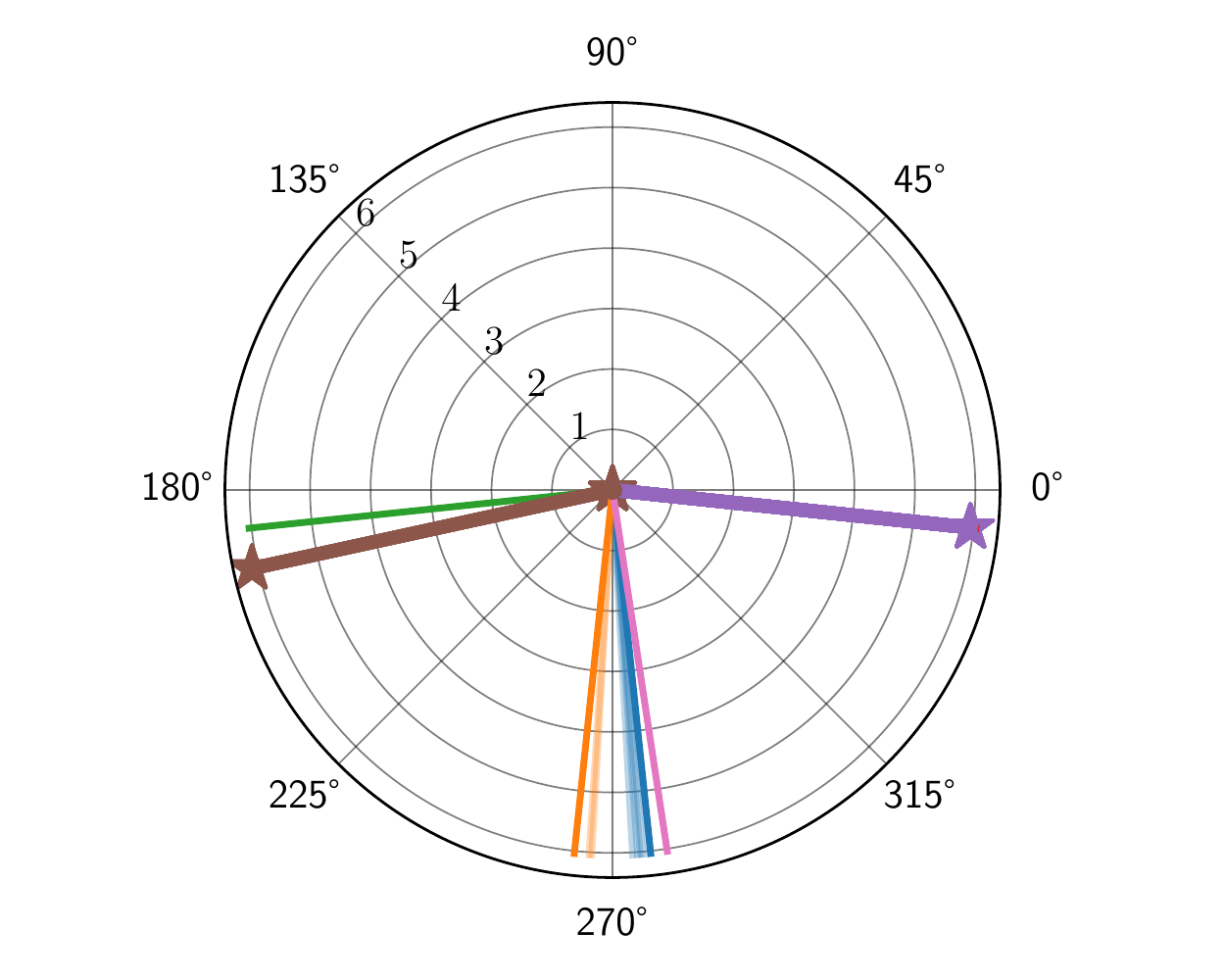}}
    \subfigure[\small $t=400000$]
    {\includegraphics[width=4cm]{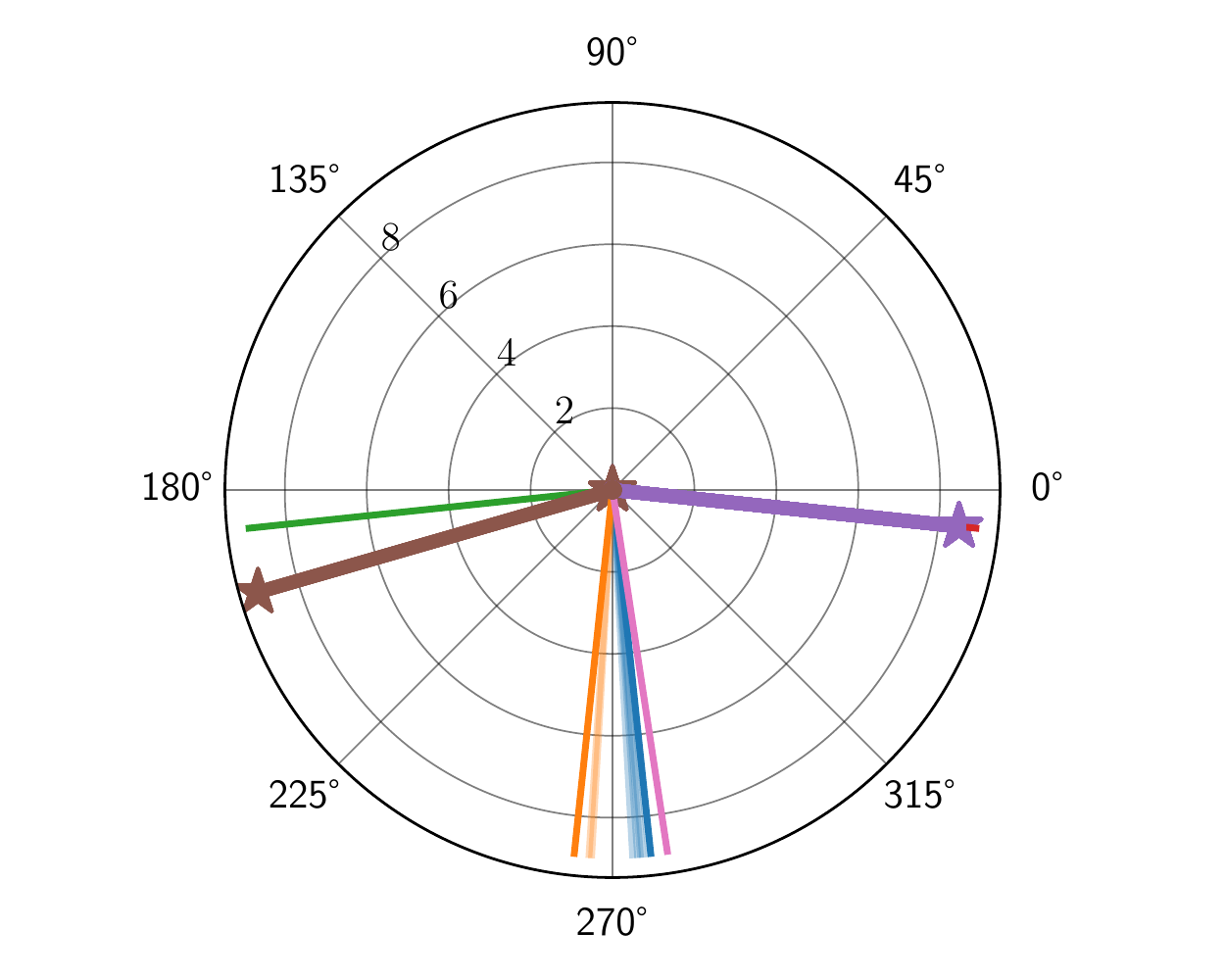}}
    \caption{ The experimental results for { noisy data} with $\Delta=\pi/15$, $n_+=12$, $n_-=3$. 
    { Fig (a)}: Some key data directions, including $\boldsymbol{x}_+$, $\boldsymbol{x}_+$+noise, $\boldsymbol{x}_-$, $\boldsymbol{x}_-$+noise, $\boldsymbol{\mu}$, $\boldsymbol{x}_+^\perp$, and $\boldsymbol{x}_-^\perp$. { Fig (b)}: The evolution of training accuracy. { Figs (c)$\sim$(j)}: the evolution of the projections of all neurons $\{\bb_k(t)\}_{k\in[m]}$ onto the $2$d subspace ${\rm span}\{\bx_+,\bx_-\}$ during training (from $t=0$ to $t=400000$). Each purple star represents a positive neuron $(k\in[m/2])$, while each brown star represents a negative neuron $(k\in[m]-[m/2])$. {\bf Four-phase dynamics}: from Fig (c) to (d) is Phase I; from Fig (d) to (g) is Phase II; from Fig (g) to (h) is Phase III; from Fig (h) to (j) is Phase IV. To compare these results with noiseless data, please refer to Figures~\ref{fig: data, loss, acc} and~\ref{fig: whole dynamics}.}
    \label{fig: noisy data dynamics}
   
\end{figure}


\newpage
\section{Proof Preparation}
\label{appendix: proof preparation}

{\bf Selection of initialization parameters.}

For the data satisfying Assumption~\ref{ass: data}, we consider the regime that $\Delta\ll1$ with $p\cos\Delta>1$.
During the entire proof, we select the initialization scale $\kappa_1,\kappa_2$ as follows:
\begin{equation}\label{equ: parameter selection GF}
\begin{gathered}
\kappa_2=\mathcal{O}(1),\quad\frac{\kappa_1}{\kappa_2}=\mathcal{O}\left(\Delta^8\right).
\end{gathered}
\end{equation}

{\bf Gradient Flow.} 
In general, for any $k\in[m]$, the GF dynamics of $\boldsymbol{b}_k(t)$ can be written as 
\begin{equation}\label{equ: dynamics}
\begin{aligned}
&\frac{\mathrm{d}\boldsymbol{b}_k(t)}{\mathrm{d} t}
\in\frac{\sgrad\cL(\btheta)}{\partial \bb_k}
=\frac{\mathrm{s}_k\kappa_2}{\sqrt{m}}\frac{1}{n}\sum_{i=1}^n e^{-y_if_i(t)}\partial^{\circ}\sigma(\left<\boldsymbol{b}_k(t),\bx_i\right>) y_i\bx_i
\\
=&\frac{\mathrm{s}_k\kappa_2}{\sqrt{m}}\Big(\frac{p}{1+p}e^{-f_+(t)}\sgrad\sigma(\left<\boldsymbol{b}_k(t),\bx_+\right>)\bx_+-\frac{1}{1+p} e^{f_-(t)}\sgrad\sigma(\left<\boldsymbol{b}_k(t),\bx_-\right>)\bx_-\Big),
\end{aligned}
\end{equation}

where $\sgrad$ is Clarke's subdifferential defined in Defnition~\ref{def: clark subdifferential}.
Notice that if $\cL$ is continuously differentiable at $\btheta$, then $\partial^{\circ}\cL(\btheta)=\{\nabla\cL(\btheta)\}$ is unique.

However, for discontinuous differentiable points of $\cL$, the differential inclusion flow $\frac{\rd\btheta}{\rd t}\in\partial^{\circ}\cL(\btheta)$ may not be unique. To study a more specific dynamics, we also utilize Definition~\ref{def: discontinuous system solution} to determine GF at some of such points, which overcomes non-uniqueness of GF trajectories to some extent. It is worth noting that Definition~\ref{def: discontinuous system solution} and Definition~\ref{def: clark subdifferential} are compatible and specifically, the dynamics defined in Definition~\ref{def: discontinuous system solution}(Case I, III) lie in the convex hull defined in Definition~\ref{def: clark subdifferential}.

\begin{remark}
    In~\citep{lyu2021gradient}, the non-branching starting point Assumption is employed to address the technical challenge of non-uniqueness in GF trajectories. By comparison, in this work, we do not need this assumption. We adopt Definition~\ref{def: discontinuous system solution} to uniquely determine the Gradient Flow trajectories theoretically near some discontinuous differential regions, such as ``Ridge'', ``Valley'', and ``Refraction edge'' discussed in Section I.2 in~\citep{lyu2021gradient}.
\end{remark}

Additionally, in the following sections, we may rewrite this dynamics accordingly, such as specific forms and dynamics decomposition.

{\bf Additional Notations.} As a similar description to $\sgn_k^+(\cdot)$ and $\sgn_k^-(\cdot)$, we also employ the following six manifolds to characterize activation patterns (by judging which manifold the neuron $\bw_k$ belongs to):
\begin{gather*}
    \mathcal{M}_{+}^+:=\{\boldsymbol{w}\in\mathbb{S}^{d-1}:\left<\boldsymbol{w},\boldsymbol{x}_+\right>>0\},
    \quad
    \mathcal{M}_-^+:=\{\boldsymbol{w}\in\mathbb{S}^{d-1}:\left<\boldsymbol{w},\boldsymbol{x}_-\right>>0\},
    \\
    \mathcal{M}_{+}^0:=\{\boldsymbol{w}\in\mathbb{S}^{d-1}:\left<\boldsymbol{w},\boldsymbol{x}_+\right>=0\},\quad
    \mathcal{M}_{-}^0:=\{\boldsymbol{w}\in\mathbb{S}^{d-1}:\left<\boldsymbol{w},\boldsymbol{x}_-\right>=0\},
    \\
    \mathcal{M}_{+}^-:=\{\boldsymbol{w}\in\mathbb{S}^{d-1}:\left<\boldsymbol{w},\boldsymbol{x}_+\right>\leq0\},\quad
    \mathcal{M}_-^-:=\{\boldsymbol{w}\in\mathbb{S}^{d-1}:\left<\boldsymbol{w},\boldsymbol{x}_-\right>\leq0\}.
\end{gather*}

As one of our interested manifolds, the border $\partial (\mathcal{M}_+^+\cap\mathcal{M}_-^+)$ can be divided into
\begin{align*}
\partial(\mathcal{M}_+^+\cap\mathcal{M}_-^+)
=&\{
\boldsymbol{w}\in\mathbb{S}^{d-1}:\left<\boldsymbol{w},\boldsymbol{x}_+\right>=0\ \text{or}\ \left<\boldsymbol{w},\boldsymbol{x}_-\right>=0
\}
\\=&\{
\boldsymbol{w}\in\mathbb{S}^{d-1}:\left<\boldsymbol{w},\boldsymbol{x}_+\right>=0, \left<\boldsymbol{w},\boldsymbol{x}_-\right>>0
\}\bigcup\{
\boldsymbol{w}\in\mathbb{S}^{d-1}:\left<\boldsymbol{w},\boldsymbol{x}_+\right>>0, \left<\boldsymbol{w},\boldsymbol{x}_-\right>=0
\}
\\&\bigcup\{
\boldsymbol{w}\in\mathbb{S}^{d-1}:\left<\boldsymbol{w},\boldsymbol{x}_+\right>=0, \left<\boldsymbol{w},\boldsymbol{x}_-\right>=0
\}
\\=&\big(\mathcal{M}_+^0\cap\mathcal{M}_{-}^+\big)\cup\big(\mathcal{M}_-^0\cap\mathcal{M}_{+}^+\big)\cup\big(\mathcal{M}_+^0\cap\mathcal{M}_{-}^0\big).
\end{align*}

Furthermore, we also utilize the following notations:
\begin{gather*}
    \mathcal{P}_{+}^+:=\{\boldsymbol{b}\in\mathbb{R}^{d}:\left<\boldsymbol{b},\boldsymbol{x}_+\right>>0\},
    \quad
    \mathcal{P}_-^+:=\{\boldsymbol{b}\in\mathbb{R}^{d}:\left<\boldsymbol{b},\boldsymbol{x}_-\right>>0\},
    \\
    \mathcal{P}_{+}^0:=\{\boldsymbol{b}\in\mathbb{R}^{d}:\left<\boldsymbol{b},\boldsymbol{x}_+\right>=0\},\quad
    \mathcal{P}_{-}^0:=\{\boldsymbol{b}\in\mathbb{R}^{d-1}:\left<\boldsymbol{b},\boldsymbol{x}_-\right>=0\},
    \\
    \mathcal{P}_{+}^-:=\{\boldsymbol{b}\in\mathbb{R}^{d}:\left<\boldsymbol{b},\boldsymbol{x}_+\right>\leq0\},\quad
    \mathcal{P}_-^-:=\{\boldsymbol{b}\in\mathbb{R}^{d}:\left<\boldsymbol{b},\boldsymbol{x}_-\right>\leq0\}.
\end{gather*}

Notice that for the direction of a neuron $\bb_k\ne\bzero$, using $\cP_+^+,\cP_+^0,\cP_+^-,\cP_-^+,\cP_-^0,\cP_-^-$ to describe $\bb_k$ is equivalent to using $\cM_+^+,\cM_+^0,\cM_+^-,\cM_-^+,\cM_-^0,\cM_-^-$ to describe $\bw_k$.

\begin{lemma}[Dead neurons keep dead]\label{lemma: GF dead neurons keep dead}
For the $k$-th neuron, if there exists a $t_0\geq0$, s.t. $\boldsymbol{w}_k(t_0)\in\mathcal{M}_+^-\cap\mathcal{M}_-^-$, then it dies and remains unchanged during the remaining training: $\boldsymbol{w}_k(t)\in\mathcal{M}_+^-\cap\mathcal{M}_-^-$ and $\bb_k(t)\equiv\bb_k(t_0)$ for any $t\geq t_0$.
\end{lemma}
\begin{proof}[Proof of Lemma \ref{lemma: GF dead neurons keep dead}] A straightforward calculation.
\end{proof}

The above fact is a basic fact in our setting, which illustrates that if the neuron $\bb_k$ is deactivated for both data $\bx_+$ and $\bx_-$ at some time, then it remains ``dead'' forever. 

It is worth mentioning that if the neuron $\bb_k$ is deactivated on $\bx_+$ but activated on $\bx_-$ at some time, it can still reactivate on $\bx_+$ later, which is one of the important reasons why our ReLU optimization dynamics are complicated.

\newpage
\section{Proofs of Optimization Dynamics in Phase I}\label{appendix: proof: Phase I}

In this section, we conduct a detailed analysis of the training dynamics of each neuron in Phase I. 
The main proof idea is to decompose neurons' dynamics into tangential and radial dynamics. 
For small initialization, in Phase I, the radial increasing of neurons is much slower than their tangential velocity, which can result in condensation. 
However, the main challenges arise from the initial direction's randomness and ReLU's discontinuous derivative, leading to eight categories of neuron dynamics.
Moreover, it is also nontrivial and requires  meticulous analysis to estimate the number of two classes of living neurons at the end of Phase I.

We define the time
\begin{equation}
    T_{\rm I}:=10\sqrt{\frac{\kappa_1}{\kappa_2}},
\end{equation}
and call $t\in[0,T_{\rm I}]$ ``Phase I''.

Recall that the selection~\eqref{equ: parameter selection GF} about initialization parameters can guarantee the {\bf whole} four-phase optimization dynamics.
Nevertheless, when we focus on Phase I, $\kappa_2=\cO(1)$ and $\kappa_1/\kappa_2=\cO(1)$ suffice to ensure the dynamics in Phase I. Specifically, during the proof of Phase I, we can use the following selection~\eqref{equ: parameter selection GF Phase I} on $\kappa_1,\kappa_2$, which is much weaker than~\eqref{equ: parameter selection GF}:

\begin{equation}\label{equ: parameter selection GF Phase I}
\begin{gathered}
\kappa_2=\cO(1),\quad \frac{\kappa_1}{\kappa_2}=\cO(1).
\end{gathered}
\end{equation}


For simplicity, we assume $\Delta\leq\frac{1}{5}$.
And for convenience, we assume $p\geq5$. 
It is worth mentioning that our proof approach also applies for $p=2,3,4$, with at most one absolute constant difference.

Prepared for the analysis in Phase I, we decompose the dynamics of $\boldsymbol{b}_k(t)$ into the radial movement $\rho_k(t)\in\mathbb{R}$ and the tangential movement $\bw_k(t)\in\mathbb{S}^{d-1}$ satisfied to $\boldsymbol{b}_k(t)=\rho_k(t)\bw_k(t)$. 

\begin{lemma}[Dynamics decomposition]\label{lemma: dynamics decomposition}
For any $k\in[m]$, the dynamics of $\boldsymbol{b}_k(t)$ can be decomposed into the radial movement $\rho_k(t)\in\mathbb{R}$ and the tangential movement $\bw_k(t)\in\mathbb{S}^{d-1}$: 
\begin{equation}\label{equ: dynamics rewrite}
    \begin{aligned}
        \frac{\mathrm{d}\bw_k(t)}{\mathrm{d} t}&\in\frac{\mathrm{s}_k\kappa_2}{\sqrt{m}\rho_k(t)}\Big(\bF_k(t)-\left<\bF_k(t),\bw_k(t)\right>\bw_k(t)\Big),
        \\
        \frac{\mathrm{d} \rho_k(t)}{\mathrm{d} t}&\in\frac{\mathrm{s}_k\kappa_2}{\sqrt{m}}\left<\bF_k(t),\bw_k(t)\right>,
    \end{aligned}
\end{equation}
where $\rho_k(t)=\left\|\boldsymbol{b}_k(t)\right\|$, $\bw_k(t)=\boldsymbol{b}_k(t)/\left\|\boldsymbol{b}_k(t)\right\|$ and 
\begin{equation}\label{equ: definition F}
    \bF_k(t)=\frac{1}{n}\sum_{i=1}^n e^{-y_if_i(t)}\sgrad\sigma\bracket{\left<\bw_k(t),\bx_i\right>}y_i\bx_i.
\end{equation}
\end{lemma}

\begin{proof}[Proof of Lemma \ref{lemma: dynamics decomposition}]\ \\
From the dynamics of $\boldsymbol{b}_k(t)$:
\begin{align*}
    \frac{\rd\bb_k(t)}{\rd t}\in\frac{\mathrm{s}_k\kappa_2}{\sqrt{m}}\bF_k(t)=\frac{\mathrm{s}_k\kappa_2}{\sqrt{m}}\frac{1}{n}\sum_{i=1}^n e^{-y_if_i(t)}\sgrad\sigma\bracket{\left<\bw_k(t),\bx_i\right>}y_i\bx_i,
\end{align*}
we have
\begin{align*}
    &\frac{\mathrm{d} \rho_k(t)}{\mathrm{d} t}=\frac{1}{2\left\|\boldsymbol{b}_k(t)\right\|}\frac{\mathrm{d}\left\|\boldsymbol{b}_k(t)\right\|^2}{\mathrm{d}t}=\frac{1}{\left\|\boldsymbol{b}_k(t)\right\|}\left<\boldsymbol{b}_k(t),\frac{\mathrm{d}\boldsymbol{b}_k(t)}{\mathrm{d}t}\right>
    \\\in&
    \frac{1}{\left\|\boldsymbol{b}_k(t)\right\|}\left<\boldsymbol{b}_k(t),\frac{\mathrm{s}_k\kappa_2}{\sqrt{m}}\bF_k(t)\right>
    =\frac{\mathrm{s}_k\kappa_2}{\sqrt{m}}\left<\bF_k(t),\bw_k(t)\right>,
\end{align*}
\begin{align*}
    &\frac{\mathrm{d} \bw_k(t)}{\mathrm{d} t}
    =\frac{\left\|\boldsymbol{b}_k(t)\right\|\frac{\mathrm{d}\boldsymbol{b}_k(t)}{\mathrm{d}t}-\frac{\mathrm{d}\left\|\boldsymbol{b}_k(t)\right\|}{\mathrm{d}t}\boldsymbol{b}_k(t)}{\left\|\boldsymbol{b}_k(t)\right\|^2}=\frac{1}{\rho_k(t)}\left(\frac{\mathrm{d}\boldsymbol{b}_k(t)}{\mathrm{d}t}-\frac{\mathrm{d}\rho_k(t)}{\mathrm{d}t}\bw_k(t)\right)
    \\\in&\frac{1}{\rho_k(t)}\Big(\frac{\mathrm{s}_k\kappa_2}{\sqrt{m}}\bF_k(t)-\frac{\mathrm{s}_k\kappa_2}{\sqrt{m}}\left<\bF_k(t),\bw_k(t)\right>\bw_k(t)\Big)
    =\frac{\mathrm{s}_k\kappa_2}{\sqrt{m}\rho_k(t)}\Big(\bF_k(t)-\left<\bF_k(t),\bw_k(t)\right>\bw_k(t)\Big).
\end{align*}

\end{proof}

Prepared for the analysis of the neurons' dynamics, we establish a rough estimate about the norm and prediction growth of each neuron in Phase I, and we will improve it later.

\begin{lemma}[A Rough Estimate of Norm and Prediction in Phase I]\label{lemma: GF Phase I norm estimate}\ \\
For any $t\leq T_{\rm I}$, $k\in[m]$, and $i\in[n]$, we have the following estimates:
\begin{gather*}
    \rho_k(t)\leq\frac{\kappa_1+1.1\kappa_2 t}{\sqrt{m}},
    \\
    |f_i(t)|
    \leq\sqrt{\kappa_1\kappa_2},
    \\
    \left|e^{-y_i f_i(t)}-1\right|\leq1.1\sqrt{\kappa_1\kappa_2}.
\end{gather*}
\end{lemma}

\begin{proof}[Proof of Lemma \ref{lemma: GF Phase I norm estimate}]\ \\
First, we define the hitting time
\[
T_{\sqrt{\kappa_1\kappa_2}}:=\inf\Big\{t>0:\max\limits_{i\in[n]}\left|f_i(t)\right|>\sqrt{\kappa_1\kappa_2}\Big\}.
\]
From $|f_i(0)|\leq m \frac{\kappa_2}{\sqrt{m}}\frac{\kappa_1}{\sqrt{m}}=\kappa_1\kappa_2<\sqrt{\kappa_1\kappa_2}$ and the continuity of $f_i(\cdot)$, we know $T_{\sqrt{\kappa_1\kappa_2}}>0$.
Then we will prove $T_{\rm I}=10\sqrt{\frac{\kappa_1}{\kappa_2}}\leq T_{\sqrt{\kappa_1\kappa_2}}$.

For any $k\in[m]$, $i\in[n]$, and $t\leq T_{\sqrt{\kappa_1\kappa_2}}$, we have the following estimates.

Recalling the definition of $\bF_k(t)$ \eqref{equ: definition F}, we have
\begin{align*}
    \left\|\bF_k(t)\right\|
    =&\left\|\frac{1}{n}\sum_{i=1}^n e^{-y_if_i(t)}\sgrad\sigma\bracket{\left<\bw_k(t),\bx_i\right>} y_i\bx _i\right|\leq\max_{i\in[n]}\left|e^{-y_i f_i(t)}\right\|
    \\\leq&e^{\max\limits_{i\in[n]}\left|f_i(t)\right|}\leq e^{\sqrt{\kappa_1\kappa_2}}\leq 1.1.
\end{align*}

Recalling the dynamics \eqref{equ: dynamics rewrite}, for any $t\leq T_{\sqrt{\kappa_1\kappa_2}}$,
\begin{align*}
    \frac{\mathrm{d} \rho_k(t)}{\mathrm{d} t}\leq \frac{\kappa_2}{\sqrt{m}}\left\|\left<\bF_k(t),\bw_k(t)\right>\right\|\leq \frac{\kappa_2}{\sqrt{m}}\left\|\bF_k(t)\right\|\leq\frac{1.1\kappa_2}{\sqrt{m}}.
\end{align*}

Then combining $\rho_k(0)=\frac{\kappa_1}{\sqrt{m}}$, for any $t\leq T_{\sqrt{\kappa_1\kappa_2}}$,
\begin{equation}\label{equ of proof: norm rough, lemma: GF Phase I norm estimate}
\begin{gathered}
    \rho_k(t)\leq\frac{\kappa_1+1.1\kappa_2 t}{\sqrt{m}},
    \\
    |f_i(t)|
    \leq\sum_{k=1}^m |a_k|\rho_k(t)
    \leq m\frac{\kappa_2}{\sqrt{m}}\frac{\kappa_1+1.1\kappa_2 t}{\sqrt{m}}
    \leq\kappa_2\big(\kappa_1+1.1\kappa_2 t\big).
\end{gathered}
\end{equation}
So for any $t\leq T_{\rm I}$, we have
\begin{equation}\label{equ of proof: prediction rough, lemma: GF Phase I norm estimate}
    |f_i(t)|
    \leq\kappa_2\big(\kappa_1+1.1\kappa_2 T_{\rm I}\big)
    \leq\kappa_2\Big(\kappa_1+11\sqrt{\kappa_1\kappa_2}\Big)
    \leq12\kappa_2\sqrt{\kappa_1\kappa_2}
    \leq\sqrt{\kappa_1\kappa_2}.
\end{equation}

From the definition of $T_{\sqrt{\kappa_1\kappa_2}}$, we have proved $T_{\rm I}\leq T_{\sqrt{\kappa_1\kappa_2}}$.

Moreover, from this proof, we know \eqref{equ of proof: norm rough, lemma: GF Phase I norm estimate} holds for any $t\leq T_{\rm I}$. Moreover, by Mean Value Theorem,
\begin{align*}
\left|e^{-y_i f_i(t)}-1\right|\leq
\big(\max_{z\in[0,\sqrt{\kappa_1\kappa_2}]}e^z\big)|-y_if_i(t)|
\leq1.1|f_i(t)|
\leq1.1\sqrt{\kappa_1\kappa_2}.
\end{align*}
\end{proof}

Now we delve into the optimization dynamics of all neurons in Phase I in the following Section~\ref{appendix: subsection: Phase I: positive} and \ref{appendix: subsection: Phase I: negative}.

\subsection{The Dynamics of Positive Neurons}

According to the initial direction, all positive neurons (${\rm s}_k=1$) can be divided into the following four classes.

\begin{align*}
    [m/2]=&\left\{k\in[m/2]: \bw_k(0)\in\cM_+^+\cap\cM_-^+\right\}
    \bigcup
    \left\{k\in[m/2]:\bw_k(0)\in \cM_+^-\cap\cM_-^+\right\}
    \\&
    \bigcup
    \left\{k\in[m/2]:\bw_k(0)\in \cM_+^+\cap\cM_-^-\right\}
    \bigcup
    \left\{k\in[m/2]:\bw_k(0)\in \cM_+^-\cap\cM_-^-\right\}.
\end{align*}

In the following four lemmas, we will prove the dynamics for these four classes of positive neurons.
In summary, in Phase I ($t<T_{\rm I}$), some of positive neurons align well with the direction $\bmu$, and their norms experiment a small but significant increase, while other positive neurons go dead.

\begin{lemma}[Positive, $\cM_+^+\cap\cM_-^+$]\label{lemma: GF Phase I positive S+ S-}\ \\
For positive neuron $k\in\{k\in[m/2]:\bw_k(0)\in\cM_+^+\cap\cM_-^+\}$, at the end of Phase I, it holds that 
\begin{align*}
    \text{\bf (Direction).}& \text{ It is aligned with } \bmu:
    \left<\bw_k(T_{\rm I}),\bmu\right>\geq\Big(1-\mathcal{O}(\sqrt{\kappa_1\kappa_2})\Big)\left(1-\mathcal{O}\Big((\frac{\kappa_1}{\kappa_2})^{0.55}\Big)\right);
    \\
    \text{\bf (Norm).}& \text{ It has a small but significant norm}: \rho_k(T_{\rm I})=\Theta\Big(\frac{\sqrt{\kappa_1\kappa_2}}{\sqrt{m}}\Big).
\end{align*}

\end{lemma}

\begin{proof}[Proof of Lemma \ref{lemma: GF Phase I positive S+ S-}]\ \\
We do the following analysis for any $k\in\{k\in[m/2]:\bw_k(0)\in\cM_+^+\cap\cM_-^+\}$, i.e. ${\rm s}_k=1$, $\left<\bw_k(0),\bx_+\right>>0$, and $\left<\bw_k(0),\bx_-\right>>0$.

\underline{{\bf Step I.} The neuron stays in $\cM_+^+\cap\cM_-^+$ for any $t\leq T_{\rm I}$.}

First, we define the hitting time 
\[
T_{\rm hit}:=\inf\Big\{t\in(0, T_{\rm I}]:\bw_k(t)\notin\partial(\cM_+^+\cap\cM_-^+)\Big\},
\] 
and we aim to prove that $T_{\rm hit}$ does not exist. From the definition of $T_{\rm hit}$ and \eqref{equ: dynamics}, the dynamics of the neuron is:
\begin{align*}
    \frac{\mathrm{d}\boldsymbol{b}_k(t)}{\mathrm{d}t}=\frac{\kappa_2}{\sqrt{m}}\Big(\frac{p}{1+p}e^{-f_+(t)}\bx_+-\frac{1}{1+p}e^{f_-(t)}\bx_-\Big),\ t\leq T_{\rm hit}.
\end{align*}

From $T_{\rm hit}\leq T_{\rm I}$ and Lemma \ref{lemma: GF Phase I norm estimate}, we have $\left|e^{-y_i f_i(t)}-1\right|\leq1.1\sqrt{\kappa_1\kappa_2}$ for any $t\leq T_{\rm hit}$. Then we have:
\begin{align*}
    &\frac{\mathrm{d}\left<\boldsymbol{b}_k(t),\bx_+\right>}{\mathrm{d}t}=\frac{\kappa_2}{\sqrt{m}}\Big(\frac{pe^{-f_+(t)}}{1+p}-\frac{e^{f_-(t)}}{1+p}\cos\Delta\Big)
    \\\geq&\frac{\kappa_2}{\sqrt{m}}\Big(\frac{p(1-1.1\sqrt{\kappa_1\kappa_2})}{1+p}-\frac{(1+1.1\sqrt{\kappa_1\kappa_2})\cos{\Delta}}{1+p}\Big){>}0,
\end{align*}
\begin{align*}
    &\frac{\mathrm{d}\left<\boldsymbol{b}_k(t),\bx_-\right>}{\mathrm{d}t}=\frac{\kappa_2}{\sqrt{m}}\Big(\frac{p e^{-f_+(t)}\cos\Delta}{1+p}-\frac{e^{f_-(t)}}{1+p}\Big)
    \\\geq&\frac{\kappa_2}{\sqrt{m}}\Big(\frac{p(1-1.1\sqrt{\kappa_1\kappa_2})\cos\Delta}{1+p}-\frac{1+1.1\sqrt{\kappa_1\kappa_2})}{1+p}\Big){>}0.
\end{align*}
Hence, for any $t\leq T_{\rm hit}$, $\left<\boldsymbol{b}_k(t),\bx_+\right>>\left<\boldsymbol{b}_k(0),\bx_+\right>>0$ and $\left<\boldsymbol{b}_k(t),\bx_-\right>>\left<\boldsymbol{b}_k(0),\bx_-\right>>0$. According to the definition of $T_{\rm hit}$, we have proved that $T_{\rm hit}$ does not exist, which means the neuron stays in $\cM_+^+\cap\cM_-^+$ for any $t\leq T_{\rm I}$.

\underline{{\bf Step II.}  Estimate the evolution of $\left<\bw_k(t),\bmu\right>$.}

With the help of {\bf Step I}, we were able to determine the dynamics for $t\leq T_{\rm I}$. For any $t\leq T_{\rm I}$, we can do the following estimate.

From \eqref{equ: dynamics rewrite}, the tangential dynamics of the neuron is 
\begin{gather*}
\frac{\mathrm{d}\bw_k(t)}{\mathrm{d} t}=\frac{\kappa_2}{\sqrt{m}\rho_k(t)}\Big(\bF_k(t)-\left<\bF_k(t),\bw_k(t)\right>\bw_k(t)\Big),
\\
\text{where}\quad\bF_k(t):=\frac{p}{1+p}e^{-f_+(t)}\bx_+-\frac{1}{1+p}e^{f_-(t)}\bx_-.
\end{gather*}

Recalling the definitions of $\bmu$ and $\boldsymbol{z}$, we can estimate the difference between $\bF_k(t)$ and $\boldsymbol{z}$:
\begin{align*}
    &\left<\bF_k(t),\boldsymbol{z}\right>=\left<\boldsymbol{z}+\frac{p}{1+p}\big(e^{-f_+(t)}-1\big)\bx_+-\frac{1}{1+p}\big(e^{f_-(t)}-1\big)\bx_-,\boldsymbol{z}\right>
    \\=&\left\|\boldsymbol{z}\right\|^2+\left<\frac{p}{1+p}\big(e^{-f_+(t)}-1\big)\bx_+-\frac{1}{1+p}\big(e^{f_-(t)}-1\big)\bx_-,    \frac{p}{1+p}\bx_+-\frac{1}{1+p}\bx_-\right>
    \\\geq&\left\|\boldsymbol{z}\right\|^2-\frac{p^2}{(1+p)^2}\left|e^{-f_+(t)}-1\right|-\frac{1}{(1+p)^2}\left|e^{f_-(t)}-1\right|-\frac{p\cos\Delta}{(1+p)^2}\Big(\left|e^{-f_+(t)}-1\right|+\left|e^{f_-(t)}-1\right|\Big)
    \\\overset{\text{Lemma \ref{lemma: GF Phase I norm estimate}}}{\geq}&\left\|\boldsymbol{z}\right\|^2-1.1\sqrt{\kappa_1\kappa_2}\Big(\frac{p^2}{(1+p)^2}+\frac{1}{(1+p)^2}+\frac{2p\cos\Delta}{(1+p)^2}\Big)\geq\left\|\boldsymbol{z}\right\|^2-1.1\sqrt{\kappa_1\kappa_2},
\end{align*}
\begin{align*}
    &\left\|\bF_k(t)-\boldsymbol{z}\right\|=\left\|\frac{p}{1+p}\big(e^{-f_+(t)}-1\big)\bx_+-\frac{1}{1+p}\big(e^{f_-(t)}-1\big)\bx_-\right\|
    \\\leq&\left\|\frac{p}{1+p}\big(e^{-f_+(t)}-1\big)\bx_+\right\|+\left\|\frac{1}{1+p}\big(e^{f_-(t)}-1\big)\bx_-\right\|
    \\\overset{\text{Lemma \ref{lemma: GF Phase I norm estimate}}}{\leq}&1.1\sqrt{\kappa_1\kappa_2}\frac{p}{1+p}+1.1\sqrt{\kappa_1\kappa_2}\frac{1}{1+p}=1.1\sqrt{\kappa_1\kappa_2}.
\end{align*}

The dynamics of $\left<\bw_k(t),\bmu\right>$ is:
\begin{equation*}
    \frac{\mathrm{d}\left<\bw_k(t),\bmu\right>}{\mathrm{d}t}=\frac{1}{\left\|\boldsymbol{z}\right\|}\frac{\mathrm{d}\left<\bw_k(t),\boldsymbol{z}\right>}{\mathrm{d}t}=\frac{\kappa_2}{\left\|\boldsymbol{z}\right\|\sqrt{m}\rho_k(t)}\Big(\left<\bF_k(t),\boldsymbol{z}\right>-\left<\bF_k(t),\bw_k(t)\right>\left<\bw_k(t),\boldsymbol{z}\right>\Big).
\end{equation*}
And it can be estimated by:
\begin{equation}\label{equ of proof lemma: GF Phase I positive S+ S-: estimate dynamics of <w,mu>}
\begin{aligned}
    &\frac{\mathrm{d}\left<\bw_k(t),\bmu\right>}{\mathrm{d}t}
    \\=&
    \frac{\kappa_2}{\left\|\boldsymbol{z}\right\|\sqrt{m}\rho_k(t)}\Big(\left<\bF_k(t),\boldsymbol{z}\right>-\left<\boldsymbol{z},\bw_k(t)\right>\left<\bw_k(t),\boldsymbol{z}\right>-\left<\bF_k(t)-\boldsymbol{z},\bw_k(t)\right>\left<\bw_k(t),\boldsymbol{z}\right>\Big)
    \\\geq&
    \frac{\kappa_2}{\left\|\boldsymbol{z}\right\|\sqrt{m}\rho_k(t)}\Big(\left<\bF_k(t),\boldsymbol{z}\right>-\left<\bw_k(t),\boldsymbol{z}\right>^2-\left\|\bF_k(t)-\boldsymbol{z}\right\|\left\|\boldsymbol{z}\right\|\Big)
    \\\geq&
    \frac{\kappa_2}{\left\|\boldsymbol{z}\right\|\sqrt{m}\rho_k(t)}\Big(\left\|\boldsymbol{z}\right\|^2-1.1\sqrt{\kappa_1\kappa_2}-\left<\bw_k(t),\boldsymbol{z}\right>^2-1.1\sqrt{\kappa_1\kappa_2}\left\|\boldsymbol{z}\right\|\Big)
    \\=&
    \frac{\kappa_2}{\left\|\boldsymbol{z}\right\|\sqrt{m}\rho_k(t)}\Big(\left\|\boldsymbol{z}\right\|^2-\left\|\boldsymbol{z}\right\|^2\left<\bw_k(t),\bmu\right>^2-1.1\sqrt{\kappa_1\kappa_2}\big(1+\left\|\boldsymbol{z}\right\|\big)\Big)
    \\=&
     \frac{\kappa_2\left\|\boldsymbol{z}\right\|}{\sqrt{m}\rho_k(t)}\Big(1-\left<\bw_k(t),\bmu\right>^2-\frac{1.1\sqrt{\kappa_1\kappa_2}\big(1+\left\|\boldsymbol{z}\right\|\big)}{\left\|\boldsymbol{z}\right\|^2}\Big)
     \\\overset{\text{Lemma \ref{lemma: basic: norm of z}}}{\geq}&
     \frac{2\kappa_2}{3\sqrt{m}\rho_k(t)}\left(1-\left<\bw_k(t),\bmu\right>^2-1.1\sqrt{\kappa_1\kappa_2}\Big(1.5+1.5^2\Big)\right)
     >\frac{2\kappa_2}{3\sqrt{m}\rho_k(t)}\Big(1-4.2\sqrt{\kappa_1\kappa_2}-\left<\bw_k(t),\bmu\right>^2\Big)
     \\\overset{\text{Lemma \ref{lemma: GF Phase I norm estimate}}}{\geq}&\frac{2\kappa_2}{3\sqrt{m}\frac{\kappa_1+1.1\kappa_2 t}{\sqrt{m}}}\Big(1-4.2\sqrt{\kappa_1\kappa_2}-\left<\bw_k(t),\bmu\right>^2\Big)=\frac{2}{3\big(\frac{\kappa_1}{\kappa_2}+1.1 t\big)}\Big(1-4.2\sqrt{\kappa_1\kappa_2}-\left<\bw_k(t),\bmu\right>^2\Big)
     .
\end{aligned}
\end{equation}

Noticing $\bw_k(0)\in\cM_+^+\cap\cM_-^+$, we have:
\begin{equation}\label{equ of proof lemma: GF Phase I positive S+ S-: estimate of <w(0),mu>}
    \begin{aligned}
        \left<\bw_k(0),\bmu\right>=\left<\bw_k(0),\frac{p}{1+p}\bx_+-\frac{1}{1+p}\bx_-\right>>-\left<\bw_k(0),\frac{1}{1+p}\bx_-\right>\geq-\frac{1}{1+p}.
    \end{aligned}
\end{equation}

Now we consider the following auxiliary ODE:
\begin{equation}\label{equ of proof lemma: GF Phase I positive S+ S-: upper ODE of <w,mu>}
\begin{cases}
    \frac{\mathrm{d}U(t)}{\mathrm{d}t}=\frac{2}{3\big(\frac{\kappa_1}{\kappa_2}+1.1 t\big)}\Big(1-4.2\sqrt{\kappa_1\kappa_2}-U^2(t)\Big)
    \\
    U(0)=-\frac{1}{1+p}
\end{cases},
\end{equation}
and let $U(t)$ is the solution of \eqref{equ of proof lemma: GF Phase I positive S+ S-: upper ODE of <w,mu>}.
Due to \eqref{equ of proof lemma: GF Phase I positive S+ S-: estimate dynamics of <w,mu>} \eqref{equ of proof lemma: GF Phase I positive S+ S-: estimate of <w(0),mu>}, we know that $\left<\bw_k(t),\bmu\right>$ is an upper solution of ODE \eqref{equ of proof lemma: GF Phase I positive S+ S-: upper ODE of <w,mu>}. From the Comparison Principle of ODEs, we know this means:
\[
\left<\bw_k(t),\bmu\right>>U(t),\text{ for any } t\leq T_{\rm I}.
\]

Hence, in order to estimate $\left<\bw_k(t),\bmu\right>$,
we only need to study the solution of ODE \eqref{equ of proof lemma: GF Phase I positive S+ S-: upper ODE of <w,mu>}. It is easy to verify that the solution of \eqref{equ of proof lemma: GF Phase I positive S+ S-: upper ODE of <w,mu>} satisfies
\begin{align*}
    \log\left(\frac{1-4.2\sqrt{\kappa_1\kappa_2}+U(t)}{1-4.2\sqrt{\kappa_1\kappa_2}-U(t)}\right)-\log\left(\frac{1-4.2\sqrt{\kappa_1\kappa_2}-\frac{1}{1+p}}{1-4.2\sqrt{\kappa_1\kappa_2}+\frac{1}{1+p}}\right)=\frac{4(1-4.2\sqrt{\kappa_1\kappa_2})}{3.3}\log\left(1+\frac{1.1\kappa_2}{\kappa_1}t\right).
\end{align*}
Then we have:
\begin{align*}
    &\log\left(\frac{1-4.2\sqrt{\kappa_1\kappa_2}+U(t)}{1-4.2\sqrt{\kappa_1\kappa_2}-U(t)}\right)\geq\log\left(\frac{1-4.2\sqrt{\kappa_1\kappa_2}-\frac{1}{6}}{1-4.2\sqrt{\kappa_1\kappa_2}
    +\frac{1}{6}}\right)+\frac{4(1-4.2\sqrt{\kappa_1\kappa_2})}{3.3}\log\left(1+\frac{1.1\kappa_2}{\kappa_1}t\right)
    \\\geq&\log\left(0.7\left(1+\frac{1.1\kappa_2}{\kappa_1}t\right)^{1.15}\right)\geq\log\left(0.7\left(1+\frac{1.1\kappa_2}{\kappa_1}t\right)^{1.15}\right),
\end{align*}
which means
\begin{align*}
    U(t)\geq\Big(1-4.2\sqrt{\kappa_1\kappa_2}\Big)\left(1-\frac{2}{1+0.7\left(1+\frac{1.1\kappa_2}{\kappa_1}t\right)^{1.15}}\right).
\end{align*}

Hence, we have the estimate of $\left<\bw_k(t),\bmu\right>$:
\begin{equation}\label{equ of proof lemma: GF Phase I positive S+ S-: estimate result of <w,mu>}
    \left<\bw_k(t),\bmu\right>\geq\Big(1-4.2\sqrt{\kappa_1\kappa_2}\Big)\left(1-\frac{2}{1+0.7\left(1+\frac{1.1\kappa_2}{\kappa_1}t\right)^{1.15}}\right).
\end{equation}

Specifically, we have:
\begin{align*}
    \left<\bw_k(T_{\rm I}),\bmu\right>\geq\Big(1-4.2\sqrt{\kappa_1\kappa_2}\Big)\left(1-\frac{2}{1+0.7\left(1+11\sqrt{\frac{\kappa_2}{\kappa_1}}\right)^{1.15}}\right).
\end{align*}

\underline{{\bf Step III.} A finer estimate of $\rho_k(T_{\rm I})$.}

In this step, we will estimate the lower bound and upper bound for $\rho_k(T_{\rm I})$.

First, lemma \ref{lemma: GF Phase I norm estimate} gives us the upper bound for $\rho_k(T_{\rm I})$:
\begin{align*}
    \rho_k(T_{\rm I})\leq \frac{\kappa_1+1.1\kappa_2 T_{\rm I}}{\sqrt{m}}\leq\frac{\kappa_1+11\sqrt{\kappa_1\kappa_2}}{\sqrt{m}}\leq\frac{12\sqrt{\kappa_1\kappa_2}}{\sqrt{m}}.
\end{align*}

Now we focus on the estimate of the lower bound. Recalling the dynamics of $\rho_k(t)$ \eqref{equ: dynamics rewrite}, for any $t\leq T_{\rm I}$,
\begin{align*}
    &\frac{\mathrm{d} \rho_k(t)}{\mathrm{d} t}=\frac{\kappa_2}{\sqrt{m}}\left<\bF_k(t),\bw_k(t)\right>=\frac{\kappa_2}{\sqrt{m}}\Big(\left<\boldsymbol{z},\bw_k(t)\right>+\left<\bF_k(t)-\boldsymbol{z},\bw_k(t)\right>\Big)
    \\\geq&\frac{\kappa_2}{\sqrt{m}}\Big(\left<\boldsymbol{z},\bw_k(t)\right>-\left\|\bF_k(t)-\boldsymbol{z}\right\|\Big)\geq\frac{\kappa_2}{\sqrt{m}}\Big(\left\|\boldsymbol{z}\right\|\left<\bw_k(t),\bmu\right>-1.1\sqrt{\kappa_1\kappa_2}\Big)
    \\
    \overset{\eqref{equ of proof lemma: GF Phase I positive S+ S-: estimate result of <w,mu>}}{\geq}&\frac{\kappa_2}{\sqrt{m}}\left(\frac{2\Big(1-4.2\sqrt{\kappa_1\kappa_2}\Big)}{3}\left(1-\frac{2}{1+0.7\left(1+\frac{1.1\kappa_2}{\kappa_1}t\right)^{1.15}}\right)-1.1\sqrt{\kappa_1\kappa_2}\right)
    \\\geq&\frac{\kappa_2}{\sqrt{m}}\Bigg(0.627-\frac{1.278}{1+0.7\big(1+\frac{1.1\kappa_2}{\kappa_1}t\big)^{1.15}}\Bigg).
\end{align*}
We denote $T_0=\frac{10\kappa_1}{\kappa_2}$, and it is easy to verify:
\[
0.627-\frac{1.278}{1+0.7\big(1+\frac{1.1\kappa_2}{\kappa_1}t\big)^{1.15}}>0.53,\text{ for any } t\in[T_{0}, T_{\rm I}].
\]
So we have:
\begin{align*}
    &\rho_k(T_{\rm I})\geq\rho_k(T_0)+\int_{T_0}^{T_{\rm I}}\frac{\kappa_2}{\sqrt{m}}\Bigg(0.627-\frac{1.278}{1+0.7\big(1+\frac{1.1\kappa_2}{\kappa_1}t\big)^{1.15}}\Bigg)\mathrm{d}t
    \\>&0+\int_{\frac{10\kappa_1}{\kappa_2}}^{T_{\rm I}}\frac{0.53\kappa_2}{\sqrt{m}}\mathrm{d}t=\frac{0.53(10\sqrt{\frac{\kappa_1}{\kappa_2}}-\frac{10\kappa_1}{\kappa_2})\kappa_2}{\sqrt{m}}\overset{\eqref{equ: parameter selection GF Phase I}}{\geq}\frac{5.3\cdot0.9\sqrt{\kappa_1\kappa_2}}{\sqrt{m}}\geq\frac{4.77\sqrt{\kappa_1\kappa_2}}{\sqrt{m}}.
\end{align*}

\end{proof}

\begin{lemma}[Positive, $\cM_+^+\cap\cM_-^-$]\label{lemma: GF Phase I positive S+}\ \\
For positive neuron $k\in\{k\in[m/2]:\bw_k(0)\in\cM_+^+\cap\cM_-^-\}$, at the end of Phase I, it holds that
\begin{align*}
    \text{\bf (Direction).}& \text{ It is aligned with } \bmu:
    \left<\bw_k(T_{\rm I}),\bmu\right>\geq\Big(1-\mathcal{O}(\sqrt{\kappa_1\kappa_2})\Big)\left(1-\mathcal{O}\Big((\frac{\kappa_1}{\kappa_2})^{0.55}\Big)\right);
    \\
    \text{\bf (Norm).}& \text{ It has a small but significant norm}: \rho_k(T_{\rm I})=\Theta\Big(\frac{\sqrt{\kappa_1\kappa_2}}{\sqrt{m}}\Big).
\end{align*} 

\end{lemma}

\begin{proof}[Proof of Lemma \ref{lemma: GF Phase I positive S+}]\ \\
We do the following analysis for any $k\in\{k\in[m/2]:\bw_k(0)\in\cM_+^+\cap\cM_-^-\}$, i.e. ${\rm s}_k=1$, $\left<\bw_k(0),\bx_+\right>>0$, and $\left<\bw_k(0),\bx_-\right>\leq0$.

\underline{{\bf Step I.} The neuron must arrives in $\cM_+^+\cap\cM_-^0$ in $\mathcal{O}\Big(\frac{\kappa_1 \Delta}{\kappa_2}\Big)$ time.}

The case $\left<\bw_k(0),\bx_-\right>=0$ is trivial.
Then we only need to consider the case $\left<\bw_k(0),\bx_-\right><0$.

First, we define the hitting time 
\[
T_{\rm hit}:=\inf\Big\{t\in(0, T_{\rm I}]:\bw_k(t)\notin\cM_+^+\cap\cM_-^-\Big\},
\] 
and we aim to estimate $T_{\rm hit}$ and prove $\bw_k(T_{\rm hit})\in\cM_+^+\cap\cM_-^0$. 

We focus on the dynamics of $\left<\boldsymbol{b}_k(t),\bx_+\right>$ and $\left<\boldsymbol{b}_k(t),\bx_-\right>$. 

From the definition of $T_{\rm hit}$ and \eqref{equ: dynamics}, the dynamics of the neuron is:
\[
\frac{\mathrm{d}\boldsymbol{b}_k(t)}{\mathrm{d}t}=\frac{\kappa_2}{\sqrt{m}}\frac{p}{1+p}e^{-f_+(t)}\bx_+,\ t\leq T_{\rm hit}.
\]
Then we have
\begin{align*}
    \frac{\mathrm{d}\left<\boldsymbol{b}_k(t),\bx_+\right>}{\mathrm{d}t}
    =&\left<\frac{\kappa_2}{\sqrt{m}}\frac{p}{1+p}e^{-f_+(t)}\bx_+,\bx_+\right>=\frac{\kappa_2 p}{\sqrt{m}(1+p)}e^{-f_+(t)},
    \\
    \frac{\mathrm{d}\left<\boldsymbol{b}_k(t),\bx_-\right>}{\mathrm{d}t}
    =&\left<\frac{\kappa_2}{\sqrt{m}}\frac{p}{1+p}e^{-f_+(t)}\bx_+,\bx_-\right>=\frac{\kappa_2 p\cos\Delta}{\sqrt{m}(1+p)}e^{-f_+(t)}.
\end{align*}

It is clear $ \frac{\mathrm{d}\left<\boldsymbol{b}_k(t),\bx_+\right>}{\mathrm{d}t}>0$, so $\left<\boldsymbol{b}_k(t),\bx_+\right>>\left<\boldsymbol{b}_k(0),\bx_+\right>>0$ for any $t\leq T_{\rm hit}$. If we denote
\[
T_{\rm hit,-}:=\inf\Big\{t\in(0,T_{\rm I}]:\left<\bw_k(t),\bx_-\right>\geq0\Big\},
\]
then it holds:
\[
T_{\rm hit}=T_{\rm hit,-}.
\]

So we only need to estimate $T_{\rm hit,-}$. Due to $T_{\rm hit}\leq T_{\rm I}\leq T_{\rm init}$ and Lemma \ref{lemma: GF Phase I norm estimate}, for any $t\leq T_{\rm hit}$, we have $\left|e^{-y_i f_i(t)}-1\right|\leq0.11$.
Then for any $t\leq T_{\rm hit}$, we have:
\[
\frac{\mathrm{d}\left<\boldsymbol{b}_k(t),\bx_-\right>}{\mathrm{d}t}\geq\frac{0.89\kappa_2 p\cos\Delta}{\sqrt{m}(1+p)}.
\]
Recalling $\left<\bw_k(0),\bx_+\right>>0$ and $\left<\bw_k(0),\bx_-\right><0$, with the help of Lemma \ref{lemma: theta in span x1, x2}, we have $\left<\bw_k(0),\bx_-\right>>-\sin\Delta$. Combining the two estimate, we have:
\begin{align*}
&\left<\boldsymbol{b}_k(t),\bx_-\right>\geq\left<\boldsymbol{b}_k(0),\bx_-\right>+\int_{0}^{t}\frac{0.89\kappa_2 p\cos\Delta}{\sqrt{m}(1+p)}\mathrm{d}t
\\>&-\rho_k(0)\sin\Delta+\frac{0.89\kappa_2 p\cos\Delta}{\sqrt{m}(1+p)}t=-\frac{\kappa_1\sin\Delta}{\sqrt{m}}+\frac{0.89\kappa_2 p\cos\Delta}{\sqrt{m}(1+p)}t.
\end{align*}
Hence, 
\[
T_{\rm hit}=T_{\rm hit,-}\leq\frac{(1+p)\tan\Delta}{0.89 p}\frac{\kappa_1}{\kappa_2}\leq2\Delta\frac{\kappa_1}{\kappa_2}.
\]

\underline{{\bf Step II.} Dynamics after arriving in the manifold $\cM_+^+\cap\cM_{-}^0$.}

In this step, we analyze the training dynamics after
$\bw_k(T_{\rm hit})\in\cM_+^+\cap\cM_{-}^0$. 

First, we will prove $\bw_k(t)$ passes immediately from one side of the surface $\cM_+^+\cap\cM_-^0$ to the other, i.e.  $\bw_k(t)$ enters into $\cM_+^+\cap\cM_-^+$ at time $T_{\rm hit}$. 
Equivalently, we only need to prove $\boldsymbol{b}_k(t)$ passes immediately from one side of the surface $\cP_+^+\cap\cP_-^0$ to the other, i.e.  $\boldsymbol{b}_k(t)$ enters into $\cP_+^+\cap\cP_-^+$ at time $T_{\rm hit}$.

For any $\tilde{\bb}\in\cP_+^+\cap\cP_{-}^0$ and $0<\delta_0\ll1$, we know that $\cP_+^+\cap\cP_{-}^0$ separates its neighborhood $\mathcal{B}(\tilde{\bb },\delta_0)$ into two domains $\mathcal{G}_-=\{\bb \in\mathcal{B}(\tilde{\bb},\delta_0):\left<\bb ,\bx_-\right><0\}$ and $\mathcal{G}_+=\{\bb\in\mathcal{B}(\tilde{\bb},\delta_0):\left<\bb,\bx_-\right>>0\}$. Following Definition \ref{def: discontinuous system solution}, we calculate the limited vector field on $\tilde{\bb}$ from $\mathcal{G}_-$ and $\mathcal{G}_+$.

(i) The limited vector field $\bF^-$ on $\tilde{\bb}$ (from $\mathcal{G}_-$):
\begin{gather*}
\frac{\mathrm{d}\bb }{\mathrm{d} t}=\bF^-,\text{ where }
\bF^-=\frac{\kappa_2}{\sqrt{m}}\bracket{\frac{p}{1+p}e^{-f_+(t)}\bx_+}.
\end{gather*}

(ii) The limited vector field $\bF^+$ on $\tilde{\bb}$ (from $\mathcal{G}_+$):
\begin{gather*}
\frac{\mathrm{d}\bb}{\mathrm{d} t}=\bF^+,\text{ where }
\bF^+=\frac{\kappa_2}{\sqrt{m}}\bracket{\frac{p e^{-f_+(t)}}{1+p}\bx_+-\frac{e^{f_-(t)}}{1+p}\bx_-}.
\end{gather*}

(iii) Then we calculate the
projections of $\bF^-$ and $\bF^+$ onto $\bx_-$ (the normal to the surface $\cP_+^+\cap\cP_-^0$):
\begin{gather*}
F_N^{-}=\left<\bF^-,\bx_-\right>
=\frac{\kappa_2 p e^{-f_+(t)}}{\sqrt{m}(1+p)}\cos\Delta,
\\
F_N^{+}=\left<\bF^+,\bx_-\right>=
\frac{\kappa_2 p e^{-f_+(t)}}{\sqrt{m}(1+p)}\cos\Delta-\frac{\kappa_2 e^{f_-(t)}}{\sqrt{m}(1+p)}.
\end{gather*}
From $T_{\rm I}<T_{\rm init}$ and Lemma \ref{lemma: GF Phase I norm estimate}, we know $|e^{-y_if_i(t)}-1|\leq0.11$, so 
$p e^{-f_+(t)}\cos\Delta-e^{f_-(t)}\geq 0.89p\cos\Delta-1.11{>}0$, which means $F_N^{+}>0$. 
And it is clear that $F_N^{-}>0$. 
Hence, the dynamics corresponds to Case (II) in Definition \ref{def: discontinuous system solution} ($F_N^{-}>0$ and $F_N^{+}>0$).

(iv) Hence, $\bb_k(t)$ passes immediately from one side of the surface $\cP_+^+\cap\cP_-^0$ to the other, i.e.  $\bb_k(t)$ enters into $\cP_+^+\cap\cP_-^+$ at time $T_{\rm hit}$.

Second, proceeding as in the proof of {\bf Step I$\sim$III} of the Proof of Theorem \ref{lemma: GF Phase I positive S+ S-}, we have the results:
\begin{align*}
    &\left<\bw_k(T_{\rm I}),\bmu\right>\geq\Big(1-4.2\sqrt{\kappa_1\kappa_2}\Big)\left(1-\frac{2}{1+0.7\left(1+1.1(T_{\rm I}-T_{\rm hit})\right)^{1.15}}\right)
    \\\geq&\Big(1-4.2\sqrt{\kappa_1\kappa_2}\Big)\left(1-\frac{2}{1+0.7\left(1+1.1\frac{\kappa_2}{\kappa_1}(10\sqrt{\frac{\kappa_1}{\kappa_2}}-2\Delta\frac{\kappa_1}{\kappa_2})\right)^{1.15}}\right)
    \\\geq&\Big(1-4.2\sqrt{\kappa_1\kappa_2}\Big)\left(1-\frac{2}{1+0.7\left(1+9.9\sqrt{\frac{\kappa_2}{\kappa_1}}\right)^{1.15}}\right).
\end{align*}
\begin{align*}
    \rho_k(T_{\rm I})\leq \frac{\kappa_1+1.1\kappa_2 T_{\rm I}}{\sqrt{m}}\leq\frac{\kappa_1+11\sqrt{\kappa_1\kappa_2}}{\sqrt{m}}\leq\frac{12\sqrt{\kappa_1\kappa_2}}{\sqrt{m}}.
\end{align*}
\begin{align*}
    &\rho_k(T_{\rm I})\geq\rho_k(T_0)+\int_{T_{\rm hit}+\frac{10\kappa_1}{\kappa_2}}^{T_{\rm I}}\frac{\kappa_2}{\sqrt{m}}\Bigg(0.627-\frac{1.278}{1+0.7\big(1+\frac{1.1\kappa_2}{\kappa_1}t\big)^{1.15}}\Bigg)\mathrm{d}t
    \\>&0+\int_{T_{\rm hit}+\frac{10\kappa_1}{\kappa_2}}^{T_{\rm I}}\frac{0.53\kappa_2}{\sqrt{m}}\mathrm{d}t=\frac{0.53(10\sqrt{\frac{\kappa_1}{\kappa_2}}-\frac{(10+2\Delta)\kappa_1}{\kappa_2})\kappa_2}{\sqrt{m}}\overset{\eqref{equ: parameter selection GF Phase I}}{\geq}\frac{5.3\cdot0.88\sqrt{\kappa_1\kappa_2}}{\sqrt{m}}\geq\frac{4.66\sqrt{\kappa_1\kappa_2}}{\sqrt{m}}.
\end{align*}

\end{proof}

\begin{lemma}[Positive, $\cM_+^-\cap\cM_-^+$]\label{lemma: GF Phase I positive S-}\ \\
For positive neuron $k\in\{k\in[m/2]:\bw_k(0)\in\cM_+^-\cap\cM_-^+\}$, 
after $\mathcal{O}\Big(\frac{\kappa_1 p\Delta}{\kappa_2}\Big)$ time, 
it goes dead: 
\[\bw_k(t)\in\cM_+^-\cap\cM_-^-, \text{ for any } t\geq T_{\rm I}> \mathcal{O}\Big(\frac{\kappa_1 p\Delta}{\kappa_2}\Big).
\]
\end{lemma}

\begin{proof}[Proof of Lemma \ref{lemma: GF Phase I negative S+}]\ \\
We do the following analysis for any $k\in\{k\in[m/2]:\bw_k(0)\in\cM_+^+\cap\cM_-^-\}$, i.e. ${\rm s}_k=1$, $\left<\bw_k(0),\bx_+\right>>0$, and $\left<\bw_k(0),\bx_-\right>\leq0$. 

First, we define the hitting time 
\[
T_{\rm hit}:=\inf\Big\{t\in(0, T_{\rm I}]:\bw_k(t)\notin\cM_+^-\cap\cM_-^+\Big\},
\] 
and we aim to estimate $T_{\rm hit}$ and prove $\bw_k(T_{\rm hit})\in\cM_+^-\cap\cM_-^-$. 

From the definition of $T_{\rm hit}$ and \eqref{equ: dynamics}, the dynamics of the neuron is:
\[
\frac{\mathrm{d}\boldsymbol{b}_k(t)}{\mathrm{d}t}=-\frac{\kappa_2}{\sqrt{m}}\frac{1}{1+p}e^{f_-(t)}\bx_-,\ t\leq T_{\rm hit}.
\]
Then we have
\begin{align*}
    \frac{\mathrm{d}\left<\boldsymbol{b}_k(t),\bx_+\right>}{\mathrm{d}t}
    =&\left<-\frac{\kappa_2}{\sqrt{m}}\frac{1}{1+p}e^{f_-(t)}\bx_-,\bx_+\right>=-\frac{\kappa_2}{\sqrt{m}}\frac{1}{1+p}e^{f_-(t)}\cos\Delta,
    \\
    \frac{\mathrm{d}\left<\boldsymbol{b}_k(t),\bx_-\right>}{\mathrm{d}t}
    =&\left<-\frac{\kappa_2}{\sqrt{m}}\frac{1}{1+p}e^{f_-(t)}\bx_-,\bx_-\right>=-\frac{\kappa_2}{\sqrt{m}}\frac{1}{1+p}e^{f_-(t)}.
\end{align*}

It is clear $ \frac{\mathrm{d}\left<\boldsymbol{b}_k(t),\bx_+\right>}{\mathrm{d}t}<0$, so $\left<\boldsymbol{b}_k(t),\bx_+\right><\left<\boldsymbol{b}_k(0),\bx_+\right>\leq0$ for any $t\leq T_{\rm hit}$. If we denote
\[
T_{\rm hit,-}:=\inf\Big\{t\in(0,T_{\rm I}]:\left<\bw_k(t),\bx_-\right>\leq0\Big\},
\]
then it holds:
\[
T_{\rm hit}=T_{\rm hit,-}.
\]

So we only need to estimate $T_{\rm hit,-}$. Due to $T_{\rm hit}\leq T_{\rm I}\leq T_{\rm init}$ and Lemma \ref{lemma: GF Phase I norm estimate}, for any $t\leq T_{\rm hit}$, we have $\left|e^{-y_i f_i(t)}-1\right|\leq0.11$.
Then for any $t\leq T_{\rm hit}$, we have:
\[
\frac{\mathrm{d}\left<\boldsymbol{b}_k(t),\bx_-\right>}{\mathrm{d}t}\leq-\frac{0.89\kappa_2}{\sqrt{m}(1+p)}.
\]
Recalling $\left<\bw_k(0),\bx_+\right>\leq0$ and $\left<\bw_k(0),\bx_-\right>>0$, with the help of Lemma \ref{lemma: theta in span x1, x2}, we have $\left<\bw_k(0),\bx_-\right>\leq\sin\Delta$. Combining the two estimate, we have:
\begin{align*}
&\left<\boldsymbol{b}_k(t),\bx_-\right>\leq\left<\boldsymbol{b}_k(0),\bx_-\right>-\int_{0}^{t}\frac{0.89\kappa_2}{\sqrt{m}(1+p)}\mathrm{d}t
\\\leq&\rho_k(0)\sin\Delta-\frac{0.89\kappa_2}{\sqrt{m}(1+p)}t=\frac{\kappa_1\sin\Delta}{\sqrt{m}}-\frac{0.89\kappa_2}{\sqrt{m}(1+p)}t.
\end{align*}
Hence, 
\[
T_{\rm hit}=T_{\rm hit,+}\leq\frac{(1+p)\sin\Delta}{0.89}\frac{\kappa_1}{\kappa_2}<T_{\rm I}=10\sqrt{\frac{\kappa_1}{\kappa_2}}.
\]
Moreover, the analysis gives us $\bw_k(T_{\rm hit})\in\cM_+^-\cap\cM_-^-$. By Lemma \ref{lemma: GF dead neurons keep dead}, we obtain:
\[
\bw_k(t)\in\cM_+^-\cap\cM_-^-,\text{ for any } t\geq T_{\rm hit}.
\]

\end{proof}

\begin{lemma}[Positive, $\cM_+^-\cap\cM_-^-$]\label{lemma: GF Phase I positive dead}\ \\
For positive neuron $k\in\{k\in[m/2]:\bw_k(0)\in\cM_+^-\cap\cM_-^-\}$, it keeps dead forever.
\end{lemma}

\begin{proof}[Proof of Lemma \ref{lemma: GF Phase I positive dead}] Due to Lemma \ref{lemma: GF dead neurons keep dead}, this lemma is trivial.
\end{proof}

\label{appendix: subsection: Phase I: positive}
\subsection{The Dynamics of Negative Neurons}

According to the initial direction, all negative neurons $({\rm d}_k=-1)$ can be divided into the following four classes.

\begin{align*}
    &[m]-[m/2]
    \\=&\left\{k\in[m]-[m/2]: \bw_k(0)\in\cM_+^+\cap\cM_-^+\right\}
    \bigcup
    \left\{k\in[m]-[m/2]:\bw_k(0)\in \cM_+^-\cap\cM_-^+\right\}
    \\&
    \bigcup
    \left\{k\in[m]-[m/2]:\bw_k(0)\in \cM_+^+\cap\cM_-^-\right\}
    \bigcup
    \left\{k\in[m]-[m/2]:\bw_k(0)\in \cM_+^-\cap\cM_-^-\right\}.
\end{align*}

In the following four lemmas, we will prove the dynamics of these four classes of negative neurons.
In summary, in Phase I ($t<T_{\rm I}$), some of the negative neurons move to the manifold $\cM_+^0\cap\cM_-^+$ in a shorter time and then remain on this manifold, and their norms grow slowly, while other negative neurons go dead.

\begin{lemma}[Negative, $\cM_+^+\cap\cM_-^+$]\label{lemma: GF Phase I negative S+ S-}\ \\
For negative neuron $k\in\{k\in[m]-[m/2]:\bw_k(0)\in\cM_+^+\cap\cM_-^+\}$, in Phase I $(t\leq T_{\rm I})$, it's dynamics must belong to one of the following two cases:
\begin{align*}
    &\text{{\bf(i. Living)}. (S1). }\bw_k(t)\in\cM_+^{0}\cap\cM_-^{+}
    \text{ for any } t\geq\mathcal{O}(\frac{\kappa_1}{\kappa_2}), \\&\quad\quad\quad\quad\quad\quad \text{(S2). It has a small norm}: \rho_k(T_{\rm I})=\mathcal{O}\Big(\frac{\sqrt{\kappa_1\kappa_2}}{\sqrt{m}}\big(\sqrt{\frac{\kappa_1}{\kappa_2}}+\frac{\Delta}{p}\big)\Big),
    \\&\quad\quad\quad\quad\quad\quad \text{(S3). It is weakly aligned with $\bx_+^{\perp}$}:
    \left<\bw_k(T_{\rm I}),\bx_+^{\perp}\right>\geq1-\mathcal{O}\Big((\sqrt{\frac{\kappa_1}{\kappa_2}}\frac{p}{\Delta})^{1.6}\Big)
    ;
    \\
    &\text{{\bf(ii. Dead)}. }\bw_k(t)\in\cM_+^-\cap\cM_-^- \text{ for any } t\geq\mathcal{O}(\frac{\kappa_1}{\kappa_2}).
\end{align*}
{\bf Moreover}, if $\left<\bw_k(0),\bx_-\right>>\frac{(1+\mathcal{O}(\kappa_1\kappa_2))p\cos\Delta-(1-\mathcal{O}(\kappa_1\kappa_2))}{(1-\mathcal{O}(\kappa_1\kappa_2))p-(1+\mathcal{O}(\kappa_1\kappa_2))\cos\Delta}\left<\bw_k(0),\bx_+\right>$, it must belongs to Case {\rm (i)}; if $\left<\bw_k(0),\bx_+\right>>\frac{(1+\mathcal{O}(\kappa_1\kappa_2))p-(1-\mathcal{O}(\kappa_1\kappa_2))\cos\Delta}{(1-\mathcal{O}(\kappa_1\kappa_2))p\cos\Delta-(1+\mathcal{O}(\kappa_1\kappa_2))}\left<\bw_k(0),\bx_-\right>$, it must belongs to Case {\rm (ii)}.
\end{lemma}

\begin{proof}[Proof of Lemma \ref{lemma: GF Phase I negative S+ S-}]\ \\
We do the following analysis for any $k\in\{k\in[m]-[m/2]:\bw_k(0)\in\cM_+^+\cap\cM_-^+\}$, i.e. ${\rm s}_k=-1$, $\left<\bw_k(0),\bx_+\right>>0$, and $\left<\bw_k(0),\bx_-\right>>0$.

\underline{{\bf Step I.}  Neuron must arrives in the border $\partial(\cM_+^+\cap\cM_-^+)$ in $\mathcal{O}(\frac{\kappa_1}{\kappa_2})$ time.} 

First, we define the hitting time 
\[
T_{\rm hit}:=\inf\Big\{t\in(0, T_{\rm I}]:\bw_k(t)\in\partial(\cM_+^+\cap\cM_-^+)\Big\},
\] 
and we aim to prove $T_{\rm hit}$ exists and estimate $T_{\rm hit}$. 

Recalling the decoupling $\partial(\cM_+^+\cap\cM_-^+)=\big(\cM_+^0\cap\cM_{-}^+\big)\cup\big(\cM_-^0\cap\cM_{+}^+\big)\cup\big(\cM_+^0\cap\cM_{-}^0\big)$, we only need to focus on the dynamics of $\left<\boldsymbol{b}_k(t),\bx_+\right>$ and $\left<\boldsymbol{b}_k(t),\bx_-\right>$. 

From the definition of $T_{\rm hit}$ and \eqref{equ: dynamics}, the dynamics of the neuron is:
\[
\frac{\mathrm{d}\boldsymbol{b}_k(t)}{\mathrm{d}t}=-\frac{\kappa_2}{\sqrt{m}}\frac{1}{n}\sum_{i=1}^n e^{-y_if_i(t)}y_i\bx _i
=-\frac{\kappa_2}{\sqrt{m}}\Big(\frac{p}{1+p}e^{-f_+(t)}\bx_+-\frac{1}{1+p}e^{f_-(t)}\bx_-\Big),\ t\leq T_{\rm hit}.
\]
Then we have
\begin{align*}
    \frac{\mathrm{d}\left<\boldsymbol{b}_k(t),\bx_+\right>}{\mathrm{d}t}
    =&\left<-\frac{\kappa_2}{\sqrt{m}}\Big(\frac{p}{1+p}e^{-f_+(t)}\bx_+-\frac{1}{1+p}e^{f_-(t)}\bx_-\Big),\bx_+\right>
    \\=&
    -\frac{\kappa_2}{\sqrt{m}}\Big(\frac{p}{1+p}e^{-f_+(t)}-\frac{1}{1+p}e^{f_-(t)}\cos\Delta\Big),
\end{align*}
\begin{align*}
    \frac{\mathrm{d}\left<\boldsymbol{b}_k(t),\bx_-\right>}{\mathrm{d}t}
    =&\left<-\frac{\kappa_2}{\sqrt{m}}\Big(\frac{p}{1+p}e^{-f_+(t)}\bx_+-\frac{1}{1+p}e^{f_-(t)}\bx_-\Big),\bx_-\right>
    \\=&
    -\frac{\kappa_2}{\sqrt{m}}\Big(\frac{p}{1+p}e^{-f_+(t)}\cos\Delta-\frac{1}{1+p}e^{f_-(t)}\Big).
\end{align*}

Due to $T_{\rm hit}\leq T_{\rm I}\leq T_{\rm init}$ and Lemma \ref{lemma: GF Phase I norm estimate}, for any $t\leq T_{\rm hit}$, we have $\left|e^{-y_i f_i(t)}-1\right|\leq1.1\sqrt{\kappa_1\kappa_2}$.  
Then for any $t\leq T_{\rm hit}$, we have:
\begin{align*}
    &\frac{\mathrm{d}\left<\boldsymbol{b}_k(t),\bx_+\right>}{\mathrm{d}t}
    \leq
    -\frac{\kappa_2}{\sqrt{m}}\Big(\frac{(1-1.1\sqrt{\kappa_1\kappa_2})p}{1+p}-\frac{1+1.1\sqrt{\kappa_1\kappa_2}}{1+p}\cos{\Delta}\Big)
    \\\leq&
    -\frac{\kappa_2\Big((1-1.1\sqrt{\kappa_1\kappa_2})p-(1+1.1\sqrt{\kappa_1\kappa_2})\cos\Delta\Big)}{\sqrt{m}(1+p)}
    \leq
    -\frac{\kappa_2(0.98p-1.02)}{\sqrt{m}(1+p)},
\end{align*}
\begin{align*}
    &\frac{\mathrm{d}\left<\boldsymbol{b}_k(t),\bx_-\right>}{\mathrm{d}t}
    \leq
    -\frac{\kappa_2}{\sqrt{m}}\Big(\frac{(1-1.1\sqrt{\kappa_1\kappa_2})p}{1+p}\cos{\Delta}-\frac{1+1.1\sqrt{\kappa_1\kappa_2}}{1+p}\Big)
    \\\leq&
    -\frac{\kappa_2\Big((1-1.1\sqrt{\kappa_1\kappa_2})p\cos{\Delta}-(1+1.1\sqrt{\kappa_1\kappa_2})\Big)}{\sqrt{m}(1+p)}
    \leq-\frac{0.98\kappa_2(p\cos\Delta-1)}{2\sqrt{m}(1+p)}.
\end{align*}
Now we consider the time 
\[
T_{\rm test}:=\frac{3\kappa_1}{\kappa_2}.
\]
If we assume $T_{\rm test}<T_{\rm hit}$, then we have the estimate:
\begin{align*}
    &\left<\boldsymbol{b}_k(T_{\rm test}),\bx_+\right>\leq\left<\boldsymbol{b}_k(0),\bx_+\right>-\int_{0}^{T_{\rm test}}\frac{\kappa_2(0.98p-1.02)}{\sqrt{m}(1+p)}\mathrm{d}t
    \\\leq& \frac{\kappa_1}{\sqrt{m}}-\frac{\kappa_2(0.98p-1.02)}{\sqrt{m}(1+p)}\frac{3\kappa_1}{\kappa_2}<0,
\end{align*}
which is contradict to the definition of $T_{\rm hit}$. Hence, we have:
\[
T_{\rm hit}\leq T_{\rm test}\leq\frac{3\kappa_1}{\kappa_2},
\]
which means neurons must arrive in the border $\partial(\cM_+^+\cap\cM_-^+)$ in $\mathcal{O}(\frac{\kappa_1}{\kappa_2})$ time.

Because $\partial(\cM_+^+\cap\cM_-^+)=\big(\cM_+^0\cap\cM_{-}^+\big)\cup\big(\cM_-^0\cap\cM_{+}^+\big)\cup\big(\cM_+^0\cap\cM_{-}^0\big)$,
the neuron must arrives in $\cM_+^0\cap\cM_{-}^+$ or $\cM_-^0\cap\cM_{+}^+$ or $\cM_+^0\cap\cM_{-}^0$. 
If the neuron arrives in $\cM_+^0\cap\cM_{-}^0$, it goes dead forever (Lemma \ref{lemma: GF dead neurons keep dead}). We will analyze the training dynamics after arriving in $\cM_+^0\cap\cM_{-}^+$ or $\cM_-^0\cap\cM_{+}^+$ in the following Step II and Step III.

\underline{{\bf Step II.} Dynamics after arriving in the manifold $\cM_+^0\cap\cM_{-}^+$.}

In this step, we will analyze the training dynamics after $\bw_k(T_{\rm hit})\in\cM_+^0\cap\cM_{-}^+$, i.e. after $\bb_k(T_{\rm hit})\in\cP_+^0\cap\cP_{-}^+$.

We first analysis the vector field around the manifold $\cP_+^0\cap\cP_-^+$ for $T_{\rm hit}\leq t\leq T_{\rm I}$. 

For any $\tilde{\bb}\in\cP_+^0\cap\cP_{-}^+$ and $0<\delta_0\ll1$, we know that $\cP_+^0\cap\cP_{-}^+$ separates its neighborhood $\mathcal{B}(\tilde{\bb },\delta_0)$ into two domains $\mathcal{G}_-=\{\bb \in\mathcal{B}(\tilde{\bb},\delta_0):\left<\bb ,\bx_+\right><0\}$ and $\mathcal{G}_+=\{\bb\in\mathcal{B}(\tilde{\bb},\delta_0):\left<\bb,\bx_+\right>>0\}$. Following Definition \ref{def: discontinuous system solution}, we calculate the limited vector field on $\tilde{\bb}$ from $\mathcal{G}_-$ and $\mathcal{G}_+$.

(i) The limited vector field $\bF^-$ on $\tilde{\bb}$ (from $\mathcal{G}_-$):
\begin{gather*}
\frac{\mathrm{d}\bb }{\mathrm{d} t}=\bF^-,\text{ where }
\bF^-=\frac{\kappa_2}{\sqrt{m}}\frac{1}{1+p}e^{f_-(t)}\bx_-.
\end{gather*}

(ii) The limited vector field $\bF^+$ on $\tilde{\bb}$ (from $\mathcal{G}_+$):
\begin{gather*}
\frac{\mathrm{d}\bb}{\mathrm{d} t}=\bF^+,\text{ where }
\bF^+=-\frac{\kappa_2}{\sqrt{m}}\bracket{\frac{p e^{-f_+(t)}}{1+p}\bx_+-\frac{e^{f_-(t)}}{1+p}\bx_-}.
\end{gather*}




(iii) Then we calculate the
projections of $\bF^-$ and $\bF^+$ onto $\bx_+$ (the normal to the surface $\cP_+^0\cap\cP_-^+$):
\begin{gather*}
F_N^{-}=\left<\bF^-,\bx_+\right>
=\frac{\kappa_2 e^{f_-(t)}}{\sqrt{m}(1+p)}\cos\Delta,
\\
F_N^{+}=\left<\bF^+,\bx_+\right>=\frac{\kappa_2 e^{f_-(t)}}{\sqrt{m}(1+p)}\cos\Delta-\frac{\kappa_2 p e^{-f_+(t)}}{\sqrt{m}(1+p)}.
\end{gather*}
From $T_{\rm I}<T_{\rm init}$ and Lemma \ref{lemma: GF Phase I norm estimate}, we know $|e^{-y_if_i(t)}-1|\leq0.11$, so 
$p e^{-f_+(t)}\cos\Delta-e^{f_-(t)}\geq 0.89p\cos\Delta-1.11{>}0$, which means $F_N^{+}<0$. 
And it is clear that $F_N^{-}>0$. 
Hence, the dynamics corresponds to Case (I) in Definition \ref{def: discontinuous system solution} ($F_N^{-}>0$ and $F_N^{+}<0$).

(iv) Hence, $\bb_k(t)$ can not leave $\cP_+^0\cap\cP_-^+$ for $T_{\rm hit}\leq t\leq T_{\rm I}$. 

(v) Moreover, the dynamics of $\bb_k$ on $\cP_+^0\cap\cP_-^+$ satisfies:
\begin{align*}
    \frac{\mathrm{d}\bb}{\mathrm{d}t}=\alpha\bF^++(1-\alpha)\bF^-,\quad \alpha=\frac{{f}_N^-}{{f}_N^--{f}_N^+},
\end{align*}
which is
\begin{align*}
    \frac{\mathrm{d}\bb_k(t)}{\mathrm{d}t}\frac{\kappa_2 e^{f_-(t)}}{\sqrt{m}(1+p)}\Big(\bx_--\bx_+\cos\Delta\Big).
\end{align*}

By Lemma~\ref{lemma: dynamics decomposition}, we know that the dynamics of $\bw_k(t)$ on $\cM_+^0\cap\cM_-^+$ and the dynamics of $\rho_k(t)$ are:
\begin{equation}\label{equ of proof lemma: GF Phase I: dynamics - w manifold}
    \frac{\mathrm{d}\bw_k(t)}{\mathrm{d}t}=\frac{\kappa_2 e^{f_-(t)}}{\rho_k(t)\sqrt{m}(1+p)}\Big(\bx_--\left<\bw_k(t),\bx_-\right>\bw_k-\bx_+\cos\Delta\Big).
\end{equation}
\begin{equation}\label{equ of proof lemma: GF Phase I: dynamics - rho manifold}
    \frac{\mathrm{d}\rho_k(t)}{\mathrm{d}t}=\frac{\kappa_2 e^{f_-(t)}}{\sqrt{m}(1+p)}\left<\bw_k(t),\boldsymbol{x_-}\right>.
\end{equation}

(vi) In this step, we aim to estimate $\rho_k(T_{\rm I})$ and $\left<\bw_k(t),\bx_+^{\perp}\right>$. 

From $\bw_k(T_{\rm hit})\in\cM_{+}^0\cap\cM_-^+$, it holds that $\left<\bw_k(T_{\rm hit}),\bx_+\right>=0$ and $\left<\bw_k(T_{\rm hit}),\bx_-\right>>0$. Using lemma \ref{lemma: theta in span x1, x2}, we have $0<\left<\bw_k(T_{\rm hit}),\bx_-\right>\leq\sin\Delta$.

Recalling Lemma \ref{lemma: GF Phase I norm estimate} and the estimate of $T_{\rm hit}$ in {\bf Step I}, we have:
\[
0\leq\rho_k(T_{\rm hit})\leq\frac{\kappa_1+1.1\kappa_2 T_{\rm hit}}{\sqrt{m}},
\]
and we can estimate the dynamics for $T_{\rm hit}\leq t\leq T_{\rm I}$ by (v)(vi):
\begin{gather*}
0\leq\frac{\mathrm{d}\rho_k(t)}{\mathrm{d}t}=\frac{\kappa_2 e^{f_-(t)}}{\sqrt{m}(1+p)}\left<\bw_k(t),\boldsymbol{x_-}\right>\leq\frac{\kappa_2 e^{f_-(t)}}{\sqrt{m}(1+p)}\sin\Delta,
\\
e^{f_-(t)}\leq1+0.11=1.11.
\end{gather*}
Then we obtain the estimate of $\rho_k(t)$ for any $T_{\rm hit}<t\leq T_{\rm I}$:
\begin{align*}
    &\rho_k(t)=\rho_k(T_{\rm hit})+\int_{T_{\rm hit}}^{t}\frac{\mathrm{d}\rho_k(s)}{\mathrm{d}t}\mathrm{d}s\leq
    \frac{\kappa_1+1.1\kappa_2 T_{\rm hit}}{\sqrt{m}}+\frac{\kappa_2 e^{f_-(t)}\sin\Delta}{\sqrt{m}(1+p)}(t-T_{\rm hit})
    \\\leq&
    \frac{\kappa_1+1.1\kappa_2 T_{\rm hit}}{\sqrt{m}}+\frac{1.11\kappa_2\sin\Delta}{\sqrt{m}(1+p)}(T_{\rm I}-T_{\rm hit})\leq\frac{\kappa_1+1.1\kappa_2 \frac{3\kappa_1}{\kappa_2}}{\sqrt{m}}+\frac{1.11\kappa_2\sin\Delta}{\sqrt{m}(1+p)}(T_{\rm I}-\frac{3\kappa_1}{\kappa_2})
    \\\leq&\frac{4.3\kappa_1}{\sqrt{m}}+\frac{1.11\kappa_2\sin\Delta}{\sqrt{m}(1+p)}t.
\end{align*}
Specifically, we have:
\[
\rho_k(T_{\rm I})\leq\frac{4.3\kappa_1}{\sqrt{m}}+\frac{1.11\kappa_2\sin\Delta}{\sqrt{m}(1+p)}T_{\rm I}\leq\frac{4.3\kappa_1}{\sqrt{m}}+\frac{11.1\sqrt{\kappa_1\kappa_2}\sin\Delta}{\sqrt{m}(1+p)}=\frac{\sqrt{\kappa_1\kappa_2}}{\sqrt{m}}\Big(4.3\sqrt{\frac{\kappa_1}{\kappa_2}}+\frac{11.1\sin\Delta}{1+p}\Big).
\]

For any $\bw \in\cM_+^0\cap\cM_-^+$, we have $\left<\bw ,\bx_+\right>=0$, so
\begin{align*}
    \left<\bw ,\bx_+^{\perp}\right>=\left<\bw ,\frac{\bx_--\bx_+\cos\Delta}{\left\|\bx_--\bx_+\cos\Delta\right\|}\right>=\frac{\left<\bw ,\bx_-\right>}{\left\|\bx_--\bx_+\cos\Delta\right\|}=\frac{1}{\sin\Delta}\left<\bw ,\bx_-\right>.
\end{align*}

So we only need to focus on the dynamics of $\left<\bw_k(t),\bx_-\right>$ to derive the dynamics of $\left<\bw_k(t),\bx_+^{\perp}\right>$.

By \eqref{equ of proof lemma: GF Phase I: dynamics - w manifold} and the estimate of $\rho_k(t)$, for any $T_{\rm hit}\leq t\leq T_{\rm I}$ we have:
\begin{align*}
    &\frac{\mathrm{d}\left<\bw_k(t),\bx_-\right>}{\mathrm{d}t}=\left<\frac{e^{f_-(t)}}{\rho_k(t)\sqrt{m}(1+p)}\Big(\bx_--\left<\bw_k(t),\bx_-\right>\bw_k(t)-\bx_+\cos\Delta\Big),\bx_-\right>
    \\=&\frac{\kappa_2e^{f_-(t)}}{\rho_k(t)\sqrt{m}(1+p)}\Big(\sin^2\Delta-\left<\bw_k(t),\bx_-\right>^2\Big)\geq
    \frac{(1-0.11)\kappa_2}{(1+p)\Big(4.3\kappa_1+\frac{1.11\kappa_2\sin\Delta}{(1+p)}t\Big)}\Big(\sin^2\Delta-\left<\bw_k(t),\bx_-\right>^2\Big)
    \\\geq&
    \frac{0.89\kappa_2}{4.3(1+p)\kappa_1+1.11\kappa_2 t\sin\Delta }\Big(\sin^2\Delta-\left<\bw_k(t),\bx_-\right>^2\Big).
\end{align*}
And we have $0<\left<\bw_k(\frac{3\kappa_1}{\kappa_2}),\bx_-\right><\sin\Delta$.

Now we consider the following auxiliary ODE:
\begin{equation}\label{equ of proof lemma: GF Phase I pnegative S+ S-: upper ODE of <w,x->}
\begin{cases}
    \frac{\mathrm{d}U(t)}{\mathrm{d}t}=\frac{0.89\kappa_2}{4.3(1+p)\kappa_1+1.11\kappa_2 t\sin\Delta}\Big(\sin^2\Delta-U^2(t)\Big)
    \\
    U(0)=0
\end{cases},
\end{equation}
and let $U(t)$ is the solution of \eqref{equ of proof lemma: GF Phase I positive S+ S-: upper ODE of <w,mu>}.
We know that $\left<\bw_k(t),\bx_-\right>$ is an upper solution of ODE \eqref{equ of proof lemma: GF Phase I pnegative S+ S-: upper ODE of <w,x->}. From the Comparison Principle of ODEs, we know this means:
\[
\left<\bw_k(t),\bx_-\right>>U(t),\text{ for any } t\leq T_{\rm I}.
\]
In order to estimate $\left<\bw_k(t),\bx_-\right>$,
we only need to study the solution of ODE \eqref{equ of proof lemma: GF Phase I pnegative S+ S-: upper ODE of <w,x->}. It is easy to verify that the solution of \eqref{equ of proof lemma: GF Phase I pnegative S+ S-: upper ODE of <w,x->} satisfies
\begin{align*}
    \log\left(\frac{\sin\Delta+U(t)}{\sin\Delta-U(t)}\right)-\log\left(\frac{\sin\Delta}{\sin\Delta}\right)=\frac{1.78\kappa_2\Delta}{1.11\kappa_2\sin\Delta}\log\left(\frac{4.3(1+p)\kappa_1+1.11\kappa_2t\sin\Delta}{4.3(1+p)\kappa_1+3.33\kappa_1\sin\Delta}\right)
\end{align*}
Then we have:
\begin{align*}
     &\log\left(\frac{\sin\Delta+U(T_{\rm I})}{\sin\Delta-U(T_{\rm I})}\right)\geq\frac{1.78}{1.11}\log\left(\frac{4.3(1+p)\kappa_1+11.1\sqrt{\kappa_1\kappa_2}\sin\Delta}{4.3(1+p)\kappa_1+3.33\kappa_1\sin\Delta}\right)
     \\>&1.6\log\left(1+\frac{\Big(11.1\sqrt{\frac{\kappa_2}{\kappa_1}}-3.33\Big)\frac{\sin\Delta}{1+p}}{4.3+3.33\frac{\sin\Delta}{1+p}}\right)
     >1.6\log\left(1+\frac{\Big(11.1\sqrt{\frac{\kappa_2}{\kappa_1}}-3.33\Big)\frac{\sin\Delta}{1+p}}{4.6}\right)
     \\>&1.6\log\left(1+\frac{10.7}{4.6}\sqrt{\frac{\kappa_2}{\kappa_1}}\frac{\sin\Delta}{1+p}\right),
\end{align*}
which means
\begin{align*}
U(T_{\rm I})>\left(1-\frac{2}{\left(1+\frac{10.7}{4.6}\sqrt{\frac{\kappa_2}{\kappa_1}}\frac{\sin\Delta}{1+p}\right)^{1.6}+1}\right)\sin\Delta.
\end{align*}
Hence, we have the estimate of $\left<\bw_k(t),\bx_+^{\perp}\right>$:
\begin{align*}
\left<\bw_k(T_{\rm I}),\bx_+^{\perp}\right>=\frac{1}{\sin\Delta}\left<\bw_k(T_{\rm I}),\bx_-\right>>\frac{1}{\sin\Delta}U(T_{\rm I})>1-\frac{2}{\left(1+2.32\sqrt{\frac{\kappa_2}{\kappa_1}}\frac{\sin\Delta}{1+p}\right)^{1.6}+1}.
\end{align*}

\underline{{\bf Step III.} Dynamics after arriving in the manifold $\cM_+^+\cap\cM_{-}^0$.}

In this step, we analyze the training dynamics after
$\bw_k(T_{\rm hit})\in\cM_+^+\cap\cM_{-}^0$, i.e. $\bb_k(T_{\rm hit})\in\cP_+^+\cap\cP_{-}^0$.

For any $\tilde{\bb}\in\cP_+^+\cap\cP_{-}^0$ and $0<\delta_0\ll1$, we know that $\cP_+^+\cap\cP_{-}^0$ separates its neighborhood $\mathcal{B}(\tilde{\bb },\delta_0)$ into two domains $\mathcal{G}_-=\{\bb \in\mathcal{B}(\tilde{\bb},\delta_0):\left<\bb ,\bx_-\right><0\}$ and $\mathcal{G}_+=\{\bb\in\mathcal{B}(\tilde{\bb},\delta_0):\left<\bb,\bx_-\right>>0\}$. Following Definition \ref{def: discontinuous system solution}, we calculate the limited vector field on $\tilde{\bb}$ from $\mathcal{G}_-$ and $\mathcal{G}_+$.

For any $\tilde{\bb}\in\cP_+^+\cap\cP_{-}^0$ and $0<\delta_0\ll1$, we know that $\cP_+^+\cap\cP_{-}^0$ separates its neighborhood $\mathcal{B}(\tilde{\bb },\delta_0)$ into two domains $\mathcal{G}_-=\{\bb \in\mathcal{B}(\tilde{\bb},\delta_0):\left<\bb ,\bx_-\right><0\}$ and $\mathcal{G}_+=\{\bb\in\mathcal{B}(\tilde{\bb},\delta_0):\left<\bb,\bx_-\right>>0\}$. Following Definition \ref{def: discontinuous system solution}, we calculate the limited vector field on $\tilde{\bb}$ from $\mathcal{G}_-$ and $\mathcal{G}_+$.

(i) The limited vector field $\bF^-$ on $\tilde{\bb}$ (from $\mathcal{G}_-$):
\begin{gather*}
\frac{\mathrm{d}\bb }{\mathrm{d} t}=\bF^-,\text{ where }
\bF^-=-\frac{\kappa_2}{\sqrt{m}}\frac{p}{1+p}e^{-f_+(t)}\bx_+.
\end{gather*}

(ii) The limited vector field $\bF^+$ on $\tilde{\bb}$ (from $\mathcal{G}_+$):
\begin{gather*}
\frac{\mathrm{d}\bb}{\mathrm{d} t}=\bF^+,\text{ where }
\bF^+=-\frac{\kappa_2}{\sqrt{m}}\bracket{\frac{p e^{-f_+(t)}}{1+p}\bx_+-\frac{e^{f_-(t)}}{1+p}\bx_-}.
\end{gather*}

(iii) Then we calculate the
projections of $\bF^-$ and $\bF^+$ onto $\bx_-$ (the normal to the surface $\cP_+^+\cap\cP_-^0$):
\begin{gather*}
F_N^{-}=\left<\bF^-,\bx_-\right>
=-\frac{\kappa_2 p e^{-f_+(t)}}{\sqrt{m}(1+p)}\cos\Delta,
\\
F_N^{+}=\left<\bF^+,\bx_-\right>=-\bracket{\frac{\kappa_2 p e^{-f_+(t)}}{\sqrt{m}(1+p)}\cos\Delta-\frac{\kappa_2 e^{f_-(t)}}{\sqrt{m}(1+p)}}.
\end{gather*}
From $T_{\rm I}<T_{\rm init}$ and Lemma \ref{lemma: GF Phase I norm estimate}, we know $|e^{-y_if_i(t)}-1|\leq0.11$, so 
$p e^{-f_+(t)}\cos\Delta-e^{f_-(t)}\geq 0.89p\cos\Delta-1.11{>}0$, which means $F_N^{+}>0$. 
And it is clear that $F_N^{-}>0$. 
Hence, the dynamics corresponds to Case (II) in Definition \ref{def: discontinuous system solution} ($F_N^{-}>0$ and $F_N^{+}>0$).

(iv) Hence, $\bb_k(t)$ passes immediately from one side of the surface $\cP_+^+\cap\cP_-^0$ to the other, i.e.  $\bb_k(t)$ enters into $\cP_+^+\cap\cP_-^+$ at time $T_{\rm hit}$.

Then the dynamics of $\bb_k$ in $\cP_+^+\cap\cP_-^-$ satisfies:
\[
\frac{\mathrm{d}\boldsymbol{b}_k(t)}{\mathrm{d}t}=-\frac{\kappa_2}{\sqrt{m}}\frac{p}{1+p}e^{-f_+(t)}\bx_+.
\]

(v) We define the following time, and our aim is to estimate $T_{\rm test,1}$:
\begin{align*}
    T_{\rm test,1}:&=\inf\Big\{t\in(T_{\rm hit},T_{\rm I}]:\left<\bw_k(t),\bx_+\right>\leq0\text{ or }\left<\bw_k(t),\bx_-\right>\geq0\Big\},
    \\
    T_{\rm test,2}:&=\inf\Big\{t\in(T_{\rm hit},T_{\rm I}]:\left<\bw_k(t),\bx_+\right>\leq0\Big\}
\end{align*} 
It is clear $ T_{\rm test,1}\leq  T_{\rm test,2}$. Moreover,
due to $\frac{\mathrm{d}\left<\boldsymbol{b}_k(t),\bx_-\right>}{\mathrm{d}t}=-\frac{\kappa_2}{\sqrt{m}}\frac{p}{1+p}e^{-f_+(t)}\cos\Delta<0$ and $\left<\boldsymbol{b}_k(T_{\rm hit}),\bx_-\right>=0$, we know $\left<\boldsymbol{b}_k(t),\bx_-\right><0$ holds for any $t\leq T_{\rm test,1}$. Hence, we have
\[
T_{\rm test,1}=T_{\rm test,2},
\quad
\left<\bw_k(T_{\rm test,1}),\bx_+\right>=0
,\quad
\left<\bw_k(T_{\rm test,1}),\bx_-\right><0.
\]
And we only need to estimate $T_{\rm test,2}$. For any $T_{\rm hit}<t\leq T_{\rm test,1}=T_{\rm test,2}$, we have
\[
    \frac{\mathrm{d}\left<\boldsymbol{b}_k(t),\bx_+\right>}{\mathrm{d}t}=-\frac{\kappa_2}{\sqrt{m}}\frac{p}{1+p}e^{-f_+(t)}\overset{\text{Lemma \ref{lemma: GF Phase I norm estimate}}}{\leq}-\frac{\kappa_2}{\sqrt{m}}\frac{p}{1+p}(1-0.11)=-\frac{0.89\kappa_2 p}{\sqrt{m}(1+p)}.
\]
Recalling Lemma \ref{lemma: GF Phase I norm estimate} and the estimate of $T_{\rm hit}$ in {\bf Step I}, we have:
\begin{align*}
  \left<\boldsymbol{b}_k(T_{\rm hit}),\bx_+\right>\leq\left\|\rho_k(t)\right\|
  \leq\frac{\kappa_1+1.1\kappa_2 T_{\rm hit}}{\sqrt{m}}
  \leq\left\|\rho_k(t)\right\|
  \leq\frac{\kappa_1+1.1\kappa_2\frac{3\kappa_1}{\kappa_2}}{\sqrt{m}}=\frac{4.3\kappa_1}{\sqrt{m}}.
\end{align*}
Then for any $T_{\rm hit}<t\leq T_{\rm test,2}$, we have:
\begin{align*}
    \left<\boldsymbol{b}_k(t),\bx_+\right>
    \leq\left<\boldsymbol{b}_k(T_{\rm hit}),\bx_+\right>-\int_{T_{\rm hit}}^t\frac{0.89\kappa_2 p}{\sqrt{m}(1+p)}\mathrm{d}s\leq\frac{4.3\kappa_1}{\sqrt{m}}-\frac{0.89\kappa_2 p(t-T_{\rm hit})}{\sqrt{m}(1+p)}.
\end{align*}
So we have the estimate 
\[
T_{\rm test,1}=T_{\rm test,2}\leq T_{\rm hit}+\frac{4.3\kappa_1(1+p)}{0.89\kappa_2p}\leq\Big(3+\frac{4.3\cdot6}{0.89\cdot5}\Big)\frac{\kappa_1}{\kappa_2}\leq \frac{9\kappa_1}{\kappa_2}<T_{\rm I}.
\]
Recalling $\bw_k(T_{\rm test,1})\in\cM_+^-\cap\cM_-^-$ and Lemma \ref{lemma: GF dead neurons keep dead}, the neuron $\boldsymbol{b}_k(t)$ keeps dead for any $t\geq T_{\rm I}$.

\underline{{\bf Step IV.} Which subspace does the neuron select?}

From {\bf Step II}, we know that the neuron $\bw_k(t)$ must arrives in $\cM_+^0\cap\cM_{-}^+$ or $\cM_-^0\cap\cM_{+}^+$ or $\cM_+^0\cap\cM_{-}^0$. 
In this step, we will analyze which subspace does the neuron select.

We only need to compare the following two times:
\begin{align*}
&T_{\rm hit,+}:=\inf\Big\{t\in(0,T_{\rm I}]:\left<\boldsymbol{b}_k(t),\bx_+\right>\leq0\Big\},
\\
&T_{\rm hit,-}:=\inf\Big\{t\in(0,T_{\rm I}]:\left<\boldsymbol{b}_k(t),\bx_-\right>\leq0\Big\}.
\end{align*}
From the definition of $T_{\rm hit}$, we know $T_{\rm hit,+}=T_{\rm hit}$ or $T_{\rm hit,-}=T_{\rm hit}$.

Recalling the proof in {\bf Step II}, we compare the following two dynamics for $t< T_{\rm hit}$:
\begin{align*}
    &\frac{\mathrm{d}\left<\boldsymbol{b}_k(t),\bx_+\right>}{\mathrm{d}t}=
    -\frac{\kappa_2}{\sqrt{m}}\Big(\frac{p}{1+p}e^{-f_+(t)}-\frac{1}{1+p}e^{f_-(t)}\cos\Delta\Big),\\
    &\frac{\mathrm{d}\left<\boldsymbol{b}_k(t),\bx_-\right>}{\mathrm{d}t}=
    -\frac{\kappa_2}{\sqrt{m}}\Big(\frac{p}{1+p}e^{-f_+(t)}\cos\Delta-\frac{1}{1+p}e^{f_-(t)}\Big).
\end{align*}

With the help of Lemma \ref{lemma: GF Phase I norm estimate} and the estimate of $T_{\rm hit}$, for any $t\leq T_{\rm hit}$,
\[
|e^{-y_if_i(t)}-1|\leq1.1\kappa_2(\kappa_1+1.1\kappa_2 T_{\rm hit})
\leq1.1\kappa_2(\kappa_1+3.3\kappa_1)=4.73\kappa_1\kappa_2.
\]
Hence, we have the estimate of the dynamics:
\begin{align*}
    &-\frac{\kappa_2\Big((1+4.73\kappa_1\kappa_2)p-(1-4.73\kappa_1\kappa_2)\cos\Delta\Big)}{\sqrt{m}(1+p)}
    \leq\frac{\mathrm{d}\left<\boldsymbol{b}_k(t),\bx_+\right>}{\mathrm{d}t}\leq
    -\frac{\kappa_2\Big((1-4.73\kappa_1\kappa_2)p-(1+4.73\kappa_1\kappa_2)\cos\Delta\Big)}{\sqrt{m}(1+p)},\\
    &-\frac{\kappa_2\Big((1+4.73\kappa_1\kappa_2)p\cos\Delta-(1-4.73\kappa_1\kappa_2)\Big)}{\sqrt{m}(1+p)}
    \leq\frac{\mathrm{d}\left<\boldsymbol{b}_k(t),\bx_-\right>}{\mathrm{d}t}\leq
    -\frac{\kappa_2\Big((1-4.73\kappa_1\kappa_2)p\cos\Delta-(1+4.73\kappa_1\kappa_2)\Big)}{\sqrt{m}(1+p)}.
\end{align*}

(i) If the initialization satisfies $\left<\boldsymbol{b}_k(0),\bx_-\right>>\frac{(1+4.73\kappa_1\kappa_2)p\cos\Delta-(1-4.73\kappa_1\kappa_2)}{(1-4.73\kappa_1\kappa_2)p-(1+4.73\kappa_1\kappa_2)\cos\Delta}\left<\boldsymbol{b}_k(0),\bx_+\right>$, we will prove that the neuron selects $\cM_+^0\cap\cM_-^+$ at $T_{\rm hit}$.

For any $t<T_{\rm hit}$, we have the estimate:
\begin{align*}
    \left<\boldsymbol{b}_k(t),\bx_+\right>\leq\left<\boldsymbol{b}_k(0),\bx_+\right>-\frac{\kappa_2\Big((1-4.73\kappa_1\kappa_2)p-(1+4.73\kappa_1\kappa_2)\cos\Delta\Big)}{\sqrt{m}(1+p)}t,
\end{align*}
\begin{align*}
    &\left<\boldsymbol{b}_k(t),\bx_-\right>\geq
    \left<\boldsymbol{b}_k(0),\bx_-\right>
    -\frac{\kappa_2\Big((1+4.73\kappa_1\kappa_2)p\cos\Delta-(1-4.73\kappa_1\kappa_2)\Big)}{\sqrt{m}(1+p)}t
    \\>&\frac{(1+4.73\kappa_1\kappa_2)p\cos\Delta-(1-4.73\kappa_1\kappa_2)}{(1-4.73\kappa_1\kappa_2)p-(1+4.73\kappa_1\kappa_2)\cos\Delta}\Bigg(\left<\boldsymbol{b}_k(0),\bx_+\right>-\frac{\kappa_2\Big((1-4.73\kappa_1\kappa_2)p-(1+4.73\kappa_1\kappa_2)\cos\Delta\Big)}{\sqrt{m}(1+p)}t\Bigg),
\end{align*}
Comparing these two inequalities, we have:
\[
T_{\rm hit,+}=T_{\rm hit}<T_{\rm hit,-},
\]
which means
\[
\left<\boldsymbol{b}_k(T_{\rm hit}),\bx_+\right>=0,\quad
\left<\boldsymbol{b}_k(T_{\rm hit}),\bx_-\right>>0.
\]
So the neuron $\bw_k(T_{\rm hit})\in\cM_+^0\cap\cM_-^+$.

(ii) If the initialization satisfies $\left<\boldsymbol{b}_k(0),\bx_+\right>>\frac{(1+4.73\kappa_1\kappa_2)p-(1-4.73\kappa_1\kappa_2)\cos\Delta}{(1-4.73\kappa_1\kappa_2)p\cos\Delta-(1+4.73\kappa_1\kappa_2)}\left<\boldsymbol{b}_k(0),\bx_-\right>$, we will prove that the neuron selects $\cM_+^+\cap\cM_-^-$ at $T_{\rm hit}$.

For any $t<T_{\rm hit}$, we have the estimate:
\begin{align*}
    &\left<\boldsymbol{b}_k(t),\bx_-\right>\leq
    \left<\boldsymbol{b}_k(0),\bx_-\right>
    -\frac{\kappa_2\Big((1-4.73\kappa_1\kappa_2)p\cos\Delta-(1+4.73\kappa_1\kappa_2)\Big)}{\sqrt{m}(1+p)}t,
\end{align*}
\begin{align*}
    &\left<\boldsymbol{b}_k(t),\bx_+\right>\geq\left<\boldsymbol{b}_k(0),\bx_+\right>-\frac{\kappa_2\Big((1+4.73\kappa_1\kappa_2)p-(1-4.73\kappa_1\kappa_2)\cos\Delta\Big)}{\sqrt{m}(1+p)}t
    \\>&\frac{(1+4.73\kappa_1\kappa_2)p-(1-4.73\kappa_1\kappa_2)\cos\Delta}{(1-4.73\kappa_1\kappa_2)p\cos\Delta-(1+4.73\kappa_1\kappa_2)}\Bigg(\left<\boldsymbol{b}_k(0),\bx_-\right>-\frac{\kappa_2\Big((1-4.73\kappa_1\kappa_2)p\cos\Delta-(1+4.73\kappa_1\kappa_2)\Big)}{\sqrt{m}(1+p)}t\Bigg),
\end{align*}
Comparing these two inequalities, we have:
\[
T_{\rm hit,-}=T_{\rm hit}<T_{\rm hit,+},
\]
which means
\[
\left<\boldsymbol{b}_k(T_{\rm hit}),\bx_-\right>=0,\quad
\left<\boldsymbol{b}_k(T_{\rm hit}),\bx_+\right>>0.
\]
So the neuron $\bw_k(T_{\rm hit})\in\cM_+^+\cap\cM_-^0$.

\end{proof}

\begin{lemma}[Negative, $\cM_+^-\cap\cM_-^+$]\label{lemma: GF Phase I negative S-}\ \\
For negative neuron $k\in\{k\in[m]-[m/2]:\bw_k(0)\in\cM_+^-\cap\cM_-^+\}$, in Phase I $(t\leq T_{\rm I})$, we have:
\begin{align*}
    &\text{(S1). }\bw_k(t)\in\cM_+^{0}\cap\cM_-^{+}
    \text{ for any }t\leq \mathcal{O}\Big(\frac{\kappa_1}{\kappa_2}p\Delta\Big), \\&\text{(S2). It has a small norm}: \rho_k(T_{\rm I})=\mathcal{O}\Big(\frac{\sqrt{\kappa_1\kappa_2}}{\sqrt{m}}\big(\sqrt{\frac{\kappa_1}{\kappa_2}}+\frac{\Delta}{p}\big)\Big),
    \\&\text{(S3). It is aligned with $\bx_+^{\perp}$}:
    \left<\bw_k(T_{\rm I}),\bx_+^{\perp}\right>\geq1-\mathcal{O}\Big((\sqrt{\frac{\kappa_1}{\kappa_2}}\frac{p}{\Delta})^{1.6}\Big).
\end{align*}

\end{lemma}

\begin{proof}[Proof of Lemma \ref{lemma: GF Phase I negative S-}]\ \\
We do the following analysis for any $k\in\{k\in[m]-[m/2]:\bw_k(0)\in\cM_+^-\cap\cM_-^+\}$, i.e. ${\rm s}_k=-1$, $\left<\bw_k(0),\bx_+\right>\leq0$, and $\left<\bw_k(0),\bx_-\right>>0$. 

\underline{{\bf Step I.} The neuron must arrives in $\cM_+^0\cap\cM_-^+$ in $\mathcal{O}\Big(\frac{\kappa_1 p\Delta}{\kappa_2}\Big)$ time.}

The case $\left<\bw_k(0),\bx_+\right>=0$ is trivial.
Then we only need to consider the case $\left<\bw_k(0),\bx_+\right><0$.

First, we define the hitting time 
\[
T_{\rm hit}:=\inf\Big\{t\in(0, T_{\rm I}]:\bw_k(t)\notin\cM_+^-\cap\cM_-^+\Big\},
\] 
and we aim to estimate $T_{\rm hit}$ and prove $\bw_k(T_{\rm hit})\in\cM_+^0\cap\cM_-^+$. 

We focus on the dynamics of $\left<\boldsymbol{b}_k(t),\bx_+\right>$ and $\left<\boldsymbol{b}_k(t),\bx_-\right>$. 

From the definition of $T_{\rm hit}$ and \eqref{equ: dynamics}, the dynamics of the neuron is:
\[
\frac{\mathrm{d}\boldsymbol{b}_k(t)}{\mathrm{d}t}=\frac{\kappa_2}{\sqrt{m}}\frac{1}{1+p}e^{f_-(t)}\bx_-,\ t\leq T_{\rm hit}.
\]
Then we have
\begin{align*}
    \frac{\mathrm{d}\left<\boldsymbol{b}_k(t),\bx_+\right>}{\mathrm{d}t}
    =&\left<\frac{\kappa_2}{\sqrt{m}}\frac{1}{1+p}e^{f_-(t)}\bx_-,\bx_+\right>=\frac{\kappa_2\cos\Delta}{\sqrt{m}(1+p)}e^{f_-(t)},
    \\
    \frac{\mathrm{d}\left<\boldsymbol{b}_k(t),\bx_-\right>}{\mathrm{d}t}
    =&\left<\frac{\kappa_2}{\sqrt{m}}\frac{1}{1+p}e^{f_-(t)}\bx_-,\bx_-\right>=\frac{\kappa_2}{\sqrt{m}(1+p)}e^{f_-(t)}.
\end{align*}
It is clear $ \frac{\mathrm{d}\left<\boldsymbol{b}_k(t),\bx_-\right>}{\mathrm{d}t}>0$, so $\left<\boldsymbol{b}_k(t),\bx_-\right>>\left<\boldsymbol{b}_k(0),\bx_-\right>>0$ for any $t\leq T_{\rm hit}$. If we denote
\[
T_{\rm hit,+}:=\inf\Big\{t\in(0,T_{\rm I}]:\left<\bw_k(t),\bx_+\right>\leq0\Big\},
\]
then it holds:
\[
T_{\rm hit}=T_{\rm hit,+}.
\]
So we only need to estimate $T_{\rm hit,+}$. Due to $T_{\rm hit}\leq T_{\rm I}\leq T_{\rm init}$ and Lemma \ref{lemma: GF Phase I norm estimate}, for any $t\leq T_{\rm hit}$, we have $\left|e^{-y_i f_i(t)}-1\right|\leq0.11$.
Then for any $t\leq T_{\rm hit}$, we have:
\[
\frac{\mathrm{d}\left<\boldsymbol{b}_k(t),\bx_+\right>}{\mathrm{d}t}\geq\frac{0.89\kappa_2\cos\Delta}{\sqrt{m}(1+p)}.
\]
Recalling $\left<\bw_k(0),\bx_+\right><0$ and $\left<\bw_k(0),\bx_-\right>>0$, with the help of Lemma \ref{lemma: theta in span x1, x2}, we have $\left<\bw_k(0),\bx_+\right>>-\sin\Delta$ and $\left<\bw_k(0),\bx_-\right><\sin\Delta$ . Combining the two estimate, we have:
\begin{align*}
&\left<\boldsymbol{b}_k(t),\bx_+\right>\geq\left<\boldsymbol{b}_k(0),\bx_+\right>+\int_{0}^{t}\frac{0.89\kappa_2\cos\Delta}{\sqrt{m}(1+p)}\mathrm{d}t
\\>&-\rho_k(0)\sin\Delta+\frac{0.89\kappa_2\cos\Delta}{\sqrt{m}(1+p)}t=-\frac{\kappa_1\sin\Delta}{\sqrt{m}}+\frac{0.89\kappa_2\cos\Delta}{\sqrt{m}(1+p)}t.
\end{align*}
Hence, 
\[
T_{\rm hit}=T_{\rm hit,+}\leq\frac{(1+p)\tan\Delta}{0.89}\frac{\kappa_1}{\kappa_2}\leq2p\Delta\frac{\kappa_1}{\kappa_2}<T_{\rm I}=10\sqrt{\frac{\kappa_1}{\kappa_2}}.
\]
Moreover, we can estimate of $\rho_k(T_{\rm hit})$.

Since $\left<\bw_k(t),\bx_+\right><0$ and $\left<\bw_k(t),\bx_-\right>>0$ hold for any $t\leq T_{\rm hit}$, with the help of Lemma \ref{lemma: theta in span x1, x2}, we have $\left<\bw_k(t),\bx_-\right><\sin\Delta$.
Combining \eqref{equ: dynamics rewrite}, for any $t\leq T_{\rm hit}$, we have
\begin{align*}
    &\rho_k(t)\leq\rho_k(0)+\int_0^{t}\frac{\kappa_2}{\sqrt{m}(1+p)}e^{f_-(t)}\left<\bw_k(t),\bx_-\right>\mathrm{d}t
    \\\leq&\frac{\kappa_1}{\sqrt{m}}+\int_0^{t}\frac{1.11\kappa_2}{\sqrt{m}(1+p)}\sin\Delta\mathrm{d}t
    \leq\frac{\kappa_1}{\sqrt{m}}+\frac{1.11\kappa_2\sin\Delta}{\sqrt{m}(1+p)} t.
\end{align*}

\underline{{\bf Step II.} Dynamics after arriving in the manifold $\cM_+^0\cap\cM_{-}^+$.}

Proceeding as in the proof of {\bf Step II} in the Proof of Theorem \ref{lemma: GF Phase I negative S+ S-}, we have:

$\bw_k(t)$ can not leave $\cM_+^0\cap\cM_-^+$ for $T_{\rm hit}\leq t\leq T_{\rm I}$. Moreover, the dynamics of $\bw_k$ on $\cM_+^0\cap\cM_-^+$ satisfies:
\begin{align*}
    \frac{\mathrm{d}\bw_k(t)}{\mathrm{d}t}=
    &\frac{\kappa_2 e^{f_-(t)}}{\rho_k(t)\sqrt{m}(1+p)}\Big(\bx_--\left<\bw_k,\bx_-\right>\bw_k-\bx_+\cos\Delta\Big),
    \\
    \frac{\mathrm{d}\rho_k(t)}{\mathrm{d}t}=
    &\frac{\kappa_2 e^{f_-(t)}}{\sqrt{m}(1+p)}\left<\bw_k(t),\boldsymbol{x_-}\right>,
    \\
    \frac{\mathrm{d}\boldsymbol{b}_k(t)}{\mathrm{d}t}=&
    \frac{\kappa_2 e^{f_-(t)}}{\sqrt{m}(1+p)}\Big(\bx_--\bx_+\cos\Delta\Big).
\end{align*}

Recalling the estimate of $T_{\rm hit}$ in {\bf Step I}, we have
\[
\rho_k(T_{\rm hit})\leq\frac{\kappa_1}{\sqrt{m}}+\frac{1.11\kappa_2\sin\Delta}{\sqrt{m}(1+p)} T_{\rm hit}.
\]
As the proof of {\bf Step II} in the Proof of Theorem \ref{lemma: GF Phase I negative S+ S-}, for any $T_{\rm hit}<t\leq T_{\rm I}$, we have
\begin{align*}
    &\rho_k(t)=\rho_k(T_{\rm hit})+\int_{T_{\rm hit}}^{t}\frac{\mathrm{d}\rho_k(s)}{\mathrm{d}s}\mathrm{d}s\leq
    \frac{\kappa_1}{\sqrt{m}}+\frac{1.11\kappa_2\sin\Delta}{\sqrt{m}(1+p)} T_{\rm hit}+\frac{\kappa_2 e^{f_-(t)}\sin\Delta}{\sqrt{m}(1+p)}(t-T_{\rm hit})
    \\\leq&\frac{\kappa_1}{\sqrt{m}}+\frac{1.11\kappa_2\sin\Delta}{\sqrt{m}(1+p)} T_{\rm hit}+\frac{1.11\kappa_2\sin\Delta}{\sqrt{m}(1+p)}(t-T_{\rm hit})
    =\frac{\kappa_1}{\sqrt{m}}+\frac{1.11\kappa_2\sin\Delta}{\sqrt{m}(1+p)}t.
\end{align*}
Combining the estimate in {\bf Step I}, for any $0<t\leq T_{\rm I}$, we have:
\[
\rho_k(t)\leq\frac{\kappa_1}{\sqrt{m}}+\frac{1.11\kappa_2\sin\Delta}{\sqrt{m}(1+p)}t.
\]
Specifically, we have:
\begin{align*}
    \rho_k(T_{\rm I})\leq
    \frac{\kappa_1}{\sqrt{m}}+\frac{11.1\sqrt{\kappa_1\kappa_2}\sin\Delta}{\sqrt{m}(1+p)}\leq\frac{\sqrt{\kappa_1\kappa_2}}{\sqrt{m}}\Big(\sqrt{\frac{\kappa_1}{\kappa_2}}+11.1\frac{\Delta}{1+p}\Big).
\end{align*}

Similar to the proof of {\bf Step II} in the Proof of Theorem \ref{lemma: GF Phase I negative S+ S-}, we have the estimate of the dynamics of $\left<\bw_k(t),\bx_-\right>$:
\begin{align*}
    &\frac{\mathrm{d}\left<\bw_k(t),\bx_-\right>}{\mathrm{d}t}\geq
    \frac{0.89\kappa_2}{(1+p)\kappa_1+1.11\kappa_2 t\sin\Delta }\Big(\sin^2\Delta-\left<\bw_k(t),\bx_-\right>^2\Big),\quad 0<t\leq T_{\rm I},
    \\&0<\left<\bw_k(0),\bx_-\right><\sin\Delta.
\end{align*}
In the same way, we can derive
\begin{align*}
    \left<\bw_k(T_{\rm I}),\bx_-\right>>\left(1-\frac{2}{\left(1+11.1\sqrt{\frac{\kappa_2}{\kappa_1}}\frac{\sin\Delta}{1+p}\right)^{1.6}+1}\right)\sin\Delta
\end{align*}

Hence, we have the estimate of $\left<\bw_k(t),\bx_+^{\perp}\right>$:
\begin{align*}
\left<\bw_k(T_{\rm I}),\bx_+^{\perp}\right>=\frac{1}{\sin\Delta}\left<\bw_k(T_{\rm I}),\bx_-\right>>1-\frac{2}{\left(1+11.1\sqrt{\frac{\kappa_2}{\kappa_1}}\frac{\sin\Delta}{1+p}\right)^{1.6}+1}.
\end{align*}

\end{proof}

\begin{lemma}[Negative, $\cM_+^+\cap\cM_-^-$]\label{lemma: GF Phase I negative S+}\ \\
For negative neuron $k\in\{k\in[m]-[m/2]:\bw_k(0)\in\cM_+^+\cap\cM_-^-\}$, it keeps dead: 
\[\bw_k(t)\in\cM_+^-\cap\cM_-^-, \text{ for any } t\geq T_{\rm I}> \mathcal{O}\Big(\frac{\kappa_1\Delta}{\kappa_2}\Big).
\]
\end{lemma}

\begin{proof}[Proof of Lemma \ref{lemma: GF Phase I negative S+}]\ \\
We do the following analysis for any $k\in\{k\in[m]-[m/2]:\bw_k(0)\in\cM_+^+\cap\cM_-^-\}$, i.e. ${\rm s}_k=-1$, $\left<\bw_k(0),\bx_+\right>>0$, and $\left<\bw_k(0),\bx_-\right>\leq0$. 

First, we define the hitting time 
\[
T_{\rm hit}:=\inf\Big\{t\in(0, T_{\rm I}]:\bw_k(t)\notin\cM_+^+\cap\cM_-^-\Big\},
\] 
and we aim to estimate $T_{\rm hit}$ and prove $\bw_k(T_{\rm hit})\in\cM_+^-\cap\cM_-^-$. 

From the definition of $T_{\rm hit}$ and \eqref{equ: dynamics}, the dynamics of the neuron is:
\[
\frac{\mathrm{d}\boldsymbol{b}_k(t)}{\mathrm{d}t}=-\frac{\kappa_2}{\sqrt{m}}\frac{p}{1+p}e^{-f_+(t)}\bx_+,\ t\leq T_{\rm hit}.
\]
Then we have
\begin{align*}
    \frac{\mathrm{d}\left<\boldsymbol{b}_k(t),\bx_+\right>}{\mathrm{d}t}
    =&\left<-\frac{\kappa_2}{\sqrt{m}}\frac{p}{1+p}e^{-f_+(t)}\bx_+,\bx_+\right>=-\frac{\kappa_2}{\sqrt{m}}\frac{p}{1+p}e^{-f_+(t)},
    \\
    \frac{\mathrm{d}\left<\boldsymbol{b}_k(t),\bx_-\right>}{\mathrm{d}t}
    =&\left<-\frac{\kappa_2}{\sqrt{m}}\frac{p}{1+p}e^{-f_+(t)}\bx_+,\bx_-\right>=-\frac{\kappa_2}{\sqrt{m}}\frac{p}{1+p}e^{-f_+(t)}\cos\Delta.
\end{align*}

It is clear $ \frac{\mathrm{d}\left<\boldsymbol{b}_k(t),\bx_-\right>}{\mathrm{d}t}<0$, so $\left<\boldsymbol{b}_k(t),\bx_-\right><\left<\boldsymbol{b}_k(0),\bx_-\right>\leq0$ for any $t\leq T_{\rm hit}$. If we denote
\[
T_{\rm hit,+}:=\inf\Big\{t\in(0,T_{\rm I}]:\left<\bw_k(t),\bx_+\right>\leq0\Big\},
\]
then it holds:
\[
T_{\rm hit}=T_{\rm hit,+}.
\]
So we only need to estimate $T_{\rm hit,+}$. Due to $T_{\rm hit}\leq T_{\rm I}\leq T_{\rm init}$ and Lemma \ref{lemma: GF Phase I norm estimate}, for any $t\leq T_{\rm hit}$, we have $\left|e^{-y_i f_i(t)}-1\right|\leq0.11$.
Then for any $t\leq T_{\rm hit}$, we have:
\[
\frac{\mathrm{d}\left<\boldsymbol{b}_k(t),\bx_+\right>}{\mathrm{d}t}\leq-\frac{0.89\kappa_2 p}{\sqrt{m}(1+p)}.
\]
Recalling $\left<\bw_k(0),\bx_+\right>>0$ and $\left<\bw_k(0),\bx_-\right>\leq0$, with the help of Lemma \ref{lemma: theta in span x1, x2}, we have $\left<\bw_k(0),\bx_+\right>\leq\sin\Delta$. Combining the two estimate, we have:
\begin{align*}
&\left<\boldsymbol{b}_k(t),\bx_+\right>\leq\left<\boldsymbol{b}_k(0),\bx_+\right>-\int_{0}^{t}\frac{0.89\kappa_2 p}{\sqrt{m}(1+p)}\mathrm{d}t
\\\leq&\rho_k(0)\sin\Delta-\frac{0.89\kappa_2 p}{\sqrt{m}(1+p)}t=\frac{\kappa_1\sin\Delta}{\sqrt{m}}-\frac{0.89\kappa_2 p}{\sqrt{m}(1+p)}t.
\end{align*}
Hence, 
\[
T_{\rm hit}=T_{\rm hit,+}\leq\frac{(1+p)\sin\Delta}{0.89 p}\frac{\kappa_1}{\kappa_2}<T_{\rm I}=10\sqrt{\frac{\kappa_1}{\kappa_2}}.
\]
Moreover, the analysis gives us $\bw_k(T_{\rm hit})\in\cM_+^-\cap\cM_-^-$. By Lemma \ref{lemma: GF dead neurons keep dead}, we obtain:
\[
\bw_k(t)\in\cM_+^-\cap\cM_-^-,\text{ for any } t\geq T_{\rm hit}.
\]
\end{proof}

\begin{lemma}[Negative, $\cM_+^-\cap\cM_-^-$]\label{lemma: GF Phase I negative dead}\ \\
For negative neuron $k\in\{k\in[m]-[m/2]:\bw_k(0)\in\cM_+^-\cap\cM_-^-\}$, it keeps dead: $\bw_k(t)\in\cM_+^-\cap\cM_-^-$ for any $t\geq0$.
\end{lemma}

\begin{proof}[Proof of Lemma \ref{lemma: GF Phase I negative dead}] Due to Lemma \ref{lemma: GF dead neurons keep dead}, this lemma is trivial.
\end{proof}

\label{appendix: subsection: Phase I: negative}
\subsection{Initialization Estimation and Proof of Theorem \ref{thm: GF Phase I}}

To get the number of neurons in the eight classes in the subsection above, we also need to estimate the initial positions of these neurons under the random initialization.

\begin{lemma}[Initialization Estimation]\label{lemma: initial position estimate}\ \\
If $m=\Omega\big(\log(1/\delta)\big)$, then with probability at least $1-\delta$, we have:
\begin{align*}
   \left|\#\Big\{k\in[m/2]:\left<\bw_k(0),\bx_+\right>>0\Big\}-\frac{m}{4}\right|\leq&0.04m,
    \\
    \#\Big\{k\in[m]-[m/2]:\left<\bw_k(0),\bx_-\right>>0,
    \left<\bw_k(0),\bx_-\right>> A\left<\bw_k(0),\bx_+\right>\Big\}\geq&0.075m,
    \\
    \#\Big\{k\in[m]-[m/2]:\left<\bw_k(0),\bx_-\right>>0,
    \left<\bw_k(0),\bx_+\right>\leq B\left<\bw_k(0),\bx_-\right>\Big\}
    \leq&0.205m.
\end{align*}
where $A=\frac{(1+4.73\kappa_1\kappa_2)p\cos\Delta-(1-4.73\kappa_1\kappa_2)}{(1-4.73\kappa_1\kappa_2)p-(1+4.73\kappa_1\kappa_2)\cos\Delta}$ and $B=\frac{(1+4.73\kappa_1\kappa_2)p-(1-4.73\kappa_1\kappa_2)\cos\Delta}{(1-4.73\kappa_1\kappa_2)p\cos\Delta-(1+4.73\kappa_1\kappa_2)}$ (mentioned in Lemma~\ref{lemma: GF Phase I negative S+ S-}).

\end{lemma}

\begin{proof}[Proof of Lemma \ref{lemma: initial position estimate}]\ \\
(i) 
By Hoeffding's Inequality (Lemma \ref{lemma: hoeffding}), for any $\epsilon>0$ we have:
\begin{align*}
    &\mathbb{P}\Bigg(\left|\#\Big\{k\in[m/2]:\left<\bw_k(0),\bx_+\right>>0\Big\}-\frac{m}{4}\right|\geq\frac{m\epsilon}{2}\Bigg)
    =\mathbb{P}\Bigg(\left|\frac{2}{m}\sum\limits_{k\in[m/2]}\mathbb{I}\Big\{\left<\bw_k(0),\bx_+\right>>0\Big\}-\frac{1}{2}\right|\geq\epsilon\Bigg)
    \\=&\mathbb{P}\Bigg(\left|\frac{2}{m}\sum\limits_{k\in[m/2]}\mathbb{I}\big\{\left<\bw_k(0),\bx_+\right>>0\big\}-\mathbb{E}\Big[\mathbb{I}\big\{\left<\bw_1(0),\bx_+\right>>0\big\}\Big]\right|\geq\epsilon\Bigg)
    \\\leq&2\exp\Big(-\frac{2(\frac{m}{2})^2\epsilon^2}{\frac{m}{2}}\Big)=2\exp(-m\epsilon^2).
\end{align*}

(ii) From $\bw_k(0)\sim\mathbb{U}(\mathbb{S}^{d-1})$, without loss of generality, we can let $\bx_-=\boldsymbol{e}_1$ and $\bx_+=\boldsymbol{e}_1\cos\Delta+\boldsymbol{e}_2\sin\Delta$.

So we have:
\begin{align*}
    &\Big\{k\in[m]-[m/2]:\left<\bw_k(0),\bx_-\right>>0,\left<\bw_k(0),\bx_-\right>>A\left<\bw_k(0),\bx_+\right>\Big\}
    \\=&
    \Big\{k\in[m]-[m/2]:{w}_{k,1}(0)>0,{w}_{k,1}(0)>A\Big({w}_{k,1}(0)\cos\Delta+{w}_{k,2}(0)\sin\Delta\Big)\Big\}
    \\=&
    \Big\{k\in[m]-[m/2]:{w}_{k,1}(0)>0,(1-A\cos\Delta){w}_{k,1}(0)>A{w}_{k,2}(0)\sin\Delta\Big\}.
\end{align*}

From \eqref{equ: parameter selection GF Phase I},
we have $A>0$ and 
\begin{align*}
    &A=1+\frac{\Big((1+4.73\kappa_1\kappa_2)\cos\Delta-(1-4.73\kappa_1\kappa_2)\Big)(p+1)}{(1-4.73\kappa_1\kappa_2)p-(1+4.73\kappa_1\kappa_2)\cos\Delta}
    \\\leq&1+\frac{4.73\kappa_1\kappa_2(1+\cos\Delta)}{1-4.73\kappa_1\kappa_2}\frac{p+1}{p-\frac{10}{9}}
    \leq1+\frac{9.46\kappa_1\kappa_2}{1-\frac{1}{19}}\frac{90}{71}\leq1+12.66\kappa_1\kappa_2,
\end{align*}
\begin{align*}
    &A=1+\frac{\Big((1+4.73\kappa_1\kappa_2)\cos\Delta-(1-4.73\kappa_1\kappa_2)\Big)(p+1)}{(1-4.73\kappa_1\kappa_2)p-(1+4.73\kappa_1\kappa_2)\cos\Delta}
    \\\geq&1-\frac{4.73\kappa_1\kappa_2(1+\cos\Delta)+(1-\cos\Delta)}{1-4.73\kappa_1\kappa_2}\frac{p+1}{p-\frac{10}{9}}
    \geq1-\frac{\frac{2}{19}+\frac{\Delta^2}{2}}{1-\frac{1}{19}}\frac{90}{71}
    \\\geq&1-\frac{\frac{2}{19}+\frac{1}{19}}{\frac{18}{19}}\frac{90}{71}\geq0.78,
\end{align*}
\begin{align*}
    \frac{1-A\cos\Delta}{A\sin\Delta}\geq\frac{1-A}{A\sin\Delta}\geq-\frac{12.66\kappa_1\kappa_2}{A\sin\Delta}\geq-\frac{12.66\kappa_1\kappa_2}{0.78\sin\Delta}\geq-\frac{12.66\kappa_1\kappa_2}{0.78\frac{2}{\pi}\Delta}\geq-\frac{25.5\kappa_1\kappa_2}{\Delta}\geq-\frac{1}{100}.
\end{align*}
For simplicity, we denote the event
\[
A_k:=\Big\{{w}_{k,1}(0)>0,-\frac{1}{100}{w}_{k,1}(0)>{w}_{k,2}(0)\Big\},\ k\in[m]-[m/2].
\]
Then we have the estimate:
\begin{align*}
    &\#\Big\{k\in[m]-[m/2]:{w}_{k,1}(0)>0,(1-A\cos\Delta){w}_{k,1}(0)>A{w}_{k,2}(0)\sin\Delta\Big\}
    \\\geq&\#\Big\{k\in[m]-[m/2]:{w}_{k,1}(0)>0,-\frac{1}{100}{w}_{k,1}(0)>{w}_{k,2}(0)\Big\}
    =\sum_{k\in[m]-[m/2]}\mathbb{I}\{A_k\}.
\end{align*}

We first estimate the lower bound for $\mathbb{E}[\mathbb{I}\{A_m\}]$:
\begin{align*}
    &\mathbb{E}\Big[\mathbb{I}\{A_m\}\Big]=\mathbb{P}(A_m)
    =\mathbb{P}\Big({w}_{m,1}(0)>0,-\frac{1}{100}{w}_{m,1}(0)>{w}_{k,2}(0)\Big)
    \\\overset{\boldsymbol{g}\sim\mathcal{N}(\boldsymbol{0},\mathbf{I}_d)}{=}&\mathbb{P}\Bigg(\frac{{g}_{1}}{\left\|\boldsymbol{g}\right\|}>0,-\frac{1}{100}\frac{{g}_{1}}{\left\|\boldsymbol{g}\right\|}>\frac{{g}_{2}}{\left\|\boldsymbol{g}\right\|}\Bigg)
    =\mathbb{P}\Big({g}_{1}>0,{g}_{1}<-100{g}_{2}\Big)=\mathbb{P}\Big({g}_{1}>0,{g}_{1}<100{g}_{2}\Big)
    \\=&\mathbb{P}\Big(100{g}_{2}>g_1>0\Big)\geq\sup\limits_{t>0}\mathbb{P}\Big({g}_{2}>t, 100t>g_1>0\Big)\overset{g\sim\mathcal{N}(0,1)}{=}\sup\limits_{t>0}\mathbb{P}\Big({g}>t\Big)\mathbb{P}\Big(100t>g>0\Big)
    \\\geq&\mathbb{P}\Big({g}>\frac{1}{10}\Big)\mathbb{P}\Big(10>g>0\Big)\geq0.23.
\end{align*}

Secondly, by Hoeffding's inequality (Lemma \ref{lemma: hoeffding}), for any $\epsilon>0$, we have
\begin{align*}
    &\mathbb{P}\Bigg(\sum_{k\in[m]-[m/2]}\mathbb{I}\{A_k\}-0.115 m\leq-
    \frac{m}{2}\epsilon\Bigg)\leq
    \mathbb{P}\Bigg(\sum_{k\in[m]-[m/2]}\mathbb{I}\{A_k\}-\frac{m}{2}\mathbb{E}\Big[\mathbb{I}\{A_m\}\Big]\leq-
    \frac{m}{2}\epsilon\Bigg)
    \\=&\mathbb{P}\Bigg(\frac{2}{m}\sum_{k\in[m]-[m/2]}\mathbb{I}\{A_k\}-\mathbb{E}\Big[\mathbb{I}\{A_m\}\Big]\leq-\epsilon\Bigg)\leq \exp\Big(-\frac{2(\frac{m}{2})^2\epsilon^2}{\frac{m}{2}}\Big)=\exp(-m\epsilon^2)
\end{align*}

(iii) This proof is similar to (ii).
From $\bw_k(0)\sim\mathbb{U}(\mathbb{S}^{d-1})$, without loss of generality, we can let $\bx_-=\boldsymbol{e}_1$ and $\bx_+=\boldsymbol{e}_1\cos\Delta+\boldsymbol{e}_2\sin\Delta$.

so we have:
\begin{align*}
    &\Big\{k\in[m]-[m/2]:\left<\bw_k(0),\bx_-\right>>0,\left<\bw_k(0),\bx_+\right>\leq B\left<\bw_k(0),\bx_-\right>\Big\}
    \\=&
    \Big\{k\in[m]-[m/2]:{w}_{k,1}(0)>0,{w}_{k,1}(0)\cos\Delta+{w}_{k,2}(0)\sin\Delta\leq B{w}_{k,1}(0)\Big\}
    \\=&
    \Big\{k\in[m]-[m/2]:{w}_{k,1}(0)>0,(B-\cos\Delta){w}_{k,1}(0)>{w}_{k,2}(0)\sin\Delta\Big\}.
\end{align*}

From \eqref{equ: parameter selection GF Phase I},
we have $B>0$ and 
\begin{align*}
    &B-\cos\Delta=\frac{(1+4.73\kappa_1\kappa_2)p-(1-4.73\kappa_1\kappa_2)\cos\Delta}{(1-4.73\kappa_1\kappa_2)p\cos\Delta-(1+4.73\kappa_1\kappa_2)}-\cos\Delta
    \\=&\frac{(1+4.73\kappa_1\kappa_2)(p+\cos\Delta)-(1-4.73\kappa_1\kappa_2)(1+p\cos\Delta)\cos\Delta}{(1-4.73\kappa_1\kappa_2)p\cos\Delta-(1+4.73\kappa_1\kappa_2)}
    \\=&
    \frac{p\sin^2\Delta+4.73\kappa_1\kappa_2(p+2\cos\Delta+p\cos^2\Delta)}{(1-4.73\kappa_1\kappa_2)p\cos\Delta-(1+4.73\kappa_1\kappa_2)}
    \leq\frac{\sin^2\Delta+9.46\kappa_1\kappa_2}{1-4.73\kappa_1\kappa_2}\frac{p+1}{p\cos\Delta-\frac{10}{9}}
    \\\leq&\frac{\sin^2\Delta+9.46\frac{\Delta}{2550}}{1-\frac{1}{19}}\frac{p+1}{\frac{9}{10}p-\frac{10}{9}}\leq\frac{\sin^2\Delta+9.46\frac{\pi\sin\Delta}{5100}}{1-\frac{1}{19}}\frac{p+1}{\frac{9}{10}p-\frac{10}{9}}
    \\\leq&\frac{0.315+\frac{9.46\pi}{5100}}{\frac{18}{19}}\frac{10}{\frac{81}{10}-\frac{10}{9}}\sin\Delta\leq\frac{\sin\Delta}{2},
\end{align*}
For simplicity, we denote the event
\[
B_k:=\Big\{{w}_{k,1}(0)>0,\frac{1}{2}{w}_{k,1}(0)>{w}_{k,2}(0)\Big\},\ k\in[m]-[m/2].
\]
Then we have the estimate:
\begin{align*}
    &\#\Big\{k\in[m]-[m/2]:{w}_{k,1}(0)>0,(B-\cos\Delta){w}_{k,1}(0)>{w}_{k,2}(0)\sin\Delta\Big\}
    \\\leq&\#\Big\{k\in[m]-[m/2]:{w}_{k,1}(0)>0,\frac{1}{2}{w}_{k,1}(0)>{w}_{k,2}(0)\Big\}
    =\sum_{k\in[m]-[m/2]}\mathbb{I}\{B_k\}.
\end{align*}

We first estimate the lower bound for $\mathbb{E}[\mathbb{I}\{B_m\}]$:
\begin{align*}
    &\mathbb{E}\Big[\mathbb{I}\{B_m\}\Big]=\mathbb{P}(B_m)
    =\mathbb{P}\Big({w}_{m,1}(0)>0,\frac{1}{2}{w}_{m,1}(0)>{w}_{k,2}(0)\Big)
    \\\overset{\boldsymbol{g}\sim\mathcal{N}(\boldsymbol{0},\mathbf{I}_d)}{=}&\mathbb{P}\Bigg(\frac{{g}_{1}}{\left\|\boldsymbol{g}\right\|}>0,\frac{1}{2}\frac{{g}_{1}}{\left\|\boldsymbol{g}\right\|}>\frac{{g}_{2}}{\left\|\boldsymbol{g}\right\|}\Bigg)
    =\mathbb{P}\Big({g}_{1}>0,{g}_{1}>2{g}_{2}\Big)=\mathbb{P}\Big({g}_{1}>0,{g}_{2}\leq0\Big)+\mathbb{P}\Big(g_1>2g_2>0\Big)
    \\=&\frac{1}{4}+\mathbb{P}\Big(g_1>2g_2>0\Big)\leq\frac{1}{4}+\frac{1}{2\pi}\int_{0}^{+\infty}e^{-\frac{x^2}{2}}\int_{0}^{\frac{x}{2}}e^{-\frac{y^2}{2}}\mathrm{d}y\mathrm{d}x
    \leq\frac{1}{4}+\frac{1}{2\pi}\int_{0}^{+\infty}\frac{x}{2}e^{-\frac{x^2}{2}}\mathrm{d}x
    \\\leq&\frac{1}{4}+\frac{1}{4\pi}.
\end{align*}

Secondly, by Hoeffding's inequality (Lemma \ref{lemma: hoeffding}), for any $\epsilon>0$, we have
\begin{align*}
    &\mathbb{P}\Bigg(\sum_{k\in[m]-[m/2]}\mathbb{I}\{B_k\}-\Big(\frac{1}{8}+\frac{1}{8\pi}\Big)m\geq
    \frac{m}{2}\epsilon\Bigg)\leq
    \mathbb{P}\Bigg(\sum_{k\in[m]-[m/2]}\mathbb{I}\{B_k\}-\frac{m}{2}\mathbb{E}\Big[\mathbb{I}\{B_m\}\Big]\geq
    \frac{m}{2}\epsilon\Bigg)
    \\=&\mathbb{P}\Bigg(\frac{2}{m}\sum_{k\in[m]-[m/2]}\mathbb{I}\{B_k\}-\mathbb{E}\Big[\mathbb{I}\{B_m\}\Big]\geq\epsilon\Bigg)\leq \exp\Big(-\frac{2(\frac{m}{2})^2\epsilon^2}{\frac{m}{2}}\Big)=\exp(-m\epsilon^2).
\end{align*}

let $\epsilon=0.08$ and $\delta=4\exp(-m\epsilon^2/2)$.
Combining the uniform bounds in (i)(ii)(iii), we obtain this theorem:

If $m\geq\frac{2\log(4/\delta)}{0.08^2}$, then with probability at least $1-\delta$, we have:
\begin{align*}
    \left|\#\Big\{k\in[m/2]:\left<\bw_k(0),\bx_+\right>>0\Big\}-\frac{m}{4}\right|\leq&0.04m,
    \\
    \#\Big\{k\in[m]-[m/2]:{w}_{k,1}(0)>0,(1-A\cos\Delta){w}_{k,1}(0)>A{w}_{k,2}(0)\sin\Delta\Big\}\geq&0.075m,
    \\
    \#\Big\{k\in[m]-[m/2]:{w}_{k,1}(0)>0,(B-\cos\Delta){w}_{k,1}(0)>{w}_{k,2}(0)\sin\Delta\Big\}
    \leq&0.205m.
\end{align*}
\end{proof}

So far, Lemma \ref{lemma: GF Phase I positive S+ S-}, \ref{lemma: GF Phase I positive S+}, \ref{lemma: GF Phase I positive S-}, \ref{lemma: GF Phase I positive dead}, \ref{lemma: GF Phase I negative S+ S-}, \ref{lemma: GF Phase I negative S-}, \ref{lemma: GF Phase I negative S+}, \ref{lemma: GF Phase I negative dead} characterize the training dynamics of each neuron in Phase I, and Lemma~\ref{lemma: initial position estimate} estimate the initial positions of the neurons. Now we can prove our main theorem in Phase I.

\begin{theorem}[Restatement of Theorem \ref{thm: GF Phase I}]\label{thm: restatement of GF Phase I}\ \\
Under the data Assumption \ref{ass: data}, let the two-layer network trained by Gradient Flow \eqref{equ: alg GF} starting from random initialization. Let the width $m=\Omega\left(\log(1/\delta)\right)$, the initialization scales satisfy \eqref{equ: parameter selection GF Phase I}. Then with probability at least $1-\delta$, the following results {\bf (S1)$\sim$(S5)} hold at the end of Phase I $(T_{\rm I}=10\sqrt{\frac{\kappa_1}{\kappa_2}})$:

{\bf (S1).} For positive neurons $k\in[m/2]$ $({\rm s}_k=1)$, let $\cK_+$ be the index set of living neurons, i.e. $\cK_+:=\{k\in[m/2]:\bw_k(T_{\rm I})\in\cM_+^+\cup\cM_-^+\}$. Then $0.21m\leq|\cK_+|\leq 0.29m$. Moreover, for any neuron $k\in\cK_+$, it has the following properties {\bf (P1)(P2)}.
\begin{align*}
    \text{\bf (P1).}& \text{ Its norm is small but significant}: \frac{4.66\sqrt{\kappa_1\kappa_2}}{\sqrt{m}}\leq\rho_k(T_{\rm I})\leq\frac{12\sqrt{\kappa_1\kappa_2}}{\sqrt{m}}.
    \\
    \text{\bf (P2).}& \text{ Its direction is strongly aligned with } \bmu:
    \\&\quad\quad\quad\quad\quad\left<\bw_k(T_{\rm I}),\bmu\right>\geq\Big(1-4.2\sqrt{\kappa_1\kappa_2}\Big)\left(1-\frac{2}{1+0.7\left(1+9.9\sqrt{\frac{\kappa_2}{\kappa_1}}\right)^{1.15}}\right).
\end{align*} 
{\bf (S2).} For negative neurons $k\in[m]-[m/2]$ $({\rm s}_k=-1)$, let $\cK_-$ be the index set of living neurons, i.e. $\cK_-:=\{k\in[m]-[m/2]:\bw_k(T_{\rm I})\in\cM_+^+\cup\cM_-^+\}$. Then $0.075m\leq|\cK_-|\leq 0.205m$. Moreover, for any neuron $k\in\cK_-$, it has the following properties {\bf (N1)(N2)(N3)}.
\begin{align*}
    &\text{{\bf (N1).} Its norm is tiny}: \rho_k(T_{\rm I})\leq\frac{\sqrt{\kappa_1\kappa_2}}{\sqrt{m}}\Big(4.3\sqrt{\frac{\kappa_1}{\kappa_2}}+\frac{11.1\sin\Delta}{1+p}\Big).
    \\&\text{{\bf (N2).} It lies on a manifold perpendicular to $\bx_+$}:\bw_k(t)\in\cM_+^{0}\cap\cM_-^{+}.
    \\&\text{{\bf (N3).} Its direction is weakly aligned with $\bx_+^{\perp}$}:\left<\bw_k(T_{\rm I}),\bx_+^{\perp}\right>>1-\frac{2}{\left(1+2.32\sqrt{\frac{\kappa_2}{\kappa_1}}\frac{\sin\Delta}{1+p}\right)^{1.6}+1}.
\end{align*}
{\bf (S3).} For other neurons $k\notin\cK_+\cup\cK_-$, it will remain dead forever:
\begin{align*}
    \bw_k(T_{\rm I})\in\cM_+^-\cap\cM_-^-,\quad \bb_k(t)\equiv\bb_k(T_{\rm I}),\quad \forall t\in[T_{\rm I},+\infty)
\end{align*}

{\bf (S4).} The predictions for $\bx_+$ and $\bx_-$ have the estimate: 
\begin{align*}
0.978\kappa_2\sqrt{\kappa_1\kappa_2}\Big((\frac{p-1}{p+1})^2-0.11\Big)\leq&f_+(T_{\rm I})\leq3.85\kappa_2\sqrt{\kappa_1\kappa_2},
\\
0.947\left((\frac{p-1}{p+1})^2\cos\Delta-0.2\right)\leq&f_-(T_{\rm I})\leq3.85\kappa_2\sqrt{\kappa_1\kappa_2},
\end{align*}
and the training accuracy is $\Acc(T_{\rm I})=\frac{p}{1+p}$.

{\bf (S5).}
$0<0.258\leq\frac{0.075}{0.29}\leq\frac{|\cK_-|}{|\cK_+|}\leq\frac{0.205}{0.21}\leq0.977<1$.
\end{theorem}

\begin{proof}[Proof of Theorem \ref{thm: restatement of GF Phase I}]\ \\
This theorem is a corollary of Lemma \ref{lemma: GF Phase I positive S+ S-}, Lemma \ref{lemma: GF Phase I positive S+}, Lemma \ref{lemma: GF Phase I positive S-}, Lemma \ref{lemma: GF Phase I positive dead}, Lemma \ref{lemma: GF Phase I negative S+ S-}, Lemma \ref{lemma: GF Phase I negative S-}, Lemma \ref{lemma: GF Phase I negative S+}, Lemma \ref{lemma: GF Phase I negative dead}, and Lemma \ref{lemma: initial position estimate}.
We focus on the end of Phase I: $T_{\rm I}=10\sqrt{\frac{\kappa_1}{\kappa_2}}$.

\underline{Proof of {\bf (S1)(S2)}.}
From Lemma \ref{lemma: initial position estimate}, we know that: if $m=\Omega\left(\log(1/\delta)\right)$, then with probability at least $1-\delta$, we have:
\begin{align*}
   \left|\#\Big\{k\in[m/2]:\left<\bw_k(0),\bx_+\right>>0\Big\}-\frac{m}{4}\right|\leq&0.04m,
    \\
    \#\Big\{k\in[m]-[m/2]:\left<\bw_k(0),\bx_-\right>>0,
    \left<\bw_k(0),\bx_-\right>> A\left<\bw_k(0),\bx_+\right>\Big\}\geq&0.075m,
    \\
    \#\Big\{k\in[m]-[m/2]:\left<\bw_k(0),\bx_-\right>>0,
    \left<\bw_k(0),\bx_+\right>\leq B\left<\bw_k(0),\bx_-\right>\Big\}
    \leq&0.205m.
\end{align*}
where $A=\frac{(1+4.73\kappa_1\kappa_2)p\cos\Delta-(1-4.73\kappa_1\kappa_2)}{(1-4.73\kappa_1\kappa_2)p-(1+4.73\kappa_1\kappa_2)\cos\Delta}$ and $B=\frac{(1+4.73\kappa_1\kappa_2)p-(1-4.73\kappa_1\kappa_2)\cos\Delta}{(1-4.73\kappa_1\kappa_2)p\cos\Delta-(1+4.73\kappa_1\kappa_2)}$.

Recalling the dynamics analysis in Lemma \ref{lemma: GF Phase I positive S+ S-},~\ref{lemma: GF Phase I positive S+},~\ref{lemma: GF Phase I positive S-}, and~\ref{lemma: GF Phase I positive dead}, we have:
\begin{align*}
    0.21m\leq|\cK_+|=\#\Big\{k\in[m/2]:\left<\bw_k(0),\bx_+\right>>0\Big\}\leq 0.29 m.
\end{align*}

Recalling the dynamics analysis in Lemma \ref{lemma: GF Phase I negative S+ S-},~\ref{lemma: GF Phase I negative S-}~\ref{lemma: GF Phase I negative S+}, and~\ref{lemma: GF Phase I negative dead}, we have:
\begin{align*}
&|\cK_-|\geq\#\Big\{k\in[m]-[m/2]:\left<\bw_k(0),\bx_-\right>>0,
\left<\bw_k(0),\bx_-\right>> A\left<\bw_k(0),\bx_+\right>\Big\}\geq0.075m,
\\
&|\cK_-|\leq\#\Big\{k\in[m]-[m/2]:\left<\bw_k(0),\bx_-\right>>0,
\left<\bw_k(0),\bx_+\right>\leq B\left<\bw_k(0),\bx_-\right>\Big\}
\leq0.205 m.
\end{align*}

Moreover, the estimates in Lemma \ref{lemma: GF Phase I positive S+ S-},~\ref{lemma: GF Phase I positive S+},~\ref{lemma: GF Phase I positive S-}, and~\ref{lemma: GF Phase I positive dead} ensure that for any $k\in\cK_+$, the following results hold:
\begin{gather*}
    \left<\bw_k(T_{\rm I}),\bmu\right>\geq\Big(1-4.2\sqrt{\kappa_1\kappa_2}\Big)\left(1-\frac{2}{1+0.7\left(1+9.9\sqrt{\frac{\kappa_2}{\kappa_1}}\right)^{1.15}}\right);
    \\
    \frac{4.66\sqrt{\kappa_1\kappa_2}}{\sqrt{m}}\leq\rho_k(T_{\rm I})\leq\frac{12\sqrt{\kappa_1\kappa_2}}{\sqrt{m}}.
\end{gather*}

Similarly, the estimates in Lemma~\ref{lemma: GF Phase I negative S+ S-},~\ref{lemma: GF Phase I negative S-}~\ref{lemma: GF Phase I negative S+}, and~\ref{lemma: GF Phase I negative dead} ensure that for any $k\in\cK_-$, the following results hold:

\begin{gather*}
\rho_k(T_{\rm I})\leq\frac{\sqrt{\kappa_1\kappa_2}}{\sqrt{m}}\Big(4.3\sqrt{\frac{\kappa_1}{\kappa_2}}+\frac{11.1\sin\Delta}{1+p}\Big);
\\
\bw_k(T_{\rm I})\in\cM_+^0\cap\cM_-^+;
\\
\left<\bw_k(T_{\rm I}),\bx_+^{\perp}\right>>1-\frac{2}{\left(1+2.32\sqrt{\frac{\kappa_2}{\kappa_1}}\frac{\sin\Delta}{1+p}\right)^{1.6}+1}.
\end{gather*}
\underline{Proof of {\bf (S3)}.} A direct corollary of Lemma \ref{lemma: GF Phase I positive S+ S-}, \ref{lemma: GF Phase I positive S+}, \ref{lemma: GF Phase I positive S-}, \ref{lemma: GF Phase I positive dead}, \ref{lemma: GF Phase I negative S+ S-}, \ref{lemma: GF Phase I negative S-}, \ref{lemma: GF Phase I negative S+}, \ref{lemma: GF Phase I negative dead}.

\underline{Proof of {\bf (S4)}.} {\bf (S4)} are direct corollaries of {\bf (S1)(S2)}.

For $f_+(T_{\rm I})$, we have the following estimate:
\begin{align*}
    &f_+(T_{\rm I})=\sum_{k\in\cK_+} a_k\sigma\left(\boldsymbol{b}_k(T_{\rm I})^\top\bx_+\right)+\sum_{k\in\cK_-} a_k\sigma\left(\boldsymbol{b}_k(T_{\rm I})^\top\bx_+\right)
    \\=&\sum_{k\in\cK_+} a_k\sigma\left(\boldsymbol{b}_k(T_{\rm I})^\top\bx_+\right)+0
    =\sum_{k\in\cK_+} \frac{\kappa_2}{\sqrt{m}}\rho_k(T_{\rm I})\sigma\left(\bw_k(T_{\rm I})^\top\bx_+\right)
    \\\geq&\sum_{k\in\cK_+}\frac{\kappa_2}{\sqrt{m}}\frac{4.66\sqrt{\kappa_1\kappa_2}}{\sqrt{m}}\left<\bw_k(T_{\rm I}),\bx_+\right>
    \geq\sum_{k\in\cK_+}\frac{\kappa_2}{\sqrt{m}}\frac{4.66\sqrt{\kappa_1\kappa_2}}{\sqrt{m}}\Big(\left<\bmu,\bx_+\right>-\left\|\bw_k(T_{\rm I})-\bmu\right\|\Big)
    \\=&\sum_{k\in\cK_+}\frac{\kappa_2}{\sqrt{m}}\frac{4.66\sqrt{\kappa_1\kappa_2}}{\sqrt{m}}\Big(\left\|\boldsymbol{z}\right\|\left<\boldsymbol{z},\bx_+\right>-2+2\left<\bw_k(T_{\rm I}),\bmu\right>\Big)
    \\\geq&|\cK_+|\frac{4.66\kappa_2\sqrt{\kappa_1\kappa_2}}{m}\Bigg(\frac{p-1}{p+1}\frac{p-\cos\Delta}{p+1}-2+2\Big(1-4.2\sqrt{\kappa_1\kappa_2}-\frac{2}{1+0.7\left(1+9.9\sqrt{\frac{\kappa_2}{\kappa_1}}\right)^{1.15}}\Big)\Bigg)
    \\\geq&0.21\cdot4.66\kappa_2\sqrt{\kappa_1\kappa_2}\Bigg(\frac{p-1}{p+1}\frac{p-\cos\Delta}{p+1}-8.4\sqrt{\kappa_1\kappa_2}-\frac{4}{1+0.7\left(1+9.9\sqrt{\frac{\kappa_2}{\kappa_1}}\right)^{1.15}}\Bigg)
    \\\geq&
    0.978\kappa_2\sqrt{\kappa_1\kappa_2}\Bigg(\left(\frac{p-1}{p+1}\right)^2-8.4\sqrt{\kappa_1\kappa_2}-\frac{4}{1+0.7\left(1+9.9\sqrt{\frac{\kappa_2}{\kappa_1}}\right)^{1.15}}\Bigg)
    \\\overset{\eqref{equ: parameter selection GF Phase I}}{\geq}&0.978\kappa_2\sqrt{\kappa_1\kappa_2}\Big((\frac{p-1}{p+1})^2-0.11\Big);
\end{align*}
\begin{align*}
    &f_+(T_{\rm I})=\sum_{k\in\cK_+} \frac{\kappa_2}{\sqrt{m}}\rho_k(T_{\rm I})\sigma\left(\bw_k(T_{\rm I})^\top\bx_+\right)
    \\\leq&\sum_{k\in\cK_+}\frac{\kappa_2}{\sqrt{m}}\frac{12\sqrt{\kappa_1\kappa_2}}{\sqrt{m}}\left<\bw_k(T_{\rm I}),\bx_+\right>
    \leq\sum_{k\in\cK_+}\frac{\kappa_2}{\sqrt{m}}\frac{12\sqrt{\kappa_1\kappa_2}}{\sqrt{m}}\Big(\left<\bmu,\bx_+\right>+\left\|\bw_k(T_{\rm I})-\bmu\right\|\Big)
    \\=&\sum_{k\in\cK_+}\frac{\kappa_2}{\sqrt{m}}\frac{12\sqrt{\kappa_1\kappa_2}}{\sqrt{m}}\Big(\left\|\boldsymbol{z}\right\|\left<\boldsymbol{z},\bx_+\right>+2-2\left<\bw_k(T_{\rm I}),\bmu\right>\Big)
    \\\leq&|\cK_+|\frac{12\kappa_2\sqrt{\kappa_1\kappa_2}}{m}\Bigg(1\cdot\frac{p-\cos\Delta}{p+1}+2-2+8.4\sqrt{\kappa_1\kappa_2}+\frac{4}{1+0.7\left(1+9.9\sqrt{\frac{\kappa_2}{\kappa_1}}\right)^{1.15}}\Bigg)
    \\\leq&0.29\cdot12\kappa_2\sqrt{\kappa_1\kappa_2}\Bigg(1\cdot\frac{p-\cos\Delta}{p+1}+8.4\sqrt{\kappa_1\kappa_2}+\frac{4}{1+0.7\left(1+9.9\sqrt{\frac{\kappa_2}{\kappa_1}}\right)^{1.15}}\Bigg)
    \\\overset{\eqref{equ: parameter selection GF Phase I}}{\leq}&3.48\kappa_2\sqrt{\kappa_1\kappa_2}\Big(1+0.084+\frac{4}{1+0.7(1+99)^{1.15}}\Big)\leq3.85\kappa_2\sqrt{\kappa_1\kappa_2}.
\end{align*}
Then we have:
\[
0.978\kappa_2\sqrt{\kappa_1\kappa_2}\Big((\frac{p-1}{p+1})^2-0.11\Big)\leq f_+(T_{\rm I})\leq3.85\kappa_2\sqrt{\kappa_1\kappa_2}.
\]

In the same way, we can estimate $f_-(T_{\rm I})$:
\begin{align*}
    &f_-(T_{\rm I})=\sum_{k\in\cK_+} a_k\sigma\left(\boldsymbol{b}_k(T_{\rm I})^\top\bx_-\right)+\sum_{k\in\cK_-} a_k\sigma\left(\boldsymbol{b}_k(T_{\rm I})^\top\bx_-\right)
    \\\geq&\sum_{k\in\cK_+}\frac{\kappa_2}{\sqrt{m}}\frac{4.66\sqrt{\kappa_1\kappa_2}}{\sqrt{m}}\left<\bw_k(T_{\rm I}),\bx_-\right>-\sum_{k\in\cK_-}\frac{\kappa_2}{\sqrt{m}}\frac{\sqrt{\kappa_1\kappa_2}}{\sqrt{m}}\Big(4.3\sqrt{\frac{\kappa_1}{\kappa_2}}+\frac{11.1\sin\Delta}{1+p}\Big)\left<\bw_k(T_{\rm I}),\bx_-\right>
    \\\geq&\sum_{k\in\cK_+}\frac{\kappa_2}{\sqrt{m}}\frac{4.66\sqrt{\kappa_1\kappa_2}}{\sqrt{m}}\Big(\left<\bmu,\bx_-\right>-\left\|\bw_k(T_{\rm I})-\bmu\right\|\Big)
    -\sum_{k\in\cK_-}\frac{\kappa_2}{\sqrt{m}}\frac{\sqrt{\kappa_1\kappa_2}}{\sqrt{m}}\Big(0.43+\frac{11.1}{12}\Big)\sin\Delta
    \\\geq&|\cK_+|\frac{4.66\kappa_2\sqrt{\kappa_1\kappa_2}}{m}\Bigg(\frac{p-1}{p+1}\frac{p\cos\Delta-1}{p+1}-2+2\Big(1-4.2\sqrt{\kappa_1\kappa_2}-\frac{2}{1+0.7\left(1+9.9\sqrt{\frac{\kappa_2}{\kappa_1}}\right)^{1.15}}\Big)\Bigg)
    \\&-|\cK_-|\frac{\kappa_2\sqrt{\kappa_1\kappa_2}}{m}\cdot1.355\sin\Delta
    \\\geq&0.21\cdot4.66\kappa_2\sqrt{\kappa_1\kappa_2}\Bigg(\frac{p-1}{p+1}\frac{p\cos\Delta-1}{p+1}-0.11\Bigg)-0.205\cdot1.355\sin\Delta\kappa_2\sqrt{\kappa_1\kappa_2}
    \\\geq&
    \kappa_2\sqrt{\kappa_1\kappa_2}\left(0.978(\frac{p-1}{p+1})^2\cos\Delta-0.11-0.28\sin\Delta\right){\geq}\kappa_2\sqrt{\kappa_1\kappa_2}\left(0.947(\frac{p-1}{p+1})^2\cos\Delta-0.11-0.07\right)
    \\\geq&0.947\left((\frac{p-1}{p+1})^2\cos\Delta-0.2\right);
\end{align*}
\begin{align*}
    &f_-(T_{\rm I})=\sum_{k\in\cK_+} a_k\sigma\left(\boldsymbol{b}_k(T_{\rm I})^\top\bx_-\right)+\sum_{k\in\cK_-} a_k\sigma\left(\boldsymbol{b}_k(T_{\rm I})^\top\bx_-\right)
    \\\leq&\sum_{k\in\cK_+}\frac{\kappa_2}{\sqrt{m}}\frac{12\sqrt{\kappa_1\kappa_2}}{\sqrt{m}}\left<\bw_k(T_{\rm I}),\bx_-\right>-0
    \\=&\sum_{k\in\cK_+}\frac{\kappa_2}{\sqrt{m}}\frac{12\sqrt{\kappa_1\kappa_2}}{\sqrt{m}}\Big(\left\|\boldsymbol{z}\right\|\left<\boldsymbol{z},\bx_-\right>+2-2\left<\bw_k(T_{\rm I}),\bmu\right>\Big)
    \\\leq&|\cK_+|\frac{12\kappa_2\sqrt{\kappa_1\kappa_2}}{m}\Bigg(1\cdot\frac{p\cos\Delta-1}{p+1}+2-2+8.4\sqrt{\kappa_1\kappa_2}+\frac{4}{1+0.7\left(1+9.9\sqrt{\frac{\kappa_2}{\kappa_1}}\right)^{1.15}}\Bigg)
    \\\leq&0.29\cdot12\kappa_2\sqrt{\kappa_1\kappa_2}\Bigg(1\cdot\frac{p\cos\Delta-1}{p+1}+8.4\sqrt{\kappa_1\kappa_2}+\frac{4}{1+0.7\left(1+9.9\sqrt{\frac{\kappa_2}{\kappa_1}}\right)^{1.15}}\Bigg)
    \\\overset{\eqref{equ: parameter selection GF Phase I}}{\leq}&3.48\kappa_2\sqrt{\kappa_1\kappa_2}\Big(1+0.084+\frac{4}{1+0.7(1+99)^{1.15}}\Big)\leq3.85\kappa_2\sqrt{\kappa_1\kappa_2}.
\end{align*}
Them we have:
\[
0.947\left((\frac{p-1}{p+1})^2\cos\Delta-0.2\right)\leq f_-(T_{\rm I})\leq3.85\kappa_2\sqrt{\kappa_1\kappa_2}.
\]

Due to $f_+(T_{\rm I})>0$ and $f_-(T_{\rm I})>0$, we obtain ${\rm ACC}(T_{\rm I})=\frac{p}{1+p}$.

Moreover, from Theorem \ref{thm: restatement of GF Phase I} (S1)(S2), we have
\begin{equation}
\label{equ of proof lemma: GF Phase V: m-m+ bound}
0<0.258\leq\frac{0.075m}{0.29m}\leq\alpha=\frac{|\cK_-|}{|\cK_+|}=\frac{m_-}{m_+}\leq\frac{0.205m}{0.21m}\leq0.977<1.
\end{equation}

\end{proof}

\begin{remark} 
The results in the following proofs are all based on the occurrence of the events in Theorem \ref{thm: restatement of GF Phase I}.
All of these results use the same settings as Theorem \ref{thm: GF Phase I}, except using a stronger condition on the initialization parameters~\eqref{equ: parameter selection GF} than~\eqref{equ: parameter selection GF Phase I} in Theorem~\ref{thm: restatement of GF Phase I}.
So they all hold with probability at least $1-\delta$.
\end{remark}

\newpage
\section{Proofs of Optimization Dynamics in Phase II}\label{appendix: proof: Phase II}

In this phase, we study the dynamics before the patterns of living neurons change again after Phase I. Specifically, we define 
\begin{align*}
    T_{\rm II}:=\inf\{t>T_{\rm I}:\exists k\in\cK_+\cup\cK_-,\text{\rm\sgn}_k^+(t)\ne\text{\rm\sgn}_k^+(T_{\rm I})\text{ or }\text{\rm\sgn}_k^-(t)\ne\text{\rm\sgn}_k^-(T_{\rm I})\},
\end{align*} 
and call $t\in(T_{\rm I},T_{\rm II}]$ Phase II.

Recalling the results in Theorem \ref{thm: GF Phase I}, during Phase II, the activation patterns do not change with $\sgn_k^+(t)=\sgn_k^-(t)=1$ for $k\in\cK_+$ and $\sgn_k^+(t)=0,\sgn_k^-(t)=1$ for $k\in\cK_-$. Theorem \ref{thm: GF Phase II} demonstrates that at the end of Phase II, except for one of living positive neuron $k_0\in\cK_+$ precisely changes its pattern on $\bx_-$, all other activation patterns remain unchanged.

Recall that at the end of Phase I, (i) the neuron $k\notin\cK_+\cup\cK_-$ is dead forever;
(ii) as for the living neuron $k\in\cK_+\cup\cK_-$, the activation patterns are:
\begin{align*}
    &\sgn_k^+(t)=\sgn_k^-(t)=1 \text{ for } k\in\cK_+;
    \\&\sgn_k^+(t)=0,\sgn_k^-(t)=1 \text{ for } k\in\cK_-.
\end{align*}

In this section, we will focus on the Phase when the negative neuron $k\in\cK_-$ still stays on the manifold $\mathcal{M}_+^0\cap\mathcal{M}_-^+$ (hence is still dead for $\boldsymbol{x}_+$) and the positive neuron $k\in\cK_+$ is still activated for $\boldsymbol{x}_-$. In general, we aim to estimate the 

As stated in the main text, we define as following hitting time:
\begin{equation}\label{equ: GF Phase II hitting time}
\begin{aligned}
T_{\rm II}:=
\inf\Big\{t>T_{\rm I}:&\ \exists k\in\cK_+\cup\cK_-,\text{\rm\sgn}_k^+(t)\ne\text{\rm\sgn}_k^+(T_{\rm I})\text{ or }\text{\rm\sgn}_k^-(t)\ne\text{\rm\sgn}_k^-(T_{\rm I})\Big\}
\\
=\inf\Big\{t>T_{\rm I}:&\ \exists k\in\cK_+,\text{  s.t. } \left<\boldsymbol{w}_k(t),\boldsymbol{x}_+\right>\leq0\text{ or }\left<\boldsymbol{w}_k(t),\boldsymbol{x}_-\right>\leq0,
\\&\text{or}\ \exists k\in\cK_-,\text{  s.t. } \left<\boldsymbol{w}_k(t),\boldsymbol{x}_+\right>\ne0\text{ or }\left<\boldsymbol{w}_k(t),\boldsymbol{x}_-\right>\leq0\Big\},
\end{aligned}
\end{equation}

and we call $T_{\rm I}\leq t\leq T_{\rm II}$ ``Phase II''.

First, we define a more relaxed hitting time than $T_{\rm II}$, only about the change of living positive neurons:
\begin{equation}\label{equ: GF Phase II hitting time +}
T_{\rm II}^+:=\inf\Big\{t>T_{\rm I}:\ \exists k\in\cK_+,\text{ s.t. } \left<\boldsymbol{w}_k(t),\boldsymbol{x}_+\right>\leq0\text{ or }\left<\boldsymbol{w}_k(t),\boldsymbol{x}_-\right>\leq0\Big\}.
\end{equation}

Noticing that the changes in activation partitions are essentially caused by the change of discontinuous vector fields, we first define the following auxiliary hitting time:
\begin{equation}\label{equ: GF Phase II auxiliary hitting time}
\begin{gathered}
T_{\rm II}^*:=T_{\rm II}^+\wedge\inf\big\{t>T_{\rm I}:
\left<\boldsymbol{F}_+(t),\boldsymbol{x}_+\right>\leq0\big\},
\\
\text{where}\ 
\boldsymbol{F}_+(t)=\frac{p}{1+p}e^{-f_+(t)}\boldsymbol{x}_{+}-\frac{1}{1+p}e^{f_-(t)}\boldsymbol{x}_-.
\end{gathered}
\end{equation}

We call $T_{\rm I}\leq t\leq T_{\rm II}^*$ ``Phase II*''.

In the subsequent proof, we first meticulously characterize the optimization dynamics in Phase II* and then prove $T_{\rm II}=T_{\rm II}^+=T_{\rm II}^*$. 
The crucial proof technique is fine-grained prior estimations for $2$d ODEs on $f_+(t)$ and $f_-(t)$, leading to the vector field estimation.


To begin with, we establish the following lemma about the optimization dynamics of living neurons.

\begin{lemma}[Dynamics of living neurons in Phase II*]\label{lemma: GF Phase II neuron dynamics}\ \\
In Phase II*, $t\in[T_{\rm I}, T_{\rm II}^*]$, we have the following dynamics for each neuron $k\in\cK_-\cup\cK_+$.

(S1) For positive neuron $k\in\cK_+$, we have:
\begin{align*}
&\boldsymbol{w}_k(t)\in\mathcal{M}_+^+\cap\mathcal{M}_-^+,
\\
&\frac{\mathrm{d} \boldsymbol{b}_k(t)}{\mathrm{d}t}=\frac{\kappa_2}{\sqrt{m}}\boldsymbol{F}_+(t)=\frac{\kappa_2}{\sqrt{m}}\left(\frac{p}{1+p}e^{-f_+(t)}\boldsymbol{x}_{+}-\frac{1}{1+p}e^{f_-(t)}\boldsymbol{x}_-\right).
\end{align*}
(S2). For negative neuron $k\in\cK_-$, we have:
\begin{align*}
    &\boldsymbol{w}_k(t)\in\mathcal{M}_+^0\cap\mathcal{M}_-^+,
    \\
    &\frac{\mathrm{d} \boldsymbol{b}_k(t)}{\mathrm{d}t}=\frac{\kappa_2e^{f_-(t)}}{\sqrt{m}(1+p)}\Big(\boldsymbol{x}_--\boldsymbol{x}_+\cos\Delta\Big).
\end{align*}
\end{lemma}

\begin{proof}[Proof of Lemma \ref{lemma: GF Phase II neuron dynamics}]\ \\
\underline{(S1) Let $k\in\cK_+$.}
Recalling the definition of $T_{\rm II}^*$, it holds that $\left<\boldsymbol{w}_k(t),\boldsymbol{x}_+\right>>0$ and $\left<\boldsymbol{w}_k(t),\boldsymbol{x}_-\right>>0$ for any $T_{\rm I}\leq t\leq T_{\rm II}^*$, so the dynamics holds.

\underline{(S2) Let $k\in\cK_-$.} Recalling the definition of $T_{\rm II}^*$, it holds $\left<\boldsymbol{F}_+(t),\boldsymbol{x}_+\right>>0$ for any $T_{\rm I}\leq t\leq T_{\rm II}^*$. 
Due to $\bw_k(T_{\rm I})\in\mathcal{M}_+^0\cap\mathcal{M}_-^+$, we first analysis the vector field around the manifold $\cP_+^0\cap\cP_-^+$ for $T_{\rm I}\leq t\leq T_{\rm II}^*$.

For any $\tilde{\bb}\in\cP_+^0\cap\cP_{-}^+$ and $0<\delta_0\ll1$, we know that $\cP_+^0\cap\cP_{-}^+$ separates its neighborhood $\mathcal{B}(\tilde{\bb },\delta_0)$ into two domains $\mathcal{G}_-=\{\bb \in\mathcal{B}(\tilde{\bb},\delta_0):\left<\bb ,\bx_+\right><0\}$ and $\mathcal{G}_+=\{\bb\in\mathcal{B}(\tilde{\bb},\delta_0):\left<\bb,\bx_+\right>>0\}$. Following Definition \ref{def: discontinuous system solution}, we calculate the limited vector field on $\tilde{\bb}$ from $\mathcal{G}_-$ and $\mathcal{G}_+$.

(i) The limited vector field $\bF^-$ on $\tilde{\bb}$ (from $\mathcal{G}_-$):
\begin{gather*}
\frac{\mathrm{d}\bb }{\mathrm{d} t}=\bF^-,\text{ where }
\bF^-=\frac{\kappa_2}{\sqrt{m}}\frac{1}{1+p}e^{f_-(t)}\bx_-.
\end{gather*}

(ii) The limited vector field $\bF^+$ on $\tilde{\bb}$ (from $\mathcal{G}_+$):
\begin{gather*}
\frac{\mathrm{d}\bb}{\mathrm{d} t}=\bF^+,\text{ where }
\bF^+=-\frac{\kappa_2}{\sqrt{m}}\bracket{\frac{p e^{-f_+(t)}}{1+p}\bx_+-\frac{e^{f_-(t)}}{1+p}\bx_-}.
\end{gather*}

(iii) Then we calculate the
projections of $\bF^-$ and $\bF^+$ onto $\bx_+$ (the normal to the surface $\cP_+^0\cap\cP_-^+$):
\begin{gather*}
F_N^{-}=\left<\bF^-,\bx_+\right>
=\frac{\kappa_2 e^{f_-(t)}}{\sqrt{m}(1+p)}\cos\Delta,
\\
F_N^{+}=\left<\bF^+,\bx_+\right>=\frac{\kappa_2 e^{f_-(t)}}{\sqrt{m}(1+p)}\cos\Delta-\frac{\kappa_2 p e^{-f_+(t)}}{\sqrt{m}(1+p)}.
\end{gather*}

We further define the hitting time to check whether $\boldsymbol{w}_k(t)\in\cM_+^0\cap\cM_-^+$ for $T_{\rm I}\leq t\leq T_{\rm II}^*$. 
\begin{align*}
    \tau_{-}^+:=T_{\rm II}^*\wedge\inf\{t>T_{\rm I}:\exists k\in\cK_-,\text{ s.t. }\<\bw_k(t),\bx_-\>\leq0\}.
\end{align*}

From the definition of $T_{\rm II}^*$, we know $\left<\boldsymbol{F}_+(t),\boldsymbol{x}_+\right>>0$ for any $T_{\rm I}\leq t\leq T_{\rm II}^*$, so $F_N^{+}<0$. 
And it is clear that $F_N^{-}>0$. 
Hence, the dynamics corresponds to Case (I) in Definition \ref{def: discontinuous system solution} ($F_N^{-}>0$ and $F_N^{+}<0$), which means $\boldsymbol{b}_k(t)$ can not leave $\mathcal{P}_+^0$ (i.e., $\<\bb_k(t),\bx_+\>=0$) for $t\in[T_{\rm I},\tau_-^+]$, and the dynamics of $\bb_k(t)$ satisfies:
\begin{align*}
    \frac{\mathrm{d}\bb}{\mathrm{d}t}=\alpha\bF^++(1-\alpha)\bF^-,\quad \alpha=\frac{{f}_N^-}{{f}_N^--{f}_N^+},\ t\in[T_{\rm I},\tau_-^+],
\end{align*}
which is
\begin{align*}
    \frac{\mathrm{d}\bb_k(t)}{\mathrm{d}t}\frac{\kappa_2 e^{f_-(t)}}{\sqrt{m}(1+p)}\Big(\bx_--\bx_+\cos\Delta\Big),\ t\in[T_{\rm I},\tau_-^+].
\end{align*}
By Lemma~\ref{lemma: dynamics decomposition}, we know that the dynamics of $\bw_k(t)$ on $\cM_+^0\cap\cM_-^+$ and the dynamics of $\rho_k(t)$ are:
\begin{equation*}
    \frac{\mathrm{d}\bw_k(t)}{\mathrm{d}t}=\frac{\kappa_2 e^{f_-(t)}}{\rho_k(t)\sqrt{m}(1+p)}\Big(\bx_--\left<\bw_k(t),\bx_-\right>\bw_k-\bx_+\cos\Delta\Big).
\end{equation*}
\begin{equation*}
    \frac{\mathrm{d}\rho_k(t)}{\mathrm{d}t}=\frac{\kappa_2 e^{f_-(t)}}{\sqrt{m}(1+p)}\left<\bw_k(t),\boldsymbol{x_-}\right>.
\end{equation*}

From this dynamics, for any $t\in[T_{\rm I},\tau_-^+]$, we have
\begin{align*}
    &\<\boldsymbol{b}_k(t),\bx_-\>=\<\boldsymbol{b}_k(0),\bx_-\>+\int_0^t\<\frac{\mathrm{d}\boldsymbol{b}_k(s)}{\mathrm{d}s},\bx_-\>\rd s
    \\=&\<\boldsymbol{b}_k(0),\bx_-\>+\int_0^t\frac{\kappa_2 e^{f_-(s)}}{\sqrt{m}(1+p)}\sin^2\Delta\rd s>\<\boldsymbol{b}_k(0),\bx_-\>>0,
\end{align*}
which means that $T_{\rm II}^*=\tau_-^+$.

\end{proof}


    

Noticing that Lemma~\ref{lemma: GF Phase II neuron dynamics} determines the activation patterns for living neurons in Phase II*, the next lemma gives the first-order dynamics of $f_+(t)$ and $f_-(t)$.

\begin{lemma}[First-order dynamics of predictions in Phase II*]\label{lemma: GF Phase II 1-order f dynamics}\ \\
In Phase II* $(T_{\rm I}\leq t\leq T_{\rm II}^*)$, we have the following dynamics for $f_+(t)$ and $f_-(t)$:
\begin{align*}
    \frac{\mathrm{d}f_+(t)}{\mathrm{d}t}&=\kappa_2^2\frac{m_+}{m}\Big(\frac{p e^{-f_+(t)}}{1+p}-\frac{e^{f_-(t)}}{1+p}\cos\Delta\Big),
    \\
    \frac{\mathrm{d}f_-(t)}{\mathrm{d}t}&=\kappa_2^2\frac{m_+}{m}\Big(\frac{p e^{-f_+(t)}}{1+p}\cos\Delta-\frac{e^{f_-(t)}}{1+p}\Big)-\kappa_2^2\frac{m_-}{m}\frac{e^{f_-(t)}}{1+p}\sin^2\Delta.
\end{align*}
\end{lemma}

\begin{proof}[Proof of Lemma \ref{lemma: GF Phase II 1-order f dynamics}]\ \\
From the definition of $T_{\rm II}^*$, for any $T_{\rm I}\leq t\leq T_{\rm II}^*$, we have
\begin{align*}
    f_+(t)&=\sum_{k\in\cK_+}\frac{\kappa_2}{\sqrt{m}}\sigma(\boldsymbol{b}_k^\top(t)\boldsymbol{x}_+)-\sum_{k\in\cK_-}\frac{\kappa_2}{\sqrt{m}}\sigma(\boldsymbol{b}_k^\top(t)\boldsymbol{x}_+)=\sum_{k\in\cK_+}\frac{\kappa_2}{\sqrt{m}}\boldsymbol{b}_k^\top(t)\boldsymbol{x}_+,
    \\
    f_-(t)&=\sum_{k\in\cK_+}\frac{\kappa_2}{\sqrt{m}}\sigma(\boldsymbol{b}_k^\top(t)\boldsymbol{x}_-)-\sum_{k\in\cK_-}\frac{\kappa_2}{\sqrt{m}}\sigma(\boldsymbol{b}_k^\top(t)\boldsymbol{x}_-)=\sum_{k\in\cK_+}\frac{\kappa_2}{\sqrt{m}}\boldsymbol{b}_k^\top(t)\boldsymbol{x}_--\sum_{k\in\cK_-}\frac{\kappa_2}{\sqrt{m}}\boldsymbol{b}_k^\top(t)\boldsymbol{x}_-.
\end{align*}

With the help of Lemma \ref{lemma: GF Phase II neuron dynamics}, we have the dynamics of predictions:
\begin{align*}
    \frac{\mathrm{d}f_+(t)}{\mathrm{d}t}=&\sum_{k\in\cK_+}\frac{\kappa_2}{\sqrt{m}}\left<\frac{\mathrm{d}\boldsymbol{b}_k(t)}{\mathrm{d}t},\boldsymbol{x}_+\right>=\frac{\kappa_2^2}{m}\sum_{k\in\cK_+}\Big(\frac{p}{1+p}e^{-f_+(t)}-\frac{1}{1+p}e^{f_-(t)}\cos\Delta\Big)
    \\=&
    \frac{m_+}{m}\kappa_2^2\Big(\frac{p e^{-f_+(t)}}{1+p}-\frac{e^{f_-(t)}}{1+p}\cos\Delta\Big).
\end{align*}
\begin{align*}
    \frac{\mathrm{d}f_-(t)}{\mathrm{d}t}=&\sum_{k\in\cK_+}\frac{\kappa_2}{\sqrt{m}}\left<\frac{\mathrm{d}\boldsymbol{b}_k(t)}{\mathrm{d}t},\boldsymbol{x}_-\right>-\sum_{k\in\cK_-}\frac{\kappa_2}{\sqrt{m}}\left<\frac{\mathrm{d}\boldsymbol{b}_k(t)}{\mathrm{d}t},\boldsymbol{x}_-\right>
    \\=&\frac{\kappa_2^2}{m}\sum_{k\in\cK_+}\Big(\frac{p e^{-f_+(t)}}{1+p}\cos\Delta-\frac{e^{f_-(t)}}{1+p}\Big)-\frac{\kappa_2^2}{m}\sum_{k\in\cK_-}\frac{e^{f_-(t)}}{1+p}\Big(1-\cos^2\Delta\Big)
    \\=&\frac{m_+}{m}\kappa_2^2\Big(\frac{p e^{-f_+(t)}}{1+p}\cos\Delta-\frac{e^{f_-(t)}}{1+p}\Big)-\frac{m_-}{m}\kappa_2^2\frac{e^{f_-(t)}}{1+p}\sin^2\Delta.
\end{align*}

\end{proof}

Due to the specificity of the first-order dynamics, the following lemma gives an second-order \textbf{autonomous} dynamics of predictions, which is is the core dynamics in this phase.

\begin{lemma}[Second-order Autonomous Dynamics of predictions in Phase II*]\label{lemma: GF Phase II 2-order f dynamics}\ \\
Consider the following two variables:
\[
\begin{cases}
    \mathcal{U}(t):=\kappa_2^2\frac{m_+}{m}\frac{p}{1+p}e^{-f_+(t)},
    \\
    \mathcal{V}(t):=\kappa_2^2\frac{m_+}{m}\frac{1}{1+p}e^{f_-(t)}.
\end{cases}
\]
Then the following autonomous dynamics of $\mathcal{U}(t)$ and $\mathcal{V}(t)$ hold in Phase II* $(T_{\rm I}\leq t\leq T_{\rm II}^*)$:
\[
\begin{cases}
    \frac{\mathrm{d}\mathcal{U}(t)}{\mathrm{d}t}
    =\cU(t)\cV(t)\cos\Delta-\cU^2(t),
    \\
    \frac{\mathrm{d}\mathcal{V}(t)}{\mathrm{d}t}=\mathcal{U}(t)\mathcal{V}(t)\cos\Delta-\mathcal{V}^2(t)\left(1+\alpha\sin^2\Delta\right).
\end{cases}
\]
\end{lemma}
\begin{proof}[Proof of Lemma \ref{lemma: GF Phase II 2-order f dynamics}]\ \\
Recall the first-order dynamics in Lemma \ref{lemma: GF Phase II 1-order f dynamics}:
\[
\begin{cases}
    \frac{\rd f_+(t)}{\rd t}&=\kappa_2^2\frac{m_+}{m}\left(\frac{p e^{-f_+(t)}}{1+p}-\frac{e^{f_-(t)}}{1+p}\cos\Delta\right),
    \\
    \frac{\rd f_-(t)}{\rd t}&=\kappa_2^2\frac{m_+}{m}\left(\frac{p e^{-f_+(t)}}{1+p}\cos\Delta-\frac{e^{f_-(t)}}{1+p}\right)-\kappa_2^2\frac{m_-}{m}\frac{e^{f_-(t)}}{1+p}\sin^2\Delta.
\end{cases}
\]
Then this proof is a straight-forward calculation:
\begin{align*}
    \frac{\mathrm{d}\mathcal{U}(t)}{\mathrm{d}t}
    =&\kappa_2^2\frac{m_+}{m}\frac{p}{1+p}\frac{\mathrm{d}e^{-f_+(t)}}{\mathrm{d}t}
    =-\kappa_2^2\frac{m_+}{m}\frac{p}{1+p}e^{-f_+(t)}\frac{\rd f_+(t)}{\rd t}
    \\=&-\mathcal{U}(t)\frac{\rd f_+(t)}{\rd t}
    =\cU(t)\cV(t)\cos\Delta-\cU^2(t),
\end{align*}
\begin{align*}
    \frac{\mathrm{d}\mathcal{V}(t)}{\mathrm{d}t}=&\kappa_2^2\frac{m_+}{m}\frac{1}{1+p}\frac{\mathrm{d}e^{f_-(t)}}{\mathrm{d}t}
    =\kappa_2^2\frac{m_+}{m}\frac{1}{1+p}e^{f_-(t)}\frac{\mathrm{d}f_-(t)}{\mathrm{d}t}
    \\=&\mathcal{V}(t)\frac{\mathrm{d}f_-(t)}{\mathrm{d}t}=\mathcal{U}(t)\mathcal{V}(t)\cos\Delta-\mathcal{V}^2(t)\left(1+\alpha\sin^2\Delta\right).
\end{align*}

\end{proof}

Lemma \ref{lemma: GF Phase II 2-order f dynamics} enlighten us that we only need to study the dynamics of $\mathcal{U}(t)$ and $\mathcal{V}(t)$ to study the dynamics in Phase II, where $\mathcal{U}(t),\mathcal{V}(t)$ satisfies the following autonomous dynamics:
\begin{equation}\label{equ: GF Phase II U V dynamics}
\begin{aligned}
    &\begin{cases}
    \frac{\mathrm{d}\mathcal{U}(t)}{\mathrm{d}t}
    =\mathcal{U}(t)\mathcal{V}(t)\cos\Delta-\mathcal{U}^2(t);
    \\
    \frac{\mathrm{d}\mathcal{V}(t)}{\mathrm{d}t}=\mathcal{U}(t)\mathcal{V}(t)\cos\Delta-\mathcal{V}^2(t)\left(1+\alpha\sin^2\Delta\right),
    \end{cases}\quad t\geq T_{\rm I};
    \\
    &\begin{cases}
    \mathcal{U}(T_{\rm I})=\kappa_2^2\frac{m_+}{m}\frac{p}{1+p}e^{-f_+(T_{\rm I})},
    \\
    \mathcal{V}(T_{\rm I})=\kappa_2^2\frac{m_+}{m}\frac{1}{1+p}e^{f_-(T_{\rm I})}.
    \end{cases}
\end{aligned}
\end{equation}

The next lemma provides a fine-grained prior estimate of the dynamics \eqref{equ: GF Phase II U V dynamics}.

\begin{lemma}[Fine-grained prior estimate of the dynamics \eqref{equ: GF Phase II U V dynamics}]\label{lemma: GF Phase II U V dynamics}\ \\
For the dynamics \eqref{equ: GF Phase II U V dynamics},
then we have the following results:

{\bf (S1).} $\mathcal{U}(T_{\rm I})=\Theta(\kappa_2^2)$ and $\mathcal{V}(T_{\rm I})=\Theta\Big(\frac{\kappa_2^2}{p}\Big)$.

{\bf (S2).} For any $t\geq T_{\rm I}$, we have $\mathcal{U}(t)>\mathcal{V}(t)>0$.

{\bf (S3).} If we define the hitting time
$\tau_1:=\inf\Big\{t\geq T_{\rm I}:\mathcal{U}(t)\cos\Delta\leq\mathcal{V}(t)\left(1+\alpha\sin^2\Delta\right)\Big\}$, then 
\begin{gather*}
    \mathcal{U}(\tau_1)=\frac{1+\alpha\sin^2\Delta}{\cos\Delta}\mathcal{V}(\tau_1),\quad \mathcal{V}(\tau_1)=\Theta\left(\kappa_2^2p^{-\frac{1}{1+\cos\Delta}}\right),
    \\
    \tau_1=\mathcal{O}\left(\frac{p^{\frac{1}{1+\cos\Delta}}\log(1/\Delta)}{\kappa_2^2}\right)=\Omega\left(\frac{p^{\frac{1}{1+\cos\Delta}}}{\kappa_2^2}\right)=\tilde{\Theta}\left(\frac{p^{\frac{1}{1+\cos\Delta}}}{\kappa_2^2}\right).
\end{gather*}
{\bf (S4).} For any $t\geq\tau_1$, we have
\begin{gather*}
    1+\frac{m_-}{2m_+}\sin^2\Delta<\frac{\mathcal{U}(t)}{\mathcal{V}(t)}<\frac{1+2\alpha\sin^2\Delta}{\cos\Delta},
    \\
    \mathcal{U}(t)=\Theta\left(\frac{1}{\frac{p^{\frac{1}{1+\cos\Delta}}}{\kappa_2^2}+\Delta^2(t-\tau_1)}\right),\quad
    \mathcal{V}(t)=\Theta\left(\frac{1}{\frac{p^{\frac{1}{1+\cos\Delta}}}{\kappa_2^2}+\Delta^2(t-\tau_1)}\right).
\end{gather*}

{\bf (S5).} For any $t\geq\tau_1$, we have $\mathcal{U}(t)-\mathcal{V}(t)\cos\Delta=\Theta\left(\Delta^2\cV(t)\right)>0$.

{\bf (S6).} For any $t\geq\tau_2=\Theta\left(\frac{p^{\frac{1}{1+\cos\Delta}}\log(1/\Delta)}{\kappa_2^2}\right)\geq2\tau_1$, we have 
\[
\mathcal{U}(t)\cos\Delta-\mathcal{V}(t)=-\Theta\left(\Delta^2\cV(t)\right)<0.
\]
    
\end{lemma}

\begin{proof}[Proof of Lemma \ref{lemma: GF Phase II U V dynamics}]\ \\
For simplicity, in this proof, we denote
\[
    \epsilon:=\alpha\sin^2\Delta.
\]
\underline{Step I. Preparation.}
From Theorem \ref{thm: restatement of GF Phase I} (S4), we know $0<f_+(T_{\rm I}),f_-(T_{\rm I})\leq3.85\kappa_2\sqrt{\kappa_1\kappa_2}\leq0.04\log(1.1)$, so 
\[
1<e^{f_-(T_{\rm I})}\leq1+e^{0.04\log(1.1)}3.85\kappa_2\sqrt{\kappa_1\kappa_2}\leq1+1.004\cdot3.85\kappa_2\sqrt{\kappa_1\kappa_2}\leq1+3.87\kappa_2\sqrt{\kappa_1\kappa_2},
\]
\[
1>e^{-f_+(T_{\rm I})}\geq1-e^{0.04\log(1.1)}3.85\kappa_2\sqrt{\kappa_1\kappa_2}\geq1-1.004\cdot3.85\kappa_2\sqrt{\kappa_1\kappa_2}\geq1-3.87\kappa_2\sqrt{\kappa_1\kappa_2}.
\]

Notice that $\mathcal{U}(T_{\rm I})=\kappa_2^2\frac{m_+}{m}\frac{p}{1+p}e^{-f_+(T_{\rm I})}$ and $\mathcal{V}(T_{\rm I})=\kappa_2^2\frac{m_+}{m}\frac{1}{1+p}e^{f_-(T_{\rm I})}$. Then we have the estimate:
\begin{align*}
    1-3.87\kappa_2\sqrt{\kappa_1\kappa_2}\leq&\frac{\mathcal{U}(T_{\rm I})}{\kappa_2^2\frac{m_+}{m}\frac{p}{1+p}}\leq 1,
    \\
    1\leq&\frac{\mathcal{V}(T_{\rm I})}{\kappa_2^2\frac{m_+}{m}\frac{1}{1+p}}\leq 1+3.87\kappa_2\sqrt{\kappa_1\kappa_2}.
\end{align*}

From Theorem \ref{thm: restatement of GF Phase I} (S1)(S2), we have
\begin{equation}\label{equ of proof lemma: GF Phase II: m-m+ bound}
0.258\leq\frac{0.075m}{0.29m}\leq\alpha\leq\frac{0.205m}{0.21m}\leq0.977.
\end{equation}
For $t=T_{\rm I}$, it holds that
\begin{align*}
    &\mathcal{U}(T_{\rm I})\mathcal{V}(T_{\rm I})\cos\Delta-\mathcal{U}^2(T_{\rm I})<0,
    \\
    &\mathcal{U}(T_{\rm I})\mathcal{V}(T_{\rm I})\cos\Delta-\mathcal{V}^2(T_{\rm I})\left(1+\epsilon\right)>0.
\end{align*}

\underline{Step II. A rough estimate on $\mathcal{U}(t)$ and $\mathcal{V}(t)$.} In this step, we aim to prove:
\[
\mathcal{U}(t)>\mathcal{V}(t)>0,\quad,\cU(t)+\cV(t)\leq\cU(T_{\rm I})+\cV(T_{\rm I}),\quad\forall t\in[ T_{\rm I},\infty).
\]
First, from the definition of $\cU(t)$ and $\cV(t)$, we have $\cU(t)>0$ and $\cV(t)>0$.

Then we consider the dynamics of $\mathcal{U}(t)+\mathcal{V}(t)$. From
\begin{align*}
    &\frac{\mathrm{d}}{\mathrm{d}t}\Big(\mathcal{U}(t)+\mathcal{V}(t)\Big)=2\mathcal{U}(t)\mathcal{V}(t)\cos\Delta-\mathcal{U}^2(t)-\mathcal{V}^2(t)\left(1+\epsilon\right)
    \\
    =&-\left(\mathcal{U}(t)-\mathcal{V}(t)\right)^2\cos\Delta-(1-\cos\Delta)\mathcal{U}^2(t)-\mathcal{V}^2(t)\left(1+\epsilon-\cos\Delta\right)
    <0,
\end{align*}
we have 
\[\mathcal{U}(t)+\mathcal{V}(t)\leq\mathcal{U}(T_{\rm I})+\mathcal{V}(T_{\rm I}),\quad \forall t\geq T_{\rm I}.\]
Then we consider the dynamics of $\mathcal{U}(t)-\mathcal{V}(t)$. We define the hitting time
\[
\tau_{\mathcal{U}-\mathcal{V}}:=\inf\Big\{t\geq T_{\rm I}:\mathcal{U}(t)\leq\mathcal{V}(t)\Big\}.
\]
For any $t\in[T_{\rm I},\tau_{\mathcal{U}-\mathcal{V}})$, we have:
\begin{align*}
    &\frac{\mathrm{d}}{\mathrm{d}t}\Big(\mathcal{U}(t)-\mathcal{V}(t)\Big)=-\mathcal{U}^2(t)+\mathcal{V}^2(t)\left(1+\epsilon\right)
    =-(\mathcal{U}(t)+\mathcal{V}(t))(\mathcal{U}(t)-\mathcal{V}(t))+\epsilon\mathcal{V}^2(t)
    \\>&
    -(\mathcal{U}(t)+\mathcal{V}(t))(\mathcal{U}(t)-\mathcal{V}(t))
    \geq
    -(\mathcal{U}(T_{\rm I})+\mathcal{V}(T_{\rm I}))
    (\mathcal{U}(t)-\mathcal{V}(t)),
\end{align*}
We consider the auxiliary ODE: $\frac{d}{\mathrm{d}t}\mathcal{P}(t)=-(\mathcal{U}(T_{\rm I})+\mathcal{V}(T_{\rm I}))\mathcal{P}(t)$, where $\mathcal{P}(T_{\rm I})=\mathcal{U}(T_{\rm I})-\mathcal{V}(T_{\rm I})>0$. From the Comparison Principle of ODEs, we have:
\begin{align*}
\mathcal{U}(t)-\mathcal{V}(t) 
\geq\mathcal{P}(t)
=\left(\mathcal{U}(T_{\rm I})-\mathcal{V}(T_{\rm I})\right)\exp\Big(-(\mathcal{U}(T_{\rm I})+\mathcal{V}(T_{\rm I}))(t-T_{\rm I})\Big)>0,\ \forall t\in[T_{\rm I},\tau_{\mathcal{U}-\mathcal{V}}).
\end{align*}
From the definition of $\tau_{\mathcal{U}-\mathcal{V}}$, we have proved 
\begin{gather*}
\tau_{\mathcal{U}-\mathcal{V}}=+\infty;
\\
\mathcal{U}(t)>\mathcal{V}(t),\ \forall t\in[T_{\rm I},+\infty).
\end{gather*}

\underline{Step III. Finer estimate in the early Phase $t\in[T_{\rm I},\tau_1]$.} 
Define the following hitting time
\begin{align*}
\tau_1:=\inf\Big\{t\geq T_{\rm I}:\mathcal{U}(t)\cos\Delta\leq\mathcal{V}(t)\left(1+\epsilon\right)\Big\}.
\end{align*}
From Step I, we know $\tau_1$ exists and $\tau_1> T_{\rm I}$. From \eqref{equ: GF Phase II U V dynamics}, we have $\frac{\mathrm{d}\mathcal{U}(t)}{\mathrm{d}t}<0$ and $\frac{\mathrm{d}\mathcal{V}(t)}{\mathrm{d}t}>0$ when $t\in[T_{\rm I},\tau_1)$. Moreover, we have the following dynamics for $t\in[T_{\rm I},\tau_1)$:
\begin{align*}
    \frac{\mathrm{d}\mathcal{U}}{\mathrm{d}\mathcal{V}}
    =\frac{\mathcal{U}\mathcal{V}\cos\Delta-\mathcal{U}^2}{\mathcal{U}\mathcal{V}\cos\Delta-\mathcal{V}^2\left(1+\epsilon\right)}
    =\frac{\frac{\mathcal{U}}{\mathcal{V}}\cos\Delta-\left(\frac{\mathcal{U}}{\mathcal{V}}\right)^2}{\frac{\mathcal{U}}{\mathcal{V}}\cos\Delta-\left(1+\epsilon\right)}.
\end{align*}
If we define $\mathcal{Z}(t):=\frac{\mathcal{U}(t)}{\mathcal{V}(t)}$, then we have $\mathrm{d}\mathcal{U}=\mathcal{Z}\mathrm{d}\mathcal{V}+\mathcal{V}\mathrm{d}\mathcal{Z}$. 

The dynamics above can be transformed to:
\begin{align*}
    \mathcal{V}\frac{\mathrm{d}\mathcal{Z}}{\mathrm{d}\mathcal{V}}=\frac{\mathcal{Z}\cos\Delta-\mathcal{Z}^2}{\mathcal{Z}\cos\Delta-(1+\epsilon)}-\mathcal{Z},
\end{align*}
which means
\begin{align*}
\frac{1}{\mathcal{V}}\mathrm{d}\mathcal{V}=-\frac{1}{{1+\cos\Delta+\epsilon}}\left(\frac{1+\epsilon}{\mathcal{Z}}+\frac{\sin^2\Delta+\epsilon}{(1+\cos\Delta+\epsilon)-\mathcal{Z}(1+\cos\Delta)}\right)\mathrm{d}\mathcal{Z}.
\end{align*}
Integrating this equation from $T_{\rm I}$ to $t\in[T_{\rm I},\tau_1)$, we have:
\begin{equation}\label{equ of proof: lemma: Phase II U V dynamics: phase I: equ V Z}
\begin{aligned}
    \log\left(\frac{\mathcal{V}(t)}{\mathcal{V}(T_{\rm I})}\right)
    =&
    -\frac{1+\epsilon}{1+\cos\Delta+\epsilon}\log\left(\frac{\mathcal{Z}(t)}{\mathcal{Z}(T_{\rm I})}\right)
    \\&+\frac{\sin^2\Delta+\epsilon}{(1+\cos\Delta+\epsilon)(1+\cos\Delta)}\log\left(\frac{(1+\cos\Delta)\mathcal{Z}(t)-(1+\cos\Delta+\epsilon)}{(1+\cos\Delta)\mathcal{Z}(T_{\rm I})-(1+\cos\Delta+\epsilon)}\right),\ t\in[T_{\rm I},\tau_1).
\end{aligned}
\end{equation}
From the continuity of $\mathcal{U}(t)$, $\mathcal{V}(t)$ and $\mathcal{Z}(t)$, we have
\begin{equation}\label{equ of proof: lemma: Phase II U V dynamics: phase I: tau V Z}
\tau_1=\inf\Big\{t\geq T_{\rm I}:\mathcal{Z}(t)\leq\frac{1+\epsilon}{\cos\Delta}\Big\}.
\end{equation}
Combining \eqref{equ of proof: lemma: Phase II U V dynamics: phase I: tau V Z} and \eqref{equ of proof: lemma: Phase II U V dynamics: phase I: equ V Z}, let $t\to\tau_1^{-}$. Then we have:
\begin{gather*}
    \mathcal{Z}(\tau_1)=\frac{1+\epsilon}{\cos\Delta};
    \\\mathcal{V}(\tau_1)=\mathcal{V}(T_{\rm I})\left(\frac{\mathcal{Z}(\tau_1)}{\mathcal{Z}(T_{\rm I})}\right)^{-\frac{1+\epsilon}{1+\cos\Delta+\epsilon}}\left(\frac{(1+\cos\Delta)\mathcal{Z}(\tau_1)-(1+\cos\Delta+\epsilon)}{(1+\cos\Delta)\mathcal{Z}(T_{\rm I})-(1+\cos\Delta+\epsilon)}\right)^{\frac{\sin^2\Delta+\epsilon}{(1+\cos\Delta+\epsilon)(1+\cos\Delta)}}>0,
\end{gather*}
where 
\[
\left(\frac{1-3.87\kappa_2\sqrt{\kappa_1\kappa_2}}{1+3.87\kappa_2\sqrt{\kappa_1\kappa_2}}\right)p\leq \mathcal{Z}(T_{\rm I})=\frac{\mathcal{U}(T_{\rm I})}{\mathcal{V}(T_{\rm I})}\leq p.
\]
Therefore,
\begin{align*}
&\frac{\mathcal{V}(\tau_1)}{\mathcal{V}(T_{\rm I})}
\leq\left(\frac{p\cos\Delta}{1+\epsilon}\right)^{\frac{1+\epsilon}{1+\cos\Delta+\epsilon}}
\left(\frac{\sin^2\Delta+\epsilon}{\left((1+\cos\Delta)p-(1+\cos\Delta+\epsilon)\right)\cos\Delta}\right)^{\frac{\sin^2\Delta+\epsilon}{(1+\cos\Delta+\epsilon)(1+\cos\Delta)}}
\\&
\frac{\mathcal{V}(\tau_1)}{\mathcal{V}(T_{\rm I})}
\geq
\left(\frac{(1-3.87\kappa_2\sqrt{\kappa_1\kappa_2})p\cos\Delta}{(1+3.87\kappa_2\sqrt{\kappa_1\kappa_2})(1+\epsilon)}\right)^{\frac{1+\epsilon}{1+\cos\Delta+\epsilon}}
\left(\frac{\sin^2\Delta+\epsilon}{\left(\frac{(1-3.87\kappa_2\sqrt{\kappa_1\kappa_2})(1+\cos\Delta)}{1+3.87\kappa_2\sqrt{\kappa_1\kappa_2}}p-(1+\cos\Delta+\epsilon)\right)\cos\Delta}\right)^{\frac{\sin^2\Delta+\epsilon}{(1+\cos\Delta+\epsilon)(1+\cos\Delta)}}
\end{align*}
and
\[
\mathcal{U}(\tau_1)=\frac{1+\epsilon}{\cos\Delta}\mathcal{V}(\tau_1),
\]
where $1\leq\frac{\mathcal{V}(T_{\rm I})}{\kappa_2^2\frac{m_+}{m}\frac{1}{1+p}}\leq 1+3.87\kappa_2\sqrt{\kappa_1\kappa_2}$ is estimated in Step I.

\underline{Step IV. Nearly tight bounds for $\tau_1$}

From the definition of $\tau_1$, we have $\frac{\mathrm{d}\mathcal{V}(t)}{\mathrm{d}t}>0$ for any $t\in[T_{\rm I},\tau_1)$, thus $\mathcal{V}(T_{\rm I})<\mathcal{V}(t)<\mathcal{V}(\tau_1)$, $\forall t\in(T_{\rm I},\tau_1)$. So we have 
\begin{align*}
    \mathcal{U}(t)\mathcal{V}(T_{\rm I})\cos\Delta-\mathcal{U}^2(t)<\frac{\mathrm{d}\mathcal{U}(t)}{\mathrm{d}t}<\mathcal{U}(t)\mathcal{V}(\tau_1)\cos\Delta-\mathcal{U}^2(t),\ \forall t\in(T_{\rm I},\tau_1).
\end{align*}
We first estimate the upper bound for $\tau_1$. Consider the following dynamics and the hitting time
\begin{align*}
    &\begin{cases}
        \frac{\mathrm{d}\phi(t)}{\mathrm{d}t}=\phi(t)\mathcal{V}(\tau_1)\cos\Delta-\phi^2(t),\quad t\geq T_{\rm I},
        \\
        \phi(T_{\rm I})=\mathcal{U}(T_{\rm I}).
    \end{cases}
    \\
    &\tau_1^{\rm u}:=\inf\Big\{t>T_{\rm I}:\phi(t)\leq\mathcal{U}(\tau_1)\Big\}
\end{align*}
Then $\tau_1<\tau_1^{\rm u}$. From the dynamics of $\phi(t)$, for any $t\in(T_{\rm I},\tau_1^{\rm u}]$, it holds
\begin{align*}
    \log\left(\frac{\phi(t)}{\phi(t)-\mathcal{V}(\tau_1)\cos\Delta}\right)\Bigg|_{T_{\rm I}}^t=(t-T_{\rm I})\mathcal{V}(\tau_1)\cos\Delta.
\end{align*}
Therefore,
\begin{align*}
    \tau_1^{\rm u}-T_{\rm I}=&\frac{1}{\mathcal{V}(\tau_1)\cos\Delta}\log\left(\frac{\phi(\tau_1^{\rm u})}{\phi(T_{\rm I})}\frac{\left(\phi(T_{\rm I})-\mathcal{V}(\tau_1)\cos\Delta\right)}{\left(\phi(\tau_1^{\rm u})-\mathcal{V}(\tau_1)\cos\Delta\right)}\right)
    \\=&\frac{1}{\mathcal{V}(\tau_1)\cos\Delta}\log\left(\frac{\mathcal{U}(\tau_1)}{\mathcal{U}(T_{\rm I})}\frac{\left(\mathcal{U}(T_{\rm I})-\mathcal{V}(\tau_1)\cos\Delta\right)}{\left(\mathcal{U}(\tau_1)-\mathcal{V}(\tau_1)\cos\Delta\right)}\right).
\end{align*}
With the help of Theorem \ref{thm: restatement of GF Phase I}, we have $\frac{m_+}{m}=\Theta(1)$ and $\frac{m_-}{m}=\Theta(1)$.
From Step I, we have $\mathcal{C}=\Theta\left(\frac{\kappa_2^2}{\Delta^2}\right)$, $\mathcal{U}(T_{\rm I})=\Theta(\kappa_2^2)$ and $\mathcal{V}(T_{\rm I})=\Theta\Big(\frac{\kappa_2^2}{p}\Big)$. Moreover, it holds
\begin{gather*}
    \frac{\mathcal{V}(\tau_1)}{\mathcal{V}(T_{\rm I})}=\Theta\left(p^{\frac{1+\epsilon}{1+\cos\Delta+\epsilon}}\Big(\frac{\Delta^2}{p}\Big)^{\frac{\sin^2\Delta+\epsilon}{(1+\cos\Delta+\epsilon)(1+\cos\Delta)}}\right)
    =\Theta\left(p^{\frac{\cos\Delta}{1+\cos\Delta}}\Delta^{\frac{2\sin^2\Delta+\epsilon}{(1+\cos\Delta+\epsilon)(1+\cos\Delta)}}\right)
    =\Theta\left(p^{\frac{\cos\Delta}{1+\cos\Delta}}\right),
    \\
    \mathcal{V}(\tau_1)=\Theta\left(\mathcal{V}(T_{\rm I})p^{\frac{\cos\Delta}{1+\cos\Delta}}\right)=\Theta\left(\kappa_2^2p^{-\frac{1}{1+\cos\Delta}}\right).
\end{gather*}
It is easy to verify
\begin{gather*}
    \frac{\mathcal{U}(\tau_1)}{\mathcal{U}(T_{\rm I})}=\frac{1+\epsilon}{\cos\Delta}\frac{\mathcal{V}(\tau_1)}{\mathcal{U}(T_{\rm I})}=\Theta\left(\frac{\mathcal{V}(T_{\rm I})}{\mathcal{U}(T_{\rm I})}p^{\frac{\cos\Delta}{1+\cos\Delta}}\right)
    =\Theta\left(p^{-\frac{1}{1+\cos\Delta}}\right);
    \\
    \frac{\mathcal{U}(T_{\rm I})-\mathcal{V}(\tau_1)\cos\Delta}{\mathcal{U}(\tau_1)-\mathcal{V}(\tau_1)\cos\Delta}=\Theta\left(\frac{\kappa_2^2\left(1-p^{-\frac{1}{1+\cos\Delta}}\right)}{\left(\frac{1+\epsilon}{\cos\Delta}-\cos\Delta\right)\kappa_2^2p^{-\frac{1}{1+\cos\Delta}}}\right)
    =\Theta\left(\frac{\kappa_2^2\left(1-p^{-\frac{1}{1+\cos\Delta}}\right)}{\kappa_2^2\Delta^2p^{-\frac{1}{1+\cos\Delta}}}\right)=\Theta\left(\frac{p^{\frac{1}{1+\cos\Delta}}}{\Delta^2}\right).
\end{gather*}
Hence, we obtain the upper bound for $\tau_1$:
\begin{align*}
    &\tau_1\leq\tau_1^{\rm u}=T_{\rm I}+\Theta\left(\frac{1}{\kappa_2^2p^{-\frac{1}{1+\cos\Delta}}}\log\left(p^{-\frac{1}{1+\cos\Delta}}\frac{p^{\frac{1}{1+\cos\Delta}}}{\Delta^2}\right)\right)
    \\=&\mathcal{O}\left(\sqrt{\frac{\kappa_1}{\kappa_2}}\right)+\Theta\left(\frac{p^{\frac{1}{1+\cos\Delta}}\log(1/\Delta)}{\kappa_2^2}\right)=\Theta\left(\frac{p^{\frac{1}{1+\cos\Delta}}\log(1/\Delta)}{\kappa_2^2}\right).
\end{align*}
In a similar way, we can derive the lower bound for $\tau_1$.
 Consider the following dynamics and the hitting time
\begin{align*}
    &\begin{cases}
        \frac{\mathrm{d}\psi(t)}{\mathrm{d}t}=\psi(t)\mathcal{V}(T_{\rm I})\cos\Delta-\psi^2(t),\quad t\geq T_{\rm I},
        \\
        \psi(T_{\rm I})=\mathcal{U}(T_{\rm I}).
    \end{cases}
    \\
    &\tau_1^{\rm l}:=\inf\Big\{t>T_{\rm I}:\psi(t)\leq\mathcal{U}(\tau_1)\Big\}
\end{align*}
Then $\tau_1>\tau_1^{\rm l}$. From the dynamics of $\phi(t)$, for any $t\in(T_{\rm I},\tau_1^{\rm l}]$, it holds
\begin{align*}
    \log\left(\frac{\psi(t)}{\psi(t)-\mathcal{V}(T_{\rm I})\cos\Delta}\right)\Bigg|_{T_{\rm I}}^t=(t-T_{\rm I})\mathcal{V}(T_{\rm I})\cos\Delta.
\end{align*}
Therefore,
\begin{align*}
    \tau_1^{\rm l}-T_{\rm I}=&\frac{1}{\mathcal{V}(T_{\rm I})\cos\Delta}\log\left(\frac{\phi(\tau_1^{\rm u})}{\phi(T_{\rm I})}\frac{\left(\phi(T_{\rm I})-\mathcal{V}(T_{\rm I})\cos\Delta\right)}{\left(\phi(\tau_1^{\rm u})-\mathcal{V}(T_{\rm I})\cos\Delta\right)}\right)
    \\=&\frac{1}{\mathcal{V}(T_{\rm I})\cos\Delta}\log\left(\frac{\mathcal{U}(\tau_1)}{\mathcal{U}(T_{\rm I})}\frac{\left(\mathcal{U}(T_{\rm I})-\mathcal{V}(T_{\rm I})\cos\Delta\right)}{\left(\mathcal{U}(\tau_1)-\mathcal{V}(T_{\rm I})\cos\Delta\right)}\right)
    \\=&\frac{1}{\mathcal{V}(T_{\rm I})\cos\Delta}\log\left(1+\frac{\mathcal{V}(T_{\rm I})}{\mathcal{U}(T_{\rm I})}\frac{\left(\mathcal{U}(T_{\rm I})-\mathcal{U}(\tau_1)\right)\cos\Delta}{\left(\mathcal{U}(\tau_1)-\mathcal{V}(T_{\rm I})\cos\Delta\right)}\right).
\end{align*}
It is easy to verify
\begin{align*}
    \frac{\mathcal{V}(T_{\rm I})}{\mathcal{U}(T_{\rm I})}\frac{\left(\mathcal{U}(T_{\rm I})-\mathcal{U}(\tau_1)\right)\cos\Delta}{\left(\mathcal{U}(\tau_1)-\mathcal{V}(T_{\rm I})\cos\Delta\right)}=\Theta\left(\frac{1}{p}\frac{\kappa_2^2(1-p^{-\frac{1}{1+\cos\Delta}})}{\kappa_2^2(p^{-\frac{1}{1+\cos\Delta}}-p^{-1})}\right)=\Theta\left(p^{-\frac{\cos\Delta}{1+\cos\Delta}}\right),
\end{align*}
thus
\begin{align*}
    \log\left(1+\frac{\mathcal{V}(T_{\rm I})}{\mathcal{U}(T_{\rm I})}\frac{\left(\mathcal{U}(T_{\rm I})-\mathcal{U}(\tau_1)\right)\cos\Delta}{\left(\mathcal{U}(\tau_1)-\mathcal{V}(T_{\rm I})\cos\Delta\right)}\right)=\Theta\left(p^{-\frac{\cos\Delta}{1+\cos\Delta}}\right),
\end{align*}
Hence, we obtain the lower bound for $\tau_1$:
\begin{align*}
    \tau_1\geq\tau_1^{\rm l}=T_{\rm I}+\Theta\left(\frac{1}{\frac{\kappa_2^2}{p}}p^{-\frac{\cos\Delta}{1+\cos\Delta}}\right)
    =\mathcal{O}\left(\sqrt{\frac{\kappa_1}{\kappa_2}}\right)+\Theta\left(\frac{p^{\frac{1}{1+\cos\Delta}}}{\kappa_2^2}\right)=\Theta\left(\frac{p^{\frac{1}{1+\cos\Delta}}}{\kappa_2^2}\right).
\end{align*}

\underline{Step V. Finer Estimate in the late Phase $t>\tau_1$.}

In this step, we focus on the dynamics when $t>\tau_1$. 

First, we will prove the following nearly tight bound about the ratio of $\mathcal{U}(t)$ to $\mathcal{V}(t)$:
\[
1+\frac{\epsilon}{2}<\frac{\mathcal{U}(t)}{\mathcal{V}(t)}<\frac{1+2\epsilon}{\cos\Delta},\quad\forall t\in[\tau_1,+\infty).
\]
For the right inequality, we define the hitting time 
\[
\tau_{\mathcal{U}/\mathcal{V}}^{r}:=\inf\Big\{t>\tau_1:{\mathcal{U}(t)}\geq\frac{1+2\epsilon}{\cos\Delta}{\mathcal{V}(t)}\Big\}.
\]
From $\frac{\mathcal{U}(\tau_1)}{\mathcal{V}(\tau_1)}=\frac{1+\epsilon}{\cos\Delta}<\frac{1+2\epsilon}{\cos\Delta}$, we know $\tau_{\mathcal{U}/\mathcal{V}}^{r}$ exists and $\tau_{\mathcal{U}/\mathcal{V}}^{r}>\tau_1$.

For any $t\in(\tau_1,\tau_{\mathcal{U}/\mathcal{V}}^{r})$, consider
\begin{align*}
    &\frac{\mathrm{d}}{\mathrm{d}t}\left(\mathcal{U}(t)-\frac{1+2\epsilon}{\cos\Delta}\mathcal{V}(t)\right)
    =\left(1-\frac{1+2\epsilon}{\cos\Delta}\right)\mathcal{U}(t)\mathcal{V}(t)\cos\Delta
    -\mathcal{U}^2(t)+\frac{1+2\epsilon}{\cos\Delta}(1+\epsilon)\mathcal{V}^2(t)
    \\=&-\left(\mathcal{U}(t)-\frac{1+2\epsilon}{\cos\Delta}\mathcal{V}(t)\right)\left(\mathcal{U}(t)+\left((1+2\epsilon)(1+\frac{1}{\cos\Delta})-\cos\Delta\right)\mathcal{V}(t)\right)
    \\&\quad+\left(\cos\Delta-(1+2\epsilon)(1+\frac{1}{\cos\Delta})+\frac{1+2\epsilon}{\cos\Delta}(1+\epsilon)\right)\mathcal{V}^2(t)
    \\=&-\left(\mathcal{U}(t)-\frac{1+2\epsilon}{\cos\Delta}\mathcal{V}(t)\right)\left(\mathcal{U}(t)+\left((1+2\epsilon)(1+\frac{1}{\cos\Delta})-\cos\Delta\right)\mathcal{V}(t)\right)
    \\&\quad+\left((\cos\Delta-1)+(\frac{1+2\epsilon}{\cos\Delta}\epsilon-2\epsilon)\right)\mathcal{V}^2(t)
    \\<&
    -\left(\mathcal{U}(t)-\frac{1+2\epsilon}{\cos\Delta}\mathcal{V}(t)\right)\left(\mathcal{U}(t)+\left((1+2\epsilon)(1+\frac{1}{\cos\Delta})-\cos\Delta\right)\mathcal{V}(t)\right)
    \\\overset{{\text{Step II}}}{<}&
    -\left((1+2\epsilon)(1+\frac{1}{\cos\Delta})-\cos\Delta\right)\left(\mathcal{U}(t)+\mathcal{V}(t)\right)\left(\mathcal{U}(t)-\frac{1+2\epsilon}{\cos\Delta}\mathcal{V}(t)\right)
    \\\overset{{\text{Step II}}}{\leq}&
    -\left((1+2\epsilon)(1+\frac{1}{\cos\Delta})-\cos\Delta\right)\left(\mathcal{U}(T_{\rm I})+\mathcal{V}(T_{\rm I})\right)\left(\mathcal{U}(t)-\frac{1+2\epsilon}{\cos\Delta}\mathcal{V}(t)\right).
\end{align*}
For simplicity, we denote $C_{1}:=\left((1+2\epsilon)(1+\frac{1}{\cos\Delta})-\cos\Delta\right)\left(\mathcal{U}(T_{\rm I})+\mathcal{V}(T_{\rm I})\right)>0$. We consider the auxiliary ODE: $\frac{d}{\mathrm{d}t}\mathcal{P}(t)=-C_{1}\mathcal{P}(t)$, where $\mathcal{P}(\tau_1)=\mathcal{U}(\tau_1)-\frac{1+2\epsilon}{\cos\Delta}\mathcal{V}(\tau_1)<0$. From the Comparison Principle of ODEs, we have:
\begin{align*}
\mathcal{U}(t)-\frac{1+2\epsilon}{\cos\Delta}\mathcal{V}(t)
\leq\mathcal{P}(t)
=\mathcal{P}(\tau_1)e^{-C_1(t-\tau_1)}<0,\ \forall t\in(\tau_1, \tau_{\mathcal{U}/\mathcal{V}}^{r}).
\end{align*}
From the definition of $\tau_{\mathcal{U}/\mathcal{V}}^{r}$, we have proved 
\begin{gather*}
\tau_{\mathcal{U}/\mathcal{V}}^{r}=+\infty;
\\
\mathcal{U}(t)<\frac{1+2\epsilon}{\cos\Delta}\mathcal{V}(t),\ \forall t\in[\tau_1,+\infty).
\end{gather*}
In the same way, it can be proved that
\[
\mathcal{U}(t)>(1+\frac{\epsilon}{2})\mathcal{V}(t),\ \forall t\in[\tau_1,+\infty).
\]
Moreover, we also need to derive a tight bound for $\mathcal{U}(t)$ and $\mathcal{V}(t)$ when $t>\tau_1$, respectively.

For any $t>\tau_1$, we have
\begin{align*}
    &\frac{\mathrm{d}}{\mathrm{d}t}\mathcal{V}(t)=\mathcal{V}(t)\Big(\mathcal{U}(t)\cos\Delta-(1+\epsilon)\mathcal{V}(t)\Big)
    \\>&\mathcal{V}(t)\Big((1+\frac{\epsilon}{2})\mathcal{V}(t)\cos\Delta-(1+\epsilon)\mathcal{V}(t)\Big)=-\Big((1+\epsilon)-(1+\frac{\epsilon}{2})\cos\Delta\Big)\mathcal{V}^2(t).
\end{align*}
We consider the auxiliary ODE: $\frac{d}{\mathrm{d}t}\mathcal{P}(t)=-\Big((1+\epsilon)-(1+\frac{\epsilon}{2})\cos\Delta\Big)\mathcal{P}^2(t)$, where $\mathcal{P}(\tau_1)=\mathcal{V}(\tau_1)$. From the Comparison Principle of ODEs, we have the lower bound for $\mathcal{V}(t)$:
\begin{align*}
\mathcal{V}(t)\geq\mathcal{P}(t)
=\frac{1}{\frac{1}{\mathcal{V}(\tau_1)}+\Big((1+\epsilon)-(1+\frac{\epsilon}{2})\cos\Delta\Big)(t-\tau_1)},\ \forall t\in(\tau_1,+\infty).
\end{align*}
In the same way, for any $t>\tau_1$, we have
\begin{align*}
    &\frac{\mathrm{d}}{\mathrm{d}t}\mathcal{U}(t)=\mathcal{U}(t)\Big(\mathcal{V}(t)\cos\Delta-\mathcal{U}(t)\Big)
    \\<&\mathcal{U}(t)\Big(\mathcal{U}(t)\frac{\cos\Delta}{1+\frac{\epsilon}{2}}-\mathcal{U}(t)\Big)=-\Big(1-\frac{\cos\Delta}{1+\frac{\epsilon}{2}}\Big)\mathcal{U}^2(t).
\end{align*}
We consider the auxiliary ODE: $\frac{d}{\mathrm{d}t}\mathcal{P}(t)=-\Big(1-\frac{\cos\Delta}{1+\frac{\epsilon}{2}}\Big)\mathcal{P}^2(t)$, where $\mathcal{P}(\tau_1)=\mathcal{U}(\tau_1)$. From the Comparison Principle of ODEs, we have the upper bound for $\mathcal{U}(t)$:
\begin{align*}
\mathcal{U}(t)\leq\mathcal{P}(t)
=\frac{1}{\frac{1}{\mathcal{U}(\tau_1)}+\Big(1-\frac{\cos\Delta}{1+\frac{\epsilon}{2}}\Big)(t-\tau_1)},\ \forall t\in(\tau_1,+\infty).
\end{align*}

The upper bound for $\mathcal{V}(t)$ and the lower bound for $\mathcal{U}(t)$ can be estimated by:
\begin{gather*}
    \mathcal{V}(t)<\frac{\mathcal{U}(t)}{1+\frac{\epsilon}{2}}\leq\frac{1}{1+\frac{\epsilon}{2}}\frac{1}{\frac{1}{\mathcal{U}(\tau_1)}+\Big(1-\frac{\cos\Delta}{1+\frac{\epsilon}{2}}\Big)(t-\tau_1)},\ \forall t\in(\tau_1,+\infty),
    \\
    \mathcal{U}(t)>\left(1+\frac{\epsilon}{2}\right)\mathcal{V}(t)\geq
    \frac{1+\frac{\epsilon}{2}}{\frac{1}{\mathcal{V}(\tau_1)}+\mathcal{C}\Big((1+\epsilon)-(1+\frac{\epsilon}{2})\cos\Delta\Big)(t-\tau_1)},\ \forall t\in(\tau_1,+\infty).
\end{gather*}

Hence, we obtain the tight bound for $\mathcal{U}(t)$ and $\mathcal{V}(t)$:
\begin{gather*}
    \frac{1}{\frac{1}{(1+\frac{\epsilon}{2})\mathcal{V}(\tau_1)}+\Big(\frac{1+\epsilon}{1+\frac{\epsilon}{2}}-\cos\Delta\Big)(t-\tau_1)}<\mathcal{U}(t)\leq\frac{1}{\frac{1}{\mathcal{U}(\tau_1)}+\Big(1-\frac{\cos\Delta}{1+\frac{\epsilon}{2}}\Big)(t-\tau_1)},\ \forall t\in(\tau_1,+\infty);
    \\
    \frac{1}{\frac{1}{\mathcal{V}(\tau_1)}+\Big((1+\epsilon)-(1+\frac{\epsilon}{2})\cos\Delta\Big)(t-\tau_1)}\leq\mathcal{V}(t)<\frac{1}{\frac{1+\frac{\epsilon}{2}}{\mathcal{U}(\tau_1)}+\Big(1+\frac{\epsilon}{2}-\cos\Delta\Big)(t-\tau_1)},\ \forall t\in(\tau_1,+\infty).
\end{gather*}
It means
\begin{gather*}
    \mathcal{U}(t)=\Theta\left(\frac{1}{\frac{p^{\frac{1}{1+\cos\Delta}}}{\kappa_2^2}+\Delta^2(t-\tau_1)}\right),\ \forall t\geq\tau_1=\mathcal{O}\left(\frac{p^{\frac{1}{1+\cos\Delta}}\log(1/\Delta)}{\kappa_2^2}\right);
    \\
    \mathcal{V}(t)=\Theta\left(\frac{1}{\frac{p^{\frac{1}{1+\cos\Delta}}}{\kappa_2^2}+\Delta^2(t-\tau_1)}\right),\ \forall t\geq\tau_1=\mathcal{O}\left(\frac{p^{\frac{1}{1+\cos\Delta}}\log(1/\Delta)}{\kappa_2^2}\right).
\end{gather*}

\underline{Step VI. The tight bound for $\mathcal{U}(t)-\mathcal{V}(t)\cos\Delta$.} 

From $1+\frac{\epsilon}{2}<\frac{\mathcal{U}(t)}{\mathcal{V}(t)}<\frac{1+2\epsilon}{\cos\Delta}$ proved in Step V, the two-side bound is straight-forward: for any $ t\geq\tau_1$,
\begin{align*}
    &\mathcal{U}(t)-\mathcal{V}(t)\cos\Delta
    >\Big(1+\frac{\epsilon}{2}-\cos\Delta\Big)\mathcal{V}(t)
    =\Theta\left(\Delta^2\cV(t)\right),
    \\
    &\mathcal{U}(t)-\mathcal{V}(t)\cos\Delta
    <\Big(\frac{1+2\epsilon}{\cos\Delta}-\cos\Delta\Big)\mathcal{V}(t)
    =\Theta\left(\Delta^2\cV(t)\right).
\end{align*}
Then we obtain
\begin{align*}
    \mathcal{U}(t)-\mathcal{V}(t)\cos\Delta=\Theta\left(\Delta^2\cV(t)\right)=\Theta\left(\frac{1}{\frac{p^{\frac{1}{1+\cos\Delta}}}{\kappa_2^2\Delta^2}+(t-\tau_1)}\right),\quad\forall t\geq\tau_1.
\end{align*}

\underline{Step VII. The tight bound of $\mathcal{U}(t)\cos\Delta-\mathcal{V}(t)$.} 

If we follow the proof in Step VI, $1+\frac{\epsilon}{2}<\frac{\mathcal{U}(t)}{\mathcal{V}(t)}<\frac{1+2\epsilon}{\cos\Delta}$ can only gives us a loose two-side bound for $\mathcal{U}(t)\cos\Delta-\mathcal{V}(t)$:
\begin{align*}
    &\mathcal{U}(t)\cos\Delta-\mathcal{V}(t)
    >\Big(\cos\Delta+\frac{m_-}{2m_+}\cos\Delta-1\Big)\mathcal{V}(t)
    \overset{\eqref{equ of proof lemma: GF Phase II: m-m+ bound}}{>}-\Theta\left(\Delta^2\cV(t)\right),
    \\
    &\mathcal{U}(t)\cos\Delta-\mathcal{V}(t)\cos\Delta
    <\Big(1+2\epsilon-1\Big)\mathcal{V}(t)
    =\Theta\left(\Delta^2\cV(t)\right).
\end{align*}

Hence, we need more fine-grained analysis to derive its sharper bounds. 

We first focus on its sharper upper bound. From the dynamics \eqref{equ: GF Phase II U V dynamics}, for any $t\geq T_{\rm I}$, we have
\begin{align*}
    &\frac{\mathrm{d}}{\mathrm{d}t}\Big(\mathcal{U}(t)\cos\Delta-\mathcal{V}(t)\Big)
    \\=&\left(-\Big(\mathcal{U}(t)\cos\Delta-\mathcal{V}(t)\Big)\Big(\mathcal{U}(t)+(\frac{1}{\cos\Delta}+1-\cos\Delta)\mathcal{V}(t)\Big)-\Big(\frac{1}{\cos\Delta}-\cos\Delta-\epsilon\Big)\mathcal{V}^2(t)\right)
    \\=&\left(-\Big(\mathcal{U}(t)\cos\Delta-\mathcal{V}(t)\Big)\Big(\mathcal{U}(t)+(\frac{1}{\cos\Delta}+1-\cos\Delta)\mathcal{V}(t)\Big)-\mathcal{V}^2(t)\Big(\frac{1}{\cos\Delta}-\alpha\Big)\sin^2\Delta\right).
\end{align*}
We define the hitting time 
\[
\tau_{\mathcal{U}/\mathcal{V}}^{+}:=\inf\Big\{t>\tau_1:\mathcal{U}(t)\cos\Delta-\mathcal{V}(t)\leq0\Big\}.
\]
From $\frac{\mathcal{U}(\tau_1)}{\mathcal{V}(\tau_1)}=\frac{1+\epsilon}{\cos\Delta}$, we know $\tau_{\mathcal{U}/\mathcal{V}}^{+}$ exists and $\tau_{\mathcal{U}/\mathcal{V}}^{+}>\tau_1$. 

Then for any $t\in(\tau_1,\tau_{\mathcal{U}/\mathcal{V}}^{+})$, we have $\mathcal{U}(t)\cos\Delta-\mathcal{V}(t)>0$, so
\begin{align*}
     &\frac{\mathrm{d}}{\mathrm{d}t}\Big(\mathcal{U}(t)\cos\Delta-\mathcal{V}(t)\Big)
     \\=&\left(-\Big(\mathcal{U}(t)\cos\Delta-\mathcal{V}(t)\Big)\Big(\mathcal{U}(t)+(\frac{1}{\cos\Delta}+1-\cos\Delta)\mathcal{V}(t)\Big)-\mathcal{V}^2(t)\Big(\frac{1}{\cos\Delta}-\alpha\Big)\sin^2\Delta\right)
     \\\overset{\text{Step IV}}{<}&
     \left(-\Big(\mathcal{U}(t)\cos\Delta-\mathcal{V}(t)\Big)\Big(\frac{1}{\cos\Delta}+2+\frac{\epsilon}{2}-\cos\Delta\Big)\mathcal{V}(t)-\mathcal{V}^2(t)\Big(\frac{1}{\cos\Delta}-\alpha\Big)\sin^2\Delta\right)
     \\\overset{\text{Step IV}}{\leq}&
     -\frac{\Big(\frac{\sin^2\Delta}{\cos\Delta}+2+\frac{\epsilon}{2}\Big)\Big(\mathcal{U}(t)\cos\Delta-\mathcal{V}(t)\Big)}{\frac{1}{\mathcal{V}(\tau_1)}+\Big((1+\epsilon)-(1+\frac{\epsilon}{2})\cos\Delta\Big)(t-\tau_1)}
     -\frac{\Big(\frac{1}{\cos\Delta}-\alpha\Big)\sin^2\Delta}{\left(\frac{1}{\mathcal{V}(\tau_1)}+\Big((1+\epsilon)-(1+\frac{\epsilon}{2})\cos\Delta\Big)(t-\tau_1)\right)^2}.
\end{align*}
For simplicity, we denote $A=\frac{1}{\mathcal{V}(\tau_1)}$, $B=(1+\epsilon)-(1+\frac{\epsilon}{2})\cos\Delta$, $C_1=\frac{\sin^2\Delta}{\cos\Delta}+2+\frac{\epsilon}{2}$, $C_2=\Big(\frac{1}{\cos\Delta}-\alpha\Big)\sin^2\Delta$. And we consider the auxiliary ODE:
\begin{align*}
     &\frac{\mathrm{d}\mathcal{E}(t)}{\mathrm{d}t}
     =
     -\frac{C_1\mathcal{E}(t)}{A+B(t-\tau_1)}
     -\frac{C_2}{\left(A+B(t-\tau_1)\right)^2},
     \\
     &\text{where }\mathcal{E}(\tau_1)=\mathcal{U}(\tau_1)\cos\Delta-\mathcal{V}(\tau_1)=\epsilon\mathcal{V}(\tau_1).
\end{align*}
Its solution is
\begin{align*}
    \mathcal{E}(t)=\Big(1+\frac{B}{A}(t-\tau_1)\Big)^{-\frac{C_1}{B}}\left(\mathcal{E}(\tau_1)-\frac{C_2}{A(C_1-B)}\left(\Big(1+\frac{B}{A}(t-\tau_1)\Big)^{\frac{C_1}{B}-1}-1\right)\right).
\end{align*}
From the Comparison Principle of ODEs, for any $t\in(\tau_1,\tau_{\mathcal{U}/\mathcal{V}}^{+})$, we have
\[
\mathcal{U}(t)\cos\Delta-\mathcal{V}(t)\leq\mathcal{E}(t).
\]
Let $T_{\mathcal{E}}-\tau_1=\frac{A}{B}\left(\left(1+\frac{(C_1-B)\epsilon}{C_2}\right)^{\frac{1}{\frac{C_1}{B}-1}}-1\right)$,
it is easy to verify $\mathcal{E}(T_{\mathcal{E}})=0$.
Moreover,
\begin{gather*}
    \frac{A}{B}=\Theta\left(\frac{1}{\mathcal{V}(\tau_1)\Big((1+\epsilon)-(1+\frac{\epsilon}{2})\cos\Delta\Big)}\right)=\Theta\left(\frac{1}{\kappa_2^2p^{-\frac{1}{1+\cos\Delta}}\Delta^2}\right)
    =\Theta\left(\frac{p^{\frac{1}{1+\cos\Delta}}}{\kappa_2^2\Delta^2}\right);
    \\
    \frac{1}{\frac{C_1}{B}-1}\log\left(1+\frac{(C_1-B)\epsilon}{C_2}\right)=\Theta\left(\Delta^2\log\left(1+\frac{\Theta(\Delta^2)}{\Theta(\Delta^2)}\right)\right)=\Theta(\Delta^2);
    \\
    \left(1+\frac{(C_1-B)\epsilon}{C_2}\right)^{\frac{1}{\frac{C_1}{B}-1}}-1=\exp\left(\frac{1}{\frac{C_1}{B}-1}\log\left(1+\frac{(C_1-B)\epsilon}{C_2}\right)\right)-1=\Theta(\Delta^2).
\end{gather*}
Therefore,
\begin{align*}
    &\tau_{\mathcal{U}/\mathcal{V}}^{+}\leq T_{\mathcal{E}}=\tau_1+\Theta\left(\frac{p^{\frac{1}{1+\cos\Delta}}}{\kappa_2^4\Delta^2}\Delta^2\right)
    \\=&\mathcal{O}\left(\frac{p^{\frac{1}{1+\cos\Delta}}\log(1/\Delta)}{\kappa_2^2}\right)+\Theta\left(\frac{p^{\frac{1}{1+\cos\Delta}}}{\kappa_2^2}\right)=\mathcal{O}\left(\frac{p^{\frac{1}{1+\cos\Delta}}\log(1/\Delta)}{\kappa_2^2}\right).
\end{align*}
Then we define the next hitting time 
\[
\tau_{\mathcal{U}/\mathcal{V}}^{-}:=\inf\Big\{t\geq\tau_{\mathcal{U}/\mathcal{V}}^{+}:\mathcal{U}(t)\cos\Delta-\mathcal{V}(t)\geq0\Big\}.
\]
From $\mathcal{U}(\tau_{\mathcal{U}/\mathcal{V}}^{+})\cos\Delta-\mathcal{V}(\tau_{\mathcal{U}/\mathcal{V}}^{+})=0$ and $\frac{\mathrm{d}}{\mathrm{d}t}\Big(\mathcal{U}(t)\cos\Delta-\mathcal{V}(t)\Big)\Big|_{t=\tau_{\mathcal{U}/\mathcal{V}}^{+}}<0$, we know $\tau_{\mathcal{U}/\mathcal{V}}^{-}$ exists and $\tau_{\mathcal{U}/\mathcal{V}}^{-}>\tau_{\mathcal{U}/\mathcal{V}}^{+}$.

For any $t\in(\tau_{\mathcal{U}/\mathcal{V}}^{+},\tau_{\mathcal{U}/\mathcal{V}}^{-})$, we have $\mathcal{U}(t)\cos\Delta-\mathcal{V}(t)<0$, so
\begin{align*}
     &\frac{\mathrm{d}}{\mathrm{d}t}\Big(\mathcal{U}(t)\cos\Delta-\mathcal{V}(t)\Big)
     \\=&\left(-\Big(\mathcal{U}(t)\cos\Delta-\mathcal{V}(t)\Big)\Big(\mathcal{U}(t)+(\frac{1}{\cos\Delta}+1-\cos\Delta)\mathcal{V}(t)\Big)-\mathcal{V}^2(t)\Big(\frac{1}{\cos\Delta}-\alpha\Big)\sin^2\Delta\right)
     \\\overset{\text{Step IV}}{<}&
     \left(-\Big(\mathcal{U}(t)\cos\Delta-\mathcal{V}(t)\Big)\Big(1+\frac{\frac{1}{\cos\Delta}+1-\cos\Delta}{1+\frac{\epsilon}{2}}\Big)\mathcal{U}(t)-\mathcal{V}^2(t)\Big(\frac{1}{\cos\Delta}-\alpha\Big)\sin^2\Delta\right)
     \\\overset{\text{Step IV}}{\leq}&
     -\frac{\Big(1+\frac{\frac{1}{\cos\Delta}+1-\cos\Delta}{1+\frac{\epsilon}{2}}\Big)\Big(\mathcal{U}(t)\cos\Delta-\mathcal{V}(t)\Big)}{\frac{1}{\mathcal{U}(\tau_1)}+\Big(1-\frac{\cos\Delta}{1+\frac{\epsilon}{2}}\Big)(t-\tau_1)}
     -\frac{\Big(\frac{1}{\cos\Delta}-\alpha\Big)\sin^2\Delta}{\left(\frac{1}{\mathcal{V}(\tau_1)}+\Big((1+\epsilon)-(1+\frac{\epsilon}{2})\cos\Delta\Big)(t-\tau_1)\right)^2}
     \\\leq&
     -\frac{\Big(1+\frac{\frac{1}{\cos\Delta}+1-\cos\Delta}{1+\frac{\epsilon}{2}}\Big)\Big(\mathcal{U}(t)\cos\Delta-\mathcal{V}(t)\Big)}{\frac{\cos\Delta}{1+\epsilon}\frac{1}{\mathcal{V}(\tau_1)}+\frac{1}{3}\Big((1+\epsilon)-(1+\frac{\epsilon}{2})\cos\Delta\Big)(t-\tau_1)}
     -\frac{\Big(\frac{1}{\cos\Delta}-\alpha\Big)\sin^2\Delta}{\left(\frac{1}{\mathcal{V}(\tau_1)}+\Big((1+\epsilon)-(1+\frac{\epsilon}{2})\cos\Delta\Big)(t-\tau_1)\right)^2}
     \\\leq&
     -\frac{3\Big(1+\frac{\frac{1}{\cos\Delta}+1-\cos\Delta}{1+\frac{\epsilon}{2}}\Big)\Big(\mathcal{U}(t)\cos\Delta-\mathcal{V}(t)\Big)}{\frac{1}{\mathcal{V}(\tau_1)}+\Big((1+\epsilon)-(1+\frac{\epsilon}{2})\cos\Delta\Big)(t-\tau_1)}
     -\frac{\Big(\frac{1}{\cos\Delta}-\alpha\Big)\sin^2\Delta}{\left(\frac{1}{\mathcal{V}(\tau_1)}+\Big((1+\epsilon)-(1+\frac{\epsilon}{2})\cos\Delta\Big)(t-\tau_1)\right)^2}.
\end{align*}
For simplicity, we denote $C_3=3\Big(1+\frac{\frac{1}{\cos\Delta}+1-\cos\Delta}{1+\frac{\epsilon}{2}}\Big)$. And we consider the auxiliary ODE:
\begin{align*}
     &\frac{\mathrm{d}\mathcal{F}(t)}{\mathrm{d}t}
     =
     -\frac{C_3\mathcal{F}(t)}{A+B(t-\tau_1)}
     -\frac{C_2}{\left(A+B(t-\tau_1)\right)^2},
     \\
     &\text{where }\mathcal{F}(\tau_{\mathcal{U}/\mathcal{V}}^{+})=\mathcal{U}(\tau_{\mathcal{U}/\mathcal{V}}^{+})\cos\Delta-\mathcal{V}(\tau_{\mathcal{U}/\mathcal{V}}^{+})=0.
\end{align*}
Its solution is
\begin{align*}
    &\mathcal{F}(t)=-\frac{C_2}{A(C_3-B)\Big(1+\frac{B}{A}(t-\tau_1)\Big)^{\frac{C_3}{B}}}\left(\Big(1+\frac{B}{A}(t-\tau_1)\Big)^{\frac{C_3}{B}-1}-1\right)
    \\\leq&-\frac{C_2}{AC_3\Big(1+\frac{B}{A}(t-\tau_1)\Big)}\left(1-\frac{1}{\Big(1+\frac{B}{A}(t-\tau_1)\Big)^{\frac{C_3}{B}-1}}\right).
\end{align*}
Let $\tau_1'=\tau_1+\Theta\left(\frac{p^{\frac{1}{1+\cos\Delta}}\log(1/\Delta)}{\kappa_2^2}\right)\geq2\tau_1$. Then for any $t\geq\tau_1'$, it holds
\begin{align*}
    &\frac{1}{\Big(1+\frac{B}{A}(t-\tau_1)\Big)^{\frac{C_3}{B}-1}}\leq\frac{1}{\Big(1+\frac{B}{A}(\tau_1'-\tau_1)\Big)^{\frac{C_3}{B}-1}}
    =\exp\left(-\Big(\frac{C_3}{B}-1\Big)\log\Big(1+\frac{B}{A}(\tau_1'-\tau_1)\Big)\right)
    \\=&\exp\left(-\Theta\left(\frac{1}{\Delta^2}\right)\log\left(1+\Theta\left(\frac{\kappa_2^4\Delta^2}{p^{\frac{1}{1+\cos\Delta}}}\frac{p^{\frac{1}{1+\cos\Delta}}\log(1/\Delta)}{\kappa_2^2}\right)\right)\right)
    \\=&\exp\left(-\Theta\left(\frac{1}{\Delta^2}\right)\log\left(1+\Theta\left(\Delta^2\log(1/\Delta)\right)\right)\right)
    =\exp\left(-\Theta\left(\frac{1}{\Delta^2}\right)\Theta\left(\Delta^2\log(1/\Delta)\right)\right)
    \\=&\exp\left(-\Theta\left(\log(1/\Delta)\right)\right)\leq\frac{1}{2},
\end{align*}
thus,
\begin{align*}
    &\mathcal{F}(t)\leq-\frac{C_2}{AC_3\Big(1+\frac{B}{A}(t-\tau_1)\Big)}\Big(1-\frac{1}{2}\Big)
    =-\Theta\left(\frac{\Delta^2}{\frac{p^{\frac{1}{1+\cos\Delta}}}{\kappa_2^2}+\Delta^2(t-\tau_1)}\right),
    \\&\ \forall t\geq\tau_1'=\tau_1+\Theta\left(\frac{p^{\frac{1}{1+\cos\Delta}}\log(1/\Delta)}{\kappa_2^2}\right)\geq2\tau_1.
\end{align*}
If we let $\tau_2=\Theta\left(\frac{p^{\frac{1}{1+\cos\Delta}}\log(1/\Delta)}{\kappa_2^2}\right)\geq2\tau_1'$, then we have:
\begin{align*}
    &\mathcal{F}(t)\leq-\Theta\left(\frac{\Delta^2}{\frac{p^{\frac{1}{1+\cos\Delta}}}{\kappa_2^2}+\Delta^2(t-\tau_1)}\right),\ \forall t\geq\tau_2.
\end{align*}

From the Comparison Principle of ODEs, for any $t\in(\tau_{\mathcal{U}/\mathcal{V}}^{+},\tau_{\mathcal{U}/\mathcal{V}}^{-})$, we have
\[
\mathcal{U}(t)\cos\Delta-\mathcal{V}(t)\leq\mathcal{F}(t)<0.
\]
From the definition of $\tau_{\mathcal{U}/\mathcal{V}}^{-}$, we obtain
\[
\tau_{\mathcal{U}/\mathcal{V}}^{-}=+\infty.
\]
Moreover,
\[
\mathcal{U}(t)\cos\Delta-\mathcal{V}(t)\leq-\Theta\left(\frac{\Delta^2}{\frac{p^{\frac{1}{1+\cos\Delta}}}{\kappa_2^2}+\Delta^2(t-\tau_1)}\right),\ \forall t\geq\tau_2=\Theta\left(\frac{p^{\frac{1}{1+\cos\Delta}}\log(1/\Delta)}{\kappa_2^2}\right).
\]
Recalling the lower bound at the beginning of Step VII, we obtain the tight bound:
\begin{align*}
    \mathcal{U}(t)\cos\Delta-\mathcal{V}(t)=-\Theta\left(\frac{\Delta^2}{\frac{p^{\frac{1}{1+\cos\Delta}}}{\kappa_2^2}+\Delta^2(t-\tau_1)}\right)=-\Theta\bracket{\Delta^2\cV(t)},\ \forall t\geq\tau_2=\Theta\left(\frac{p^{\frac{1}{1+\cos\Delta}}\log(1/\Delta)}{\kappa_2^2}\right).
\end{align*}  
\end{proof}

\begin{lemma}[Hitting time relationship]\label{lemma: GF Phase II hitting time transformation}
\begin{align*}
  &T_{\rm II}^+=T_{\rm II}^*
  =\inf\left\{t\geq T_{\rm I}:\exists k\in\cK_+,\ {\rm s.t. }
  \left<\boldsymbol{b}_k(t),\boldsymbol{x}_-\right>\leq0
  \right\}
  \\=&\inf\left\{t\geq T_{\rm I}:\exists k\in\cK_+,\ {\rm s.t. }
  \left<\boldsymbol{b}_k(T_{\rm I}),\boldsymbol{x}_-\right>+\int_{T_{\rm I}}^t\frac{\sqrt{m}}{\kappa_2m_+}\Big(\mathcal{U}(s)\cos\Delta-\mathcal{V}(s)\Big)\mathrm{d}s\leq0
  \right\},
\end{align*}
where $T_{\rm II}^+$ and $T_{\rm II}^*$ are defined in \eqref{equ: GF Phase II hitting time +}\eqref{equ: GF Phase II auxiliary hitting time}, and $\mathcal{U}(t),\mathcal{V}(t)$ satisfy \eqref{equ: GF Phase II U V dynamics}.
\end{lemma}

\begin{proof}[Proof of Lemma \ref{lemma: GF Phase II hitting time transformation}]\ \\
Recall the definitions of $T_{\rm II}^+$ and $T_{\rm II}^*$:
\begin{align*}
T_{\rm II}^+=\inf\Big\{t>T_{\rm I}:&\ \exists k\in\cK_+,\text{  s.t. } \left<\boldsymbol{w}_k(t),\boldsymbol{x}_+\right>\leq0\text{ or }\left<\boldsymbol{w}_k(t),\boldsymbol{x}_-\right>\leq0\Big\}
\\=\inf\Big\{t>T_{\rm I}:&\ \exists k\in\cK_+,\text{  s.t. } \left<\boldsymbol{b}_k(t),\boldsymbol{x}_+\right>\leq0\text{ or }\left<\boldsymbol{b}_k(t),\boldsymbol{x}_-\right>\leq0\Big\},
\\
T_{\rm II}^*=\inf\Big\{t>T_{\rm I}:&\left<\boldsymbol{F}_+(t),\boldsymbol{x}_+\right>\leq0\text{ or }
\exists k\in\cK_+,
\text{  s.t. } \left<\boldsymbol{b}_k(t),\boldsymbol{x}_+\right>\leq0\text{ or }\left<\boldsymbol{b}_k(t),\boldsymbol{x}_-\right>\leq0\Big\},
\\
\text{where}\quad
\boldsymbol{F}_+(t)=&\frac{p}{1+p}e^{-f_+(t)}\boldsymbol{x}_{+}-\frac{1}{1+p}e^{f_-(t)}\boldsymbol{x}_-.
\end{align*}
Notice
\begin{align*}
    \left<\boldsymbol{F}_+(t),\boldsymbol{x}_+\right>=\frac{p}{1+p}e^{-f_+(t)}-\frac{1}{1+p}e^{f_-(t)}\cos\Delta
    =\frac{m}{\kappa_2^2m_+}\Big(\mathcal{U}(t)-\mathcal{V}(t)\cos\Delta\Big).
\end{align*}
And for any $k\in\cK_+$,
\begin{align*}
    &\left<\boldsymbol{b}_k(t),\boldsymbol{x}_+\right>
    =\left<\boldsymbol{b}_k(T_{\rm I}),\boldsymbol{x}_+\right>+\int_{T_{\rm I}}^t\left<\frac{\mathrm{d}\boldsymbol{b}_k(s)}{\mathrm{d}s},\boldsymbol{x}_+\right>\mathrm{d}s
    \\=&\left<\boldsymbol{b}_k(T_{\rm I}),\boldsymbol{x}_+\right>+\int_{T_{\rm I}}^t\frac{\sqrt{m}}{\kappa_2m_+}\Big(\mathcal{U}(s)-\mathcal{V}(s)\cos\Delta\Big)\mathrm{d}s;
\end{align*}
\begin{align*}
    &\left<\boldsymbol{b}_k(t),\boldsymbol{x}_-\right>
    =\left<\boldsymbol{b}_k(T_{\rm I}),\boldsymbol{x}_-\right>+\int_{T_{\rm I}}^t\left<\frac{\mathrm{d}\boldsymbol{b}_k(s)}{\mathrm{d}s},\boldsymbol{x}_-\right>\mathrm{d}s
    \\=&\left<\boldsymbol{b}_k(T_{\rm I}),\boldsymbol{x}_-\right>+\int_{T_{\rm I}}^t\frac{\sqrt{m}}{\kappa_2m_+}\Big(\mathcal{U}(s)\cos\Delta-\mathcal{V}(s)\Big)\mathrm{d}s.
\end{align*}
So we have
\begin{align*}
    T_{\rm II}^*=\sup\Big\{t>T_{\rm I}:&\mathcal{U}(t)-\mathcal{V}(t)\cos\Delta>0;
    \\&\left<\boldsymbol{b}_k(T_{\rm I}),\boldsymbol{x}_+\right>+\int_{T_{\rm I}}^t\frac{\sqrt{m}}{\kappa_2m_+}\Big( \mathcal{U}(s)-\mathcal{V}(s)\cos\Delta\Big)\mathrm{d}s>0,\forall k\in\cK_+;
    \\&\left<\boldsymbol{b}_k(T_{\rm I}),\boldsymbol{x}_-\right>+\int_{T_{\rm I}}^t\frac{\sqrt{m}}{\kappa_2m_+}\Big(\mathcal{U}(s)\cos\Delta-\mathcal{V}(s)\Big)\mathrm{d}s>0,\forall k\in\cK_+
    \Big\}.
\end{align*}

With the help of Lemma \ref{lemma: GF Phase II U V dynamics}, we know that $\mathcal{U}(t)-\mathcal{V}(t)\cos\Delta>0$ for any $t\geq T_{\rm I}$. So $\left<\boldsymbol{b}_k(T_{\rm I}),\boldsymbol{x}_+\right>+\int_{T_{\rm I}}^t\frac{\sqrt{m}}{\kappa_2m_+\Delta}\Big(\mathcal{U}(s)-\mathcal{V}(s)\cos\Delta\Big)\mathrm{d}s>\left<\boldsymbol{b}_k(T_{\rm I}),\boldsymbol{x}_+\right>>0,\forall k\in\cK_+,\forall t\geq T_{\rm I}$. Hence, we have the transformation of the hitting time:
\begin{gather*}
    T_{\rm II}^*=\inf\left\{t\geq T_{\rm I}:\exists k\in\cK_+,\ {\rm s.t. }\left<\boldsymbol{b}_k(T_{\rm I}),\boldsymbol{x}_-\right>+\int_{T_{\rm I}}^t\frac{\sqrt{m}}{\kappa_2m_+}\Big(\mathcal{U}(s)\cos\Delta-\mathcal{V}(s)\Big)\mathrm{d}s\leq0\right\},
    \\
 T_{\rm II}^+=T_{\rm II}^*.
\end{gather*}

\end{proof}

\begin{lemma}[Time Estimate of Phase II]\label{lemma: GF Phase II hitting time estimate}
\begin{align*}
    T_{\rm II}=T_{\rm II}^+=T_{\rm II}^*=\Theta\left(\frac{p^{\frac{1}{1-\alpha\cos\Delta}}}{\kappa_2^2\Delta^2}\right).
\end{align*}
\end{lemma}

\begin{proof}[Proof of Lemma \ref{lemma: GF Phase II hitting time estimate}]\ \\
With the help of Theorem \ref{thm: restatement of GF Phase I} (S1), for any $k\in\cK_+$, we have:
\[
    \frac{4.66\sqrt{\kappa_1\kappa_2}}{\sqrt{m}}\leq\rho_k(T_{\rm I})\leq\frac{12\sqrt{\kappa_1\kappa_2}}{\sqrt{m}};
\]
\begin{align*}
\left<\boldsymbol{w}_k(T_{\rm I}),\boldsymbol{\mu}\right>\geq&\Big(1-4.2\sqrt{\kappa_1\kappa_2}\Big)\left(1-\frac{2}{1+0.7\left(1+9.9\sqrt{\frac{\kappa_2}{\kappa_1}}\right)^{1.15}}\right)
\\>&1-4.2\sqrt{\kappa_1\kappa_2}-\frac{2}{1+0.7\left(1+9.9\sqrt{\frac{\kappa_2}{\kappa_1}}\right)^{1.15}}.
\end{align*}  

With the help of Lemma \ref{lemma: mu theta to x- theta}, we have the estimate of $\left<\boldsymbol{w}_k(T_{\rm I}),\boldsymbol{x}_-\right>$:
\begin{align*}
    -2\sqrt{\epsilon}\sin\Delta-\epsilon\leq\left<\boldsymbol{w}_k(T_{\rm I}),\boldsymbol{x}_-\right>-\frac{p\cos\Delta-1}{\sqrt{p^2+1-2p\cos\Delta}}\leq2\sqrt{\epsilon}\sin\Delta,
\end{align*}
where 
\[\epsilon=4.2\sqrt{\kappa_1\kappa_2}+\frac{2}{1+0.7\left(1+9.9\sqrt{\frac{\kappa_2}{\kappa_1}}\right)^{1.15}}.\]
Then we have:
\begin{align*}
    &\left<\boldsymbol{b}_k(T_{\rm I}),\boldsymbol{x}_-\right>\leq\frac{12\sqrt{\kappa_1\kappa_2}}{\sqrt{m}}\left(\frac{p\cos\Delta-1}{\sqrt{p^2+1-2p\cos\Delta}}+2\sqrt{\epsilon}\sin\Delta\right),
    \\&\left<\boldsymbol{b}_k(T_{\rm I}),\boldsymbol{x}_-\right>\geq\frac{4.66\sqrt{\kappa_1\kappa_2}}{\sqrt{m}}\left(\frac{p\cos\Delta-1}{\sqrt{p^2+1-2p\cos\Delta}}-2\sqrt{\epsilon}\sin\Delta-\epsilon\right),
\end{align*}
which means
\[
\left<\boldsymbol{b}_k(T_{\rm I}),\boldsymbol{x}_-\right>=\Theta\left(\frac{\sqrt{\kappa_1\kappa_2}}{\sqrt{m}}\right).
\]
From the dynamics \eqref{equ: GF Phase II U V dynamics}, we have
\begin{align*}
&\mathcal{U}(t)\cos\Delta-\mathcal{V}(t)
\\=&\frac{m_-\cos\Delta}{m_-+m_+}\Big(\mathcal{U}(t)-\mathcal{V}(t)\cos\Delta\Big)+\frac{m_+}{m_-+m_+}\left(\mathcal{U}(t)\cos\Delta-\mathcal{V}(t)\left(1+\alpha\sin^2\Delta\right)\right)
\\=&
-\frac{m_-\cos\Delta}{m_-+m_+}\frac{\mathrm{d}\mathcal{U}(t)}{\mathcal{U}(t)\mathrm{d}t}
+\frac{m_+}{m_-+m_+}\frac{\mathrm{d}\mathcal{V}(t)}{\mathcal{V}(t)\mathrm{d}t},
\end{align*}
Taking integral, we obtain:
\begin{align*}
    &\int_{T_{\rm I}}^t\frac{\sqrt{m}}{\kappa_2m_+}\Big(\mathcal{U}(s)\cos\Delta-\mathcal{V}(s)\Big)\mathrm{d}s
    \\=&\frac{\sqrt{m}}{\kappa_2m_+}
    \int_{T_{\rm I}}^t\left(-\frac{m_-\cos\Delta}{m_-+m_+}\frac{\mathrm{d}\mathcal{U}(s)}{\mathcal{U}(s)}+\frac{m_+}{m_-+m_+}\frac{\mathrm{d}\mathcal{V}(s)}{\mathcal{V}(s)}\right)
    \\=&
    \frac{\sqrt{m}}{\kappa_2m_+}
    \left(-\frac{m_-\cos\Delta}{m_-+m_+}\log\left(\frac{\mathcal{U}(t)}{\mathcal{U}(T_{\rm I})}\right)
    +\frac{m_+}{m_-+m_+}\log\left(\frac{\mathcal{V}(t)}{\mathcal{V}(T_{\rm I})}\right)\right)
    \\=&
   \frac{\sqrt{m}}{\kappa_2m_+}
    \log\left(\frac{\mathcal{U}(T_{\rm I})^{\frac{m_-\cos\Delta}{(m_-+m_+)}}}{\mathcal{U}(t)^{\frac{m_-\cos\Delta}{(m_-+m_+)}}}\cdot\frac{\mathcal{V}(t)^{\frac{m_+}{m_-+m_+}}}{\mathcal{V}(T_{\rm I})^{\frac{m_+}{m_-+m_+}}}\right).
\end{align*}
From Lemma \ref{lemma: GF Phase II hitting time transformation}, we have:
\begin{align*}
  T_{\rm II}^+=T_{\rm II}^*=\inf\left\{t\geq T_{\rm I}:\exists k\in\cK_+,\ {\rm s.t. }
  \left<\boldsymbol{b}_k(T_{\rm I}),\boldsymbol{x}_-\right>+\int_{T_{\rm I}}^t\frac{\sqrt{m}}{\kappa_2m_+}\Big(\mathcal{U}(s)\cos\Delta-\mathcal{V}(s)\Big)\mathrm{d}s\leq0
\right\}.
\end{align*}
Recalling the definition of $\tau_1$ in Lemma \ref{lemma: GF Phase II U V dynamics}, we know $\mathcal{U}(s)\cos\Delta-\mathcal{V}(s)>0$ for any $t\leq \tau_1$, so $T_{\rm II}^*>\tau_1$.

From Lemma \ref{lemma: GF Phase II U V dynamics} (S4), we know
\[
    \mathcal{U}(t)=\Theta\left(\frac{1}{\frac{p^{\frac{1}{1+\cos\Delta}}}{\kappa_2^2}+\Delta^2(t-\tau_1)}\right),\quad
    \mathcal{V}(t)=\Theta\left(\frac{1}{\frac{p^{\frac{1}{1+\cos\Delta}}}{\kappa_2^2}+\Delta^2(t-\tau_1)}\right).
\]
From the proof of Lemma \ref{lemma: GF Phase II U V dynamics}, we know $\mathcal{U}(T_{\rm I})=\Theta(\kappa_2^2)$ and $\mathcal{V}(T_{\rm I})=\Theta\Big(\frac{\kappa_2^2}{p}\Big)$. 

Therefore, solving
\begin{align*}
    &0=\left<\boldsymbol{b}_k(T_{\rm I}),\boldsymbol{x}_-\right>+\int_{T_{\rm I}}^{T_{\rm II}^*}\frac{\sqrt{m}}{\kappa_2m_+}\Big(\mathcal{U}(s)\cos\Delta-\mathcal{V}(s)\Big)\mathrm{d}s
    \\=&\Theta\left(\frac{\sqrt{\kappa_1\kappa_2}}{\sqrt{m}}\right)+\frac{\sqrt{m}}{\kappa_2m_+}
    \log\left(\frac{\mathcal{U}(T_{\rm I})^{\frac{m_-\cos\Delta}{(m_-+m_+)}}}{\mathcal{U}(T_{\rm II}^*)^{\frac{m_-\cos\Delta}{(m_-+m_+)}}}\cdot\frac{\mathcal{V}(T_{\rm II}^*)^{\frac{m_+}{m_-+m_+}}}{\mathcal{V}(T_{\rm I})^{\frac{m_+}{m_-+m_+}}}\right)
    \\=&\Theta\left(\frac{\sqrt{\kappa_1\kappa_2}}{\sqrt{m}}\right)+\Theta\left(\frac{1}{\kappa_2 \sqrt{m}}\right)
    \log\left(\frac{\mathcal{U}(T_{\rm I})^{\frac{m_-\cos\Delta}{(m_-+m_+)}}}{\mathcal{U}(T_{\rm II}^*)^{\frac{m_-\cos\Delta}{(m_-+m_+)}}}\cdot\frac{\mathcal{V}(T_{\rm II}^*)^{\frac{m_+}{m_-+m_+}}}{\mathcal{V}(T_{\rm I})^{\frac{m_+}{m_-+m_+}}}\right),
\end{align*}
we obtain
\begin{align*}
    \frac{\mathcal{V}(T_{\rm II}^*)^{\frac{m_+}{m_-+m_+}}}{\mathcal{U}(T_{\rm II}^*)^{\frac{m_-\cos\Delta}{(m_-+m_+)}}}\cdot\frac{\mathcal{U}(T_{\rm I})^{\frac{m_-\cos\Delta}{(m_-+m_+)}}}{\mathcal{V}(T_{\rm I})^{\frac{m_+}{m_-+m_+}}}=\exp\left(-\Theta\left(\kappa_2\sqrt{\kappa_1\kappa_2}\right)\right)=\Theta(1).
\end{align*}
A straight-forward calculation gives us:
\begin{align*}
    &\frac{\mathcal{V}(T_{\rm II}^*)^{\frac{m_+}{m_-+m_+}}}{\mathcal{U}(T_{\rm II}^*)^{\frac{m_-\cos\Delta}{(m_-+m_+)}}}\cdot\frac{\mathcal{U}(T_{\rm I})^{\frac{m_-\cos\Delta}{(m_-+m_+)}}}{\mathcal{V}(T_{\rm I})^{\frac{m_+}{m_-+m_+}}}
    \\=&\Theta\left(\frac{1}{\left(\frac{p^{\frac{1}{1+\cos\Delta}}}{\kappa_2^2}+\Delta^2(T_{\rm II}^*-\tau_1)\right)^{\frac{m_+-m_-\cos\Delta}{m_++m_-}}}\cdot\frac{(\kappa_2^2)^{\frac{m_-\cos\Delta}{m_++m_-}}}{\left(\frac{\kappa_2^2}{p}\right)^{\frac{m_+}{m_++m_-}}}\right)
    \\=&\Theta\left(\frac{p^{\frac{m_+}{m_++m_-}}}{\left(p^{\frac{1}{1+\cos\Delta}}+\kappa_2^2\Delta^2(T_{\rm II}^*-\tau_1)\right)^{\frac{m_+-m_-\cos\Delta}{m_++m_-}}}\right).
\end{align*}
Hence, we get
\[
    T_{\rm II}^*-\tau_1=\Theta\left(\frac{1}{\kappa_2^2\Delta^2}\left(p^{\frac{m_+}{m_+-m_-\cos\Delta}}-\Theta\left(p^{\frac{1}{1+\cos\Delta}}\right)\right)\right)=\Theta\left(\frac{p^{\frac{1}{1-\alpha\cos\Delta}}}{\kappa_2^2\Delta^2}\right),
\]
Combining Lemma \ref{lemma: GF Phase II U V dynamics} (S3), we obtain
\[
    T_{\rm II}^+=T_{\rm II}^*=\tau_1+\Theta\left(\frac{p^{\frac{1}{1-\alpha\cos\Delta}}}{\kappa_2^2\Delta^2}\right)=\mathcal{O}\left(\frac{p^{\frac{1}{1+\cos\Delta}}\log(1/\Delta)}{\kappa_2^2}\right)+\Theta\left(\frac{p^{\frac{1}{1-\alpha\cos\Delta}}}{\kappa_2^2\Delta^2}\right)=\Theta\left(\frac{p^{\frac{1}{1-\alpha\cos\Delta}}}{\kappa_2^2\Delta^2}\right).
\] 

Recall the relationship between $T_{\rm II}$ and $T_{\rm II}^+$~\eqref{equ: GF Phase II hitting time}\eqref{equ: GF Phase II hitting time +}:
\begin{align*}
    T_{\rm II}=T_{\rm II}^+\wedge\inf\{t>T_{\rm I}:\exists k\in\cK_-,\text{ s.t. }\<\bw_k(t),\bx_+\>\ne0\text{ or }\<\bw_k(t),\bx_-\>\leq0\}.
\end{align*}
Then using Lemma~\ref{lemma: GF Phase II neuron dynamics} (S2), we obtain:
\begin{align*}
    T_{\rm II}=T_{\rm II}^+=\Theta\left(\frac{p^{\frac{1}{1-\alpha\cos\Delta}}}{\kappa_2^2\Delta^2}\right).
\end{align*}
\end{proof}

\begin{lemma}[Length of Plateau]\label{lemma: Plateau Estimate}
\ \\
If we define the hitting time 
$T_{\rm plat}:=\inf\Big\{t\in[T_{\rm I},T_{\rm II}]:{\rm Acc}(t)=1\Big\}$, then we have:

{\rm\bf (S1).} $T_{\rm plat}=\Theta\left(\frac{p}{\kappa_2^2\Delta^2}\right)$.

{\rm\bf (S2).} $\forall t\in[T_{\rm I},T_{\rm plat}]$, ${\rm Acc}(t)\equiv\frac{p}{1+p}$.

{\rm\bf (S3).} $\forall t\in(T_{\rm plat},T_{\rm II}]$, ${\rm Acc}(t)\equiv1$.
    
\end{lemma}

\begin{proof}[Proof of Lemma \ref{lemma: Plateau Estimate}]\ \\
It is easy to verify
\[
T_{\rm plat}=\inf\Big\{t\in[T_{\rm I},T_{\rm II}]:f_+(t)\leq0\text{ or }f_-(t)>0\Big\}.
\]
From Theorem \ref{thm: restatement of GF Phase I} (S4), we know $f_+(T_{\rm I})>0$ and $f_-(T_{\rm I})>0$.
From Lemma \ref{lemma: GF Phase II 2-order f dynamics}, we have
\begin{gather*}
    \mathcal{U}(t) =\kappa_2^2\frac{m_+}{m}\frac{p}{1+p}e^{-f_+(t)},
    \quad
    \mathcal{V}(t)=\kappa_2^2\frac{m_+}{m}\frac{1}{1+p}e^{f_-(t)}.
\end{gather*}
From the proof of Lemma \ref{lemma: GF Phase II U V dynamics}, we know $\frac{\mathrm{d}}{\mathrm{d}t}\mathcal{U}(t)<0,\ \forall t\in[T_{\rm I},T_{\rm II}]$, so 
\[
\mathcal{U}(t)\leq\mathcal{U}(T_{\rm I}),
\quad f_+(t)\geq f_+(T_{\rm I})>0,\ \forall t\in[T_{\rm I},T_{\rm II}].
\]
Recall the definition of $\tau_1$ and $\tau_2$ in Lemma \ref{lemma: GF Phase II U V dynamics}. From the proof of Lemma \ref{lemma: GF Phase II U V dynamics}, we know $\frac{\mathrm{d}}{\mathrm{d}t}\mathcal{V}(t)>0,\ \forall t\in[T_{\rm I},\tau_1)$, so 
\[
\mathcal{V}(t)\geq\mathcal{V}(T_{\rm I})
\quad,
f_-(t)\geq f_-(T_{\rm I})>0,\ \forall t\in[T_{\rm I},\tau_1].
\]
With the help of Lemma \ref{lemma: GF Phase II U V dynamics} (S4), we know 
\[
\mathcal{V}(t)=\Theta\left(\frac{1}{\frac{p^{\frac{1}{1+\cos\Delta}}}{\kappa_2^2}+\Delta^2(t-\tau_1)}\right),\ \forall t\in(\tau_1,+\infty).
\]
Because $\mathcal{V}(T_{\rm plat})=\kappa_2^2\frac{m_+}{m}\frac{1}{1+p}=\Theta\left(\frac{\kappa_2^2}{p}\right)$, we have
\[
\Theta\left(\frac{\kappa_2^2}{p}\right)=\Theta\left(\frac{1}{\frac{p^{\frac{1}{1+\cos\Delta}}}{\kappa_2^2}+\Delta^2(T_{\rm plat}-\tau_1)}\right),\ \forall t\in(\tau_1,+\infty).
\]
Therefore,
\begin{align*}
    &T_{\rm plat}=\tau_1+\Theta\left(\frac{p}{\kappa_2^2\Delta^2}\left(1-\Theta\left(\frac{1}{p^{\frac{\cos\Delta}{1+\cos\Delta}}}\right)\right)\right)
    =\tau_1+\Theta\left(\frac{p}{\kappa_2^2\Delta^2}\right)
    \\=&\mathcal{O}\left(\frac{p^{\frac{1}{1+\cos\Delta}}\log(1/\Delta)}{\kappa_2^2}\right)+\Theta\left(\frac{p}{\kappa_2^2\Delta^2}\right)=\Theta\left(\frac{p}{\kappa_2^2\Delta^2}\right).
\end{align*}
It is easy to verify $T_{\rm plat}=\Theta\left(\frac{p}{\kappa_2^2\Delta^2}\right)<T_{\rm II}=\Theta\left(\frac{p^{\frac{1}{1-\alpha\cos\Delta}}}{\kappa_2^2\Delta^2}\right)$.

Because $\tau_2-\tau_1=\Theta\left(\frac{p^{\frac{1}{1+\cos\Delta}}\log(1/\Delta)}{\kappa_2^2}\right)$, for any $t\in(\tau_1,\tau_2]$, we have
\begin{align*}
    &\mathcal{V}(t)=\Theta\left(\frac{1}{\frac{p^{\frac{1}{1+\cos\Delta}}}{\kappa_2^2}+\Delta^2(t-\tau_1)}\right)=\Theta\left(\frac{1}{\frac{p^{\frac{1}{1+\cos\Delta}}}{\kappa_2^2}+\Delta^2(\tau_2-\tau_1)}\right)
    \\=&\Theta\left(\frac{\kappa_2^2}{p^{\frac{1}{1+\cos\Delta}}}\right)\gg\Theta\left(\frac{\kappa_2^2}{p}\right)=\mathcal{V}(T_{\rm plat}).
\end{align*}
Thus,
\[
f_-(t)> f_-(T_{\rm plat})=0,\ \forall t\in(\tau_1,\tau_2].
\]
From Lemma \ref{lemma: GF Phase II U V dynamics} (S6), we know that $\mathcal{U}(t)\cos\Delta-\mathcal{V}(t)<0$, $\forall t\in(\tau_2,T_{\rm II}]$. Then $\frac{\mathrm{d}\mathcal{V}(t)}{\mathrm{d}t}<0$, $\forall t\in(\tau_2,T_{\rm II}]$. Thus,
\begin{align*}
    &f_-(t)>f_-(T_{\rm plat})=0,\ \forall t\in(\tau_2,T_{\rm plat});
    \\
    &f_-(t)<f_-(T_{\rm plat})=0,\ \forall t\in(T_{\rm plat},T_{\rm II}].
\end{align*}
Hence, we know
\begin{align*}
    &f_-(t)\geq0,\ \forall t\in[T_{\rm I},T_{\rm plat}];
    \\
    &f_-(t)<0,\ \forall t\in(T_{\rm plat},T_{\rm II}].
\end{align*}
In summary, we have proved (S1)(S2)(S3).

\end{proof}

\begin{lemma}[Prediction at end of Phase II]\label{lemma: prediction at end of Phase II}\ \\
(S1) For the predictions, we have:
\begin{gather*}
e^{-f_+(T_{\rm II})}=\Theta\left(p^{-\frac{1}{1-\alpha\cos\Delta}}\right)
,\quad e^{f_-(T_{\rm II})}=\Theta\left(p^{-\frac{\alpha\cos\Delta}{1-\alpha\cos\Delta}}\right),
\quad\cL(\btheta(T_{\rm II}))=\Theta\left(p^{-\frac{1}{1-\alpha\cos\Delta}}\right);
\\
\frac{pe^{-f_+(T_{\rm II})}}{1+p}-\frac{e^{f_-(T_{\rm II})}}{1+p}\cos\Delta=\Theta\left(\Delta^2p^{-\frac{1}{1-\alpha\cos\Delta}}\right),
\\
\frac{pe^{-f_+(T_{\rm II})}}{1+p}\cos\Delta-\frac{e^{f_-(T_{\rm II})}}{1+p}=-\Theta\left(\Delta^2p^{-\frac{1}{1-\alpha\cos\Delta}}\right).
\end{gather*}
(S2). For any $k\in\cK_+$, we have:
\[
\left<\boldsymbol{b}_k(T_{\rm II}),\boldsymbol{x}_-\right>=\mathcal{O}\left(\frac{\sqrt{\kappa_1\kappa_2}}{\sqrt{m}}\right).
\]
(S3).
For any $k\in\cK_-$, we have
$\left<\boldsymbol{b}_k(T_{\rm II}),\boldsymbol{x}_+\right>=0$.

\end{lemma}

\begin{proof}[Proof of Lemma \ref{lemma: prediction at end of Phase II}]\ \\
\underline{Proof of (S1).} Recall the definitions in Lemma \ref{lemma: GF Phase II 2-order f dynamics}:
\[
\begin{cases}
    \mathcal{U}(t)=\kappa_2^2\frac{m_+}{m}\frac{p}{1+p}e^{-f_+(t)},
    \\
    \mathcal{V}(t)=\kappa_2^2\frac{m_+}{m}\frac{1}{1+p}e^{f_-(t)}.
\end{cases}
\]
From Lemma \ref{lemma: GF Phase II U V dynamics} (S4) and Lemma \ref{lemma: GF Phase II hitting time estimate}, we have
 \begin{gather*}
     \mathcal{U}(t)=\Theta\left(\frac{1}{\frac{p^{\frac{1}{1+\cos\Delta}}}{\kappa_2^2}+\Delta^2(t-\tau_1)}\right),\quad
    \mathcal{V}(t)=\Theta\left(\frac{1}{\frac{p^{\frac{1}{1+\cos\Delta}}}{\kappa_2^2}+\Delta^2(t-\tau_1)}\right),
    \\
    \tau_1=\mathcal{O}\left(\frac{p^{\frac{1}{1+\cos\Delta}}\log(1/\Delta)}{\kappa_2^2}\right)
    ,\quad T_{\rm II}=\Theta\left(\frac{p^{\frac{1}{1-\alpha\cos\Delta}}}{\kappa_2^2\Delta^2}\right).
 \end{gather*}
Therefore, we obtain the estimate:
\begin{gather*}
    e^{-f_+(T_{\rm II})}=\Theta\left(\frac{1}{\kappa_2^2}\frac{1}{\mathcal{U}(T_{\rm II})}\right)
    =\Theta\left(\frac{1}{\kappa_2^2}\frac{1}{\frac{p^{\frac{1}{1+\cos\Delta}}}{\kappa_2^2}+\Delta^2\frac{p^{\frac{1}{1-\alpha\cos\Delta}}}{\kappa_2^2\Delta^2}}\right)=\Theta\left(p^{-\frac{1}{1-\alpha\cos\Delta}}\right),
    \\
    e^{f_-(T_{\rm II})}=\Theta\left(\frac{p}{\kappa_2^2}\frac{1}{\mathcal{V}(T_{\rm II})}\right)
    =\Theta\left(\frac{p}{\kappa_2^2}\frac{1}{\frac{p^{\frac{1}{1+\cos\Delta}}}{\kappa_2^2}+\Delta^2\frac{p^{\frac{1}{1-\alpha\cos\Delta}}}{\kappa_2^2\Delta^2}}\right)=\Theta\left(p^{-\frac{\alpha\cos\Delta}{1-\alpha\cos\Delta}}\right).
\end{gather*}
Moreover, Lemma \ref{lemma: GF Phase II U V dynamics} (S4)(S5)(S6) give us
\begin{gather*}
    \mathcal{U}(T_{\rm II})-\mathcal{V}(T_{\rm II})\cos\Delta
    =\Theta\left(\Delta^2\cV(T_{\rm II})\right)=\Theta\left(\kappa_2^2\Delta^2p^{-\frac{1}{1-\alpha\cos\Delta}}\right),
    \\
    \mathcal{U}(T_{\rm II})\cos\Delta-\mathcal{V}(T_{\rm II})=-\Theta\left(\Delta^2\cV(T_{\rm II})\right)=-\Theta\left(\kappa_2^2\Delta^2p^{-\frac{1}{1-\alpha\cos\Delta}}\right).
\end{gather*}
Hence, 
\begin{gather*}
    \frac{pe^{-f_+(T_{\rm II})}}{1+p}-\frac{e^{f_-(T_{\rm II})}}{1+p}\cos\Delta=\frac{1}{\kappa_2^2}\Big(\mathcal{U}(T_{\rm II})-\mathcal{V}(T_{\rm II})\cos\Delta\Big)=\Theta\left(\Delta^2p^{-\frac{1}{1-\alpha\cos\Delta}}\right),
    \\
    \frac{pe^{-f_+(T_{\rm II})}}{1+p}\cos\Delta-\frac{e^{f_-(T_{\rm II})}}{1+p}=\frac{1}{\kappa_2^2}\Big(\mathcal{U}(T_{\rm II})\cos\Delta-\mathcal{V}(T_{\rm II})\Big)=-\Theta\left(\Delta^2p^{-\frac{1}{1-\alpha\cos\Delta}}\right).
\end{gather*}
\underline{Proof of (S2).}
Denote $\cK_+^0:=\Big\{k\in\cK_+:\left<\boldsymbol{w}_k(T_{\rm II}),\boldsymbol{x}_-\right>=0\Big\}$. From the definition of $T_{\rm II}$ and the proof in Phase II, we know that $\left<\boldsymbol{w}_k(T_{\rm II}),\boldsymbol{x}_-\right>>0$ holds for any $k\in\cK_+-\cK_+^0$.

From the proof in Lemma \ref{lemma: GF Phase II hitting time transformation}, for any $k\in\cK_+$, it holds
\begin{align*}
    &\left<\boldsymbol{b}_k(t),\boldsymbol{x}_-\right>
    =\left<\boldsymbol{b}_k(T_{\rm I}),\boldsymbol{x}_-\right>+\int_{T_{\rm I}}^t\left<\frac{\mathrm{d}\boldsymbol{b}_k(s)}{\mathrm{d}s},\boldsymbol{x}_-\right>\mathrm{d}s
    \\=&\left<\boldsymbol{b}_k(T_{\rm I}),\boldsymbol{x}_-\right>+\int_{T_{\rm I}}^t\frac{\sqrt{m}}{\kappa_2m_+}\Big(\mathcal{U}(s)\cos\Delta-\mathcal{V}(s)\Big)\mathrm{d}s,\ \forall t\in[T_{\rm I},T_{\rm II}].
\end{align*}
Thus for any $k\in\cK_+-\cK_+^0$, we have
\begin{align*}
    \left<\boldsymbol{b}_k(T_{\rm II}),\boldsymbol{x}_-\right>-\left<\boldsymbol{b}_{k}(T_{\rm I}),\boldsymbol{x}_-\right>=\left<\boldsymbol{b}_{k_0}(T_{\rm II}),\boldsymbol{x}_-\right>-\left<\boldsymbol{b}_{k_0}(T_{\rm I}),\boldsymbol{x}_-\right>=-\left<\boldsymbol{b}_{k_0}(T_{\rm I}),\boldsymbol{x}_-\right>,
\end{align*}
so 
\[
\left<\boldsymbol{b}_k(T_{\rm II}),\boldsymbol{x}_-\right>=\left<\boldsymbol{b}_{k}(T_{\rm I}),\boldsymbol{x}_-\right>-\left<\boldsymbol{b}_{k_0}(T_{\rm I}),\boldsymbol{x}_-\right>.
\]
From the proof of Lemma \ref{lemma: GF Phase II hitting time estimate}, we know
\begin{align*}
    &\left<\boldsymbol{b}_k(T_{\rm I}),\boldsymbol{x}_-\right>\leq\frac{12\sqrt{\kappa_1\kappa_2}}{\sqrt{m}}\left(\frac{p\cos\Delta-1}{\sqrt{p^2+1-2p\cos\Delta}}+2\sqrt{\epsilon}\sin\Delta\right),
    \\&\left<\boldsymbol{b}_k(T_{\rm I}),\boldsymbol{x}_-\right>\geq\frac{4.66\sqrt{\kappa_1\kappa_2}}{\sqrt{m}}\left(\frac{p\cos\Delta-1}{\sqrt{p^2+1-2p\cos\Delta}}-2\sqrt{\epsilon}\sin\Delta-\epsilon\right),
\end{align*}
where $\epsilon=4.2\sqrt{\kappa_1\kappa_2}+\frac{2}{1+0.7\left(1+9.9\sqrt{\frac{\kappa_2}{\kappa_1}}\right)^{1.15}}$. This means
\[
\left<\boldsymbol{b}_k(T_{\rm I}),\boldsymbol{x}_-\right>=\Theta\left(\frac{\sqrt{\kappa_1\kappa_2}}{\sqrt{m}}\right).
\]
Hence, for any $k\in\cK_+-\cK_+^0$,
\begin{align*}
    0<\left<\boldsymbol{b}_k(T_{\rm II}),\boldsymbol{x}_-\right>=&\left|\left<\boldsymbol{b}_{k}(T_{\rm I}),\boldsymbol{x}_-\right>-\left<\boldsymbol{b}_{k_0}(T_{\rm I}),\boldsymbol{x}_-\right>\right|
    \leq\left|\left<\boldsymbol{b}_{k}(T_{\rm I}),\boldsymbol{x}_-\right>\right|+\left|\left<\boldsymbol{b}_{k_0}(T_{\rm I}),\boldsymbol{x}_-\right>\right|
    \\=&\Theta\left(\frac{\sqrt{\kappa_1\kappa_2}}{\sqrt{m}}\right)+\Theta\left(\frac{\sqrt{\kappa_1\kappa_2}}{\sqrt{m}}\right)=\Theta\left(\frac{\sqrt{\kappa_1\kappa_2}}{\sqrt{m}}\right).
\end{align*}
\underline{Proof of (S3).} Due to the dynamics of the neuron $k\in\cK_-$ in Phase II, this conclusion is clear.

\end{proof}

As simple corollaries of these lemmas, we can prove two theorems in Phase II.

\begin{proof}[Proof of Theorem~\ref{thm: GF Phase II} and~\ref{thm: Plateau Estimate}]\ \\
Theorem~\ref{thm: Plateau Estimate} is Lemma~\ref{lemma: Plateau Estimate}. Theorem~\ref{thm: GF Phase II} (S1) has been proven in Lemma~\ref{lemma: GF Phase II hitting time estimate}; Theorem~\ref{thm: GF Phase II} (S2) has been proven in Lemma~\ref{lemma: prediction at end of Phase II}.
Additionally, combining (i) $T_{\rm II}=T_{\rm II}^+$ in Lemma~\ref{lemma: GF Phase II hitting time estimate}, (ii) the transformation in Lemma~\ref{lemma: GF Phase II hitting time transformation}, and (iii) the definition of $T_{\rm II}$, we obtain Theorem~\ref{thm: GF Phase II} (S3).
 
\end{proof}

\newpage
\section{Proofs of Optimization Dynamics in Phase III}\label{appendix: proof: Phase III}

\subsection{Optimization Dynamics during Phase Transition}

Building upon Phase II, we will demonstrate that within a short time, all the living positive neurons $\cK_+$ change their activation patterns, corresponding to a ``phase transition''. 
After the phase transition, \textit{all} the living positive neurons $k\in\cK_+$ undergo deactivation for $\bx_-$, i.e., $\sgn_k^-(t)$ changes from 1 to 0, while other activation patterns remain unchanged.

Specifically, we define the hitting time

\begin{equation}\label{equ: GF Phase III hitting time}
\begin{aligned}
    T_{\rm II}^{\rm PT}:=&\inf\{t>T_{\rm II}:\forall k\in\cK_+,\sgn_k^-(t)=0\}
    \\=&\inf\big\{t>T_{\rm II}:\forall k\in\cK_+,\<\bw_k(t),\bx_-\>=0\big\},
\end{aligned}
\end{equation}

and we call $t\in(T_{\rm II},T_{\rm II}^{\rm PT}]$ ``Phase Transition'' from Phase II to Phase III.

Notice that the dynamics during phase transition is highly nonlinear with $|\cK_+|=\Theta(m)$ changes on activation partitions. 
Fortunately, we can keep the neurons of $\cK_+$ and $\cK_-$ close enough respectively in Phase I by using sufficiently small initialization $\kappa_1$. Moreover, their differences do not enlarge in Phase II. As a result, the phase transition can be completed quickly without significant changes in the vector field.


In order to analyze the dynamics of neurons and vector fields, we introduce the auxiliary hitting time:
\begin{equation}\label{equ: GF Phase III auxiliary hitting time}
\begin{aligned}
T_{\rm II}^{{\rm PT}*}:=T_{\rm II}^{{\rm PT}}\land
\inf\Big\{t>T_{\rm II}:&\left<\boldsymbol{F}_+(t),\boldsymbol{x}_+\right>\leq0\text{ or }\left<\boldsymbol{F}_+(t),\boldsymbol{x}_-\right>\geq0\Big\};
\\
\text{where}\quad
\boldsymbol{F}_+(t)=&\frac{p}{1+p}e^{-f_+(t)}\boldsymbol{x}_{+}-\frac{1}{1+p}e^{f_-(t)}\boldsymbol{x}_-.
\end{aligned}
\end{equation}
We call $T_{\rm II}\leq t\leq T_{\rm II}^{{\rm PT}*}$ ``Phase Transition*''.

\begin{lemma}[Dynamics of living neurons during Phase Transition*]\label{lemma: GF Phase III neuron dynamics}\ \\
In Phase Transition*, i.e., $t\in[T_{\rm II}, T_{\rm II}^{{\rm PT}*}]$, we have the following dynamics for each neuron $k\in\mathcal{K}_-\cup\mathcal{K}_+$.

(S1). For living negative neuron $k\in\mathcal{K}_-$, we have:
\begin{align*}
    &\boldsymbol{w}_k(t)\in\mathcal{M}_+^0\cap\mathcal{M}_-^+,
    \\
    &\frac{\mathrm{d} \boldsymbol{b}_k(t)}{\mathrm{d}t}=\frac{\kappa_2e^{f_-(t)}}{\sqrt{m}(1+p)}\Big(\boldsymbol{x}_--\boldsymbol{x}_+\cos\Delta\Big).
\end{align*}
(S2) For living positive neuron $k\in\mathcal{K}_+$, we define the hitting time:
\[
T_{{\rm II},k}^{{\rm PT}*}:=\inf\big\{t>T_{\rm II}:\left<\boldsymbol{w}_k(t),\boldsymbol{x}_-\right>=0\big\}\land
\inf\big\{t>T_{\rm II}:\left<\boldsymbol{F}_+(t),\boldsymbol{x}_+\right>\leq0\text{ or }\left<\boldsymbol{F}_+(t),\boldsymbol{x}_-\right>\geq0\big\}.
\]
Then it holds that:
\begin{align*}
&\text{(P0) } T_{\rm II}^{{\rm PT}*}=\max_{k\in\mathcal{K}_+}T_{{\rm II},k}^{{\rm PT}*};
\\
&\text{(P1) For any }t\in[T_{\rm II},T_{{\rm II},k}^{{\rm PT}*}),\text{ we have}
\\&\quad\quad\quad\quad\quad
\frac{\mathrm{d} \boldsymbol{b}_k(t)}{\mathrm{d}t}=\frac{\kappa_2}{\sqrt{m}}\boldsymbol{F}_+(t);
\\
&\text{(P2) 
If $T_{{\rm II},k}^{{\rm PT}*}<T_{\rm II}^{{\rm PT}*}$ strictly, then for any }t\in[T_{{\rm II},k}^{{\rm PT}*},T_{\rm II}^{{\rm PT}*}],\text{ we have }\boldsymbol{w}_k(t)\in\mathcal{M}_+^+\cap\mathcal{M}_-^0\text{ and }
\\
&\quad\quad\quad\quad\quad\quad\quad\quad\quad\quad\quad
\frac{\mathrm{d}\boldsymbol{b}_k(t)}{\mathrm{d}t}=\frac{\kappa_2pe^{-f_+(t)}}{\sqrt{m}(1+p)}\Big(\boldsymbol{x}_+-\boldsymbol{x}_-\cos\Delta\Big).
\\
&\text{(P3) Regardless of the relationship between $T_{{\rm II},k}^{{\rm PT}*}$ and $T_{\rm II}^{{\rm PT}*}$, for any $t\in[T_{\rm II},T_{{\rm II}}^{{\rm PT}*}]$, we have}
\\
&\quad\quad\quad\quad\quad\quad\quad\quad\quad
\<\bw_k(t),\bx_+\>>0,\quad\<\bw_k(t),\bx_-\>\geq0.
\end{align*}

\end{lemma}

\begin{proof}[Proof of Lemma \ref{lemma: GF Phase III neuron dynamics}]\ \\
\underline{Proof of (S1).} Recalling the definition of $T_{\rm II}^{{\rm PT}*}$, it holds $\left<\boldsymbol{F}_+(t),\boldsymbol{x}_+\right>>0$ for any $T_{\rm II}\leq t\leq T_{\rm II}^{{\rm PT}*}$. 
So (S1) can be proved in the same way as employed in the proof of Lemma \ref{lemma: GF Phase II neuron dynamics} (S2) and is omitted.

\underline{Proof of (S2)(P0) and (S2)(P1).} 
(S2)(P0) is obvious.
Moreover, for any $k\in\mathcal{K}_+$ and $t\in[T_{\rm II},T_{{\rm II},k}^{{\rm PT}*})$, we have $\left<\boldsymbol{w}_k(t),\boldsymbol{x}_+\right>>0$ and $\left<\boldsymbol{w}_k(t),\boldsymbol{x}_-\right>>0$ for any $T_{\rm I}\leq t\leq T_{\rm II}^*$, so (S2)(P1) can be proved in the same method as shown in the proof of Lemma \ref{lemma: GF Phase II neuron dynamics} (S1), and we have the dynamics:
\begin{align*}
    \frac{\mathrm{d} \boldsymbol{b}_k(t)}{\mathrm{d}t}=\frac{\kappa_2}{\sqrt{m}}\boldsymbol{F}_+(t),\quad t\in[T_{\rm II},T_{{\rm II},k}^{{\rm PT}*}).
\end{align*}
Additionally, recalling the definition of $T_{{\rm II},k}^{{\rm PT}*}$, we know $\<\bF_+(t),\bx_+\>>0$ holds for any $t\in[T_{\rm II},T_{{\rm II},k}^{{\rm PT}*})$.
Combining the dynamics of $\bb_k(t)$, we further have:
\begin{align*}
    \<\bb_k(t),\bx_+\>=&\<\bb_k(T_{\rm II}),\bx_+\>+\frac{\kappa_2}{\sqrt{m}}\int_{T_{\rm II}}^t\<\bF_+(s),\bx_+\>\rd s
    \\>&\<\bb_k(T_{\rm II}),\bx_+\> >0,\quad\forall t\in[T_{\rm II},T_{{\rm II},k}^{{\rm PT}*}].
\end{align*}

\underline{Proof of (S2)(P2).}
Let $k\in\mathcal{K}_+$. If $T_{{\rm II},k}^{{\rm PT}*}<T_{\rm II}^{{\rm PT}*}$, we have the following results:

\underline{Step I. $\boldsymbol{w}_k(T_{{\rm II},k}^{{\rm PT}*})\in\mathcal{M}_+^+\cap\mathcal{M}_-^0$.} 

Recalling the definition of $T_{{\rm II},k}^{{\rm PT}*}$ and $T_{\rm II}^{{\rm PT}*}$, $T_{{\rm II},k}^{{\rm PT}*}<T_{\rm II}^{{\rm PT}*}$ implies that $\<\boldsymbol{w}_k(T_{{\rm II},k}^{{\rm PT}*}),\bx_-\>=0$.

Then recalling our proof of (S2)(P1), we obtain $\<\bb_k(T_{{\rm II},k}^{{\rm PT}*}),\bx_+\>>0$.

Hence, we obtain $\boldsymbol{w}(T_{{\rm II},k}^{{\rm PT}*})\in\mathcal{M}_+^+\cap\mathcal{M}_-^0$.

\underline{Step II. Dynamics after $t=T_{{\rm II},k}^{{\rm PT}*}$.}

In this step, we will analyze the training dynamics after $\boldsymbol{w}_k(T_{{\rm II},k}^{{\rm PT}*})\in\mathcal{M}_+^+\cap\mathcal{M}_{-}^0$, i.e. $\boldsymbol{b}_k(T_{{\rm II},k}^{{\rm PT}*})\in\mathcal{P}_+^+\cap\mathcal{P}_{-}^0$.
We first analysis the vector field around the manifold $\mathcal{P}_+^+\cap\mathcal{P}_{-}^0$.
For any $\tilde{\bb}\in\cP_+^+\cap\cP_{-}^0$ and $0<\delta_0\ll1$, we know that $\cP_+^0\cap\cP_{-}^+$ separates its neighborhood $\mathcal{B}(\tilde{\bb },\delta_0)$ into two domains $\mathcal{G}_-=\{\bb \in\mathcal{B}(\tilde{\bb},\delta_0):\left<\bb ,\bx_-\right><0\}$ and $\mathcal{G}_+=\{\bb\in\mathcal{B}(\tilde{\bb},\delta_0):\left<\bb,\bx_-\right>>0\}$. Following Definition \ref{def: discontinuous system solution}, we calculate the limited vector field on $\tilde{\bb}$ from $\mathcal{G}_-$ and $\mathcal{G}_+$.

(i) The limited vector field $\bF^-$ on $\tilde{\bb}$ (from $\mathcal{G}_-$):
\begin{gather*}
\frac{\mathrm{d}\bb }{\mathrm{d} t}=\bF^-,\text{ where }
\bF^-=\frac{\kappa_2}{\sqrt{m}}\frac{p}{1+p}e^{-f_+(t)}\bx_+.
\end{gather*}

(ii) The limited vector field $\bF^+$ on $\tilde{\bb}$ (from $\mathcal{G}_+$):
\begin{gather*}
\frac{\mathrm{d}\bb}{\mathrm{d} t}=\bF^+,\text{ where }
\bF^+=\frac{\kappa_2}{\sqrt{m}}\bracket{\frac{p e^{-f_+(t)}}{1+p}\bx_+-\frac{e^{f_-(t)}}{1+p}\bx_-}.
\end{gather*}

(iii) Then we calculate the
projections of $\bF^-$ and $\bF^+$ onto $\bx_-$ (the normal to the surface $\cP_+^+\cap\cP_-^0$):
\begin{gather*}
F_N^{-}=\left<\bF^-,\bx_-\right>
=\frac{\kappa_2 p e^{-f_+(t)}}{\sqrt{m}(1+p)}\cos\Delta,
\\
F_N^{+}=\left<\bF^+,\bx_-\right>=\frac{\kappa_2 e^{f_-(t)}}{\sqrt{m}(1+p)}\cos\Delta-\frac{\kappa_2 p e^{-f_+(t)}}{\sqrt{m}(1+p)}.
\end{gather*}

We further define the hitting time to check whether $\boldsymbol{w}_k(t)\in\cM_+^+\cap\cM_-^0$ for $T_{{\rm II},k}^{{\rm PT}*}\leq t\leq T_{\rm II}^{{\rm PT}*}$. 
\begin{align*}
    \tau_{+,k}^+:=\inf\big\{t\in[T_{{\rm II},k}^{{\rm PT}*},T_{\rm II}^{{\rm PT}*}]:
    \<\bw_k(t),\bx_+\>\leq0\big\}.
\end{align*}

From the definition of $T_{\rm II}^{{\rm PT}*}$, we know that $\left<\boldsymbol{F}_+(t),\boldsymbol{x}_-\right>=\frac{p}{1+p}e^{-f_+(t)}\cos\Delta-\frac{1}{1+p}e^{f_-(t)}<0$ for any $t\in[T_{\rm II},T_{\rm II}^{{\rm PT}*}]$, which means $F_N^{+}<0$. 
And it is clear that $F_N^{-}>0$. 
Hence, the dynamics corresponds to Case (I) in Definition \ref{def: discontinuous system solution} ($F_N^{-}>0$ and $F_N^{+}<0$),
which means that $\boldsymbol{b}_k(t)$ can not leave $\mathcal{P}_-^0$ for any $t\in[T_{{\rm II},k}^{{\rm PT}*},\tau_{+,k}^+]$, and the dynamics of $\boldsymbol{b}_k$ for $t\in[T_{{\rm II},k}^{{\rm PT}*},\tau_{+,k}^+]$ satisfies:
\begin{align*}
    \frac{\mathrm{d}\bb}{\mathrm{d}t}=\alpha\bF^++(1-\alpha)\bF^-,\quad \alpha=\frac{{f}_N^-}{{f}_N^--{f}_N^+},
\end{align*}
which is
\begin{align*}
    \frac{\mathrm{d}\boldsymbol{b}_k(t)}{\mathrm{d}t}=
    \frac{\kappa_2 pe^{-f_+(t)}}{\sqrt{m}(1+p)}\Big(\boldsymbol{x}_+-\boldsymbol{x}_-\cos\Delta\Big),t\in[T_{{\rm II},k}^{{\rm PT}*},\tau_{+,k}^+].
\end{align*}

By Lemma~\ref{lemma: dynamics decomposition}, we know that the dynamics of $\bw_k(t)$ on $\cM_+^+\cap\cM_-^0$ and the dynamics of $\rho_k(t)$ are:
\begin{equation*}
   \frac{\mathrm{d}\boldsymbol{w}_k(t)}{\mathrm{d}t}=\frac{\kappa_2 pe^{-f_+(t)}}{\rho_k(t)\sqrt{m}(1+p)}\Big(\boldsymbol{x}_+-\left<\boldsymbol{w}_k,\boldsymbol{x}_+\right>\boldsymbol{w}_k-\boldsymbol{x}_-\cos\Delta\Big).
\end{equation*}
\begin{equation*}
    \frac{\mathrm{d}\rho_k(t)}{\mathrm{d}t}=\frac{\kappa_2 pe^{-f_+(t)}}{\sqrt{m}(1+p)}\left<\boldsymbol{w}_k(t),\boldsymbol{x_+}\right>.
\end{equation*}

Moreover, The dynamics above also ensures that:
\begin{align*}
    \<\bb_k(t),\bx_+\>=&\<\bb_k(T_{{\rm II},k}^{{\rm PT}*}),\bx_+\>+\int_{T_{{\rm II},k}^{{\rm PT}*}}^t\frac{\kappa_2 pe^{-f_+(s)}}{\sqrt{m}(1+p)}\sin^2\Delta\rd s
    \\>&\<\bb_k(T_{{\rm II},k}^{{\rm PT}*}),\bx_+\>>0,\quad \forall t\in[T_{{\rm II},k}^{{\rm PT}*},\tau_{+,k}^+].
\end{align*}
which means $\tau_{+,k}^+=T_{\rm II}^{{\rm PT}*}$. Hence, we have proved (S2)(P2).

\underline{Proof of (S2)(P3).}
Our proof for (S2)(P1) and (S2)(P2) imply this result directly.

\end{proof}

\begin{lemma}[Evolution of the prediction in Phase III*]\label{lemma: Phase III Evolution of the prediction}
\ \\
For any $t\in[T_{\rm II},T_{\rm II}^{{\rm PT}*}]$, we have
\begin{align*}
\frac{e^{-C_1}}{1+C_0e^{f_-(T_{\rm II})}(t-T_{\rm II})}\leq&e^{f_-(t)-f_-(T_{\rm II})}\leq \frac{1}{1+C_0e^{(f_-(T_{\rm II})-C_1)}(t-T_{\rm II})},
\\
\exp\left(-C_2(t-T_{\rm II})\right)\leq&e^{f_+(T_{\rm II})-f_+(t)}\leq1,
\end{align*}
where
\begin{align*}
C_0=\Theta\left(\frac{\kappa_2^2\Delta^2}{p}\right),
\ C_1=\mathcal{O}\left(\kappa_2\sqrt{\kappa_1\kappa_2}\right),
\ e^{f_-(T_{\rm II})}=\Theta\left(p^{-\frac{\alpha\cos\Delta}{1-\alpha\cos\Delta}}\right),
\ C_2=\Theta\left(\kappa_2^2p^{-\frac{1}{1-\alpha\cos\Delta}}\right).
\end{align*}

\end{lemma}

\begin{proof}[Proof of Lemma \ref{lemma: Phase III Evolution of the prediction}]\ \\
\underline{Step I. Preparation.}
With the help of Lemma \ref{lemma: GF Phase III neuron dynamics}(S1) and (S2)(P3), we know that 

(i) For $k\in\mathcal{K}_-$, we have
\[
\left<\boldsymbol{w}_k(t),\boldsymbol{x}_+\right>=0,\quad
\left<\boldsymbol{w}_k(t),\boldsymbol{x}_-\right>>0,\quad\forall t\in[T_{\rm II}, T_{\rm II}^{{\rm PT}*}].
\]
(ii) For $k\in\mathcal{K}_+$, we have 
\begin{align*}
&\left<\boldsymbol{w}_k(t),\boldsymbol{x}_+\right>>0,\quad
\left<\boldsymbol{w}_k(t),\boldsymbol{x}_-\right>\geq0,\quad\forall t\in[T_{\rm II},T_{\rm II}^{{\rm PT}*}];
\end{align*}
So $f_+(t)$ and $f_-(t)$ have the following representation for any $t\in[T_{\rm II}, T_{\rm II}^{{\rm PT}*}]$:
\begin{align*}
    f_+(t)&=\sum_{k\in\mathcal{K}_+}\frac{\kappa_2}{\sqrt{m}}\boldsymbol{b}_k^\top(t)\boldsymbol{x}_+,
    \\
    f_-(t)&=\sum_{k\in\mathcal{K}_+}\frac{\kappa_2}{\sqrt{m}}\boldsymbol{b}_k^\top(t)\boldsymbol{x}_--\sum_{k\in\mathcal{K}_-}\frac{\kappa_2}{\sqrt{m}}\boldsymbol{b}_k^\top(t)\boldsymbol{x}_-.
\end{align*}
\underline{Step II. Evolution of $f_-(t)$.}

To begin with, we need to do a rough estimate of $\sum_{k\in\mathcal{K}_+}\frac{\kappa_2}{\sqrt{m}}\boldsymbol{b}_k^\top(t)\boldsymbol{x}_-$. Let $k\in\mathcal{\mathcal{K}_+}$.
For any $t\in[T_{{\rm II},k}^{{\rm PT}*}, T_{\rm II}^{{\rm PT}*}]$,
we have $\left<\boldsymbol{b}_k(t),\boldsymbol{x}_-\right>=0$. And for any $t\in[T_{\rm II},T_{{\rm II},k}^{{\rm PT}*})$, we have:
\begin{align*}
    \frac{\mathrm{d}}{\mathrm{d}t}\left<\boldsymbol{b}_k(t),\boldsymbol{x}_-\right>\overset{\text{Lemma \ref{lemma: GF Phase III neuron dynamics}}}{=}\frac{\kappa_2}{\sqrt{m}}\left<\boldsymbol{F}_k(t),\boldsymbol{x}_-\right>\overset{\eqref{equ: GF Phase III auxiliary hitting time}}{<}0.
\end{align*}
Therefore, for any $t\in[T_{\rm II}, T_{\rm II}^{{\rm PT}*}]$, we have $    0\leq\left<\boldsymbol{b}_k(t),\boldsymbol{x}_-\right>\leq\left<\boldsymbol{b}_k(T_{\rm II}),\boldsymbol{x}_-\right>$, so
\begin{gather*}
    0\leq\sum_{k\in\mathcal{K}_+}\frac{\kappa_2}{\sqrt{m}}\boldsymbol{b}_k^\top(t)\boldsymbol{x}_-\leq\sum_{k\in\mathcal{K}_+}\frac{\kappa_2}{\sqrt{m}}\boldsymbol{b}_k^{\top}(T_{\rm II})\boldsymbol{x}_-,
    \\
    -\sum_{k\in\mathcal{K}_-}\frac{\kappa_2}{\sqrt{m}}\boldsymbol{b}_k^\top(t)\boldsymbol{x}_-\leq f_-(t)\leq\sum_{k\in\mathcal{K}_+}\frac{\kappa_2}{\sqrt{m}}\boldsymbol{b}_k^{\top}(T_{\rm II})\boldsymbol{x}_--\sum_{k\in\mathcal{K}_-}\frac{\kappa_2}{\sqrt{m}}\boldsymbol{b}_k^\top(t)\boldsymbol{x}_-.
\end{gather*}
According to Lemma \ref{lemma: GF Phase III neuron dynamics}, it follows that for any $k\in\mathcal{K}_-$, its dynamics is
$\frac{\mathrm{d} \boldsymbol{b}_k(t)}{\mathrm{d}t}=\frac{\kappa_2e^{f_-(t)}}{\sqrt{m}(1+p)}\Big(\boldsymbol{x}_--\boldsymbol{x}_+\cos\Delta\Big),$ thus
\begin{gather*}
    \boldsymbol{b}_k^\top(t)\boldsymbol{x}_-=\boldsymbol{b}_k^\top(T_{\rm II})\boldsymbol{x}_-+\int_{T_{\rm II}}^t\left<\frac{\mathrm{d} \boldsymbol{b}_k(s)}{\mathrm{d}s},\boldsymbol{x}_-\right>\mathrm{d}s
    =\boldsymbol{b}_k^\top(T_{\rm II})\boldsymbol{x}_-+\frac{\kappa_2\sin^2\Delta}{\sqrt{m}(1+p)}\int_{T_{\rm II}}^t
    e^{f_-(s)}\mathrm{d}s,
    \\
    \sum_{k\in\mathcal{K}_-}\frac{\kappa_2}{\sqrt{m}}\boldsymbol{b}_k^\top(t)\boldsymbol{x}_-=\sum_{k\in\mathcal{K}_-}\frac{\kappa_2}{\sqrt{m}}\boldsymbol{b}_k^\top(T_{\rm II})\boldsymbol{x}_-+\frac{m_-\kappa_2^2\sin^2\Delta}{m(1+p)}\int_{T_{\rm II}}^t
    e^{f_-(s)}\mathrm{d}s.
\end{gather*}
Therefore, we have two-side bounds of $f_-(t)$:
\begin{align*}
    &f_-(t)\leq-\frac{m_-\kappa_2^2\sin^2\Delta}{m(1+p)}\int_{T_{\rm II}}^t e^{f_-(s)}\mathrm{d}s-\sum_{k\in\mathcal{K}_-}\frac{\kappa_2}{\sqrt{m}}\boldsymbol{b}_k^\top(T_{\rm II})\boldsymbol{x}_-+\sum_{k\in\mathcal{K}_+}\frac{\kappa_2}{\sqrt{m}}\boldsymbol{b}_k^\top(T_{\rm II})\boldsymbol{x}_-,
    \\
    &
    f_-(t)\geq-\frac{m_-\kappa_2^2\sin^2\Delta}{m(1+p)}\int_{T_{\rm II}}^t e^{f_-(s)}\mathrm{d}s-\sum_{k\in\mathcal{K}_-}\frac{\kappa_2}{\sqrt{m}}\boldsymbol{b}_k^\top(T_{\rm II})\boldsymbol{x}_-.
\end{align*}
For simplicity, we denote $C_0:=\frac{m_-\kappa_2^2\sin^2\Delta}{m(1+p)}$, $C_-^-:=\sum_{k\in\mathcal{K}_-}\frac{\kappa_2}{\sqrt{m}}\boldsymbol{b}_k^\top(T_{\rm II})\boldsymbol{x}_-$ and $C_-^+:=\sum_{k\in\mathcal{K}_+}\frac{\kappa_2}{\sqrt{m}}\boldsymbol{b}_k^\top(T_{\rm II})\boldsymbol{x}_-$. Then we have:
\begin{align*}
    -C_0\int_{T_{\rm II}}^t e^{f_-(s)}\mathrm{d}s-C_-^-\leq f_-(t)\leq-C_0\int_{T_{\rm II}}^t e^{f_-(s)}\mathrm{d}s-C_-^-+C_-^+.
\end{align*}
Let $\Psi(t):=\int_{T_{\rm II}}^t e^{f_-(s)}\mathrm{d}s$, then $\frac{\mathrm{d}\Psi(t)}{\mathrm{d}t}=e^{f_-(t)}$. So $\Psi(T_{\rm II})=0$ and
\begin{gather*}
    -C_0\Psi(t)-C_-^-\leq\log\left(\frac{\mathrm{d}\Psi(t)}{\mathrm{d}t}\right)\leq-C_0\Psi(t)-C_-^-+C_-^+,
    \\
     e^{-C_-^-} e^{-C_0\Psi(t)}\leq\frac{\mathrm{d}\Psi(t)}{\mathrm{d}t}\leq e^{-C_-^-+C_-^+} e^{-C_0\Psi(t)}
\end{gather*}
For the right hand, for any $\epsilon\in(0,1)$, we consider the auxiliary ODE:
\[
\begin{cases}
    \frac{\mathrm{d}\mathcal{P}(t)}{\mathrm{d}t}=e^{-C_-^-+(1+\epsilon)C_-^+} e^{-C_0\mathcal{P}(t)},
    \\
    \mathcal{P}(T_{\rm II})=0.
\end{cases}
\]
The solution of this ODE is $\mathcal{P}(t)=\frac{1}{C_0}\log\left(1+C_0e^{-C_-^-+(1+\epsilon)C_-^+}(t-T_{\rm II})\right)$.
From the Comparison Principle of ODEs, we have the upper bound for $\Psi(t)$:
\[
\Psi(t)\leq \mathcal{P}(t)=\frac{1}{C_0}\log\left(1+C_0e^{-C_-^-+(1+\epsilon)C_-^+}(t-T_{\rm II})\right).
\]
Taking $\epsilon\to0$, we obtain
\[
\Psi(t)\leq\frac{1}{C_0}\log\left(1+C_0e^{-C_-^-+C_-^+}(t-T_{\rm II})\right).
\]
In the similar way, we can derive the lower bound for $\Psi(t)$:
\[
\Psi(t)\geq\frac{1}{C_0}\log\left(1+C_0e^{-C_-^-}(t-T_{\rm II})\right).
\]
Consequently, we infer that
\begin{align*}
    &f_-(t)\leq-C_0\Psi(t)-C_-^-+C_-^+\leq-\log\left(1+C_0e^{-C_-^-}(t-T_{\rm II})\right)-C_-^-+C_-^+,
    \\&
    f_-(t)\geq-C_0\Psi(t)-C_-^-\geq-\log\left(1+C_0e^{-C_-^-+C_-^+}(t-T_{\rm II})\right)-C_-^-.
\end{align*}
Noticing $f_-(T_{\rm II})=C_-^+-C_-^-$, we obtain
\begin{align*}
    -\log\left(1+C_0e^{-C_-^-+C_-^+}(t-T_{\rm II})\right)-C_-^+\leq f_-(t)-f_-(T_{\rm II})\leq -\log\left(1+C_0e^{-C_-^-}(t-T_{\rm II})\right).
\end{align*}
Noticing $f_-(T_{\rm II})=C_-^+-C_-^-$, this inequality means
\begin{align*}
\frac{e^{-C_-^+}}{1+C_0e^{f_-(T_{\rm II})}(t-T_{\rm II})}\leq e^{f_-(t)-f_-(T_{\rm II})}\leq \frac{1}{1+C_0e^{(f_-(T_{\rm II})-C_-^+)}(t-T_{\rm II})}.
\end{align*}
where $C_0=\frac{m_-\kappa_2^2\sin^2\Delta}{m(1+p)}=\Theta\left(\frac{\kappa_2^2\Delta^2}{p}\right)$.
Moreover, according to Lemma \ref{lemma: prediction at end of Phase II} (S1)(S2), we have
\begin{gather*}
e^{f_-(T_{\rm II})}=\Theta\left(p^{-\frac{\alpha\cos\Delta}{1-\alpha\cos\Delta}}\right),
\\
C_-^+=\sum_{k\in\mathcal{K}_+}\frac{\kappa_2}{\sqrt{m}}\boldsymbol{b}_k^\top(T_{\rm II})\boldsymbol{x}_-=\mathcal{O}\left(m_-\frac{\kappa_2}{\sqrt{m}}\frac{\sqrt{\kappa_1\kappa_2}}{\sqrt{m}}\right)=\mathcal{O}\left(\kappa_2\sqrt{\kappa_1\kappa_2}\right).
\end{gather*}
\underline{Step III. Evolution of $f_+(t)$.}

Let $k\in\mathcal{K}_+$.
According to Lemma \ref{lemma: GF Phase III neuron dynamics} (S2)(P2) and (S2)(P3), it follows that for any $k\in\mathcal{K}_+$, its dynamics during $t\in[T_{\rm II},T_{\rm II}^{{\rm PT}*}]$ is
\begin{gather*}
\frac{\mathrm{d} \boldsymbol{b}_k(t)}{\mathrm{d}t}=\frac{\kappa_2}{\sqrt{m}}\left<\boldsymbol{F}_+(t),\boldsymbol{x}_+\right>;
\\\text{ or }
\frac{\mathrm{d}\boldsymbol{b}_k(t)}{\mathrm{d}t}=\frac{\kappa_2pe^{-f_+(t)}}{\sqrt{m}(1+p)}\Big(\boldsymbol{x}_+-\boldsymbol{x}_-\cos\Delta\Big).
\end{gather*}
Notice that
\begin{align*}
f_+(t)=&\sum_{k\in\mathcal{K}_+}\frac{\kappa_2}{\sqrt{m}}\boldsymbol{b}_k^\top(t)\boldsymbol{x}_+
=\sum_{k\in\mathcal{K}_+}\frac{\kappa_2}{\sqrt{m}}\boldsymbol{b}_k^\top(T_{\rm II})\boldsymbol{x}_+
+\sum_{k\in\mathcal{K}_+}\int_{T_{\rm II}}^t\frac{\kappa_2}{\sqrt{m}}\left<\frac{\mathrm{d}\boldsymbol{b}_k(t)}{\mathrm{d}s},\boldsymbol{x}_+\right>\mathrm{d}s
\\=&f_+(T_{\rm II})+\frac{\kappa_2}{\sqrt{m}}\sum_{k\in\mathcal{K}_+}\int_{T_{\rm II}}^t\left<\frac{\mathrm{d}\boldsymbol{b}_k(t)}{\mathrm{d}s},\boldsymbol{x}_+\right>\mathrm{d}s.
\end{align*}
On the one hand, for any $t\in[T_{\rm II},T_{\rm II}^{{\rm PT}*}]$, we have the lower bound:
\begin{align*}
    f_+(t)\geq f_+(T_{\rm II})+\frac{\kappa_2^2}{m}\sum_{k\in\mathcal{K}_+}\int_{T_{\rm II}}^t\min\Big\{\left<\boldsymbol{F}_+(s),\boldsymbol{x}_+\right>,\frac{pe^{-f_+(s)}}{1+p}\sin^2\Delta\Big\}\mathrm{d}s\geq f_+(T_{\rm II}).
\end{align*}
On the other hand, for any $t\in[T_{\rm II},T_{\rm II}^{{\rm PT}*}]$, we can derive an upper bound:
\begin{align*}
    f_+(t)\leq&f_+(T_{\rm II})+\frac{\kappa_2^2}{m}\sum_{k\in\mathcal{K}_+}\int_{T_{\rm II}}^t\max\Big\{\left<\boldsymbol{F}_+(s),\boldsymbol{x}_+\right>,\frac{pe^{-f_+(s)}}{1+p}\sin^2\Delta\Big\}\mathrm{d}s
    \\\leq&f_+(T_{\rm II})+\frac{\kappa_2^2}{m}\sum_{k\in\mathcal{K}_+}\int_{T_{\rm II}}^t\max\Big\{\frac{pe^{-f_+(s)}}{1+p},\frac{pe^{-f_+(s)}}{1+p}\sin^2\Delta\Big\}\mathrm{d}s
    \\\leq&f_+(T_{\rm II})+\frac{\kappa_2^2}{m}\sum_{k\in\mathcal{K}_+}\int_{T_{\rm II}}^t\frac{pe^{-f_+(s)}}{1+p}\mathrm{d}s
    \leq f_+(T_{\rm II})+\frac{\kappa_2^2m_+}{m}\frac{pe^{-f_+(T_{\rm II})}}{1+p}(t-T_{\rm II}).
\end{align*}
Hence, we obtain
\[
\exp\left(-\frac{\kappa_2^2m_+}{m}\frac{pe^{-f_+(T_{\rm II})}}{1+p}(t-T_{\rm II})\right)\leq e^{f_+(T_{\rm II})-f_+(t)}\leq1,
\]
where
\[
\frac{\kappa_2^2m_+}{m}\frac{pe^{-f_+(T_{\rm II})}}{1+p}\overset{\text{Lemma \ref{lemma: prediction at end of Phase II}}}{=}\Theta\left(\kappa_2^2e^{-f_+(T_{\rm II})}\right)
=\Theta\left(\kappa_2^2p^{-\frac{1}{1-\alpha\cos\Delta}}\right).
\]

\end{proof}

\begin{lemma}[Nearly fixed vector filed in Phase III*]\label{lemma: GF Phase III Nearly fixed vector filed}\ \\
There exist absolute constants $Q_1,Q_2>0$, such that: For any time $T_{\rm fix}\in[T_{\rm II},+\infty)$, if we choose $\kappa_1,\kappa_2$ s.t.
\begin{align*}
    \kappa_2^2\Big(T_{\rm fix}\land T_{\rm II}^{{\rm PT}*}-T_{\rm II}\Big) p^{-\frac{1}{1-\alpha\cos\Delta}}=\mathcal{O}(\Delta^2),\quad\kappa_2^2\sqrt{\frac{\kappa_1}{\kappa_2}}=\mathcal{O}(\Delta^2),
\end{align*}
then for any $t\in[T_{\rm II},T_{\rm fix}\land T_{\rm II}^{{\rm PT}*}]$, we have
\begin{gather*}
    \left<\boldsymbol{F}_+(t),\boldsymbol{x}_+\right>\leq\frac{Q_1}{2}\Delta^2p^{-\frac{1}{1-\alpha\cos\Delta}},
    \quad
    \left<\boldsymbol{F}_+(t),\boldsymbol{x}_-\right>\geq-\frac{Q_2}{2}\Delta^2p^{-\frac{1}{1-\alpha\cos\Delta}}.
\end{gather*}

\end{lemma}

\begin{proof}[Proof of Lemma \ref{lemma: GF Phase III Nearly fixed vector filed}]\ \\
For simplicity, we denote $\delta_T:=T_{\rm fix}\land T_{\rm II}^{{\rm PT}*}-T_{\rm II}$
From Lemma \ref{lemma: Phase III Evolution of the prediction}, for any $t\in[T_{\rm II},T_{\rm fix}\land T_{\rm II}^{{\rm PT}*}]$, we have
\begin{align*}
&e^{f_-(t)-f_-(T_{\rm II})}-1\leq\frac{1}{1+C_0e^{(f_-(T_{\rm II})-C_1)}\delta_T}-1\leq0,
\\
&e^{f_-(t)-f_-(T_{\rm II})}-1\geq\frac{e^{-C_1}}{1+C_0e^{f_-(T_{\rm II})}\delta_T}-1=\frac{e^{-C_1}-1-C_0e^{f_-(T_{\rm II})}\delta_T}{1+C_0e^{f_-(T_{\rm II})}\delta_T}\geq-\frac{C_1+C_0e^{f_-(T_{\rm II})}\delta_T}{1+C_0e^{f_-(T_{\rm II})}\delta_T}
\\
&e^{f_+(T_{\rm II})-f_+(t)}-1\leq0,
\\
&e^{f_+(T_{\rm II})-f_+(t)}-1\geq e^{-C_2\delta_T}-1\geq-C_2\delta_T.
\end{align*}
Recalling Lemma \ref{lemma: prediction at end of Phase II} (S1), there exists absolute constants $Q_1,Q_2>0$ such that
\begin{gather*}
\left<\boldsymbol{F}_+(T_{\rm II}),\boldsymbol{x}_+\right>=\frac{pe^{-f_+(T_{\rm II})}}{1+p}-\frac{e^{f_-(T_{\rm II})}}{1+p}\cos\Delta\geq Q_1\Delta^2p^{-\frac{1}{1-\alpha\cos\Delta}},
\\
\left<\boldsymbol{F}_+(T_{\rm II}),\boldsymbol{x}_-\right>=\frac{pe^{-f_+(T_{\rm II})}}{1+p}\cos\Delta-\frac{e^{f_-(T_{\rm II})}}{1+p}\leq-Q_2\Delta^2p^{-\frac{1}{1-\alpha\cos\Delta}}.
\end{gather*}
\underline{Step I. Bounding the term $\left<\boldsymbol{F}_+(t),\boldsymbol{x}_+\right>$.} 
\begin{align*}
    &\left|\left<\boldsymbol{F}_+(t),\boldsymbol{x}_+\right>-\left<\boldsymbol{F}_+(T_{\rm II}),\boldsymbol{x}_+\right>\right|
    \\=&\left|\frac{p}{1+p}e^{-f_+(t)}-\frac{p}{1+p}e^{-f_+(T_{\rm II})}-\frac{e^{f_-(t)}}{1+p}\cos\Delta+\frac{e^{f_-(T_{\rm II})}}{1+p}\cos\Delta\right|
    \\\leq&\left|\frac{p}{1+p}e^{-f_+(t)}-\frac{p}{1+p}e^{-f_+(T_{\rm II})}\right|+\left|\frac{e^{f_-(t)}}{1+p}\cos\Delta-\frac{e^{f_-(T_{\rm II})}}{1+p}\cos\Delta\right|
    \\\leq&\frac{p}{1+p}e^{-f_+(T_{\rm II})}\left|e^{f_+(T_{\rm II})-f_+(t)}-1\right|+\frac{e^{f_-(T_{\rm II})}}{1+p}\left|e^{f_-(t)-f_-(T_{\rm II})}-1\right|
    \\\leq&\frac{p}{1+p}e^{-f_+(T_{\rm II})}C_2\delta_T+\frac{e^{f_-(T_{\rm II})}}{1+p}\frac{C_1+C_0e^{f_-(T_{\rm II})}\delta_T}{1+C_0e^{f_-(T_{\rm II})}\delta_T}
\end{align*}
To ensure $\left|\left<\boldsymbol{F}_+(t),\boldsymbol{x}_+\right>-\left<\boldsymbol{F}_+(T_{\rm II}),\boldsymbol{x}_+\right>\right|\leq\frac{1}{2}Q_1\Delta^2p^{-\frac{1}{1-\alpha\cos\Delta}}$, we need only select parameters such that
\[
\frac{p}{1+p}e^{-f_+(T_{\rm II})}C_2\delta_T+\frac{e^{f_-(T_{\rm II})}}{1+p}\frac{C_1+C_0e^{f_-(T_{\rm II})}\delta_T}{1+C_0e^{f_-(T_{\rm II})}\delta_T}\leq\frac{1}{2}Q_1\Delta^2p^{-\frac{1}{1-\alpha\cos\Delta}}.
\]
From Lemma \ref{lemma: Phase III Evolution of the prediction} and Lemma \ref{lemma: prediction at end of Phase II}, we have:
\begin{gather*}
C_0=\Theta\left(\frac{\kappa_2^2\Delta^2}{p}\right),
\ C_1=\mathcal{O}\left(\kappa_2\sqrt{\kappa_1\kappa_2}\right),
\ C_2=\Theta\left(\kappa_2^2p^{-\frac{1}{1-\alpha\cos\Delta}}\right),
\\
e^{-f_+(T_{\rm II})}=\Theta\left(p^{-\frac{1}{1-\alpha\cos\Delta}}\right)
,\quad e^{f_-(T_{\rm II})}=\Theta\left(p^{-\frac{\alpha\cos\Delta}{1-\alpha\cos\Delta}}\right).
\end{gather*}
Therefore, if we take
\[
C_0e^{f_-(T_{\rm II})}\delta_T=\Theta\left(\kappa_2^2\Delta^2p^{-\frac{1}{1-\alpha\cos\Delta}}\delta_T\right)=\mathcal{O}(1),
\]
then we have
\begin{align*}
    &\frac{p}{1+p}e^{-f_+(T_{\rm II})}C_2\delta_T+\frac{e^{f_-(T_{\rm II})}}{1+p}\frac{C_1+C_0e^{f_-(T_{\rm II})}\delta_T}{1+C_0e^{f_-(T_{\rm II})}\delta_T}
    \\=&\Theta\left(\kappa_2^2p^{-\frac{2}{1-\alpha\cos\Delta}}\delta_T\right)+\Theta\left(p^{-\frac{1}{1-\alpha\cos\Delta}}\left(\mathcal{O}(\kappa_2\sqrt{\kappa_1\kappa_2})+\kappa_2^2\Delta^2 p^{-\frac{1}{1-\alpha\cos\Delta}}\delta_T\right)\right)
    \\=&\Theta\left(\kappa_2^2p^{-\frac{1}{1-\alpha\cos\Delta}}\left(\delta_T p^{-\frac{1}{1-\alpha\cos\Delta}}+\mathcal{O}(\sqrt{\frac{\kappa_1}{\kappa_2}})\right)\right).
\end{align*}
If we can take
\begin{align*}
    \kappa_2^2\delta_T p^{-\frac{1}{1-\alpha\cos\Delta}}=\mathcal{O}(\Delta^2),\quad\kappa_2^2\sqrt{\frac{\kappa_1}{\kappa_2}}=\mathcal{O}(\Delta^2),
\end{align*}
then $\kappa_2^2\Delta^2p^{-\frac{1}{1-\alpha\cos\Delta}}\delta_T=\mathcal{O}(1)$ and 
\[\left|\left<\boldsymbol{F}_+(t),\boldsymbol{x}_+\right>-\left<\boldsymbol{F}_+(T_{\rm II}),\boldsymbol{x}_+\right>\right|=\mathcal{O}\left(\Delta^2p^{-\frac{1}{1-\alpha\cos\Delta}}\right)\leq\frac{1}{2}Q_1\Delta^2p^{-\frac{1}{1-\alpha\cos\Delta}},
\]
Hence,
\begin{align*}
\left<\boldsymbol{F}_+(t),\boldsymbol{x}_+\right>
\geq&\left<\boldsymbol{F}_+(T_{\rm II}),\boldsymbol{x}_+\right>-
\left|\left<\boldsymbol{F}_+(t),\boldsymbol{x}_+\right>-\left<\boldsymbol{F}_+(T_{\rm II}),\boldsymbol{x}_+\right>\right|
\\\geq&\frac{1}{2}Q_1\Delta^2p^{-\frac{1}{1-\alpha\cos\Delta}}=\Omega\left(\Delta^2p^{-\frac{1}{1-\alpha\cos\Delta}}\right),\quad\forall t\in[T_{\rm II},T_{\rm fix}\land T_{\rm II}^{{\rm PT}*}].
\end{align*}
\underline{Step II. Bounding the term $\left<\boldsymbol{F}_+(t),\boldsymbol{x}_-\right>$.} 

The proof can be completed by the method analogous to that used in Step I, and we omit it. The result is
\begin{align*}
\left<\boldsymbol{F}_+(t),\boldsymbol{x}_-\right>
\geq&\left<\boldsymbol{F}_+(T_{\rm II}),\boldsymbol{x}_-\right>+
\left|\left<\boldsymbol{F}_+(t),\boldsymbol{x}_+\right>-\left<\boldsymbol{F}_+(T_{\rm II}),\boldsymbol{x}_-\right>\right|
\\\geq&-\frac{1}{2}Q_2\Delta^2p^{-\frac{1}{1-\alpha\cos\Delta}}\equiv-\Omega\left(\Delta^2p^{-\frac{1}{1-\alpha\cos\Delta}}\right),\quad\forall t\in[T_{\rm II},T_{\rm fix}\land T_{\rm II}^{{\rm PT}*}].
\end{align*}

\end{proof}

\begin{lemma}[The end of Phase Transition]\label{lemma: GF Time Estimate of Phase III}\ \\
If we choose $\kappa_1,\kappa_2$ s.t $\kappa_2=\mathcal{O}(1)$ and $\sqrt{\frac{\kappa_1}{\kappa_2}}=\mathcal{O}\left(\Delta^4\right)$ \eqref{equ: parameter selection GF},
then it holds that

(S1) (Time).
\begin{align*}
&T_{\rm II}^{{\rm PT}}=
T_{\rm II}^{{\rm PT}*}=T_{\rm II}+\mathcal{O}\left(\sqrt{\frac{\kappa_1}{\kappa_2}}\frac{p^{\frac{1}{1-\alpha\cos\Delta}}}{\Delta^2}\right)=\bracket{1+\cO\bracket{\sqrt{\kappa_1\kappa_2^3}}}T_{\rm II};
\end{align*}
(S2) (Prediction).
\begin{gather*}
e^{-f_+(T_{\rm II}^{{\rm PT}})}=\Theta\left(p^{-\frac{1}{1-\alpha\cos\Delta}}\right)
,\quad e^{f_-(T_{\rm II}^{{\rm PT}})}=\Theta\left(p^{-\frac{\alpha\cos\Delta}{1-\alpha\cos\Delta}}\right);
\\
\frac{pe^{-f_+(T_{\rm II}^{{\rm PT}})}}{1+p}-\frac{e^{f_-(T_{\rm II}^{{\rm PT}})}}{1+p}\cos\Delta=\Theta\left(\Delta^2p^{-\frac{1}{1-\alpha\cos\Delta}}\right),
\\
\frac{pe^{-f_+(T_{\rm II}^{{\rm PT}})}}{1+p}\cos\Delta-\frac{e^{f_-(T_{\rm II}^{{\rm PT}})}}{1+p}=-\Theta\left(\Delta^2p^{-\frac{1}{1-\alpha\cos\Delta}}\right).
\end{gather*}

(S3) (Activation patterns).
\begin{gather*}
    \<\bw_k(T_{\rm II}^{{\rm PT}}),\bx_+\>>0,\ \<\bw_k(T_{\rm II}^{{\rm PT}}),\bx_-\>=0,\ \forall k\in\cK_+;
    \\
    \<\bw_k(T_{\rm II}^{{\rm PT}}),\bx_+\>=0,\ \<\bw_k(T_{\rm II}^{{\rm PT}}),\bx_-\>>0,\ \forall k\in\cK_-.
\end{gather*}
    
\end{lemma}

\begin{proof}[Proof of Lemma \ref{lemma: GF Time Estimate of Phase III}]\ \\
\underline{Step I. Time Estimate.} Let $k\in\mathcal{K}_+$.

Recalling the definition of $T_{{\rm II},k}^{{\rm PT}*}$ in Lemma \ref{lemma: GF Phase III neuron dynamics},
Lemma \ref{lemma: GF Phase III neuron dynamics} (S2)(P0) also gives us
\[
T_{\rm II}^{{\rm PT}*}=\max_{k\in\mathcal{K}_+}T_{{\rm II},k}^{{\rm PT}*}.
\]
From Lemma \ref{lemma: prediction at end of Phase II} (S2), we know that there exists an absolute constant $Q_3>0$, s.t. $0\leq\left<\boldsymbol{b}_k(T_{\rm II}),\boldsymbol{x}_-\right>\leq Q_3\frac{\sqrt{\kappa_1\kappa_2}}{\sqrt{m}}$. And we let $Q_2>0$ be the absolute constant $Q_2$ in Lemma \ref{lemma: GF Phase III Nearly fixed vector filed}.

First, we choose the time
\[
T_{\rm fix}=T_{\rm II}+\frac{3Q_3}{Q_2\Delta^2}\sqrt{\frac{\kappa_1}{\kappa_2}}p^{\frac{1}{1-\alpha\cos\Delta}}.
\]
then we choose $\kappa_1,\kappa_2$ s.t.
\begin{align*}
    \kappa_2=\mathcal{O}(1),\quad\sqrt{\frac{\kappa_1}{\kappa_2}}=\mathcal{O}\left(\Delta^4\right).
\end{align*}
It can ensure
\begin{align*}
    \kappa_2^2\Big(T_{\rm fix}\land T_{\rm II}^{{\rm PT}*}-T_{\rm II}\Big) p^{-\frac{1}{1-\alpha\cos\Delta}}=\Theta\left(\frac{\kappa_2^2}{\Delta^2}\sqrt{\frac{\kappa_1}{\kappa_2}}\right)=\mathcal{O}(\Delta^2),\quad\kappa_2^2\sqrt{\frac{\kappa_1}{\kappa_2}}=\mathcal{O}(\Delta^2).
\end{align*}
Then according to Lemma \ref{lemma: GF Phase III Nearly fixed vector filed}, 
it follows that
\[\left<\boldsymbol{F}_+(t),\boldsymbol{x}_-\right>\leq-\frac{Q_2}{2}\Delta^2p^{-\frac{1}{1-\alpha\cos\Delta}},\quad\forall t\in[T_{\rm II},T_{{\rm II},k}^{{\rm PT}*}\land T_{\rm fix}).
\]
Now we consider the dynamics for $t\in[T_{\rm II},T_{{\rm II},k}^{{\rm PT}*}\land T_{\rm fix})$.

Recalling lemma \ref{lemma: GF Phase III neuron dynamics}, we have
\begin{align*}
    &\left<\boldsymbol{b}_k(T_{{\rm II},k}^{{\rm PT}*}\land T_{\rm fix}),\boldsymbol{x}_-\right>=\left<\boldsymbol{b}_k(T_{\rm II}),\boldsymbol{x}_-\right>+\int_{T_{\rm II}}^t\left<\frac{\mathrm{d}\boldsymbol{b}_k(s)}{\mathrm{d}s},\boldsymbol{x}_+\right>\mathrm{d}s
    \\=&\left<\boldsymbol{b}_k(T_{\rm II}),\boldsymbol{x}_-\right>+\frac{\kappa_2}{\sqrt{m}}\int_{T_{\rm II}}^{T_{{\rm II},k}^{{\rm PT}*}\land T_{\rm fix}}\left<\boldsymbol{F}_+(s),\boldsymbol{x}_-\right>\mathrm{d}s
    \\\leq&Q_3\frac{\sqrt{\kappa_1\kappa_2}}{\sqrt{m}}-\frac{Q_2}{2}\Delta^2p^{-\frac{1}{1-\alpha\cos\Delta}}\Big(T_{{\rm II},k}^{{\rm PT}*}\land T_{\rm fix}-T_{\rm II}\Big)
    \\\leq&Q_3\frac{\sqrt{\kappa_1\kappa_2}}{\sqrt{m}}-\frac{Q_2}{2}\Delta^2p^{-\frac{1}{1-\alpha\cos\Delta}}\Big((T_{{\rm III},k}-T_{\rm II})\land\frac{3Q_3}{Q_2\Delta^2}\sqrt{\frac{\kappa_1}{\kappa_2}}p^{\frac{1}{1-\alpha\cos\Delta}}\Big).
\end{align*}
We claim $T_{{\rm II},k}^{{\rm PT}*}-T_{\rm II}\leq\frac{2Q_3}{Q_2\Delta^2}\sqrt{\frac{\kappa_1}{\kappa_2}}p^{\frac{1}{1-\alpha\cos\Delta}}$. If otherwise, then
\begin{align*}
    \left<\boldsymbol{b}_k(T_{{\rm II},k}^{{\rm PT}*}\land T_{\rm fix}),\boldsymbol{x}_-\right><Q_3\frac{\sqrt{\kappa_1\kappa_2}}{\sqrt{m}}-Q_3\frac{\sqrt{\kappa_1\kappa_2}}{\sqrt{m}}=0.
\end{align*}
From the definition of $T_{{\rm II},k}^{{\rm PT}*}$, we know $T_{{\rm II},k}^{{\rm PT}*}<T_{{\rm II},k}^{{\rm PT}*}\land T_{\rm fix}$, which leads to a contradiction.

therefore, we have proved that for any $k\in\mathcal{K}_+$,
\begin{gather*}
    T_{{\rm II},k}^{{\rm PT}*}\land T_{\rm fix}=T_{{\rm II},k}^{{\rm PT}*};
    \\
    T_{{\rm II},k}^{{\rm PT}*}\leq T_{\rm II}+\frac{2Q_3}{Q_2\Delta^2}\sqrt{\frac{\kappa_1}{\kappa_2}}p^{\frac{1}{1-\alpha\cos\Delta}}.
\end{gather*}
With the help of Lemma \ref{lemma: GF Phase III neuron dynamics} (S2)(P0), we obtain
\begin{gather*}
T_{\rm II}^{{\rm PT}*}\land T_{\rm fix}=T_{\rm II}^{{\rm PT}*};
\\
T_{\rm II}^{{\rm PT}*}=\max_{k\in\mathcal{K}_+}T_{{\rm II},k}^{{\rm PT}*}\leq
T_{\rm II}+\frac{2Q_3}{Q_2\Delta^2}\sqrt{\frac{\kappa_1}{\kappa_2}}p^{\frac{1}{1-\alpha\cos\Delta}}.
\end{gather*}
Recalling Lemma \ref{lemma: GF Phase III Nearly fixed vector filed}, $\left<\boldsymbol{F}_+(t),\boldsymbol{x}_+\right>>0$ and $\left<\boldsymbol{F}_+(t),\boldsymbol{x}_-\right><0$ hold for any $t\in[T_{\rm II},T_{\rm II}^{{\rm PT}*}\land  T_{\rm fix}]=[T_{\rm II},T_{\rm II}^{{\rm PT}*}]$. From the definitions of $T_{\rm II}^{{\rm PT}}$ and $T_{\rm II}^{{\rm PT}*}$ \eqref{equ: GF Phase III hitting time}\eqref{equ: GF Phase III auxiliary hitting time}, we obtain
\[
T_{\rm II}^{{\rm PT}}=T_{\rm II}^{{\rm PT}*}.
\]
In conclusion, we have proved:
\[
T_{\rm II}^{{\rm PT}}=
T_{\rm II}^{{\rm PT}*}=T_{\rm II}+\mathcal{O}\left(\sqrt{\frac{\kappa_1}{\kappa_2}}\frac{p^{\frac{1}{1-\alpha\cos\Delta}}}{\Delta^2}\right)
\overset{\text{Lemma~\ref{lemma: GF Phase II hitting time estimate}}}{=}\bracket{1+\cO\bracket{\sqrt{\kappa_1\kappa_2^3}}}T_{\rm II}.
\]
\underline{Step II. Prediction Estimate.} 
Step I gives us the result:
\[
\delta_T:=T_{\rm II}^{{\rm PT}}-T_{\rm II}=\mathcal{O}\left(\sqrt{\frac{\kappa_1}{\kappa_2}}\frac{p^{\frac{1}{1-\alpha\cos\Delta}}}{\Delta^2}\right).
\]
Recalling the proof of Lemma \ref{lemma: Phase III Evolution of the prediction}, we know
\begin{align*}
-\frac{C_1+C_0e^{f_-(T_{\rm II})}\delta_T}{1+C_0e^{f_-(T_{\rm II})}\delta_T}\leq& e^{f_-(t)-f_-(T_{\rm II})}-1\leq0,
\\
-C_2\delta_T\leq &e^{f_+(T_{\rm II})-f_+(t)}-1\leq0.
\end{align*}
where
\begin{align*}
C_0=\Theta\left(\frac{\kappa_2^2\Delta^2}{p}\right),
\ C_1=\mathcal{O}\left(\kappa_2\sqrt{\kappa_1\kappa_2}\right),
\ C_2=\Theta\left(\kappa_2^2p^{-\frac{1}{1-\alpha\cos\Delta}}\right).
\end{align*}
Then a straightforward calculation gives us:
\begin{align*}
    0\geq e^{f_-(t)-f_-(T_{\rm II})}-1=-\mathcal{O}\left(\kappa_2^2\sqrt{\frac{\kappa_1}{\kappa_2}}\right)-\mathcal{O}\left(\kappa_2^2\sqrt{\frac{\kappa_1}{\kappa_2}}\right)=-\mathcal{O}\left(\kappa_2^2\sqrt{\frac{\kappa_1}{\kappa_2}}\right)
\end{align*}
\begin{align*}
    0\geq e^{f_+(T_{\rm II})-f_+(T_{\rm II}^{{\rm PT}})}-1=-\mathcal{O}\left(\kappa_2^2\sqrt{\frac{\kappa_1}{\kappa_2}}\frac{1}{\Delta^2}\right).
\end{align*}

With the help of Lemma \ref{lemma: prediction at end of Phase II}, we obtain the prediction estimate at the end of Phase III:
\begin{align*}
    e^{-f_+(T_{\rm II}^{{\rm PT}})}=&e^{-f_+(T_{\rm II})}e^{f_+(T_{\rm II})-f_+(T_{\rm II}^{{\rm PT}})}=\Theta\bracket{e^{-f_+(T_{\rm II})}}=\Theta\left(p^{-\frac{1}{1-\alpha\cos\Delta}}\right),
    \\
     e^{f_-(T_{\rm II}^{{\rm PT}})}=&e^{f_-(T_{\rm II})}e^{f_-(T_{\rm II}^{{\rm PT}})-f_-(T_{\rm II})}=\Theta\bracket{e^{f_-(T_{\rm II})}}=\Theta\left(p^{-\frac{\alpha\cos\Delta}{1-\alpha\cos\Delta}}\right).
\end{align*}
Moreover,
\begin{align*}
    &\left|\bracket{\frac{pe^{-f_+(T_{\rm II}^{{\rm PT}})}}{1+p}-\frac{e^{f_-(T_{\rm II}^{{\rm PT}})}}{1+p}\cos\Delta}-\bracket{\frac{pe^{-f_+(T_{\rm II})}}{1+p}-\frac{e^{f_-(T_{\rm II})}}{1+p}\cos\Delta}\right|
    \\\leq&\left|\frac{pe^{-f_+(T_{\rm II})}}{1+p}\right|\left|\frac{\frac{pe^{-f_+(T_{\rm II}^{{\rm PT}})}}{1+p}}{\frac{pe^{-f_+(T_{\rm II})}}{1+p}}-1\right|+\left|\frac{e^{f_-(T_{\rm II})}\cos\Delta}{1+p}\right|\left|\frac{\frac{e^{f_-(T_{\rm II}^{{\rm PT}})}\cos\Delta}{1+p}}{\frac{e^{f_-(T_{\rm II})}\cos\Delta}{1+p}}-1\right|
    \\=&\cO\bracket{p^{-\frac{1}{1-\alpha\cos\Delta}}\kappa_2^2\sqrt{\frac{\kappa_1}{\kappa_2}}\frac{1}{\Delta^2}}+\cO\bracket{p^{-\frac{1}{1-\alpha\cos\Delta}}\kappa_2^2\sqrt{\frac{\kappa_1}{\kappa_2}}}
    \\=&\cO\bracket{p^{-\frac{1}{1-\alpha\cos\Delta}}\kappa_2^2\sqrt{\frac{\kappa_1}{\kappa_2}}\frac{1}{\Delta^2}}\overset{\kappa_1\kappa_2^3=\cO(\Delta^8)}{=}\cO\left(\Delta^2p^{-\frac{1}{1-\alpha\cos\Delta}}\right),
\end{align*}
which means
\begin{align*}
    \frac{pe^{-f_+(T_{\rm II}^{{\rm PT}})}}{1+p}-\frac{e^{f_-(T_{\rm II}^{{\rm PT}})}}{1+p}\cos\Delta=\Theta\left(\Delta^2p^{-\frac{1}{1-\alpha\cos\Delta}}\right).
\end{align*}
In the same way, we can obtain
\begin{gather*}
\frac{pe^{-f_+(T_{\rm II}^{{\rm PT}})}}{1+p}\cos\Delta-\frac{e^{f_-(T_{\rm II}^{{\rm PT}})}}{1+p}=-\Theta\left(\Delta^2p^{-\frac{1}{1-\alpha\cos\Delta}}\right).
\end{gather*}

\underline{Step III. Activation Patterns.}

Recall our proofs in Step I, we know that
\begin{align*}
    \<\bw_k(T_{\rm II}^{{\rm PT}}),\bx_-\>=0,\ \forall k\in\cK_+.
\end{align*}
Moreover, from the dynamics in Lemma~\ref{lemma: GF Phase III neuron dynamics} (S1) and (S2)(P3), we obtain:
\begin{gather*}
    \<\bw_k(T_{\rm II}^{{\rm PT}}),\bx_+\>>0,\ \forall k\in\cK_+;
    \\
    \<\bw_k(T_{\rm II}^{{\rm PT}}),\bx_+\>=0,\ \<\bw_k(T_{\rm II}^{{\rm PT}}),\bx_-\>>0,\ \forall k\in\cK_-.
\end{gather*}
\end{proof}

\begin{proof}[Proof of Theorem~\ref{thm: GF Phase transition II-to-III}]\ \\
Theorem~\ref{thm: GF Phase transition II-to-III} (S1) has been proven in Lemma~\ref{lemma: GF Time Estimate of Phase III} (S1), and Theorem~\ref{thm: GF Phase transition II-to-III} (S2) has been proven in Lemma~\ref{lemma: GF Time Estimate of Phase III} (S3).
    
\end{proof}

\subsection{Optimization Dynamics after Phase Transition}

After Phase Transition $(t>T_{\rm II}^{\rm PT})$, we study the dynamics before the patterns of living neurons change again. Specifically, we define the following hitting time
\begin{equation}\label{equ: GF Phase IV hitting time actually}
\begin{aligned}
T_{\rm III}:=
\inf\Big\{t>T_{\rm II}^{\rm PT}:&\ \exists k\in\cK_+\cup\cK_-,\text{\rm\sgn}_k^+(t)\ne\text{\rm\sgn}_k^+(T_{\rm I})\text{ or }\text{\rm\sgn}_k^-(t)\ne\text{\rm\sgn}_k^-(T_{\rm I})\Big\}
\\
=\inf\Big\{t>T_{\rm II}^{\rm PT}:&\ \exists k\in\cK_+,\text{  s.t. } \left<\boldsymbol{w}_k(t),\boldsymbol{x}_+\right>\leq0\text{ or }\left<\boldsymbol{w}_k(t),\boldsymbol{x}_-\right>\ne0;
\\&\text{or}\ \exists k\in\cK_-,\text{  s.t. } \left<\boldsymbol{w}_k(t),\boldsymbol{x}_+\right>\ne0\text{ or }\left<\boldsymbol{w}_k(t),\boldsymbol{x}_-\right>\leq0\Big\},
\end{aligned}
\end{equation}
and we call $t\in(T_{\rm II}^{\rm PT}, T_{\rm III})$ ``L-Phase III''.

Moreover, we call $t\in[T_{\rm II},T_{\rm III})$ ``Phase III'', i.e.. ``Phase Transition'' + ``L-Phase III''.

In order to analyze the dynamics of neurons and vector fields, we introduce the auxiliary hitting time:
\begin{equation}\label{equ: GF Phase IV hitting time}
\begin{aligned}
T_{\rm III}^*:=T_{\rm III}\land\inf\Big\{t>T_{\rm II}^{{\rm PT}}:&\left<\boldsymbol{F}_+(t),\boldsymbol{x}_+\right>\leq0\text{ or }\left<\boldsymbol{F}_+(t),\boldsymbol{x}_-\right>\geq0\Big\},
\\
\text{where}\quad
\boldsymbol{F}_+(t)=&\frac{p}{1+p}e^{-f_+(t)}\boldsymbol{x}_{+}-\frac{1}{1+p}e^{f_-(t)}\boldsymbol{x}_-.
\end{aligned}
\end{equation}
We call $t\in(T_{\rm II}^{{\rm PT}}, T_{\rm III}^*)$ ``L-Phase III*''.

Due to the almost simplest activation patterns, this phase is easier to analyze, and we only need to estimate the time and size of the changes in the vector field. Nevertheless, our challenge is to prove that all living negative neurons simultaneously change their activation patterns at $T_{\rm III}^*$, which also implies that $T_{\rm III}=T_{\rm III}^*$.

\begin{lemma}[Dynamics of activate neurons during L-Phase III*]\label{lemma: GF Phase IV neuron dynamics}\ \\
In L-Phase III* $(t\in[T_{\rm II}^{{\rm PT}}, T_{\rm III}^*))$, we have the following dynamics for each neuron $k\in\mathcal{K}_-\cup\mathcal{K}_+$.

(S1). For negative neuron $k\in\mathcal{K}_-$, we have:
\begin{align*}
    &\boldsymbol{w}_k(t)\in\mathcal{M}_+^0\cap\mathcal{M}_-^+,
    \\
    &\frac{\mathrm{d} \boldsymbol{b}_k(t)}{\mathrm{d}t}=\frac{\kappa_2e^{f_-(t)}}{\sqrt{m}(1+p)}\Big(\boldsymbol{x}_--\boldsymbol{x}_+\cos\Delta\Big).
\end{align*}
(S2) For positive neuron $k\in\mathcal{K}_+$, we have:
\begin{align*}
    &\boldsymbol{w}_k(t)\in\mathcal{M}_+^+\cap\mathcal{M}_-^0,
    \\
    &\frac{\mathrm{d}\boldsymbol{b}_k(t)}{\mathrm{d}t}=\frac{\kappa_2pe^{-f_+(t)}}{\sqrt{m}(1+p)}\Big(\boldsymbol{x}_+-\boldsymbol{x}_-\cos\Delta\Big).
\end{align*}

\end{lemma}

\begin{proof}[Proof of Lemma \ref{lemma: GF Phase IV neuron dynamics}]\ \\
From the definition of $T_{\rm III}^*$, we know that $\left<\boldsymbol{F}_+(t),\boldsymbol{x}_+\right>>0$ and $\left<\boldsymbol{F}_+(t),\boldsymbol{x}_-\right><0$ hold for any $t\in[T_{\rm II}^{{\rm PT}},T_{\rm III}^*)$. 
Moreover, Lemma~\ref{lemma: GF Time Estimate of Phase III} ensures that for $k\in\cK_+$, $\boldsymbol{w}_k(T_{\rm II}^{\rm PT})\in\mathcal{M}_+^0\cap\mathcal{M}_-^+$; for $k\in\cK_-$, $\boldsymbol{w}_k(T_{\rm II}^{\rm PT})\in\mathcal{M}_+^+\cap\mathcal{M}_-^0$. 
Hence, this lemma can be proved in the same way as shown in the proof of Lemma \ref{lemma: GF Phase III neuron dynamics} (S1) and (S2)(P2). 
We do not repeat it here.

\end{proof}

\begin{lemma}[Time and prediction estimate at the end of L-Phase III*]\label{lemma: GF Phase III* hitting time estimate}\ \\
(S1) (Time).
\begin{align*}
&T_{\rm III}^*=T_{\rm II}^{{\rm PT}}+\Theta\bracket{\frac{p^{\frac{1}{1-\alpha\cos\Delta}}}{\kappa_2^2}}=\bracket{1+\Theta(\Delta^2)}T_{\rm II}^{{\rm PT}}=\bracket{1+\Theta(\Delta^2)}T_{\rm II};
\end{align*}
(S2) (Prediction).
\begin{gather*}
e^{-f_+(T_{\rm III}^*)}=\Theta\left(p^{-\frac{1}{1-\alpha\cos\Delta}}\right)
,\quad e^{f_-(T_{\rm III}^*)}=\Theta\left(p^{-\frac{\alpha\cos\Delta}{1-\alpha\cos\Delta}}\right);
\\
\frac{pe^{-f_+(T_{\rm III}^*)}}{1+p}-\frac{e^{f_-(T_{\rm III}^*)}}{1+p}\cos\Delta=0,
\\
\frac{pe^{-f_+(T_{\rm II}^{{\rm PT}})}}{1+p}\cos\Delta-\frac{e^{f_-(T_{\rm II}^{{\rm PT}})}}{1+p}=-\Theta\left(\Delta^2p^{-\frac{1}{1-\alpha\cos\Delta}}\right).
\end{gather*}
    
\end{lemma}

\begin{proof}[Proof of Lemma \ref{lemma: GF Phase III* hitting time estimate}]\ \\
\underline{Step I. Explicit Solution to $f_+(t)$ and $f_-(t)$.}

For any $t\in[T_{\rm II}^{{\rm PT}},T_{\rm III}^*)$, we have:
\begin{align*}
    &f_+(t)=\frac{\kappa_2}{\sqrt{m}}\sum_{k\in\mathcal{K}_+}\boldsymbol{b}_k^\top(t)\boldsymbol{x}_+,
    \\
    &f_-(t)=-\frac{\kappa_2}{\sqrt{m}}\sum_{k\in\mathcal{K}_-}\boldsymbol{b}_k^\top(t)\boldsymbol{x}_-.
\end{align*}
Let us consider the dynamics of $f_+(t)$ and $f_-(t)$. With the help of Lemma \ref{lemma: GF Phase IV neuron dynamics}, these two dynamics are nearly independent:
\begin{align*}
    \frac{\mathrm{d}f_+(t)}{\mathrm{d}t}=&=\frac{\kappa_2}{\sqrt{m}}\sum_{k\in\mathcal{K}_+}\left<\frac{\kappa_2pe^{-f_+(t)}}{\sqrt{m}(1+p)}\Big(\boldsymbol{x}_+-\boldsymbol{x}_-\cos\Delta\Big),\boldsymbol{x}_+\right>
    =\frac{\kappa_2^2m_+p\sin^2\Delta}{m(1+p)}e^{-f_+(t)},
    \\
    \frac{\mathrm{d}f_-(t)}{\mathrm{d}t}=&=-\frac{\kappa_2}{\sqrt{m}}\sum_{k\in\mathcal{K}_-}\left<\frac{\kappa_2e^{f_-(t)}}{\sqrt{m}(1+p)}\Big(\boldsymbol{x}_--\boldsymbol{x}_+\cos\Delta\Big),\boldsymbol{x}_-\right>
    =-\frac{\kappa_2^2m_-\sin^2\Delta}{m(1+p)}e^{f_-(t)}.
\end{align*}
Their solutions are:
\begin{align*}
    e^{-f_+(t)}&=\frac{e^{-f_+(T_{\rm II}^{{\rm PT}})}}{1+e^{-f_+(T_{\rm II}^{{\rm PT}})}\frac{\kappa_2^2m_+p\sin^2\Delta}{m(1+p)}(t-T_{\rm II}^{{\rm PT}})},
    \\e^{f_-(t)}&=\frac{e^{f_-(T_{\rm II}^{{\rm PT}})}}{1+e^{f_-(T_{\rm II}^{{\rm PT}})}\frac{\kappa_2^2m_-\sin^2\Delta}{m(1+p)}(t-T_{\rm II}^{{\rm PT}})}.
\end{align*}

\underline{Step II. Time Estimate of $T_{\rm III}^*$.}

For simplicity, we denote $G_+:=\frac{\kappa_2^2m_+p\sin^2\Delta}{m(1+p)}$ and $G_-:=\frac{\kappa_2^2m_-\sin^2\Delta}{m(1+p)}$.

First, we consider the evolution of the vector field $\left<\boldsymbol{F}_+(t),\boldsymbol{x}_-\right>$:
\begin{align*}
    &\left<\boldsymbol{F}_+(t),\boldsymbol{x}_-\right>=\frac{pe^{-f_+(t)}}{1+p}\cos\Delta-\frac{e^{f_-(t)}}{1+p}
    \\=&\frac{1}{1+p}\left(\frac{pe^{-f_+(T_{\rm II}^{{\rm PT}})}\cos\Delta}{1+e^{-f_+(T_{\rm II}^{{\rm PT}})}G_+(t-T_{\rm II}^{{\rm PT}})}-\frac{e^{f_-(T_{\rm II}^{{\rm PT}})}}{1+e^{f_-(T_{\rm II}^{{\rm PT}})}G_-(t-T_{\rm II}^{{\rm PT}})}\right)
    \\=&\frac{(pe^{-f_+(T_{\rm II}^{{\rm PT}})}\cos\Delta-e^{f_-(T_{\rm II}^{{\rm PT}})})+e^{f_-(T_{\rm II}^{{\rm PT}}-f_+(T_{\rm II}^{{\rm PT}})}(pG_-\cos\Delta-G_+)(t-T_{\rm II}^{{\rm PT}})}{(1+p)(1+e^{-f_+(T_{\rm II}^{{\rm PT}})}G_+(t-T_{\rm II}^{{\rm PT}}))(1+e^{f_-(T_{\rm II}^{{\rm PT}})}G_-(t-T_{\rm II}^{{\rm PT}}))}
    \\=&\frac{(1+p)\left<\boldsymbol{F}_+(T_{\rm II}^{{\rm PT}}),\boldsymbol{x}_-\right>+e^{f_-(T_{\rm II}^{{\rm PT}}-f_+(T_{\rm II}^{{\rm PT}})}(pG_-\cos\Delta-G_+)(t-T_{\rm II}^{{\rm PT}})}{(1+p)(1+e^{-f_+(T_{\rm II}^{{\rm PT}})}G_+(t-T_{\rm II}^{{\rm PT}}))(1+e^{f_-(T_{\rm II}^{{\rm PT}})}G_-(t-T_{\rm II}^{{\rm PT}}))}<0.
\end{align*}
Hence, the hitting time $T_{\rm III}^*$ can be
converted to the following $T_{\rm III}^{**}$:
\begin{align*}
    T_{\rm III}^*=T_{\rm III}^{**}:=T_{\rm III}\land\inf\Big\{t>T_{\rm II}^{{\rm PT}}:&\left<\boldsymbol{F}_+(t),\boldsymbol{x}_+\right>\leq0\Big\}.
\end{align*}
Then we consider $\left<\boldsymbol{F}_+(t),\boldsymbol{x}_+\right>$:
\begin{align*}
    &\left<\boldsymbol{F}_+(t),\boldsymbol{x}_+\right>=\frac{pe^{-f_+(t)}}{1+p}-\frac{e^{f_-(t)}}{1+p}\cos\Delta
    \\=&\frac{1}{1+p}\left(\frac{pe^{-f_+(T_{\rm II}^{{\rm PT}})}}{1+e^{-f_+(T_{\rm II}^{{\rm PT}})}G_+(t-T_{\rm II}^{{\rm PT}})}-\frac{e^{f_-(T_{\rm II}^{{\rm PT}})}\cos\Delta}{1+e^{f_-(T_{\rm II}^{{\rm PT}})}G_-(t-T_{\rm II}^{{\rm PT}})}\right)
    \\=&\frac{(pe^{-f_+(T_{\rm II}^{{\rm PT}})}-e^{f_-(T_{\rm II}^{{\rm PT}})}\cos\Delta)+e^{f_-(T_{\rm II}^{{\rm PT}}-f_+(T_{\rm II}^{{\rm PT}})}(pG_--G_+\cos\Delta)(t-T_{\rm II}^{{\rm PT}})}{(1+p)(1+e^{-f_+(T_{\rm II}^{{\rm PT}})}G_+(t-T_{\rm II}^{{\rm PT}}))(1+e^{f_-(T_{\rm II}^{{\rm PT}})}G_-(t-T_{\rm II}^{{\rm PT}}))}
    \\=&\frac{(1+p)\left<\boldsymbol{F}_+(T_{\rm II}^{{\rm PT}}),\boldsymbol{x}_+\right>+e^{f_-(T_{\rm II}^{{\rm PT}})-f_+(T_{\rm II}^{{\rm PT}})}(pG_--G_+\cos\Delta)(t-T_{\rm II}^{{\rm PT}})}{(1+p)(1+e^{-f_+(T_{\rm II}^{{\rm PT}})}G_+(t-T_{\rm II}^{{\rm PT}}))(1+e^{f_-(T_{\rm II}^{{\rm PT}})}G_-(t-T_{\rm II}^{{\rm PT}}))}.
\end{align*}
From Lemma \ref{lemma: GF Time Estimate of Phase III}, we know
\begin{gather*}
    (1+p)\left<\boldsymbol{F}_+(T_{\rm II}^{{\rm PT}}),\boldsymbol{x}_+\right>
    =(1+p)\bracket{\frac{pe^{-f_+(t)}}{1+p}-\frac{e^{f_-(t)}}{1+p}\cos\Delta}=\Theta\left(\Delta^2p^{-\frac{\alpha\cos\Delta}{1-\alpha\cos\Delta}}\right),
    \\
     e^{f_-(T_{\rm II}^{{\rm PT}})-f_+(T_{\rm II}^{{\rm PT}})}=\Theta\left(e^{f_-(T_{\rm I+II})-f_+(T_{\rm I+II})}\right)=\Theta\left(p^{-\frac{1+\alpha\cos\Delta}{1-\alpha\cos\Delta}}\right),
    \\
    pG_--G_+\cos\Delta=\frac{\kappa_2^2p\sin^2\Delta}{1+p}\frac{(m_--m_+\cos\Delta)}{m}=-\Theta\left(\kappa_2^2\Delta^2\right).
\end{gather*}
These imply the hitting time:
\begin{align*}
    T_{\rm III}^*=T_{\rm III}^{**}=T_{\rm II}^{{\rm PT}}+\Theta\bracket{\frac{\Delta^2p^{-\frac{\alpha\cos\Delta}{1-\alpha\cos\Delta}}}{p^{-\frac{1+\alpha\cos\Delta}{1-\alpha\cos\Delta}}\kappa_2^2\Delta^2}}=T_{\rm II}^{{\rm PT}}+\Theta\bracket{\frac{p^{\frac{1}{1-\alpha\cos\Delta}}}{\kappa_2^2}}.
\end{align*}

\underline{Step III. Prediction estimate.}

From the explicit solution in Step I and the time estimate in Step II, it is easy to verify
\begin{align*}
    e^{-f_+(T_{\rm III}^*)}=\frac{e^{-f_+(T_{\rm II}^{{\rm PT}})}}{1+e^{-f_+(T_{\rm II}^{{\rm PT}})}\frac{\kappa_2^2m_+p\sin^2\Delta}{m(1+p)}(T_{\rm III}^*-T_{\rm II}^{{\rm PT}})}
    =\Theta\bracket{p^{-\frac{1}{1-\alpha\cos\Delta}}},
\end{align*}
\begin{align*}
    e^{f_-(T_{\rm III}^*)}=\frac{e^{f_-(T_{\rm II}^{{\rm PT}})}}{1+e^{f_-(T_{\rm II}^{{\rm PT}})}\frac{\kappa_2^2m_-\sin^2\Delta}{m(1+p)}(T_{\rm III}^*-T_{\rm II}^{{\rm PT}})}
    =\Theta\bracket{p^{-\frac{\alpha\cos\Delta}{1-\alpha\cos\Delta}}}.
\end{align*}
Recalling the calculation in Step II, we have:
\begin{align*}
    \<\bF_+(T_{\rm III}^*),\bx_+\>=\frac{pe^{-f_+(T_{\rm III}^*)}}{1+p}-\frac{e^{f_-(T_{\rm III}^*)}}{1+p}\cos\Delta=0,
\end{align*}
\begin{align*}
    & \<\bF_+(T_{\rm III}^*),\bx_-\>=\frac{pe^{-f_+(T_{\rm III}^*)}}{1+p}\cos\Delta-\frac{e^{f_-(T_{\rm III}^*)}}{1+p}
    \\=&\frac{(1+p)\left<\boldsymbol{F}_+(T_{\rm II}^{{\rm PT}}),\boldsymbol{x}_-\right>+e^{f_-(T_{\rm II}^{{\rm PT}})-f_+(T_{\rm II}^{{\rm PT}})}(pG_-\cos\Delta-G_+)(T_{\rm III}^*-T_{\rm II}^{{\rm PT}})}{(1+p)(1+e^{-f_+(T_{\rm II}^{{\rm PT}})}G_+(T_{\rm III}^*-T_{\rm II}^{{\rm PT}}))(1+e^{f_-(T_{\rm II}^{{\rm PT}})}G_-(T_{\rm III}^*-T_{\rm II}^{{\rm PT}}))}
    \\=&\Theta\bracket{\frac{-\Delta^2p^{-\frac{\alpha\cos\Delta}{1-\alpha\cos\Delta}}-\Delta^2p^{-\frac{\alpha\cos\Delta}{1-\alpha\cos\Delta}}}{p\bracket{1+\Theta(\Delta^2)}\bracket{1+\Theta(\Delta^2)}}}=-\Theta\bracket{\Delta^2p^{-\frac{1}{1-\alpha\cos\Delta}}}.
\end{align*}
    
\end{proof}

\begin{lemma}[Hitting time relationship]\label{lemma: GF Phase III hitting time transformation}
If we define the following hitting time:
\begin{align*}
    T_{\rm III}^{\rm W}=\inf\big\{t>T_{\rm II}^{\rm PT}:\forall k\in\cK_-,\<\bw_k(t),\bx_+\>>0\big\},
\end{align*}
then it holds that $T_{\rm III}=T_{\rm III}^*=T_{\rm III}^{\rm W}$.
\end{lemma}

\begin{proof}[Proof of Lemma~\ref{lemma: GF Phase III hitting time transformation}]\ \\
We define the following hitting time:
\begin{align*}
    &T_{\rm III}^{\rm F}=\inf\big\{t>T_{\rm II}^{\rm PT}:\<\bF_+(t),\bx_+\>\leq0\big\};
    \\
    &T_{\rm III}^{\rm N}=\inf\big\{t>T_{\rm II}^{\rm PT}:\exists k\in\cK_-,\text{ s.t. }\<\bw_k(t),\bx_+\>>0\big\},
    \\
    &T_{\rm III}^{\rm W}=\inf\big\{t>T_{\rm II}^{\rm PT}:\forall k\in\cK_-,\<\bw_k(t),\bx_+\>>0\big\},
\end{align*}

From the proof in Lemma~\ref{lemma: GF Time Estimate of Phase III}, we know $\left<\boldsymbol{F}_+(T_{\rm III}^*),\boldsymbol{x}_-\right><0$.
From the continuity of $\left<\boldsymbol{F}_+(\cdot),\boldsymbol{x}_-\right>$, we know that there exists $\tau_1>0$, such that $\left<\boldsymbol{F}_+(t),\boldsymbol{x}_-\right><0$ holds for any $t\in[T_{\rm III}^*,T_{\rm III}^*+\tau_1)$. 
Then in the same way as the proof of Lemma~\ref{lemma: GF Phase IV neuron dynamics} (S2), we know that for $k\in\cK_+$, $\bw(t)\in\cM_+^+\cap\cM_-^0$ for any $t\in[T_{\rm III}^*,T_{\rm III}^*+\tau_1)$.

Recalling that for any $k\in\cK_-$, $\<\bb_k(T_{\rm III}^*),\bx_-\>>0$, from the continuity, we know that there exists $\tau_2>0$ such that $\<\bb_k(t),\bx_-\>>0$ holds for any $t\in[T_{\rm III}^*,T_{\rm III}^*+\tau_2)$.

Hence, we have:
\begin{align*}
    T_{\rm III}^*=T_{\rm III}^{\rm F}\land T_{\rm III}^{\rm N}=\inf\big\{t>T_{\rm II}^{\rm PT}:\<\bF_+(t),\bx_+\>\leq0
    \text{ or }\exists k\in\cK_-,\text{ s.t. }\<\bw_k(t),\bx_+\>\ne 0\big\}.
\end{align*}

It is obvious that $T_{\rm III}^{\rm N}\geq T_{\rm III}^{\rm F}\land T_{\rm III}^{\rm N}=T_{\rm III}^*$. Now we prove $T_{\rm III}^{\rm N}=T_{\rm III}^*$.

If we assume $T_{\rm III}^{\rm N}>T_{\rm III}^*$ strictly, then the dynamics about $f_+(t)$ and $f_-(t)$ in the proof (Step I) of Lemma~\ref{lemma: GF Phase III* hitting time estimate} still hold for any $t\in[T_{\rm III}^*,T_{\rm III}^{\rm N})$. Using the same calculate about $\<\bF_+(t),\bx_+\>$ in the proof (Step II, III) of Lemma~\ref{lemma: GF Phase III* hitting time estimate}, we can obtain: $\<\bF_+(t),\bx_+\><0,\ t\in[T_{\rm III}^*,T_{\rm III}^{\rm N})$.

Then we consider the vector field around the manifold $\cM_+^0\cap\cM_-^+$ for $t\in[T_{\rm III}^*,T_{\rm III}^{\rm N})$. In the same way as the proof of Lemma~\ref{lemma: GF Phase III neuron dynamics} (S1), 
we can prove that the two-side projections onto $\bx_+$ (the normal to the surface $\cM_+^+\cap\cM_-^0$) satisfies $f_N^+(t,\tilde{\bw}),f_N^-(t,\tilde{\bw})>0$ for any $t\in[T_{\rm III}^*,T_{\rm III}^*+\tau_1)$, which satisfies (Case II) in Definition~\ref{def: discontinuous system solution}. This implies that $\bw_k(t)$ enter the manifold $\cM_+^+$, i.e., $\<\bw_k(t),\bx_+\>>0$ for any $t\in[T_{\rm III}^*,T_{\rm III}^{\rm N})$, which is contradict to the definition of $T_{\rm III}^{\rm N}$.
Hence, we have proved 
\begin{align*}
    T_{\rm III}^{\rm N}=T_{\rm III}^{\rm N}\land T_{\rm III}^{\rm F}=T_{\rm III}^*.
\end{align*}

Noticing that the change of activation patterns of $\sgn_k^+(t)$ $(k\in\cK_-)$ is due to the change of the vector field $\<\bF_+(t),\bx_+\>$, it is easy to verify that $T_{\rm III}^{\rm F}=T_{\rm III}^{\rm N}\land T_{\rm III}^{\rm F}*$.
Then we have $T_{\rm III}^{\rm N}=T_{\rm III}^{\rm F}=T_{\rm III}^{\rm N}\land T_{\rm III}^{\rm F}=T_{\rm III}^*$.

Moreover, noticing that $T_{\rm III}\leq T_{\rm III}^{\rm N}$ and $T_{\rm III}^*\leq T_{\rm III}$, we obtain $T_{\rm III}=T_{\rm III}^*=T_{\rm III}^{\rm N}=T_{\rm III}^{\rm F}$.

Lastly, noticing that all living negative neurons ($k\in\cK_-$) belong to $\cM_+^0\cap\cM_-^+$ at time $T_{\rm III}$. 
As discussed above, for each living negative neuron $k\in\cK_-$, the vector field near $\bb_k(T_{\rm III})$ is the same, with $f_N^->0$ and $f_N^+=0$ in Definition~\ref{def: discontinuous system solution} (Case II). Hence, each living positive neuron $\bw_k$ leaves from $\cM_+^0$ and enter $\cM_+^+$ instantly at $T_{\rm III}$, which means $T_{\rm III}^{\rm W}=T_{\rm III}^{\rm N}$.

Hence, we have proved $T_{\rm III}=T_{\rm III}^*=T_{\rm III}^{\rm W}=T_{\rm III}^{\rm N}=T_{\rm III}^{\rm F}$.
\end{proof}

\begin{proof}[Proof of Theorem~\ref{thm: GF Phase III}]\ \\
Combining Lemma~\ref{lemma: GF Phase III* hitting time estimate} and~\ref{lemma: GF Phase III hitting time transformation}, we obtain $T_{\rm III}=\bracket{1+\Theta(\Delta^2)}T_{\rm II}$.  
\end{proof}

\newpage
\section{Proofs of Optimization Dynamics in Phase IV}\label{appendix: proof: Phase IV}

\begin{proof}[Proof of Theorem~\ref{thm: GF Phase transition III-to-IV}]\ \\
From Lemma~\ref{lemma: GF Phase III hitting time transformation}, we know that all living negative neuron $k\in\cK_-$ simultaneously change their patterns on $\bx_+$ at $T_{\rm III}$: $\lim\limits_{t\to T_{\rm III}^-}\text{\rm\sgn}_{k}^+(t)=0$, $\lim\limits_{t\to T_{\rm III}^+}\text{\rm\sgn}_{k}^+(t)=1$. Moreover, from our proof of Lemma~\ref{lemma: GF Phase III hitting time transformation}, we know that other activation patterns remain unchanged at $T_{\rm III}$.
\end{proof}

In this phase, we study the dynamics before activation patterns change again after the phase transition in Theorem~\ref{thm: GF Phase transition III-to-IV}. We define the hitting time:
\begin{align*}
    T_{\rm IV}:=\inf\{t>T_{\rm III}:\exists k\in\cK_+\cup\cK_-,\text{\rm\sgn}_k^+(t)\ne \lim\limits_{s\to T_{\rm III}^+}\text{\rm\sgn}_k^+(s)\text{ or }\text{\rm\sgn}_k^-(t)\ne \lim\limits_{s\to T_{\rm III}^+}\text{\rm\sgn}_k^-(s)\},
    \vspace{-.2cm}
\end{align*}
and we call $t\in(T_{\rm III},T_{\rm IV})$ Phase IV.

In order to analyze the dynamics of neurons and vector fields, we introduce the auxiliary hitting time:
\begin{equation}
\begin{aligned}
T_{\rm IV}^*:=\inf\Big\{t>T_{\rm III}:&
\<\bF_+(t),\bx_+\>>0,\text{ or }\<\bF_+(t),\bx_-\>>0\Big\},
\\
\text{where}\quad
\boldsymbol{F}_+(t)=&\frac{p}{1+p}e^{-f_+(t)}\boldsymbol{x}_{+}-\frac{1}{1+p}e^{f_-(t)}\boldsymbol{x}_-,
\end{aligned}
\end{equation}
and we call $t\in(T_{\rm III},T_{\rm IV}^*)$ Phase IV*.

First, we will provide meticulous prior estimations for $2$d ODEs on $f_+(t)$ and $f_-(t)$, similar to Phase II, which can imply $T_{\rm IV}^*=+\infty$. Additionally, we can prove $T_{\rm IV}=T_{\rm IV}^*$.
Lastly, with the help of our fine-grained analysis for the 2D dynamics and the results in \citep{lyu2019gradient,ji2020directional}, we can determine the unique convergent direction from numerous KKT directions.

\subsection{Non-asymptotic Analysis of Optimization Dynamics in Phase IV*}




\begin{lemma}[Dynamics of activate neurons in Phase IV*]\label{lemma: GF Phase IV* neuron dynamics}\ \\
In Phase IV* $(t\in(T_{\rm III}, T_{\rm IV}^*))$, we have the following dynamics for each neuron $k\in\mathcal{K}_-\cup\mathcal{K}_+$.

(S1). For negative neuron $k\in\mathcal{K}_-$, we have:
\begin{align*}
&\boldsymbol{w}_k(t)\in\mathcal{M}_+^+\cap\mathcal{M}_-^+,
\\
&\frac{\mathrm{d} \boldsymbol{b}_k(t)}{\mathrm{d}t}=-\frac{\kappa_2}{\sqrt{m}}\boldsymbol{F}_+(t)=-\frac{\kappa_2}{\sqrt{m}}\left(\frac{p}{1+p}e^{-f_+(t)}\boldsymbol{x}_{+}-\frac{1}{1+p}e^{f_-(t)}\boldsymbol{x}_-\right).
\end{align*}
(S2) For positive neuron $k\in\mathcal{K}_+$, we have:
\begin{align*}
    &\boldsymbol{w}_k(t)\in\mathcal{M}_+^+\cap\mathcal{M}_-^0,
    \\
    &\frac{\mathrm{d}\boldsymbol{b}_k(t)}{\mathrm{d}t}=\frac{\kappa_2pe^{-f_+(t)}}{\sqrt{m}(1+p)}\Big(\boldsymbol{x}_+-\boldsymbol{x}_-\cos\Delta\Big).
\end{align*}

\end{lemma}

\begin{proof}[Proof of Lemma \ref{lemma: GF Phase IV* neuron dynamics}]\ \\
Using the definition of $T_{\rm IV}^*$, this lemma can be proved in the same way as shown in the proof of Lemma \ref{lemma: GF Phase II neuron dynamics}, \ref{lemma: GF Phase III neuron dynamics} and \ref{lemma: GF Phase IV neuron dynamics}.
\end{proof}

The next lemma gives the first-order dynamics of $f_+(t)$ and $f_-(t)$.

\begin{lemma}[First-order Dynamics of predictions in Phase IV*]\label{lemma: GF Phase IV* 1-order f dynamics}\ \\
In Phase IV* $(T_{\rm III}\leq t\leq T_{\rm IV}^*)$, we have the following dynamics for $f_+(t)$ and $f_-(t)$:
\begin{align*}
    \frac{\mathrm{d}f_+(t)}{\mathrm{d}t}&=\kappa_2^2\frac{m_+}{m}\frac{pe^{-f_+(t)}}{1+p}\sin^2\Delta+\kappa_2^2\frac{m_-}{m}\bracket{\frac{pe^{-f_+(t)}}{1+p}-\frac{e^{f-(t)}\cos\Delta}{1+p}},
    \\
    \frac{\mathrm{d}f_-(t)}{\mathrm{d}t}&=\kappa_2^2\frac{m_-}{m}\bracket{\frac{pe^{-f_+(t)}}{1+p}\cos\Delta-\frac{e^{f_-(t)}}{1+p}}.
\end{align*}
\end{lemma}

\begin{proof}[Proof of Lemma \ref{lemma: GF Phase IV* 1-order f dynamics}]\ \\
From the definition of $T_{\rm IV}$, for any $T_{\rm III}\leq t\leq T_{\rm IV}$, we have

\begin{align*}
    f_+(t)&=\sum_{k\in\mathcal{K}_+}\frac{\kappa_2}{\sqrt{m}}\boldsymbol{b}_k^\top(t)\boldsymbol{x}_+-\sum_{k\in\mathcal{K}_-}\frac{\kappa_2}{\sqrt{m}}\boldsymbol{b}_k^\top(t)\boldsymbol{x}_+,
    \\
    f_-(t)&=-\sum_{k\in\mathcal{K}_-}\frac{\kappa_2}{\sqrt{m}}\boldsymbol{b}_k^\top(t)\boldsymbol{x}_-.
\end{align*}

With the help of Lemma \ref{lemma: GF Phase IV* neuron dynamics}, we have the dynamics of predictions:
\begin{align*}
    \frac{\mathrm{d}f_+(t)}{\mathrm{d}t}=&\sum_{k\in\mathcal{K}_+}\frac{\kappa_2}{\sqrt{m}}\left<\frac{\mathrm{d}\boldsymbol{b}_k(t)}{\mathrm{d}t},\boldsymbol{x}_+\right>-\sum_{k\in\mathcal{K}_-}\frac{\kappa_2}{\sqrt{m}}\left<\frac{\mathrm{d}\boldsymbol{b}_k(t)}{\mathrm{d}t},\boldsymbol{x}_+\right>
    \\=&\frac{\kappa_2^2}{m}\sum_{k\in\mathcal{K}_+}\frac{p}{1+p}e^{-f_+(t)}\bracket{1-\cos^2\Delta}-\frac{\kappa_2^2}{m}\sum_{k\in\mathcal{K}_-}\bracket{\frac{\cos\Delta}{1+p}e^{f_-(t)}-\frac{p}{1+p}e^{-f_+(t)}}
    \\=&\frac{m_+}{m}\kappa_2^2\frac{p}{1+p}e^{-f_+(t)}\sin^2\Delta+\frac{m_-}{m}\kappa_2^2\bracket{\frac{pe^{-f_+(t)}}{1+p}-\frac{\cos\Delta}{1+p}e^{f-(t)}}.
\end{align*}
\begin{align*}
    \frac{\mathrm{d}f_-(t)}{\mathrm{d}t}=&-\sum_{k\in\mathcal{K}_-}\frac{\kappa_2}{\sqrt{m}}\left<\frac{\mathrm{d}\boldsymbol{b}_k(t)}{\mathrm{d}t},\boldsymbol{x}_-\right>=\frac{\kappa_2^2}{m}\sum_{k\in\mathcal{K}_-}\bracket{\frac{p}{1+p}e^{-f_+(t)}\cos\Delta-\frac{1}{1+p}e^{f_-(t)}}
    \\=&\frac{m_-}{m}\kappa_2^2\bracket{\frac{p}{1+p}e^{-f_+(t)}\cos\Delta-\frac{1}{1+p}e^{f_-(t)}}.
\end{align*}

\end{proof}

Following the proof in Phase II, we focus on the dynamics about predictions.
Due to the specificity of the first-order dynamics, the following lemma gives an second-order \textbf{autonomous} dynamics of predictions.

\begin{lemma}[Second-order Autonomous Dynamics of predictions in Phase IV*]\label{lemma: GF Phase IV* 2-order f dynamics}\ \\
If we consider the following two variables:
\[
\begin{cases}
    \cI(t):=\kappa_2^2\frac{m_-}{m}\frac{p}{1+p}e^{-f_+(t)},
    \\
    \cJ(t):=\kappa_2^2\frac{m_-}{m}\frac{1}{1+p}e^{f_-(t)},
\end{cases}
\]
then the following autonomous dynamics of $\mathcal{U}(t)$ and $\mathcal{V}(t)$ hold in Phase IV* $(T_{\rm III}\leq t\leq T_{\rm IV}^*)$:
\[
\begin{cases}
    \frac{\mathrm{d}\cI(t)}{\mathrm{d}t}=\cI(t) \cJ(t)\cos\Delta-\cI^2(t)\bracket{1+\frac{m_+}{m_-}\sin^2\Delta},
    \\
    \frac{\mathrm{d}\cJ(t)}{\mathrm{d}t}=\cI(t) \cJ(t)\cos\Delta-\cJ^2(t).
\end{cases}
\]
\end{lemma}
\begin{proof}[Proof of Lemma \ref{lemma: GF Phase IV* 2-order f dynamics}]\ \\
With the help of the first-order dynamics in Lemma \ref{lemma: GF Phase IV* 1-order f dynamics}, the proof is straight-forward.
\end{proof}

Lemma \ref{lemma: GF Phase IV* 2-order f dynamics} enlighten us that we only need to study the dynamics of $\cI(t)$ and $\cJ(t)$ to study the dynamics in Phase IV*, where $\cI(t),\cJ(t)$ satisfies the following autonomous dynamics:
\begin{equation}\label{equ: GF Phase IV* U V dynamics}
\begin{aligned}
    &\begin{cases}
    \frac{\mathrm{d}\cI(t)}{\mathrm{d}t}=\cI(t) \cJ(t)\cos\Delta-\cI^2(t)\bracket{1+\frac{m_+}{m_-}\sin^2\Delta},
    \\
    \frac{\mathrm{d}\cJ(t)}{\mathrm{d}t}=\cI(t) \cJ(t)\cos\Delta-\cJ^2(t),
    \end{cases}\quad t\geq T_{\rm III};
    \\
    &\begin{cases}
    \cI(T_{\rm III})=\kappa_2^2\frac{m_-}{m}\frac{p}{1+p}e^{-f_+(T_{\rm III})},
    \\
    \cJ(T_{\rm III})=\kappa_2^2\frac{m_-}{m}\frac{1}{1+p}e^{f_-(T_{\rm III})}.
    \end{cases}
\end{aligned}
\end{equation}

The next lemma studies the dynamics \eqref{equ: GF Phase IV* U V dynamics} for any $t\in[T_{\rm III},+\infty)$.
\begin{lemma}[Fine-grained analysis of the dynamics \eqref{equ: GF Phase IV* U V dynamics}]\label{lemma: GF Phase IV* U V dynamics}\ \\
For the dynamics \eqref{equ: GF Phase IV* U V dynamics}, we have the following results:

{\bf (S1).} Initialization.
\begin{gather*}
    \mathcal{I}(T_{\rm III})=\Theta\bracket{\kappa_2^2p^{-\frac{1}{1-\alpha\cos\Delta}}},\quad
    \mathcal{J}(T_{\rm III})=\Theta\bracket{\kappa_2^2p^{-\frac{1}{1-\alpha\cos\Delta}}},
    \\
    \mathcal{I}(T_{\rm III})-\mathcal{J}(T_{\rm III})\cos\Delta=0,\quad
    \mathcal{I}(T_{\rm III})\cos\Delta-\mathcal{J}(T_{\rm III})=-\Theta\bracket{\kappa_2^2\Delta^2p^{-\frac{1}{1-\alpha\cos\Delta}}}.
\end{gather*}

{\bf (S2).} Fine-grained two-side bound for $\cI(t)/\cJ(t)$.
\begin{align*}
     \frac{1+\cos\Delta}{1+\cos\Delta+\frac{m_+}{m_-}\sin^2\Delta}<\frac{\mathcal{I}(t)}{\mathcal{J}(t)}<\cos\Delta,\quad \forall t\in[T_{\rm III},+\infty).
\end{align*}

{\bf (S3).} The limit of $\cI(t)/\cJ(t)$.
\begin{align*}
    \lim\limits_{t\to\infty}\frac{\cI(t)}{\cJ(t)}=\frac{1+\cos\Delta}{1+\cos\Delta+\frac{m_+}{m_-}\sin^2\Delta}.
\end{align*}
{\bf (S4).} Tight estimate of $\cI(t)$ and $\cJ(t)$.
\begin{align*}
\cI(t)=&\Theta\bracket{\frac{1}{\frac{p^{\frac{1}{1-\alpha\cos\Delta}}}{\kappa_2^2}+\Delta^2(t-T_{\rm III})}},\quad\forall t\in[T_{\rm III},+\infty);
\\
\cJ(t)=&\Theta\bracket{\frac{1}{\frac{p^{\frac{1}{1-\alpha\cos\Delta}}}{\kappa_2^2}+\Delta^2(t-T_{\rm III})}},\quad\forall t\in[T_{\rm III},+\infty).
\end{align*}

\end{lemma}

\begin{proof}[Proof of Lemma \ref{lemma: GF Phase IV* U V dynamics}]\ \\
For simplicity, in this proof, we denote
\[
    \epsilon:=\frac{m_+}{m_-}\sin^2\Delta.
\]
\underline{Step I. Preparation.}
Recalling Lemma \ref{lemma: GF Phase III* hitting time estimate}, we have:
\begin{gather*}
    \mathcal{I}(T_{\rm III})=\Theta\bracket{\kappa_2^2p^{-\frac{1}{1-\alpha\cos\Delta}}},\quad
    \mathcal{J}(T_{\rm III})=\Theta\bracket{\kappa_2^2p^{-\frac{1}{1-\alpha\cos\Delta}}},
    \\
    \mathcal{I}(T_{\rm III})-\mathcal{J}(T_{\rm III})\cos\Delta=0,\quad
    \mathcal{I}(T_{\rm III})\cos\Delta-\mathcal{J}(T_{\rm III})=-\Theta\bracket{\kappa_2^2\Delta^2p^{-\frac{1}{1-\alpha\cos\Delta}}}.
\end{gather*}

\underline{Step II. A rough estimate on $\cI(t)$ and $\cJ(t)$.} In this step, we aim to prove:
\[
\cJ(t)>\cI(t)>0,\quad\cI(t)+\cJ(t)\leq\cI(T_{\rm III})+\cJ(T_{\rm III}),\quad\forall t\in[T_{\rm III},\infty).
\]
First, from the definition of $\cI(t)$ and $\cJ(t)$, we have $\cI(t)>0$ and $\cJ(t)>0$.

Then we consider the dynamics of $\cI(t)+\cJ(t)$. From
\begin{align*}
    &\frac{\mathrm{d}}{\mathrm{d}t}\Big(\mathcal{I}(t)+\mathcal{J}(t)\Big)=2\mathcal{I}(t)\mathcal{J}(t)\cos\Delta-\mathcal{I}^2(t)\left(1+\epsilon\right)-\mathcal{J}^2(t)
    \\
    =&-\left(\mathcal{I}(t)-\mathcal{J}(t)\right)^2\cos\Delta-(1-\cos\Delta)\mathcal{J}^2(t)-\mathcal{I}^2(t)\left(1+\epsilon-\cos\Delta\right)
    <0,
\end{align*}
we have 
\[\mathcal{I}(t)+\mathcal{J}(t)\leq\mathcal{I}(T_{\rm III})+\mathcal{J}(T_{\rm III}),\quad \forall t\geq T_{\rm III}.\]

Then we consider the dynamics of $\mathcal{J}(t)-\mathcal{I}(t)$. We define the hitting time
\[
\tau_{\mathcal{J}-\mathcal{I}}:=\inf\Big\{t\geq T_{\rm III}:\mathcal{J}(t)\leq\mathcal{I}(t)\Big\}.
\]
From Step I, we know $\mathcal{J}(T_{\rm III})-\mathcal{I}(T_{\rm III})=(1-\cos\Delta)\mathcal{J}(T_{\rm III})>0$. From the continuity, $\tau_{\mathcal{J}-\mathcal{I}}$ exists and $\tau_{\mathcal{J}-\mathcal{I}}>T_{\rm III}$.
 
For any $t\in[T_{\rm III},\tau_{\mathcal{J}-\mathcal{I}})$, we have $\cJ(t)-\cI(t)>0$ and
\begin{align*}
    &\frac{\mathrm{d}}{\mathrm{d}t}\Big(\mathcal{J}(t)-\mathcal{I}(t)\Big)=-\mathcal{J}^2(t)+\mathcal{I}^2(t)\left(1+\epsilon\right)
    =-(\mathcal{J}(t)+\mathcal{I}(t))(\mathcal{J}(t)-\mathcal{I}(t))+\epsilon\mathcal{I}^2(t)
    \\>&
    -(\mathcal{J}(t)+\mathcal{I}(t))(\mathcal{J}(t)-\mathcal{I}(t))
    \geq
    -(\mathcal{J}(T_{\rm III})+\mathcal{I}(T_{\rm III}))
    (\mathcal{J}(t)-\mathcal{I}(t)),
\end{align*}
We consider the auxiliary ODE: $\frac{d}{\mathrm{d}t}\mathcal{P}(t)=-(\mathcal{J}(T_{\rm III})+\mathcal{I}(T_{\rm III}))\mathcal{P}(t)$, where $\mathcal{P}(T_{\rm III})=\mathcal{J}(T_{\rm III})-\mathcal{I}(T_{\rm III})>0$. From the Comparison Principle of ODEs, we have:
\begin{align*}
\mathcal{J}(t)-\mathcal{I}(t) 
\geq\mathcal{P}(t)
=\left(\mathcal{J}(T_{\rm III})-\mathcal{I}(T_{\rm III})\right)\exp\Big(-(\mathcal{J}(T_{\rm III})+\mathcal{I}(T_{\rm III}))(t-T_{\rm III})\Big)>0,\ \forall t\in[T_{\rm I},\tau_{\mathcal{U}-\mathcal{V}}).
\end{align*}
From the definition of $\tau_{\mathcal{J}-\mathcal{I}}$, we have proved 
\begin{gather*}
\tau_{\mathcal{J}-\mathcal{I}}=+\infty;
\\
\mathcal{J}(t)>\mathcal{I}(t),\ \forall t\in[T_{\rm III},+\infty).
\end{gather*}

\underline{Step III. A rough two-side bound for $\cI(t)/\cJ(t)$.} 

In Step II, we have given a rough upper bound for $\cI(t)/\cJ(t)$: $\cI(t)/\cJ(t)<1$, $\forall t\geq T_{\rm III}$. And we want to derive a lower bound for $\cI(t)/\cJ(t)$ in this step. We aim to prove:
\begin{align*}
    \cI(t)/\cJ(t)>\frac{1+\cos\Delta}{1+\cos\Delta+\epsilon},\quad\forall t\in[T_{\rm III},+\infty).
\end{align*}
First, we define the hitting time:
\[
\tau_{\cI/\cJ}^{l}:=\inf\Big\{t\geq T_{\rm III}:\mathcal{I}(t)\leq\frac{1+\cos\Delta}{1+\cos\Delta+\epsilon}\mathcal{J}(t)\Big\}.
\]
From Step I, we know
\begin{align*}
    &\mathcal{I}(T_{\rm III})-\frac{1+\cos\Delta}{1+\cos\Delta+\epsilon}\mathcal{J}(T_{\rm III})>\bracket{\cos\Delta-\frac{1+\cos\Delta}{1+\cos\Delta+\epsilon}}\mathcal{J}(T_{\rm III})
    \\=&\frac{\bracket{\frac{m_+}{m_-}\cos\Delta-1}\sin^2\Delta}{1+\cos\Delta+\epsilon}\mathcal{J}(T_{\rm III}){\geq}\frac{\bracket{\frac{\cos\Delta}{0.977}-1}\sin^2\Delta}{1+\cos\Delta+\epsilon}\mathcal{J}(T_{\rm III})
    \\{\geq}&\frac{\bracket{\frac{0.980}{0.977}-1}\sin^2\Delta}{1+\cos\Delta+\epsilon}\mathcal{J}(T_{\rm III})>0.
\end{align*} 
From the continuity, $\tau_{\cI/\cJ}^{l}$ exists and $\tau_{\cI/\cJ}^{l}>T_{\rm III}$.

For any $t\in[T_{\rm III},\tau_{\cI/\cJ}^{l})$, we have $\mathcal{I}(t)-\frac{1+\cos\Delta}{1+\cos\Delta+\epsilon}\mathcal{J}(t)>0$ and
\begin{align*}
    &\frac{\rd}{\rd t}\bracket{\mathcal{I}(t)-\frac{1+\cos\Delta}{1+\cos\Delta+\epsilon}\mathcal{J}(t)}
    \\=&\cI(t)\cJ(t)\cos\Delta\Big(1-\frac{1+\cos\Delta}{1+\cos\Delta+\epsilon}\Big)-(1+\epsilon)\cI^2(t)+\frac{1+\cos\Delta}{1+\cos\Delta+\epsilon}\cJ^2(t)
    \\=&-\bracket{\cI(t)-\frac{1+\cos\Delta}{1+\cos\Delta+\epsilon}\mathcal{J}(t)}\bracket{(1+\epsilon)\cI(t)+\cJ(t)}
    \\>&-(1+\epsilon)\bracket{\cI(t)-\frac{1+\cos\Delta}{1+\cos\Delta+\epsilon}\mathcal{J}(t)}\bracket{\cI(t)+\cJ(t)}
    \\\geq&-(1+\epsilon)\bracket{\cI(T_{\rm III})+\cJ(T_{\rm III})}\bracket{\cI(t)-\frac{1+\cos\Delta}{1+\cos\Delta+\epsilon}\mathcal{J}(t)}.
\end{align*}

We consider the auxiliary ODE: $\frac{d}{\mathrm{d}t}\mathcal{Q}(t)=-(1+\epsilon)(\mathcal{I}(T_{\rm III})+\mathcal{J}(T_{\rm III}))\mathcal{Q}(t)$, where $\mathcal{Q}(T_{\rm III})=\cI(T_{\rm III})-\frac{1+\cos\Delta}{1+\cos\Delta+\epsilon}\mathcal{J}(T_{\rm III})>0$. From the Comparison Principle of ODEs, we have:
\begin{align*}
&\cI(t)-\frac{1+\cos\Delta}{1+\cos\Delta+\epsilon}\mathcal{J}(t) 
\geq\mathcal{Q}(t)
\\=&\left(\mathcal{I}(T_{\rm III})-\frac{1+\cos\Delta}{1+\cos\Delta+\epsilon}\mathcal{J}(T_{\rm III})\right)\exp\Big(-(1+\epsilon)(\mathcal{I}(T_{\rm III})+\mathcal{J}(T_{\rm III}))(t-T_{\rm III})\Big)>0,\ \forall t\in[T_{\rm III},\tau_{\cI/\cJ}^{l}).
\end{align*}
From the definition of $\tau_{\mathcal{J}-\mathcal{I}}$, we have proved 
\begin{gather*}
\tau_{\cI/\cJ}^{l}=+\infty;
\\
\mathcal{I}(t)/\mathcal{J}(t)>\frac{1+\cos\Delta}{1+\cos\Delta+\epsilon},\quad \forall t\in[T_{\rm III},+\infty).
\end{gather*}
Hence, we obtain the two-side bound for $\mathcal{I}(t)/\mathcal{J}(t)$:
\begin{align*}
    \frac{1+\cos\Delta}{1+\cos\Delta+\epsilon}<\frac{\mathcal{I}(t)}{\mathcal{J}(t)}<1,\quad \forall t\in[T_{\rm III},+\infty).
\end{align*}

\underline{Step IV. $\cI(t)\cos\Delta-\cJ(t)$ and $\cI(t)-\cJ(t)\cos\Delta$ are both negative.}

The estimate on $\cI(t)\cos\Delta-\cJ(t)$ is straight-forward:
\begin{align*}
    \cI(t)\cos\Delta-\cJ(t)<\cI(t)\cos\Delta-\cI(t)<0.
\end{align*}

As for $\cI(t)-\cJ(t)\cos\Delta$, we will actually prove a tighter upper bound:
\[
\frac{\cI(t)}{\cJ(t)}<\cos\Delta.
\]
We need to do finer analysis using the specific dynamics \eqref{equ: GF Phase IV* U V dynamics}. First, we define the following hitting time:

Define the following hitting time
\begin{align*}
\tau_{\cI/\cJ}^u:=\inf\Big\{t> T_{\rm I}:\cI(t)-\cJ(t)\cos\Delta\geq0\}.
\end{align*}
From $\cI(T_{\rm III})-\cJ(T_{\rm III})\cos\Delta=0$, $\frac{\rd}{\rd t}(\cI(T_{\rm III})-\cJ(T_{\rm III})\cos\Delta)<0$ and the continuity, we know $\tau_{\cI/\cJ}^u$ exists and $\tau_{\cI/\cJ}^u>T_{\rm III}$. 

Recalling $\cI(t)\cos\Delta-\cJ(t)<0$, we have
\[
\frac{\rd}{\rd t}\cJ(t)=\cJ(t)\bracket{\cI(t)\cos\Delta-\cJ(t)}<0,\quad\forall t\geq T_{\rm III}.
\]
So we can consider the following dynamics for $t\in[T_{\rm III},\tau_{\cI/\cJ}^u]$:
\begin{align*}
    \frac{\mathrm{d}\mathcal{I}}{\mathrm{d}\mathcal{J}}
    =\frac{\mathcal{I}\mathcal{J}\cos\Delta-\mathcal{I}^2(1+\epsilon)}{\mathcal{I}\mathcal{J}\cos\Delta-\mathcal{J}^2}
    =\frac{\frac{\mathcal{I}}{\mathcal{J}}\cos\Delta-\bracket{\frac{\mathcal{I}}{\mathcal{J}}}^2(1+\epsilon)}{\frac{\mathcal{I}}{\mathcal{J}}\cos\Delta-1}.
\end{align*}
If we define $\mathcal{Z}(t):=\frac{\mathcal{I}(t)}{\mathcal{J}(t)}$, then we have  $\mathrm{d}\mathcal{I}=\mathcal{Z}\mathrm{d}\mathcal{J}+\mathcal{J}\mathrm{d}\mathcal{Z}$. 

The dynamics above can be transformed to:
\begin{align*}
    \mathcal{J}\frac{\mathrm{d}\mathcal{Z}}{\mathrm{d}\mathcal{J}}=\frac{\mathcal{Z}\cos\Delta-\mathcal{Z}^2(1+\epsilon)}{\mathcal{Z}\cos\Delta-1}-\mathcal{Z}=\frac{\mathcal{Z}(1+\cos\Delta)-\mathcal{Z}^2(1+\cos\Delta+\epsilon)}{\mathcal{Z}\cos\Delta-1}.
\end{align*}
Recalling the result in Step III, we have $(1+\cos\Delta)-(1+\cos\Delta+\epsilon)\mathcal{Z}(t)<0$ holds for any $t\geq T_{\rm III}$. So the dynamics is equal to:
\begin{align*}
    \frac{\mathrm{d}\mathcal{J}}{\mathcal{J}}
    =&\bracket{\frac{\mathcal{Z}\cos\Delta-1}{\mathcal{Z}(1+\cos\Delta)-\mathcal{Z}^2(1+\cos\Delta+\epsilon)}}\rd\cZ
    \\=&-\frac{1}{1+\cos\Delta}\bracket{\frac{1}{\cZ}+\frac{\sin^2\Delta+\epsilon}{1+\cos\Delta-\cZ(1+\cos\Delta+\epsilon)}}\rd\cZ.
\end{align*}
Integrating this equation from $T_{\rm III}$ to $t\in[T_{\rm III},\tau_{\cI/\cJ}^u)$, we have:
\begin{equation}\label{equ of proof: lemma: Phase IV* U V dynamics: equ J Z}
\begin{aligned}
    \log\left(\frac{\mathcal{J}(t)}{\mathcal{J}(T_{\rm III})}\right)
    =&
    -\frac{1}{1+\cos\Delta}\log\left(\frac{\mathcal{Z}(t)}{\mathcal{Z}(T_{\rm III})}\right)
    \\&+\frac{\sin^2\Delta+\epsilon}{(1+\cos\Delta+\epsilon)(1+\cos\Delta)}\log\left(\frac{(1+\cos\Delta+\epsilon)\mathcal{Z}(t)-(1+\cos\Delta)}{(1+\cos\Delta+\epsilon)\mathcal{Z}(T_{\rm I})-(1+\cos\Delta)}\right).
\end{aligned}
\end{equation}
If we assume $\tau_{\cI/\cJ}^u<+\infty$, the continuity gives us 
\begin{align*}
\lim_{t\to{\tau_{\cI/\cJ}^u}^-}\cZ(t)=\lim_{t\to{\tau_{\cI/\cJ}^u}^-}\cI(t)/\cJ(t)=\cos\Delta=\cI(T_{\rm III})/\cJ(T_{\rm III})=\cZ(T_{\rm III}).
\end{align*}
Then letting $t\to{\tau_{\cI/\cJ}^u}^-$ in \eqref{equ of proof: lemma: Phase IV* U V dynamics: equ J Z}, we have
\begin{align*}
    \lim_{t\to{\tau_{\cI/\cJ}^u}^-}\log\left(\frac{\mathcal{J}(t)}{\mathcal{J}(T_{\rm III})}\right)=0+0=0,
\end{align*}
which means $\mathcal{J}(\tau_{\cI/\cJ}^u)=\mathcal{J}(T_{\rm III})$.

But on the other hand, we have:
\begin{align*}
    &\cJ(\tau_{\cI/\cJ}^u)=\mathcal{J}(T_{\rm III})+\int_{T_{\rm III}}^{\tau_{\cI/\cJ}^u}(\cI(t)\cJ(t)\cos\Delta-\cJ^2(t))\rd t
    \\=&\mathcal{J}(T_{\rm III})+\int_{T_{\rm III}}^{\tau_{\cI/\cJ}^u}\cJ(t)(\cI(t)\cos\Delta-\cJ(t))\rd t
     \\<&\mathcal{J}(T_{\rm III})+(\cos\Delta-1)\int_{T_{\rm III}}^{\tau_{\cI/\cJ}^u}\cJ(t)\cI(t)\rd t<\mathcal{J}(T_{\rm III}),
\end{align*}
which leads to a contradiction. Hence, we have proved
\begin{gather*}
    \tau_{\cI/\cJ}^u=+\infty;
    \\
    \cI(t)-\cJ(t)\cos\Delta<0,\quad\forall t\in[T_{\rm III},+\infty).
\end{gather*}
Moreover, we obtain a sharper two-side bound for $\cI(t)/\cJ(t)$:
\begin{equation}\label{equ of proof: lemma: Phase IV* two-side bound I/J}
\begin{aligned}
    \frac{1+\cos\Delta}{1+\cos\Delta+\epsilon}<\frac{\mathcal{I}(t)}{\mathcal{J}(t)}<\cos\Delta,\quad \forall t\in[T_{\rm III},+\infty).
\end{aligned}
\end{equation}

\underline{Step V. Tight bound for $\cI(t)$ and $\cJ(t)$.}

In this step, we aim to give a tight bound for $\cI(t)+\cJ(t)$. With the help of the two-side bound \eqref{equ of proof: lemma: Phase IV* two-side bound I/J}, we have
\begin{align*}
    &\frac{\mathrm{d}}{\mathrm{d}t}\Big(\mathcal{I}(t)+\mathcal{J}(t)\Big)=2\mathcal{I}(t)\mathcal{J}(t)\cos\Delta-\mathcal{I}^2(t)\left(1+\epsilon\right)-\mathcal{J}^2(t)
    \\=&
    -\left(\mathcal{I}(t)-\mathcal{J}(t)\right)^2\cos\Delta-(1-\cos\Delta)\mathcal{J}^2(t)-\mathcal{I}^2(t)\left(1+\epsilon-\cos\Delta\right)
    \\<&-(1-\cos\Delta)\left(\mathcal{J}^2(t)+\mathcal{I}^2(t)\right)<-\frac{(1-\cos\Delta)\left(\mathcal{I}(t)+\mathcal{J}(t)\right)^2}{2}<-\frac{\Delta^2}{6}\left(\mathcal{I}(t)+\mathcal{J}(t)\right)^2,
\end{align*}
\begin{align*}
    &\frac{\mathrm{d}}{\mathrm{d}t}\Big(\mathcal{I}(t)+\mathcal{J}(t)\Big)=2\mathcal{I}(t)\mathcal{J}(t)\cos\Delta-\mathcal{I}^2(t)\left(1+\epsilon\right)-\mathcal{J}^2(t)
    \\=&
    -\left(\mathcal{I}(t)-\mathcal{J}(t)\right)^2\cos\Delta-(1-\cos\Delta)\mathcal{J}^2(t)-\mathcal{I}^2(t)\left(1+\epsilon-\cos\Delta\right)
    \\>&
    -\bracket{1-\frac{1+\cos\Delta}{1+\cos\Delta+\epsilon}}^2\cJ^2(t)-(1-\cos\Delta)\mathcal{J}^2(t)-\mathcal{I}^2(t)\left(1+\epsilon-\cos\Delta\right)
    \\>&
    -\left(1+\epsilon-\cos\Delta\right)\bracket{\cI^2(t)+\cJ^2(t)}>-\left(\frac{2}{3}+\frac{m_+}{m_-}\right)\Delta^2\bracket{\cI(t)+\cJ(t)}^2{>}-2\Delta^2\bracket{\cI(t)+\cJ(t)}^2.
\end{align*}
For the first inequality, we consider the auxiliary ODE: $\frac{d}{\mathrm{d}t}\mathcal{P}(t)=-\frac{\Delta^2}{6}\cP^2(t)$, where $\mathcal{P}(T_{\rm III})=\cI(T_{\rm III})+\cJ(T_{\rm III})>0$. From the Comparison Principle of ODEs, we have:
\begin{align*}
\cI(t)+\cJ(t)\leq\cP(t)=\frac{1}{\frac{1}{\cI(T_{\rm III})+\cJ(T_{\rm III})}+\frac{\Delta^2}{6}(t-T_{\rm III})},
\quad\forall t\geq T_{\rm III}.
\end{align*}
In the same way, we can obtain the lower bound:
\begin{align*}
\cI(t)+\cJ(t)\geq\frac{1}{\frac{1}{\cI(T_{\rm III})+\cJ(T_{\rm III})}+2\Delta^2(t-T_{\rm III})},
\quad\forall t\geq T_{\rm III}.
\end{align*}
Recalling Step I, we have $\frac{1}{\cI(T_{\rm III})+\cJ(T_{\rm III})}=\Theta\bracket{p^{\frac{1}{1-\alpha\cos\Delta}}/\kappa_2^2}$. Hence, we obtain the tight bound:
\begin{align*}
    &\cI(t)+\cJ(t)=\Theta\bracket{\frac{1}{\frac{p^{\frac{1}{1-\alpha\cos\Delta}}}{\kappa_2^2}+\Delta^2(t-T_{\rm III})}},\quad\forall t\in[T_{\rm III},+\infty).
\end{align*}

Taking \eqref{equ of proof: lemma: Phase IV* two-side bound I/J} into the equation above, we have:
\begin{align*}
\cI(t)=&\Theta\bracket{\frac{1}{\frac{p^{\frac{1}{1-\alpha\cos\Delta}}}{\kappa_2^2}+\Delta^2(t-T_{\rm III})}},\quad\forall t\in[T_{\rm III},+\infty);
\\
\cJ(t)=&\Theta\bracket{\frac{1}{\frac{p^{\frac{1}{1-\alpha\cos\Delta}}}{\kappa_2^2}+\Delta^2(t-T_{\rm III})}},\quad\forall t\in[T_{\rm III},+\infty).
\end{align*}

\underline{Step VI. The limit of $\cI(t)/\cJ(t)$.}

By Step V and the proof of Step IV, we know $\lim\limits_{t\to\infty}\cJ(t)=0$ and $\frac{\rd}{\rd t}\cJ(t)<0$ holds for any $t>T_{\rm III}$. 

Then for any $\epsilon'>0$, there exists $T'>T_{\rm III}$ such that 
\[
\log\bracket{\frac{\cJ(t)}{\cJ(T_{\rm III})}}<\frac{1000}{\Delta^2}\log(1000\epsilon\Delta^2),\quad\forall t>T'.
\]

Taking it into \eqref{equ of proof: lemma: Phase IV* U V dynamics: equ J Z}, we obtain that for any $t>T'$,
\begin{align*}
    0<(1+\cos\Delta+\epsilon)\cZ(t)-(1+\cos\Delta)<\epsilon'.
\end{align*}
By the definition of the limit, we get
\begin{align*}
    \lim\limits_{t\to\infty}\frac{\cI(t)}{\cJ(t)}=\frac{1+\cos\Delta}{1+\cos\Delta+\epsilon}=\frac{1+\cos\Delta}{1+\cos\Delta+\frac{m_+}{m_-}\sin^2\Delta}.
\end{align*}

\end{proof}

\begin{lemma}[Time and prediction estimate]\label{lemma: GF Phase IV* prediction estimate}\ \\
{\bf (S1).} For any $t\in(T_{\rm III},+\infty)$
\begin{gather*}
pe^{-f_+(t)}=\Theta\bracket{\frac{1}{p^{\frac{1}{1-\alpha\cos\Delta}}+\kappa_2^2\Delta^2(t-T_{\rm III})}},\ 
e^{f_-(t)}=\Theta\bracket{\frac{1}{p^{\frac{1}{1-\alpha\cos\Delta}}+\kappa_2^2\Delta^2(t-T_{\rm III})}};
\\
\cL(\btheta(t))=\Theta\bracket{\frac{1}{p^{\frac{1}{1-\alpha\cos\Delta}}+\kappa_2^2\Delta^2(t-T_{\rm III})}}.
\end{gather*}
{\bf (S2).} For any $t\in(T_{\rm III},+\infty)$,
\begin{gather*} 
    \frac{1+\cos\Delta}{1+\cos\Delta+\frac{m_+}{m_-}\sin^2\Delta}<pe^{-(f_+(t)+f_-(t))}<\cos\Delta.
\end{gather*}
Moreover, $pe^{-(f_+(T_{\rm III})+f_-(T_{\rm III}))}=\cos\Delta$ and $\lim\limits_{t\to\infty}pe^{-(f_+(t)+f_-(t))}=\frac{1+\cos\Delta}{1+\cos\Delta+\frac{m_+}{m_-}\sin^2\Delta}$.

{\bf (S3).} For any $t\in(T_{\rm III},+\infty)$,
\begin{align*}
    &\<\bb_k(t),\bx_+\>>0,\ \<\bb_k(t),\bx_-\>=0,\ k\in\cK_+;
    \\
    &\<\bb_k(t),\bx_+\>>0,\ \<\bb_k(t),\bx_-\>>0,\ k\in\cK_-.
\end{align*}

{\bf (S4) (Time).}
\begin{align*}
   T_{\rm IV}=T_{\rm IV}^*=+\infty.
\end{align*}

\end{lemma}

\begin{proof}[Proof of Lemma \ref{lemma: GF Phase IV* prediction estimate}]\ \\
Notice the relationships: $pe^{-f_+(t)}=\kappa_2^2\frac{m_-}{m}\frac{\cI(t)}{1+p}$, $e^{f_-(t)}=\kappa_2^2\frac{m_-}{m}\frac{\cJ(t)}{1+p}$ and $pe^{-(f_+(t)+f_-(t))}=\cI(t)/\cJ(t)$. 
Then Lemma \ref{lemma: GF Phase IV* U V dynamics} implies that $T_{\rm IV}^*=+\infty$. Recalling the dynamics in Lemma~\ref{lemma: GF Phase IV* neuron dynamics}, then lemma (S3)(S4) hold. Then using Lemma \ref{lemma: GF Phase IV* U V dynamics} again, we obtain (S1)(S2).
\end{proof}

\begin{proof}[Proof of Theorem~\ref{thm: GF Phase IV}]\ \\
Theorem~\ref{thm: GF Phase IV} (S1), (S2), and (S3) are obtained in Lemma~\ref{lemma: GF Phase IV* prediction estimate} (S4), (S3), and (S1), respectively.
\end{proof}

\subsection{Asymptotic Directional Convergence}

In this section, we will study the final convergence direction in our setting. It mainly depends on our prior fine-grained analysis of the training dynamics in Phase IV and the following result about the final convergence direction at the end of training.

\begin{lemma}\label{lemma: KKT convergence homogeneous}
Let $f(\cdot;\btheta)$ be a homogeneous neural network parameterized by $\btheta$. Consider minimizing the exponential loss over a binary classification dataset $\{(\bx_i,y_i)\}_{i=1}^n$ $(\norm{\bx_i}_2\leq1,y_i\in\{\pm1\})$ using Gradient Flow. Assume that there exists time $t_0$ such that $\cL(\btheta(t_0))<\frac{1}{n}$. Then, 

(I) (Paraphrased from \citep{lyu2019gradient,ji2020directional}). $\btheta(t)$ converges in direction to a KKT point (Definition \ref{def: definition KKT point}) of the following maximum margin problem:
\begin{align*}
    \min:&\frac{1}{2}\left\|\boldsymbol{\theta}\right\|_2^2
    \\
    {\rm s.t.}\ &y_if(\bx_i;\btheta)\geq1.
\end{align*}
(II) \citep{lyu2019gradient,ji2020directional}).
$\norm{\btheta(t)}_2\to\infty$ and $\cL(\btheta(t))\to0$.

(III) (\cite{ji2020directional}). $-\nabla\cL(\btheta(t))$ and $\btheta(t)$ converge to the same direction, meaning the angle between $\btheta(t)$ and $-\nabla\cL(\btheta(t))$ converges
to $0$.

\end{lemma}

\begin{lemma}[Final Convergence Direction]\label{lemma: final convergence direction}
The limit $\lim\limits_{t\to+\infty}\frac{\btheta(t)}{\norm{\btheta(t)}_2}$ exists, and denoted by $\overline{\btheta}=(\bar{\bb}_1^\top,\cdots,\bar{\bb}_m^\top)^\top\in\bbS^{md-1}$, then it satisfies
\begin{align*}
    \overline{\bb}_k=&\bzero,\quad\forall k\notin\cK_+\cup\cK_-;
    \\
    \overline{\bb}_k=&C\Big(\bx_+-\bx_-\cos\Delta\Big),\quad\forall k\in\cK_+;
    \\
    \overline{\bb}_k=&C\bracket{\bracket{1+\frac{m_+\sin^2\Delta}{m_-(1+\cos\Delta)}}\bx_--\bx_+},\quad\forall k\in\cK_-;
\end{align*}
where $C>0$ is a scaling constant such that $\norm{\overline{\btheta}}_2=1$. Moreover, $f(\bx_+;\overline{\btheta})=-f(\bx_-;\overline{\btheta})>0$.
\end{lemma}

\begin{proof}[Proof of Lemma \ref{lemma: final convergence direction}]\ \\
Let $\overline{\btheta}=(\overline{\bb}_1^\top,\cdots,\overline{\bb}_m^\top)^{\top}\in\bbS^{md-1}$ be the limits point of $\left\{\frac{\btheta(t)}{\norm{\btheta(t)}_2}:t\geq t_0\right\}$. From Lemma \ref{lemma: KKT convergence homogeneous} (I), we know that there exists a scaling factor $\alpha>0$ such that $\alpha\overline{\btheta}$ satisfies KKT conditions (Definition \ref{def: definition KKT point}) of the maximum-margin problem
\begin{equation}\label{equ of proof: final direction: max margin problem}
    \begin{aligned}
    \min:&\frac{1}{2}\left\|\boldsymbol{\theta}\right\|_2^2
    \\
    {\rm s.t.}&\ f(\boldsymbol{x}_+;\boldsymbol{\theta})\geq1,\ f(\boldsymbol{x}_-;\boldsymbol{\theta})\leq-1.
\end{aligned}
\end{equation}
For simplicity, we denote $\btheta^*:=\alpha\bar{\btheta}$, where
\[
\btheta^*=({\bb_1^*}^\top,\cdots,{\bb_m^*}^\top)^{\top}.
\]
Moreover, let $\lambda_+^*,\lambda_-^*\geq0$ be the corresponding Lagrange multipliers (with respect to $\btheta^*$) in Definition \ref{def: definition KKT point}.

\underline{Step I. The rough direction of each neuron.}

Recalling the training dynamics about the dead neurons in Theorem \ref{thm: GF Phase I} (S3), 
\begin{align*}
\bb_k(t)\equiv&\bb_k(T_{\rm I}),\ 
    \<\bb_k(t),\bx_+\>\leq0,\ \<\bb_k(t),\bx_-\>\leq0,\ k\in[m/2]-\cK_+;
\\
\bb_k(t)\equiv&\bb_k(T_{\rm I}),\ 
    \<\bb_k(t),\bx_+\>\leq0,\ \<\bb_k(t),\bx_-\>\leq0,\ k\in[m]-[m/2]-\cK_-.
\end{align*}
Noticing Lemma \ref{lemma: KKT convergence homogeneous} (II) or Lemma \ref{lemma: GF Phase IV* prediction estimate} (S1), $\norm{\btheta(t)}_2\to\infty$, so
\begin{align*}
    \bb_k^*=\bzero,\quad \forall k\notin\cK_+\cup\cK_-.
\end{align*}
Then we only need to focus on $\btheta_k^*$ for $k\in\cK_+\cup\cK_-$.

Recalling in Lemma \ref{lemma: GF Phase IV* prediction estimate} (S3), for any $t>T_{\rm IV}$, 
\begin{align*}
    &\<\bb_k(t),\bx_+\>>0,\ \<\bb_k(t),\bx_-\>=0,\ k\in\cK_+;
    \\
    &\<\bb_k(t),\bx_+\>>0,\ \<\bb_k(t),\bx_-\>>0,\ k\in\cK_-.
\end{align*}
then we have
\begin{align*}
    &\<\bb_k^*,\bx_+\>\geq0,\ \<\bb_k^*,\bx_-\>=0,\quad k\in\cK_+;
     \\
     &\<\bb_k^*,\bx_+\>\geq0,\ \<\bb_k^*,\bx_-\>\geq0,\quad k\in\cK_-.
\end{align*}
Moreover,
\begin{align*}
    f(\bx_+;\btheta^*)=&\frac{\kappa_2}{\sqrt{m}}\Big(\sum_{k\in\cK_+}\sigma\bracket{\<\bb_k^*,\bx_+\>}-\sum_{k\in\cK_-}\sigma\bracket{\<\bb_k^*,\bx_+\>}\Big),
    \\
    f(\bx_-;\btheta^*)=&\frac{\kappa_2}{\sqrt{m}}\Big(\sum_{k\in\cK_+}\sigma\bracket{\<\bb_k^*,\bx_-\>}-\sum_{k\in\cK_-}\sigma\bracket{\<\bb_k^*,\bx_-\>}\Big)=-\frac{\kappa_2}{\sqrt{m}}\sum_{k\in\cK_-}\sigma\bracket{\<\bb_k^*,\bx_-\>}.
\end{align*}

\underline{Step II. Determine the direction of the neurons $k\in\cK_+$.}

Since $\btheta^*$ is a possible point,  $f(\bx_+;\btheta^*)\geq1$, which gives us
\begin{align*}
    \sum_{k\in\cK_+}\sigma(\<\bb_k^*,\bx_+\>)\geq\frac{\sqrt{m}}{\kappa_2}>0.
\end{align*}
Hence, there exists $k_1\in\cK_+$, s.t. $\<\bb_{k_1}^*,\bx_+\>>0$ strictly.

Then we study the neuron $k_2\in\cK_+$ ($k_2\ne k_1$). Lemma \ref{lemma: GF Phase IV* neuron dynamics} and Lemma \ref{lemma: GF Phase IV* prediction estimate} (S3) give use that
\begin{align*}
    \<\bb_{k_1}(t),\bx_+\>=&\<\bb_{k_1}(T_{\rm IV}),\bx_+\>+\int_{T_{\rm IV}}^t\frac{\kappa_2pe^{-f_+(t)}}{\sqrt{m}(1+p)}\sin^2\Delta\rd t
    \\=&\<\bb_{k_2}(T_{\rm IV}),\bx_+\>+\int_{T_{\rm IV}}^t\frac{\kappa_2pe^{-f_+(t)}}{\sqrt{m}(1+p)}\sin^2\Delta\rd t +\Big(\<\bb_{k_1}(T_{\rm IV}),\bx_+\>-\<\bb_{k_2}(T_{\rm IV}),\bx_+\>\Big)
    \\=&\<\bb_{k_2}(t),\bx_+\>+\Big(\<\bb_{k_1}(T_{\rm IV}),\bx_+\>-\<\bb_{k_2}(T_{\rm IV}),\bx_+\>\Big).
\end{align*}
Multiplying the above formula by $c/\norm{\btheta(t)}_2$ and taking $t$ go to infinity, we obtain
\begin{align*}
    \<\bb_{k_1}^*,\bx_+\>=\<\bb_{k_2}^*,\bx_+\>>0.
\end{align*}
Due to the arbitrariness of $k_2$, we know
\begin{align*}
    \bb_k^*\ne0,\quad \<\bb_k^*,\bx_+\>>0,\quad \<\bb_k^*,\bx_-\>=0,\quad\forall k\in\cK_+.
\end{align*}

Then we can write the KKT condition about the gradient of $\bb_k^*$ ($k\in\cK_+$) of Problem \eqref{equ of proof: final direction: max margin problem}:
\begin{align*}
    \bzero\in\bb_k^*-\lambda_+^*\frac{\kappa_2}{\sqrt{m}}\bx_++\lambda_-^*\frac{\kappa_2}{\sqrt{m}}\partial^{\circ}\sigma(0)\bx_-.
\end{align*}
It is clear that $\bb_k^*\in{\rm span}\{\bx_+,\bx_-\}$. Then combining two formulations above, we obtain:
\begin{align*}
    \bb_k^*=\lambda_+^*\frac{\kappa_2}{\sqrt{m}}\bracket{\bx_+-\bx_-\cos\Delta},\quad\forall k\in\cK_+.
\end{align*}

\underline{Step III. Determine the direction of the neurons $k\in\cK_-$.}

Since $\btheta^*$ is a possible point,  $f(\bx_-;\btheta^*)\leq-1$, which gives us
\begin{align*}
    \sum_{k\in\cK_-}\sigma(\<\bb_k^*,\bx_-\>)\geq\frac{\sqrt{m}}{\kappa_2}>0.
\end{align*}
Hence, there exists $k_1\in\cK_-$, s.t. $\<\bb_{k_1}^*,\bx_-\>>0$ strictly.

Then we study the neuron $k_2\in\cK_-$ ($k_2\ne k_1$). Lemma \ref{lemma: GF Phase IV* neuron dynamics} and Lemma \ref{lemma: GF Phase IV* prediction estimate} (S3) give use that
\begin{align*}
    &\<\bb_{k_1}(t),\bx_-\>=\<\bb_{k_1}(T_{\rm IV}),\bx_+\>+\int_{T_{\rm IV}}^t\frac{\kappa_2}{\sqrt{m}}\left(\frac{1}{1+p}e^{f_-(t)}-\frac{p}{1+p}e^{-f_+(t)}\cos\Delta\right)\rd t
    \\=&\<\bb_{k_2}(T_{\rm IV}),\bx_-\>+\int_{T_{\rm IV}}^t\frac{\kappa_2}{\sqrt{m}}\left(\frac{1}{1+p}e^{f_-(t)}-\frac{p}{1+p}e^{-f_+(t)}\cos\Delta\right)\rd t+\Big(\<\bb_{k_1}(T_{\rm IV}),\bx_-\>-\<\bb_{k_2}(T_{\rm IV}),\bx_-\>\Big)
    \\=&\<\bb_{k_2}(t),\bx_-\>+\Big(\<\bb_{k_1}(T_{\rm IV}),\bx_-\>-\<\bb_{k_2}(T_{\rm IV}),\bx_-\>\Big).
\end{align*}
Multiplying the above formula by $c/\norm{\btheta(t)}_2$ and taking $t$ go to infinity, we obtain
\begin{align*}
    \<\bb_{k_1}^*,\bx_-\>=\<\bb_{k_2}^*,\bx_-\>>0.
\end{align*}
Due to the arbitrariness of $k_2$, we know
\begin{align*}
    \bb_k^*\ne0,\quad \<\bb_k^*,\bx_-\>>0,\quad \<\bb_k^*,\bx_+\>\geq0,\quad\forall k\in\cK_-.
\end{align*}

The next difficulty in this step is to determine whether $\<\bb_k^*,\bx_+\>$ can be $0$. To prove this, we will use our fine-grained analysis of training dynamics (Lemma \ref{lemma: GF Phase IV* prediction estimate}) and Lemma \ref{lemma: KKT convergence homogeneous} (III).

Let $k\in\cK_-$.
Recalling the dynamics of $\bb_{k}(t)$ in Lemma \ref{lemma: GF Phase IV* neuron dynamics}, we know
\begin{align*}
    -\frac{\partial\cL(\btheta(t))}{\partial\bb_{k}}=&-\frac{\kappa_2}{\sqrt{m}}\left(\frac{p}{1+p}e^{-f_+(t)}\boldsymbol{x}_{+}-\frac{1}{1+p}e^{f_-(t)}\boldsymbol{x}_-\right)
    \\=&\frac{\kappa_2}{\sqrt{m}}\frac{e^{f_-(t)}}{1+p}\left(\boldsymbol{x}_--pe^{-(f_+(t)+f_-(t))}\boldsymbol{x}_{+}\right).
\end{align*}
Recalling Lemma Lemma \ref{lemma: GF Phase IV* prediction estimate} (S2), $\lim\limits_{t\to\infty}pe^{-(f_+(t)+f_-(t))}=\frac{1+\cos\Delta}{1+\cos\Delta+\frac{m_+}{m_-}\sin^2\Delta}$. Then using Lemma \ref{lemma: KKT convergence homogeneous} (III), there exists $c_1>0$, s.t. 
\begin{align*}
    \bb_{k}^*=c_1\bracket{\bx_--\frac{1+\cos\Delta}{1+\cos\Delta+\frac{m_+}{m_-}\sin^2\Delta}\bx_+}.
\end{align*}
Hence, we have proved 
\begin{align*}
    \<\bb_k^*,\bx_+\>>0,\quad\forall k\in\cK_-.
\end{align*}
Then writing the KKT condition about the gradient of $\bb_{k}^*$ of Problem \eqref{equ of proof: final direction: max margin problem}:
\begin{align*}
    \bzero=\bb_k^*+\lambda_+^*\frac{\kappa_2}{\sqrt{m}}\bx_+-\lambda_-^*\frac{\kappa_2}{\sqrt{m}}\bx_-,\quad\forall k\in\cK_-.
\end{align*}
Combining the two equations about $\bb_k^*$, we obatin
\begin{gather*}
    \bb_k^*=\lambda_-^*\frac{\kappa_2}{\sqrt{m}}\bracket{\bx_--\frac{1+\cos\Delta}{1+\cos\Delta+\frac{m_+}{m_-}\sin^2\Delta}\bx_+},\quad\forall k\in\cK_-;
    \\
    \frac{\lambda_+^*}{\lambda_-^*}=\frac{1+\cos\Delta}{1+\cos\Delta+\frac{m_+}{m_-}\sin^2\Delta}.
\end{gather*}

In summary, we have proved the final convergence direction $\overline{\btheta}=(\bar{\bb}_1^\top,\cdots,\bar{\bb}_m^\top)^\top\in\bbS^{md-1}$ satisfies
\begin{align*}
    \overline{\bb}_k=&\bzero,\quad\forall k\notin\cK_+\cup\cK_-;
    \\
    \overline{\bb}_k=&C\Big(\bx_+-\bx_-\cos\Delta\Big),\quad\forall k\in\cK_+;
    \\
    \overline{\bb}_k=&C\bracket{\bracket{1+\frac{m_+\sin^2\Delta}{m_-(1+\cos\Delta)}}\bx_--\bx_+},\quad\forall k\in\cK_-;
\end{align*}
where $C>0$ is a scaling constant such that $\norm{\overline{\btheta}}_2=1$.

Moreover, a straight-forward calculation gives us that 
\begin{align*}
    f(\bx_+;\overline{\btheta})=-f(\bx_-;\overline{\btheta})>0.
\end{align*}

\end{proof}

\begin{proof}[Proof of Theorem~\ref{thm: GF  directional convergence}]\ \\
Lemma~\ref{lemma: final convergence direction} implies Theorem~\ref{thm: GF  directional convergence} directly.
\end{proof}

\begin{proof}[Proof of Theorem~\ref{thm: GF not optimum max margin}]\ \\
From Lemma~\ref{lemma: final convergence direction}, the final convergence direction $\overline{\btheta}=(\bar{\bb}_1^\top,\cdots,\bar{\bb}_m^\top)^\top\in\bbS^{md-1}$ satisfies
\begin{align*}
    \overline{\bb}_k=&\bzero,\quad\forall k\notin\cK_+\cup\cK_-;
    \\
    \overline{\bb}_k=&C\Big(\bx_+-\bx_-\cos\Delta\Big),\quad\forall k\in\cK_+;
    \\
    \overline{\bb}_k=&C\bracket{\bracket{1+\frac{m_+\sin^2\Delta}{m_-(1+\cos\Delta)}}\bx_--\bx_+},\quad\forall k\in\cK_-;
\end{align*}
where $C>0$ is a scaling constant such that $\norm{\overline{\btheta}}_2=1$ and $f_+(\overline{\btheta})=-f_-(\overline{\btheta})>0$.

It is easy to verify that There exists a scaling factor $C_1>0$ such that $f_+(\hat{\btheta})=-f_-(\hat{\btheta})=1$, where $\hat{\btheta}=C_1\overline{\btheta}$.
For simplicity, we denote $Q:=CC_1$, then $\hat{\btheta}=(\hat{\bb}_1^\top,\cdots,\hat{\bb}_m^\top)^\top\in\bbR^{md}$ satisfies
\begin{align*}
    \hat{\bb}_k=&\bzero,\quad\forall k\notin\cK_+\cup\cK_-;
    \\
    \hat{\bb}_k=&Q\Big(\bx_+-\bx_-\cos\Delta\Big),\quad\forall k\in\cK_+;
    \\
    \hat{\bb}_k=&Q\bracket{\bracket{1+\frac{m_+\sin^2\Delta}{m_-(1+\cos\Delta)}}\bx_--\bx_+},\quad\forall k\in\cK_-;
    \\
    f_+(\hat{\btheta})=&-f_-(\hat{\btheta})=1.
\end{align*}
Therefore, $\hat{\btheta}$ is a feasible point of Problem~\eqref{equ of proof: final direction: max margin problem}. 
Moreover, from
\begin{align*}
    &-1 =-\sum_{k\in\cK_-}\frac{\kappa_2}{\sqrt{m}}\<\hat{\bb}_k,\bx_-\>
    \\=&-\frac{\kappa_2}{\sqrt{m}}m_-\left<Q\bracket{\bracket{1+\frac{m_+\sin^2\Delta}{m_-(1+\cos\Delta)}}\bx_--\bx_+},\bx_-\right>
     \\=&-Q\frac{\kappa_2}{\sqrt{m}} m_- \left(1-\cos\Delta+\frac{m_+\sin^2\Delta}{m_-(1+\cos\Delta)}\right)
     \\=&-Q\frac{\kappa_2}{\sqrt{m}}\left(m_-(1-\cos\Delta)+m_+\frac{\sin^2\Delta}{1+\cos\Delta}\right),
\end{align*}
we have
\begin{align*}
    \frac{\kappa_2}{\sqrt{m}}Q=\frac{1}{m_-(1-\cos\Delta)+m_+\frac{\sin^2\Delta}{1+\cos\Delta}}.
\end{align*}

For any $\epsilon\geq0$, now we consider another solution $\hat{\btheta}(\epsilon)$ near $\hat{\btheta}$: $\hat{\btheta}(\epsilon)=(\hat{\bb}_1^\top(\epsilon),\cdots,\hat{\bb}_m^\top(\epsilon))^\top$, where 
\begin{align*}
    \hat{\bb}_k(\epsilon)=&\bzero,\quad\forall k\notin\cK_+\cup\cK_-;
    \\
    \hat{\bb}_k(\epsilon)=&Q\Big(\bx_+-\bx_-\cos\Delta\Big)-Q\epsilon\Big(\bx_+-\bx_-\cos\Delta\Big),\quad\forall k\in\cK_+;
    \\
    \hat{\bb}_k(\epsilon)=&Q\bracket{\bracket{1+\frac{m_+\sin^2\Delta}{m_-(1+\cos\Delta)}}\bx_--\bx_+}+Q\epsilon\Big(\bx_+-\bx_-\cos\Delta\Big),\quad\forall k\in\cK_-.
\end{align*}
and it holds that $\hat{\btheta}(0)=\hat{\btheta}$. Moreover, it is easy to verify that $\hat{\btheta}(\epsilon)$ is also a feasible point of Problem~\eqref{equ of proof: final direction: max margin problem}. 
\begin{align*}
    f_-(\hat{\btheta}(\epsilon))=&f_-(\hat{\btheta}(0))=-1;
    \\
    f_+(\hat{\btheta}(\epsilon))=&f_+(\hat{\btheta}(0))=1.
\end{align*}
Then we compare the norm of $\hat{\btheta}(\epsilon)$ and $\hat{\btheta}(0)$.
\begin{align*}
    \norm{\hat{\btheta}(\epsilon)}^2
    =&m_+\left(Q\Big(\bx_+-\bx_-\cos\Delta\Big)-Q\epsilon\Big(\bx_+-\bx_-\cos\Delta\Big)\right)^2
    \\&+m_-\left(Q\bracket{\bracket{1+\frac{m_+\sin^2\Delta}{m_-(1+\cos\Delta)}}\bx_--\bx_+}+Q\epsilon\Big(\bx_+-\bx_-\cos\Delta\Big)\right)^2,
\end{align*}
At $\epsilon=0$, we can calculate that
\begin{align*}
    &\frac{\rd\norm{\hat{\btheta}(\epsilon)}^2}{2m_+Q^2\rd\epsilon}\Bigg|_{\epsilon=0}
    \\=&\frac{1}{m_+}\left(-m_+\norm{\bx_+-\bx_-\cos\Delta}^2+m_-\<\bracket{1+\frac{m_+\sin^2\Delta}{m_-(1+\cos\Delta)}}\bx_--\bx_+,\bx_+-\bx_-\cos\Delta\>\right)
    \\=&\<\left(\frac{m_-}{m_+}+\frac{\sin^2\Delta}{1+\cos\Delta}+\cos\Delta\right)\bx_--(1+\frac{m_-}{m_+})\bx_+,\bx_+-\bx_-\cos\Delta\>
    \\
    \overset{\text{Denote $\alpha=\frac{m_-}{m_+}$}}{=}&\<\left(\alpha+\frac{\sin^2\Delta}{1+\cos\Delta}+\cos\Delta\right)\bx_--(1+\alpha)\bx_+,\bx_+-\bx_-\cos\Delta\>
    \\=&
    -\left(\alpha+\frac{\sin^2\Delta}{1+\cos\Delta}+\cos\Delta\right)\cos\Delta-(1+\alpha)
    \\&+\left(\alpha+\frac{\sin^2\Delta}{1+\cos\Delta}+\cos\Delta\right)\cos\Delta+(1+\alpha)\cos^2\Delta
    \\=&-(1+\alpha)\sin^2\Delta<0
\end{align*}
Then combining the continuity of $\frac{\rd\norm{\hat{\btheta}(\epsilon)}^2}{\rd \epsilon}$, there exists $\delta>0$ such that the following inequality holds:
\begin{align*}
    \norm{\hat{\btheta}(\epsilon)}^2<\norm{\hat{\btheta}(0)}^2,\ \forall\epsilon\in(0,\delta).
\end{align*}

Hence, we have proved that $\overline{\btheta}$ is not a local optimal direction of the max-margin problem~\eqref{equ of proof: final direction: max margin problem}.

\end{proof}

\newpage

\section{Clarke Subdifferential and KKT Conditions for Non-smooth Optimization}\label{section: subdifferential and KKT}

\begin{definition}[Clarke's Subdifferential~\citep{clarke2008nonsmooth}]\label{def: clark subdifferential}
For a locally Lipschitz function $\cL:\Omega\to\bbR$, the Clarke's subdifferential at $\btheta\in \Omega$ is the convex set
\[
\partial^{\circ}\cL(\btheta):={\rm conv}\left\{\lim\limits_{i\to\infty}\nabla \cL(\btheta_i):\lim\limits_{i\to\infty}\btheta_i=\btheta,\cL\text{ is differential at }\btheta_i\right\}.
\]
\end{definition}

\begin{remark}
    Notice that if $\cL$ is continuously differentiable at $\btheta$, then $\partial^{\circ}\cL(\btheta)=\{\nabla\cL(\btheta)\}$ is unique.
    However, for discontinuous differentiable points of $\cL$, the differential inclusion flow $\frac{\rd\btheta}{\rd t}\in\partial^{\circ}\cL(\btheta)$ defined by Definition~\ref{def: clark subdifferential} may not be unique. To study a more specific dynamics, 
    we also utilize Definition~\ref{def: discontinuous system solution} to determine GF at some of such points.
\end{remark}


Now we review the definition of Karush-Kuhn-Tucker (KKT) conditions for non-smooth optimization problems~\citep{dutta2013approximate}. Consider the following constrained optimization problem (P):
\begin{align*}
    \min\limits_{\bx\in\bbR^d}:&f(\bx)
    \\
    {\rm s.t.}\ &g_i(\bx)\leq0,\quad\forall i\in[N]
\end{align*}
where $f,g_1,\cdots,g_N:\bbR^d\to\bbR$ are locally Lipschitz functions. We say that $\bx\in\bbR$ is a feasible point
of (P) if $\bx$ satisfies $g_i(\bx)\leq0$ for all $i\in[N]$.

\begin{definition}[KKT Point for Non-smooth Optimization]\label{def: definition KKT point} We say that a feasible point $\bf$ of {\rm (P)} is a KKT point if there exists $\lambda_1,\cdots,\lambda_N\geq0$ such that
\begin{align*}
    {\rm 1.}&\ \boldsymbol{0}\in\partial^{\circ}f(\bx)+\sum_{i\in[N]}\lambda_i\partial^{\circ}g_i(\bx);
    \\{\rm 2.}&\ \forall i\in[N], \lambda_ig_i(\bx)=0.
\end{align*}
\end{definition}

\newpage

\section{Solution of Discontinuous System}\label{section: discontinuous system solution}

In this section, we add some supplements about the definitions of solutions of discontinuous systems, which can overcomes non-uniqueness of GF trajectories~\eqref{equ: alg GF} to some extent. Many definitions of solutions of differential equations with discontinuous systems have been proposed. In this paper, we adopt a widely used definition of the solutions in Chapter 2.4 in \citep{filippov2013differential}.

\begin{definition}[Solutions of Discontinuous Systems, Chapter 2.4 in \citep{filippov2013differential}]\label{def: discontinuous system solution}
Consider a $n$-dimensional equation or a system ($\boldsymbol{x}\in\mathbb{R}^n$): $\frac{\mathrm{d}\boldsymbol{x}}{\mathrm{d}t}=\boldsymbol{f}(\boldsymbol{x})$
with a piecewise continuous function $\boldsymbol{f}$ in a domain $G$. We aim to define the dynamics near some discontinuous regions.

Let the function $\boldsymbol{f}$ be discontinuous on a smooth surface $S$ given by the equation $\phi(\boldsymbol{x})=0$. Let $\bx^*\in S$ and the surface $S$ separate the neighborhood of $\bx^*$ into domains $G^-$ and $G^+$. 
Let the function $\boldsymbol{f}(\boldsymbol{x})$ have the limit values:
\begin{align*}
\boldsymbol{f}^-(\boldsymbol{x}^*):=\lim\limits_{\boldsymbol{x}\in G^-,\boldsymbol{x}\to\boldsymbol{x}^*} \boldsymbol{f}(\boldsymbol{x}),\quad\boldsymbol{f}^+(\boldsymbol{x}^*):=\lim\limits_{\boldsymbol{x}\in G^+,\boldsymbol{x}^*\to\boldsymbol{x}^*} \boldsymbol{f}(\boldsymbol{x}).
\end{align*}
Here one should distinguish between two main cases. Let ${f}_N^-(\boldsymbol{x}^*)$ and ${f}_N^+(\boldsymbol{x}^*)$ be projections of the vectors $\boldsymbol{f}^-(\boldsymbol{x}^*)$ and $\boldsymbol{f}^+(\boldsymbol{x}^*)$ onto the normal to the surface $S$ at the point $\boldsymbol{x}^*$, where the normal is directed towards the domain $G^+$.

\textbf{(Case I)}. If the vectors $\boldsymbol{f}(\boldsymbol{x}^*)$ are directed to the surface $S$ on both sides, i.e. ${f}_N^-(\boldsymbol{x}^*)>0$, ${f}_N^+(\boldsymbol{x}^*)<0$, then the solution the solution starting from $\bx^*$ can not leave $S$ for some time.
Moreover, its dynamics on $S$ can be defined in the following way:
\begin{gather*}
    \frac{\mathrm{d}\boldsymbol{x}}{\mathrm{d}t}=\boldsymbol{f}^{0}(\bx),
\\
\text{where\ }\boldsymbol{f}^{0}(\bx)=\alpha\boldsymbol{f}^+(\boldsymbol{x})+(1-\alpha)\boldsymbol{f}^-(\boldsymbol{x}),\quad \alpha=\frac{{f}_N^-(\boldsymbol{x})}{{f}_N^-(\boldsymbol{x})-{f}_N^+(\boldsymbol{x})}.
\end{gather*}

\textbf{(Case II)}. If $f_N^-(\boldsymbol{x}^*)\geq0,f_N^+(\boldsymbol{x}^*)\geq0$, but $f_N^-(\boldsymbol{x}^*)$ and $f_N^+(\boldsymbol{x}^*)$ are not both $0$, then the solution starting from $\bx^*$ passes from one side of the surface $S$ to the other instantly.

\textbf{(Case III).} If ${f}_N^-(\boldsymbol{x}^*)<0$, ${f}_N^+(\boldsymbol{x}^*)>0$, then the dynamics is defined in the similar way as (Case I).

\textbf{(Case IV).} If $f_N^-(\boldsymbol{x}^*)\leq0,f_N^+(\boldsymbol{x}^*)\leq0$, but $f_N^-(\boldsymbol{x}^*)$ and $f_N^+(\boldsymbol{x}^*)$ are not both $0$, then the dynamics is defined in the similar way as (Case II).
\end{definition}

\begin{remark}
    Notice that Definition~\ref{def: discontinuous system solution} overcomes non-uniqueness of GF trajectories to some extent. It is worth noting that Definition~\ref{def: discontinuous system solution} and Definition~\ref{def: clark subdifferential} are compatible and specifically, the dynamics defined in Definition~\ref{def: discontinuous system solution}(Case I, III) lie in the convex hull defined in Definition~\ref{def: clark subdifferential}.
\end{remark}
\begin{remark}
    In~\citep{lyu2021gradient}, the non-branching starting point Assumption is employed to address the technical challenge of non-uniqueness in GF trajectories. By comparison, in this work, we do not need this assumption. We adopt Definition~\ref{def: discontinuous system solution} to uniquely determine the Gradient Flow trajectories theoretically near some discontinuous differential regions, such as ``Ridge'', ``Valley'', and ``Refraction edge'' discussed in Section I.2 in~\citep{lyu2021gradient}.
\end{remark}

\newpage
\section{Some Basic Inequalities}

\begin{lemma}[Hoeffding's Inequality]\label{lemma: hoeffding} Let $X_1,\cdots,X_n$ are independent random variables, and $X_i\in[a_i,b_i]$ for any $i\in[n]$. Define $\bar{X}=\frac{1}{n}\sum_{i=1}^n X_i$. Then for any $\epsilon>0$, we have the following probability inequalities:
\begin{gather*}
    \mathbb{P}\Big(\bar{X}-\mathbb{E}[\bar{X}]\geq \epsilon\Big)\leq\exp\Big(-\frac{2n^2\epsilon^2}{\sum_{i=1}^n(b_i-a_i)^2}\Big),\\
      \mathbb{P}\Big(\bar{X}-\mathbb{E}[\bar{X}]\leq- \epsilon\Big)\leq\exp\Big(-\frac{2n^2\epsilon^2}{\sum_{i=1}^n(b_i-a_i)^2}\Big).
\end{gather*}
\end{lemma}

\begin{lemma}\label{lemma: theta in span x1, x2} Consider $\bx_1,\bx_2,\boldsymbol{y}\in\mathbb{S}^{d-1}$, where $\left<\bx_1,\bx_2\right>=\cos\Delta$ $(\Delta\in(0,\pi/2))$. If $\left<\boldsymbol{y},\bx_1\right>\geq0$ and $\left<\boldsymbol{y},\bx_2\right>\leq0$, then we have $0\leq\left<
\boldsymbol{y},\bx_1\right>\leq\sin\Delta$ and $-\sin\Delta\leq\left<
\boldsymbol{y},\bx_2\right>\leq0$.
\end{lemma}

\begin{proof}[Proof of Lemma \ref{lemma: theta in span x1, x2}]\ \\
Denote $\mathcal{M}_{\bx}:={\rm span}\{\bx_1,\bx_2\}$.
We can do the orthogonal decomposition of $\boldsymbol{y}$:
\[
\boldsymbol{y}=\boldsymbol{y}_{\mathcal{M}}+\boldsymbol{y}_{\mathcal{M}}^{\perp},
\]
where $\boldsymbol{y}_{\mathcal{M}}\in\mathcal{M}_{\bx}$ and $\boldsymbol{y}_{\mathcal{M}}^{\perp}\perp\mathcal{M}_{\bx}$. From $\boldsymbol{y}\in{\rm span}\{\bx_1,\bx_2\}$, there exist $\alpha,\beta\in\mathbb{R}$, s.t. $\boldsymbol{y}_{\mathcal{M}}=\alpha\bx_1+\beta\bx_2$.

Due to the orthogonal decomposition, we know $\left\|\boldsymbol{y}_{\mathcal{M}}\right\|\leq1$, which means $\alpha^2+\beta^2+2\alpha\beta\cos\Delta\leq1$. Noticing $\alpha^2+\beta^2+2\alpha\beta\cos\Delta=(\alpha+\beta\cos\Delta)^2+\alpha^2\sin^2\Delta$, we know $\alpha^2\sin^2\Delta\leq1$.

Due to $\left<\boldsymbol{y},\bx_1\right>\geq0$ and $\left<\boldsymbol{y},\bx_2\right>\leq0$, we have $\left<\boldsymbol{y}_{
\mathcal{M}},\bx_1\right>\geq0$ and $\left<\boldsymbol{y}_{
\mathcal{M}},\bx_2\right>\leq0$, which means
\[
\alpha+\beta\cos\Delta\geq0,\quad \alpha\cos\Delta+\beta\leq0.
\]
So $\alpha\geq0$ and $\alpha\sin\Delta\geq0$. Recalling $\alpha^2\cos^2\Delta\leq1$, we know $0\leq\alpha\sin\Delta\leq1$.
Hence, we have $\alpha+\beta\cos\Delta\leq\alpha-\alpha\cos^2\Delta=\alpha\sin^2\Delta\leq\sin\Delta$, i.e. $\left<\boldsymbol{y},\bx_1\right>\leq\sin\Delta$.

In the same way, we have $-\sin\Delta\leq\left<
\boldsymbol{y},\bx_2\right>\leq0$.

\end{proof}

\begin{lemma}\label{lemma: basic: norm of z}
If $p\geq5$, we have $\left\|\boldsymbol{z}\right\|\geq\frac{p-1}{p+1}\geq\frac{2}{3}$.
\end{lemma}

\begin{proof}[Proof of Lemma \ref{lemma: basic: norm of z}]
\begin{align*}
    \left\|\boldsymbol{z}\right\|^2
    =&\left\|\frac{1}{n}\sum_{i=1}^n y_i\bx_i\right\|^2=\left\|\frac{p}{1+p}\bx_+-\frac{1}{1+p}\bx_-\right\|^2
    =\left(\frac{p}{1+p}\right)^2+\left(\frac{1}{1+p}\right)^2-\frac{2p}{(1+p)^2}\left<\bx_+,\bx_-\right>
    \\=&\left(\frac{p}{1+p}+\frac{1}{1+p}\right)^2-\frac{2p}{(1+p)^2}\left(\cos\Delta+1\right)
    \geq1-\frac{2p}{(p+1)^2}\cdot2=\left(\frac{p-1}{p+1}\right)^2.
\end{align*}

\end{proof}

\begin{lemma}\label{lemma: mu theta to x- theta}
Let $\boldsymbol{w}\in\mathbb
S^{d-1}$. If $\left<\boldsymbol{w},\boldsymbol{\mu}\right>\geq1-\epsilon$ $(\epsilon\in(0,1))$, $p\geq5$ and $\cos\Delta\geq4/5$, then we have
\begin{align*}
-2\sqrt{\epsilon}\sin\Delta-\epsilon\leq\left<\boldsymbol{w},\bx_-\right>-\frac{p\cos\Delta-1}{\sqrt{p^2+1-2p\cos\Delta}}\leq2\sqrt{\epsilon}\sin\Delta.
\end{align*}
\end{lemma}

\begin{proof}[Proof of Lemma \ref{lemma: mu theta to x- theta}]
\begin{align*}
    &\left<\boldsymbol{w},\bx_-\right>=\left<\boldsymbol{w},\boldsymbol{\mu}\right>+\left<\boldsymbol{w},\bx_--\boldsymbol{\mu}\right>
    =\left<\boldsymbol{w},\boldsymbol{\mu}\right>+\left<\boldsymbol{\mu},\bx_--\boldsymbol{\mu}\right>+\left<\boldsymbol{w}-\boldsymbol{\mu},\bx_--\boldsymbol{\mu}\right>.
\end{align*}
It is easy to verify 
\begin{align*}
    &\left<\boldsymbol{\mu},\bx_--\boldsymbol{\mu}\right>=\left<\boldsymbol{\mu},\bx_-\right>-1=\left<\frac{p\bx_+-\bx_-}{\left\|p\bx_+-\bx_-\right\|},\bx_-\right>-1=\frac{p\cos\Delta-1}{\sqrt{p^2+1-2p\cos\Delta}}-1;
    \\
    &\left\|\bx_--\boldsymbol{\mu}\right\|=\sqrt{2-2\left<\boldsymbol{\mu},\bx_-\right>}=\sqrt{2-2\frac{p\cos\Delta-1}{\sqrt{p^2+1-2p\cos\Delta}}};
    \\
    &\left\|\boldsymbol{w}-\boldsymbol{\mu}\right\|=\sqrt{2-2\left<\boldsymbol{w},\boldsymbol{\mu}\right>}.
\end{align*}
Thus,
\begin{align*}
    &\left|\left<\boldsymbol{w}-\boldsymbol{\mu},\bx_--\boldsymbol{\mu}\right>\right|\leq\left\|\boldsymbol{w}-\boldsymbol{\mu}\right\|\left\|\bx_--\boldsymbol{\mu}\right\|
    \\=&\sqrt{2-2\frac{p\cos\Delta-1}{\sqrt{p^2+1-2p\cos\Delta}}}\sqrt{2-2\left<\boldsymbol{w},\boldsymbol{\mu}\right>}
    \leq\sqrt{2-2\frac{p\cos\Delta-1}{\sqrt{p^2+1-2p\cos\Delta}}}\sqrt{2\epsilon}
    \\\leq&\frac{\sqrt{2}}{\sqrt[4]{p^2+1-2p\cos\Delta}}\sqrt{\frac{p^2\sin^2\Delta}{\sqrt{p^2+1-2p\cos\Delta}+p\cos\Delta-1}}
    \sqrt{2\epsilon}
    \\\leq&\frac{2\sqrt{\epsilon}}{\sqrt{p-1}}\frac{p\sin\Delta}{\sqrt{p-1+p\cos\Delta-1}}\leq2\sqrt{\epsilon}\sin\Delta.
\end{align*}
Then we have the bound:
\begin{align*}
    &\left<\boldsymbol{w},\bx_-\right>\leq\frac{p\cos\Delta-1}{\sqrt{p^2+1-2p\cos\Delta}}-1+\left<\boldsymbol{w},\boldsymbol{\mu}\right>+\left|\left<\boldsymbol{w}-\boldsymbol{\mu},\bx_--\boldsymbol{\mu}\right>\right|
    \\\leq&\frac{p\cos\Delta-1}{\sqrt{p^2+1-2p\cos\Delta}}+2\sqrt{\epsilon}\sin\Delta,
\end{align*}
\begin{align*}
    &\left<\boldsymbol{w},\bx_-\right>\geq\frac{p\cos\Delta-1}{\sqrt{p^2+1-2p\cos\Delta}}-1+\left<\boldsymbol{w},\boldsymbol{\mu}\right>-\left|\left<\boldsymbol{w}-\boldsymbol{\mu},\bx_--\boldsymbol{\mu}\right>\right|
    \\\geq&\frac{p\cos\Delta-1}{\sqrt{p^2+1-2p\cos\Delta}}-\epsilon-2\sqrt{\epsilon}\sin\Delta.
\end{align*}

\end{proof}



\begin{thebibliography}{76}
\providecommand{\natexlab}[1]{#1}
\providecommand{\url}[1]{\texttt{#1}}
\expandafter\ifx\csname urlstyle\endcsname\relax
  \providecommand{\doi}[1]{doi: #1}\else
  \providecommand{\doi}{doi: \begingroup \urlstyle{rm}\Url}\fi

\bibitem[Abbe et~al.(2022{\natexlab{a}})Abbe, Adsera, and Misiakiewicz]{abbe2022merged}
Emmanuel Abbe, Enric~Boix Adsera, and Theodor Misiakiewicz.
\newblock The merged-staircase property: a necessary and nearly sufficient condition for sgd learning of sparse functions on two-layer neural networks.
\newblock In \emph{Conference on Learning Theory}, pages 4782--4887. PMLR, 2022{\natexlab{a}}.

\bibitem[Abbe et~al.(2022{\natexlab{b}})Abbe, Cornacchia, Hazla, and Marquis]{abbe2022initial}
Emmanuel Abbe, Elisabetta Cornacchia, Jan Hazla, and Christopher Marquis.
\newblock An initial alignment between neural network and target is needed for gradient descent to learn.
\newblock In \emph{International Conference on Machine Learning}, pages 33--52. PMLR, 2022{\natexlab{b}}.

\bibitem[Abbe et~al.(2023)Abbe, Boix-Adsera, and Misiakiewicz]{abbe2023sgd}
Emmanuel Abbe, Enric Boix-Adsera, and Theodor Misiakiewicz.
\newblock Sgd learning on neural networks: leap complexity and saddle-to-saddle dynamics.
\newblock \emph{arXiv preprint arXiv:2302.11055}, 2023.

\bibitem[Ahn et~al.(2022{\natexlab{a}})Ahn, Bubeck, Chewi, Lee, Suarez, and Zhang]{ahn2022learning}
Kwangjun Ahn, S{\'e}bastien Bubeck, Sinho Chewi, Yin~Tat Lee, Felipe Suarez, and Yi~Zhang.
\newblock Learning threshold neurons via the" edge of stability".
\newblock \emph{arXiv preprint arXiv:2212.07469}, 2022{\natexlab{a}}.

\bibitem[Ahn et~al.(2022{\natexlab{b}})Ahn, Zhang, and Sra]{ahn2022understanding}
Kwangjun Ahn, Jingzhao Zhang, and Suvrit Sra.
\newblock Understanding the unstable convergence of gradient descent.
\newblock In \emph{International Conference on Machine Learning}, pages 247--257. PMLR, 2022{\natexlab{b}}.

\bibitem[Allen-Zhu et~al.(2019)Allen-Zhu, Li, and Song]{allen2019convergence}
Zeyuan Allen-Zhu, Yuanzhi Li, and Zhao Song.
\newblock A convergence theory for deep learning via over-parameterization.
\newblock In \emph{International Conference on Machine Learning}, pages 242--252. PMLR, 2019.

\bibitem[Arora et~al.(2019)Arora, Du, Hu, Li, and Wang]{arora2019fine}
Sanjeev Arora, Simon Du, Wei Hu, Zhiyuan Li, and Ruosong Wang.
\newblock Fine-grained analysis of optimization and generalization for overparameterized two-layer neural networks.
\newblock In \emph{International Conference on Machine Learning}, pages 322--332. PMLR, 2019.

\bibitem[Arpit et~al.(2017)Arpit, Jastrz{\k{e}}bski, Ballas, Krueger, Bengio, Kanwal, Maharaj, Fischer, Courville, Bengio, et~al.]{arpit2017closer}
Devansh Arpit, Stanis{\l}aw Jastrz{\k{e}}bski, Nicolas Ballas, David Krueger, Emmanuel Bengio, Maxinder~S Kanwal, Tegan Maharaj, Asja Fischer, Aaron Courville, Yoshua Bengio, et~al.
\newblock A closer look at memorization in deep networks.
\newblock In \emph{International conference on machine learning}, pages 233--242. PMLR, 2017.

\bibitem[Atanasov et~al.(2022)Atanasov, Bordelon, and Pehlevan]{atanasov2021neural}
Alexander Atanasov, Blake Bordelon, and Cengiz Pehlevan.
\newblock Neural networks as kernel learners: The silent alignment effect.
\newblock \emph{International Conference on Learning Representations}, 2022.

\bibitem[Blanc et~al.(2020)Blanc, Gupta, Valiant, and Valiant]{blanc2020implicit}
Guy Blanc, Neha Gupta, Gregory Valiant, and Paul Valiant.
\newblock Implicit regularization for deep neural networks driven by an ornstein-uhlenbeck like process.
\newblock In \emph{Conference on learning theory}, pages 483--513. PMLR, 2020.

\bibitem[Bolte et~al.(2010)Bolte, Daniilidis, Ley, and Mazet]{bolte2010characterizations}
J{\'e}r{\^o}me Bolte, Aris Daniilidis, Olivier Ley, and Laurent Mazet.
\newblock Characterizations of {\l}ojasiewicz inequalities: subgradient flows, talweg, convexity.
\newblock \emph{Transactions of the American Mathematical Society}, 362\penalty0 (6):\penalty0 3319--3363, 2010.

\bibitem[Boursier et~al.(2022)Boursier, Pillaud-Vivien, and Flammarion]{boursier2022gradient}
Etienne Boursier, Loucas Pillaud-Vivien, and Nicolas Flammarion.
\newblock Gradient flow dynamics of shallow relu networks for square loss and orthogonal inputs.
\newblock \emph{Advances in Neural Information Processing Systems}, 2022.

\bibitem[Brutzkus et~al.(2017)Brutzkus, Globerson, Malach, and Shalev-Shwartz]{brutzkus2017sgd}
Alon Brutzkus, Amir Globerson, Eran Malach, and Shai Shalev-Shwartz.
\newblock Sgd learns over-parameterized networks that provably generalize on linearly separable data.
\newblock \emph{arXiv preprint arXiv:1710.10174}, 2017.

\bibitem[Chatterji et~al.(2021)Chatterji, Long, and Bartlett]{chatterji2021doesA}
Niladri~S Chatterji, Philip~M Long, and Peter~L Bartlett.
\newblock When does gradient descent with logistic loss find interpolating two-layer networks?
\newblock \emph{Journal of Machine Learning Research}, 22\penalty0 (159):\penalty0 1--48, 2021.

\bibitem[Chen et~al.(2023)Chen, Li, Luo, Zhou, and Xu]{chen2023phase}
Zhengan Chen, Yuqing Li, Tao Luo, Zhangchen Zhou, and Zhi-Qin~John Xu.
\newblock Phase diagram of initial condensation for two-layer neural networks.
\newblock \emph{arXiv preprint arXiv:2303.06561}, 2023.

\bibitem[Chizat and Bach(2018)]{chizat2018global}
Lenaic Chizat and Francis Bach.
\newblock On the global convergence of gradient descent for over-parameterized models using optimal transport.
\newblock \emph{Advances in neural information processing systems}, 31, 2018.

\bibitem[Chizat and Bach(2020)]{chizat2020implicit}
Lenaic Chizat and Francis Bach.
\newblock Implicit bias of gradient descent for wide two-layer neural networks trained with the logistic loss.
\newblock In \emph{Conference on Learning Theory}, pages 1305--1338. PMLR, 2020.

\bibitem[Clarke et~al.(2008)Clarke, Ledyaev, Stern, and Wolenski]{clarke2008nonsmooth}
Francis~H Clarke, Yuri~S Ledyaev, Ronald~J Stern, and Peter~R Wolenski.
\newblock \emph{Nonsmooth analysis and control theory}, volume 178.
\newblock Springer Science \& Business Media, 2008.

\bibitem[Cohen et~al.(2021)Cohen, Kaur, Li, Kolter, and Talwalkar]{cohen2021gradient}
Jeremy~M Cohen, Simran Kaur, Yuanzhi Li, J~Zico Kolter, and Ameet Talwalkar.
\newblock Gradient descent on neural networks typically occurs at the edge of stability.
\newblock \emph{arXiv preprint arXiv:2103.00065}, 2021.

\bibitem[Damian et~al.(2022)Damian, Nichani, and Lee]{damian2022self}
Alex Damian, Eshaan Nichani, and Jason~D Lee.
\newblock Self-stabilization: The implicit bias of gradient descent at the edge of stability.
\newblock \emph{arXiv preprint arXiv:2209.15594}, 2022.

\bibitem[Du et~al.(2019)Du, Lee, Li, Wang, and Zhai]{du2019gradient}
Simon Du, Jason Lee, Haochuan Li, Liwei Wang, and Xiyu Zhai.
\newblock Gradient descent finds global minima of deep neural networks.
\newblock In \emph{International Conference on Machine Learning}, pages 1675--1685. PMLR, 2019.

\bibitem[Du et~al.(2018)Du, Zhai, Poczos, and Singh]{du2018gradient}
Simon~S Du, Xiyu Zhai, Barnabas Poczos, and Aarti Singh.
\newblock Gradient descent provably optimizes over-parameterized neural networks.
\newblock \emph{arXiv preprint arXiv:1810.02054}, 2018.

\bibitem[Dutta et~al.(2013)Dutta, Deb, Tulshyan, and Arora]{dutta2013approximate}
Joydeep Dutta, Kalyanmoy Deb, Rupesh Tulshyan, and Ramnik Arora.
\newblock Approximate kkt points and a proximity measure for termination.
\newblock \emph{Journal of Global Optimization}, 56\penalty0 (4):\penalty0 1463--1499, 2013.

\bibitem[Fang et~al.(2021)Fang, He, Long, and Su]{fang2021exploring}
Cong Fang, Hangfeng He, Qi~Long, and Weijie~J Su.
\newblock Exploring deep neural networks via layer-peeled model: Minority collapse in imbalanced training.
\newblock \emph{Proceedings of the National Academy of Sciences}, 118\penalty0 (43):\penalty0 e2103091118, 2021.

\bibitem[Feng and Tu(2021)]{feng2021inverse}
Yu~Feng and Yuhai Tu.
\newblock The inverse variance--flatness relation in stochastic gradient descent is critical for finding flat minima.
\newblock \emph{Proceedings of the National Academy of Sciences}, 118\penalty0 (9), 2021.

\bibitem[Filippov(2013)]{filippov2013differential}
Aleksei~Fedorovich Filippov.
\newblock \emph{Differential equations with discontinuous righthand sides: control systems}, volume~18.
\newblock Springer Science \& Business Media, 2013.

\bibitem[Han et~al.(2021)Han, Papyan, and Donoho]{han2021neural}
XY~Han, Vardan Papyan, and David~L Donoho.
\newblock Neural collapse under mse loss: Proximity to and dynamics on the central path.
\newblock \emph{arXiv preprint arXiv:2106.02073}, 2021.

\bibitem[Hochreiter and Schmidhuber(1997)]{hochreiter1997flat}
Sepp Hochreiter and J{\"u}rgen Schmidhuber.
\newblock Flat minima.
\newblock \emph{Neural computation}, 9\penalty0 (1):\penalty0 1--42, 1997.

\bibitem[Jacot et~al.(2018)Jacot, Gabriel, and Hongler]{jacot2018neural}
Arthur Jacot, Franck Gabriel, and Cl{\'e}ment Hongler.
\newblock Neural tangent kernel: Convergence and generalization in neural networks.
\newblock \emph{arXiv preprint arXiv:1806.07572}, 2018.

\bibitem[Jacot et~al.(2021)Jacot, Ged, {\c{S}}im{\c{s}}ek, Hongler, and Gabriel]{jacot2021saddle}
Arthur Jacot, Fran{\c{c}}ois Ged, Berfin {\c{S}}im{\c{s}}ek, Cl{\'e}ment Hongler, and Franck Gabriel.
\newblock Saddle-to-saddle dynamics in deep linear networks: Small initialization training, symmetry, and sparsity.
\newblock \emph{arXiv preprint arXiv:2106.15933}, 2021.

\bibitem[Jastrz{\k{e}}bski et~al.(2019)Jastrz{\k{e}}bski, Kenton, Ballas, Fischer, Bengio, and Storkey]{jastrzkebski2018relation}
Stanis{\l}aw Jastrz{\k{e}}bski, Zachary Kenton, Nicolas Ballas, Asja Fischer, Yoshua Bengio, and Amos Storkey.
\newblock On the relation between the sharpest directions of dnn loss and the sgd step length.
\newblock \emph{International Conference on Learning Representations}, 2019.

\bibitem[Ji and Telgarsky(2019)]{ji2019polylogarithmic}
Ziwei Ji and Matus Telgarsky.
\newblock Polylogarithmic width suffices for gradient descent to achieve arbitrarily small test error with shallow relu networks.
\newblock \emph{arXiv preprint arXiv:1909.12292}, 2019.

\bibitem[Ji and Telgarsky(2020)]{ji2020directional}
Ziwei Ji and Matus Telgarsky.
\newblock Directional convergence and alignment in deep learning.
\newblock \emph{Advances in Neural Information Processing Systems}, 33:\penalty0 17176--17186, 2020.

\bibitem[Keskar et~al.(2016)Keskar, Mudigere, Nocedal, Smelyanskiy, and Tang]{keskar2016large}
Nitish~Shirish Keskar, Dheevatsa Mudigere, Jorge Nocedal, Mikhail Smelyanskiy, and Ping Tak~Peter Tang.
\newblock On large-batch training for deep learning: Generalization gap and sharp minima.
\newblock In \emph{International Conference on Learning Representations}, 2016.

\bibitem[Kunin et~al.(2023)Kunin, Yamamura, Ma, and Ganguli]{kunin2022asymmetric}
Daniel Kunin, Atsushi Yamamura, Chao Ma, and Surya Ganguli.
\newblock The asymmetric maximum margin bias of quasi-homogeneous neural networks.
\newblock \emph{International Conference on Learning Representations}, 2023.

\bibitem[Li et~al.(2017)Li, Tai, and E]{li2017stochastic}
Qianxiao Li, Cheng Tai, and Weinan E.
\newblock Stochastic modified equations and adaptive stochastic gradient algorithms.
\newblock In \emph{International Conference on Machine Learning}, pages 2101--2110. PMLR, 2017.

\bibitem[Li et~al.(2019)Li, Tai, and E]{li2019stochastic}
Qianxiao Li, Cheng Tai, and Weinan E.
\newblock Stochastic modified equations and dynamics of stochastic gradient algorithms i: Mathematical foundations.
\newblock \emph{The Journal of Machine Learning Research}, 20\penalty0 (1):\penalty0 1474--1520, 2019.

\bibitem[Li et~al.(2021)Li, Wang, and Arora]{li2021happens}
Zhiyuan Li, Tianhao Wang, and Sanjeev Arora.
\newblock What happens after sgd reaches zero loss?--a mathematical framework.
\newblock \emph{arXiv preprint arXiv:2110.06914}, 2021.

\bibitem[Li et~al.(2022)Li, Wang, and Li]{li2022analyzing}
Zhouzi Li, Zixuan Wang, and Jian Li.
\newblock Analyzing sharpness along gd trajectory: Progressive sharpening and edge of stability.
\newblock \emph{arXiv preprint arXiv:2207.12678}, 2022.

\bibitem[Liu et~al.(2021)Liu, Ziyin, and Ueda]{liu2021noise}
Kangqiao Liu, Liu Ziyin, and Masahito Ueda.
\newblock Noise and fluctuation of finite learning rate stochastic gradient descent.
\newblock In \emph{International Conference on Machine Learning}, pages 7045--7056. PMLR, 2021.

\bibitem[Luo et~al.(2021)Luo, Xu, Ma, and Zhang]{luo2021phase}
Tao Luo, Zhi-Qin~John Xu, Zheng Ma, and Yaoyu Zhang.
\newblock Phase diagram for two-layer relu neural networks at infinite-width limit.
\newblock \emph{The Journal of Machine Learning Research}, 22\penalty0 (1):\penalty0 3327--3373, 2021.

\bibitem[Lyu and Li(2019)]{lyu2019gradient}
Kaifeng Lyu and Jian Li.
\newblock Gradient descent maximizes the margin of homogeneous neural networks.
\newblock \emph{arXiv preprint arXiv:1906.05890}, 2019.

\bibitem[Lyu et~al.(2021)Lyu, Li, Wang, and Arora]{lyu2021gradient}
Kaifeng Lyu, Zhiyuan Li, Runzhe Wang, and Sanjeev Arora.
\newblock Gradient descent on two-layer nets: Margin maximization and simplicity bias.
\newblock \emph{Advances in Neural Information Processing Systems}, 34, 2021.

\bibitem[Ma and Ying(2021)]{ma2021linear}
Chao Ma and Lexing Ying.
\newblock On linear stability of sgd and input-smoothness of neural networks.
\newblock \emph{Advances in Neural Information Processing Systems}, 34:\penalty0 16805--16817, 2021.

\bibitem[Ma et~al.(2022)Ma, Kunin, Wu, and Ying]{ma2022beyond}
Chao Ma, Daniel Kunin, Lei Wu, and Lexing Ying.
\newblock Beyond the quadratic approximation: the multiscale structure of neural network loss landscapes.
\newblock \emph{arXiv preprint arXiv:2204.11326}, 2022.

\bibitem[Maennel et~al.(2018)Maennel, Bousquet, and Gelly]{maennel2018gradient}
Hartmut Maennel, Olivier Bousquet, and Sylvain Gelly.
\newblock Gradient descent quantizes relu network features.
\newblock \emph{arXiv preprint arXiv:1803.08367}, 2018.

\bibitem[Mei et~al.(2019)Mei, Misiakiewicz, and Montanari]{mei2019mean}
Song Mei, Theodor Misiakiewicz, and Andrea Montanari.
\newblock Mean-field theory of two-layers neural networks: dimension-free bounds and kernel limit.
\newblock In \emph{Conference on Learning Theory}, pages 2388--2464. PMLR, 2019.

\bibitem[Mulayoff et~al.(2021)Mulayoff, Michaeli, and Soudry]{mulayoff2021implicit}
Rotem Mulayoff, Tomer Michaeli, and Daniel Soudry.
\newblock The implicit bias of minima stability: A view from function space.
\newblock \emph{Advances in Neural Information Processing Systems}, 34:\penalty0 17749--17761, 2021.

\bibitem[Nacson et~al.(2019)Nacson, Gunasekar, Lee, Srebro, and Soudry]{nacson2019lexicographic}
Mor~Shpigel Nacson, Suriya Gunasekar, Jason Lee, Nathan Srebro, and Daniel Soudry.
\newblock Lexicographic and depth-sensitive margins in homogeneous and non-homogeneous deep models.
\newblock In \emph{International Conference on Machine Learning}, pages 4683--4692. PMLR, 2019.

\bibitem[Nakkiran et~al.(2019)Nakkiran, Kaplun, Kalimeris, Yang, Edelman, Zhang, and Barak]{nakkiran2019sgd}
Preetum Nakkiran, Gal Kaplun, Dimitris Kalimeris, Tristan Yang, Benjamin~L Edelman, Fred Zhang, and Boaz Barak.
\newblock Sgd on neural networks learns functions of increasing complexity.
\newblock \emph{arXiv preprint arXiv:1905.11604}, 2019.

\bibitem[Papyan et~al.(2020)Papyan, Han, and Donoho]{papyan2020prevalence}
Vardan Papyan, XY~Han, and David~L Donoho.
\newblock Prevalence of neural collapse during the terminal phase of deep learning training.
\newblock \emph{Proceedings of the National Academy of Sciences}, 117\penalty0 (40):\penalty0 24652--24663, 2020.

\bibitem[Pesme and Flammarion(2023)]{pesme2023saddle}
Scott Pesme and Nicolas Flammarion.
\newblock Saddle-to-saddle dynamics in diagonal linear networks.
\newblock \emph{arXiv preprint arXiv:2304.00488}, 2023.

\bibitem[Phuong and Lampert(2021)]{phuong2021inductive}
Mary Phuong and Christoph~H Lampert.
\newblock The inductive bias of relu networks on orthogonally separable data.
\newblock In \emph{International Conference on Learning Representations}, 2021.

\bibitem[Rahaman et~al.(2019)Rahaman, Arpit, Draxler, Lin, Hamprecht, Bengio, and Courville]{rahaman2019spectral}
Aristide Rahaman, Nasim xd~Baratin, Devansh Arpit, Felix Draxler, Min Lin, Fred Hamprecht, Yoshua Bengio, and Aaron Courville.
\newblock On the spectral bias of neural networks.
\newblock In \emph{International Conference on Machine Learning}, pages 5301--5310. PMLR, 2019.

\bibitem[Safran et~al.(2022)Safran, Vardi, and Lee]{safran2022effective}
Itay Safran, Gal Vardi, and Jason~D Lee.
\newblock On the effective number of linear regions in shallow univariate relu networks: Convergence guarantees and implicit bias.
\newblock \emph{Advances in Neural Information Processing Systems}, 35:\penalty0 32667--32679, 2022.

\bibitem[Saxe et~al.(2022)Saxe, Sodhani, and Lewallen]{saxe2022neural}
Andrew Saxe, Shagun Sodhani, and Sam~Jay Lewallen.
\newblock The neural race reduction: Dynamics of abstraction in gated networks.
\newblock In \emph{International Conference on Machine Learning}, pages 19287--19309. PMLR, 2022.

\bibitem[Soudry et~al.(2018)Soudry, Hoffer, Nacson, Gunasekar, and Srebro]{soudry2018implicit}
Daniel Soudry, Elad Hoffer, Mor~Shpigel Nacson, Suriya Gunasekar, and Nathan Srebro.
\newblock The implicit bias of gradient descent on separable data.
\newblock \emph{The Journal of Machine Learning Research}, 19\penalty0 (1):\penalty0 2822--2878, 2018.

\bibitem[Thomas et~al.(2020)Thomas, Pedregosa, Merri{\"e}nboer, Manzagol, Bengio, and Le~Roux]{thomas2020interplay}
Valentin Thomas, Fabian Pedregosa, Bart Merri{\"e}nboer, Pierre-Antoine Manzagol, Yoshua Bengio, and Nicolas Le~Roux.
\newblock On the interplay between noise and curvature and its effect on optimization and generalization.
\newblock In \emph{International Conference on Artificial Intelligence and Statistics}, pages 3503--3513. PMLR, 2020.

\bibitem[Vardi(2023)]{vardi2023implicit}
Gal Vardi.
\newblock On the implicit bias in deep-learning algorithms.
\newblock \emph{Communications of the ACM}, 66\penalty0 (6):\penalty0 86--93, 2023.

\bibitem[Vardi et~al.(2022)Vardi, Shamir, and Srebro]{vardi2022margin}
Gal Vardi, Ohad Shamir, and Nati Srebro.
\newblock On margin maximization in linear and relu networks.
\newblock \emph{Advances in Neural Information Processing Systems}, 35:\penalty0 37024--37036, 2022.

\bibitem[Wang and Ma(2022)]{wang2022early}
Mingze Wang and Chao Ma.
\newblock Early stage convergence and global convergence of training mildly parameterized neural networks.
\newblock \emph{Advances in Neural Information Processing Systems}, 2022.

\bibitem[Wang and Wu(2023)]{wang2023noise}
Mingze Wang and Lei Wu.
\newblock The noise geometry of stochastic gradient descent: A quantitative and analytical characterization.
\newblock \emph{arXiv preprint arXiv:2310.00692}, 2023.

\bibitem[Wojtowytsch(2023)]{wojtowytsch2023stochastic}
Stephan Wojtowytsch.
\newblock Stochastic gradient descent with noise of machine learning type part i: Discrete time analysis.
\newblock \emph{Journal of Nonlinear Science}, 33\penalty0 (3):\penalty0 45, 2023.

\bibitem[Woodworth et~al.(2020)Woodworth, Gunasekar, Lee, Moroshko, Savarese, Golan, Soudry, and Srebro]{woodworth2020kernel}
Blake Woodworth, Suriya Gunasekar, Jason~D Lee, Edward Moroshko, Pedro Savarese, Itay Golan, Daniel Soudry, and Nathan Srebro.
\newblock Kernel and rich regimes in overparametrized models.
\newblock In \emph{Conference on Learning Theory}, pages 3635--3673. PMLR, 2020.

\bibitem[Wu and Su(2023)]{wu2023implicitstability}
Lei Wu and Weijie~J Su.
\newblock The implicit regularization of dynamical stability in stochastic gradient descent.
\newblock In \emph{Proceedings of the 40th International Conference on Machine Learning}, volume 202 of \emph{Proceedings of Machine Learning Research}, pages 37656--37684. PMLR, 2023.

\bibitem[Wu et~al.(2018)Wu, Ma, and E]{wu2018sgd}
Lei Wu, Chao Ma, and Weinan E.
\newblock How sgd selects the global minima in over-parameterized learning: A dynamical stability perspective.
\newblock \emph{Advances in Neural Information Processing Systems}, 31:\penalty0 8279--8288, 2018.

\bibitem[Wu et~al.(2022)Wu, Wang, and Su]{wu2022does}
Lei Wu, Mingze Wang, and Weijie Su.
\newblock When does sgd favor flat minima? a quantitative characterization via linear stability.
\newblock \emph{Advances in Neural Information Processing Systems}, 2022.

\bibitem[Xu et~al.(2019)Xu, Zhang, Luo, Xiao, and Ma]{xu2019frequency}
Zhi-Qin~John Xu, Yaoyu Zhang, Tao Luo, Yanyang Xiao, and Zheng Ma.
\newblock Frequency principle: Fourier analysis sheds light on deep neural networks.
\newblock \emph{arXiv preprint arXiv:1901.06523}, 2019.

\bibitem[Zhang et~al.(2022)Zhang, Li, Zhang, Luo, and John~Xu]{zhang2021embedding}
Yaoyu Zhang, Yuqing Li, Zhongwang Zhang, Tao Luo, and Zhi-Qin John~Xu.
\newblock Embedding principle: A hierarchical structure of loss landscape of deep neural networks.
\newblock \emph{Journal of Machine Learning}, 1\penalty0 (1):\penalty0 60--113, 2022.
\newblock ISSN 2790-2048.

\bibitem[Zhou et~al.(2022{\natexlab{a}})Zhou, Qixuan, Jin, Luo, Zhang, and Xu]{zhou2022empirical}
Hanxu Zhou, Zhou Qixuan, Zhenyuan Jin, Tao Luo, Yaoyu Zhang, and Zhi-Qin Xu.
\newblock Empirical phase diagram for three-layer neural networks with infinite width.
\newblock \emph{Advances in Neural Information Processing Systems}, 35:\penalty0 26021--26033, 2022{\natexlab{a}}.

\bibitem[Zhou et~al.(2022{\natexlab{b}})Zhou, Qixuan, Luo, Zhang, and Xu]{zhou2022towards}
Hanxu Zhou, Zhou Qixuan, Tao Luo, Yaoyu Zhang, and Zhi-Qin Xu.
\newblock Towards understanding the condensation of neural networks at initial training.
\newblock \emph{Advances in Neural Information Processing Systems}, 35:\penalty0 2184--2196, 2022{\natexlab{b}}.

\bibitem[Zhu et~al.(2022)Zhu, Wang, Wang, Zhou, and Ge]{zhu2022understanding}
Xingyu Zhu, Zixuan Wang, Xiang Wang, Mo~Zhou, and Rong Ge.
\newblock Understanding edge-of-stability training dynamics with a minimalist example.
\newblock \emph{arXiv preprint arXiv:2210.03294}, 2022.

\bibitem[Zhu et~al.(2019)Zhu, Wu, Yu, Wu, and Ma]{zhu2019anisotropic}
Zhanxing Zhu, Jingfeng Wu, Bing Yu, Lei Wu, and Jinwen Ma.
\newblock The anisotropic noise in stochastic gradient descent: Its behavior of escaping from sharp minima and regularization effects.
\newblock In \emph{International Conference on Machine Learning}, pages 7654--7663. PMLR, 2019.

\bibitem[Zhu et~al.(2021)Zhu, Ding, Zhou, Li, You, Sulam, and Qu]{zhu2021geometric}
Zhihui Zhu, Tianyu Ding, Jinxin Zhou, Xiao Li, Chong You, Jeremias Sulam, and Qing Qu.
\newblock A geometric analysis of neural collapse with unconstrained features.
\newblock \emph{Advances in Neural Information Processing Systems}, 34:\penalty0 29820--29834, 2021.

\bibitem[Ziyin et~al.(2022)Ziyin, Liu, Mori, and Ueda]{ziyin2021minibatch}
Liu Ziyin, Kangqiao Liu, Takashi Mori, and Masahito Ueda.
\newblock Strength of minibatch noise in {SGD}.
\newblock In \emph{International Conference on Learning Representations}, 2022.

\bibitem[Zou et~al.(2018)Zou, Cao, Zhou, and Gu]{zou2018stochastic}
Difan Zou, Yuan Cao, Dongruo Zhou, and Quanquan Gu.
\newblock Stochastic gradient descent optimizes over-parameterized deep relu networks.
\newblock \emph{arXiv preprint arXiv:1811.08888}, 2018.

\end{thebibliography}
\end{document}